\documentclass[a4paper,12pt]{article}
\pdfoutput=1
\usepackage{pdfsync}
\usepackage{amsmath}
\usepackage{amsfonts}
\usepackage{amssymb}
\usepackage{amsthm}
\usepackage{mathrsfs}
\newtheorem{proposition}{Proposition}[section]
\newtheorem{definition}[proposition]{Definition}

\newtheorem{theorem}[proposition]{Theorem}
\newtheorem{example}[proposition]{Example}

\newtheorem*{assumption*}{Assumption}
\newtheorem{corollary}[proposition]{Corollary}
\newtheorem*{remark}{Remark}

\newtheorem{lemma}[proposition]{Lemma}
\newtheorem*{lemma*}{Lemma}
\newtheorem*{proposition*}{Proposition}
\newtheorem*{theorem*}{Theorem}
\usepackage[normalem]{ulem}
\usepackage{extarrows}
\usepackage{xcolor}
\usepackage{caption}
\usepackage{subcaption}
\usepackage[inline]{enumitem}

\usepackage[numbers, sort]{natbib}
\makeatletter
\renewcommand*\env@matrix[1][\arraystretch]{%
  \edef\arraystretch{#1}%
  \hskip -\arraycolsep
  \let\@ifnextchar\new@ifnextchar
  \array{*\c@MaxMatrixCols c}}
\makeatother

\usepackage{framed}
\newenvironment{customabstract}{
    \begin{center}
    \begin{minipage}{0.9\textwidth}
}{%
    \end{minipage}
    \end{center}
    \vspace{1cm}
}

\DeclareMathOperator*{\argmin}{argmin}

\newcommand{\nc}[1]{{\color{blue} #1}}

\newcommand{\bit}{k}
\newcommand{\zero}{X}
\newcommand{\elll}{\ell}
\newcommand\affine{S}
\newcommand\lip{H^1}

\newcommand{\nleft}{}
\newcommand{\nright}{}

\DeclareMathOperator{\coef}{\mathcal{C}}

\newcommand\nmathbf{}
\newcommand{\cupper}{D}

\newcommand{\TODO}[1][0]{%
  \ifx#10
    $\square$
  \else
    $\boxtimes$
  \fi
}
\usepackage{float}
\usepackage{tikz}
\tikzstyle{inarrow} = [<-,very thick]
\tikzstyle{bcircle} = [circle,draw = blue]
\usetikzlibrary{decorations.pathmorphing}
\usepackage[linkcolor=blue,colorlinks=true,citecolor=blue,filecolor=black]{hyperref}
\usepackage{geometry}
\usepackage{setspace}
\newcommand\blindfootnote[1]{%
  \begingroup
  \renewcommand\thefootnote{}\footnote{#1}%
  \addtocounter{footnote}{-1}%
  \endgroup
}

\title{Three Quantization Regimes for ReLU Networks}
\author{Weigutian Ou \\ wou@mins.ee.ethz.ch \and Philipp Schenkel \\ schenkel@fzi.de \and Helmut Bölcskei \\ hboelcskei@ethz.ch}
\date{}
\begin{document}
	\maketitle 

	%!TEX root = ../draft_quantized_weight_networks.tex

\begin{customabstract}
\noindent \textbf{Abstract.} We establish the fundamental limits in the approximation of Lipschitz functions by deep ReLU neural networks with finite-precision weights. Specifically, three regimes, namely under-, over-, and proper quantization, in terms of minimax approximation error behavior as a function of network weight precision, are identified. This is accomplished by deriving nonasymptotic tight lower and upper bounds on the minimax approximation error. Notably, in the proper-quantization regime, neural networks exhibit memory-optimality in the approximation of Lipschitz functions. Deep networks have an inherent advantage over shallow networks in achieving memory-optimality.
We also develop the notion of depth-precision tradeoff, showing that networks with high-precision weights can be converted into functionally equivalent 
deeper networks with low-precision weights, while preserving memory-optimality.
This idea is reminiscent of sigma-delta analog-to-digital conversion, where oversampling rate is traded for resolution in the quantization of signal samples. We improve upon the best-known ReLU network approximation results for Lipschitz functions and describe a refinement of the bit extraction technique which could be of independent general interest. 
\end{customabstract}
\section{Introduction} % (fold)
	\label{sec:introduction}

This paper is concerned with the fundamental limits in the approximation of Lipschitz functions by deep ReLU neural networks with finite-precision weights. Specifically, we consider fully connected networks and allow depth and width to be chosen independently. 
\blindfootnote{\noindent H. B\"olcskei gratefully acknowledges support by the Lagrange Mathematics and
Computing Research Center, Paris, France.}
%of general architecture, i.e., the network depth and the individual layer widths can all be chosen independently. 
The main conceptual contribution resides in the identification of three different regimes in terms of minimax approximation error behavior as a function of the network weight precision, i.e., the number of bits $b$ needed to store each of the network weights.
%, as a function of network depth $L$ and width $W$. 
This is accomplished by deriving nonasymptotic and, in particular, tight lower and upper bounds on the minimax error. In the under-quantization regime the minimax error exhibits exponential decay in $b$, in the proper-quantization regime the decay is polynomial, and in 
the over-quantization regime we get constant behavior.
%, the error does not improve as $b$ is increased, and, finally, in the \textit{proper-quantization regime}, the error decays polynomially in $b$. 
Notably, in the proper-quantization regime, neural networks approximate Lipschitz functions in a memory-optimal fashion. In addition, 
\iffalse 
we can establish memory optimality in the following sense. For sufficiently large, but otherwise arbitrary $L$ and $W$, and $b\,\in\, [C_1 \log (W)/L , C_2 L  \log (W))$, with $C_1,C_2$ absolute constants, for every $f \in \lip ([0,1])$, there exists a ReLU network of depth $L$ and width $W$, with each weight encoded by $b$ bits, such that the corresponding approximation error behaves as $C ( W^2 L b )^{-1}$, for some absolute constant $C$. As the total number of network weights is on the order of
$W^2 L$, it takes roughly $W^2 L b$ bits to encode the overall network. On the other hand, given that the metric entropy of $\lip ([0,1])$ behaves according to $\epsilon^{-1}$, it follows that on the order of $W^2 L b$ bits are needed, by any procedure, to encode $f \in \lip ([0,1])$ at error proportional to $(W^2 L b)^{-1}$. 
\fi 
%the specifics of this result allow us to conclude that 
deep networks are found to exhibit an inherent advantage over shallow networks in achieving memory-optimality. 
%Specifically, given a total memory budget of $W^2 L b$ bits, it follows that for fixed weight precision $b$, choosing $L$ large leads to a larger
%proper quantization region $[C_1 \log (W)/L , C_2 L  \log (W))$. In contrast, varying $W$ simply offsets the region (according to $\log(W)$) without %increasing it. In summary, for fixed weight precision, it is hence preferable to let the network be deep rather than wide.

Besides the conceptual contribution of identifying the three quantization regimes, we report three technical contributions.
		First, we develop the notion of depth-precision tradeoff, showing that networks with high-precision weights can be converted into equivalent (in terms of input-output relation) deeper networks with low-precision weights, while preserving memory-optimality.  The underlying network transformation is constructive. This idea is reminiscent of the concept of sigma-delta analog-to-digital conversion \cite{Tewksbury1978}, where sampling rate is traded for resolution in the signal samples. Here, we trade network depth for network weight resolution.
  %for network depth. {\color{blue} RMK1: trade depth for resolution/ precision. In the previous sentence, resolution is traded for but in this sentence the %resolution is traded. }
  %Here, we We emphasize that if the starting high-precision network is memory optimal so is the corresponding low-precision deep network, albeit with %different constants in the minimax error.

		The second technical contribution is an improvement of the best-known neural network approximation results for $1$-Lipschitz functions on $[0,1]$. Specifically, for networks of sufficiently large, but otherwise arbitrary, network width $W$ and depth $L$, as well as weight magnitude bounded by $1$,
  %and $f \in \lip  ( [0,1] )$, 
  we show that the minimax
  %construct a ReLU network of depth no greater than $L$, width no larger than $W$, and maximum absolute value of network weights bounded by $1$, such that the %corresponding 
  approximation error behaves according to $ C ( W^2 L^2 \log(W) )^{-1}$, with $C$ an absolute constant. There is a significant body of literature on neural network approximation of $1$-Lipschitz functions on $[0,1]$, usually presented in the broader context of approximation of smooth functions on hypercubes. Specifically, the references \cite{Schmidt-Hieber2017, Chen2019, nakada2020adaptive, GUHRING2021107, schmidt2019deep, PETERSEN2018296} consider approximation with ReLU networks whose depth grows at most poly-logarithmically in network width, hence, in contrast to our results, with depth and width coupled. On the other hand, the findings reported in \cite{yarotsky2019phase, shen2021optimal, Kohler2019} allow network depth to grow faster than network width. Notably, in \cite{shen2021optimal} network width and depth can be chosen independently and the same approximation error behavior as in our case, namely $C ( W^2 L^2 \log(W) )^{-1}$, is obtained. However, the network constructions proposed in \cite{yarotsky2019phase, shen2021optimal, Kohler2019} all come with weight-magnitude growth that is at least exponential in network depth, in contrast to the weight-magnitude upper-bounded by $1$ in our case. This constant weight-magnitude upper bound will turn out to be essential in establishing memory optimality.

		The third technical contribution we report is an improvement of the bit extraction technique pioneered in \cite{bartlett1998almost,bartlett2019nearly}.
		Bit extraction refers to the recovery---through ReLU networks---of binary strings encoded into real numbers. This concept was originally used to lower-bound the VC-dimension of ReLU networks \cite{bartlett1998almost, bartlett2019nearly} and later employed in the context of neural network approximation \cite{yarotsky2019phase, shen2021optimal, Kohler2019, Vardi2022}. The legacy approach yields extraction networks whose weight magnitude grows exponentially in network depth and polynomially in network width. In contrast, the novel construction we present exhibits only 
  (polynomial) weight magnitude dependence on network width. While this refinement is essential in establishing the second technical contribution mentioned above, the technique could also be of independent interest.

		We finally note that the results in this paper are readily extended to neural network approximation of Lipschitz functions on $d$-dimensional hypercubes. For clarity of exposition we decided, however, to restrict ourselves to the one-dimensional case.

        \textit{Notation.} We denote the cardinality of a set $X$ by $\nleft| X \nright|$.  $\mathbb{N} = \{ 1,2,\dots \}$ designates the natural numbers, $\mathbb{R}$ stands for the real numbers, $\mathbb{R}_+$ for the positive real numbers, and $\emptyset$ for the empty set. For $\mathbb{A} \subseteq \mathbb{R}$, we denote its maximum, minimum, supremum, and infimum, by $\max \mathbb{A}$, $\min \mathbb{A}$, $\sup \mathbb{A}$, and $\inf \mathbb{A}$, respectively. The indicator function $1_P$ for proposition $P$ is equal to $1$ if $P$ is true and $0$ else.
      
			For a vector $b \in \mathbb{R}^d$, we write $\nleft\| b \nright\|_\infty := \max_{i = 1,\dots,d} \nleft| b_i \nright|$ and $\nleft\| b \nright\|_0 := \sum_{i = 1}^d 1_{b_i \neq 0}$. Similarly, for a matrix $A \in \mathbb{R}^{m \times n}$, we let $\nleft\| A \nright\|_\infty = \max_{i=1,...,m,j=1,...,n} \nleft| A_{i,j} \nright| $ and $\nleft\| A \nright\|_0 := \sum_{i = 1}^m \sum_{j = 1}^n  1_{A_{i,j} \neq 0}$. $1_{m}$ and $0_{m}$ stand for the
      $m$-dimensional vector with all entries equal to $1$ and $0$, respectively. 
    $I_m$ refers to the $m \times m$ identity matrix. $1_{m\times n}$ and $0_{m\times n}$ denote the $m\times n$ matrix with all entries equal to $1$ and $0$, respectively. For matrices $A_1, \dots, A_n$, possibly of different dimensions, we designate the  block-diagonal matrix with diagonal element-matrices $A_1, \dots, A_n$ by  $\text{diag} (A_1, \dots, A_n) $.

			$\log(\cdot)$ and $\ln(\cdot)$ denote the logarithm to base $2$ and base $e$, respectively.
      The ReLU activation function $\rho$ is given by $\rho(x):= \max \{ x,0 \}$, for $x \in \mathbb{R}$, and, when applied to vectors, acts elementwise.
			The sign function $\text{sgn}: \mathbb{R} \mapsto  \{ 0,1 \}$ is defined according to $\text{sgn}(x) =1$, for $x \geq 0$, and $\text{sgn} ( x ) = 0$, for $x < 0$. %\nc{We define the indication function $1_P$ for a proposition $P$ according to $1_P = 1$, if $P$ is true, and $1_P = 0$, if $P$ is not true.} 
   We use $S(A,b)$ to refer to the affine mapping $S(A,b) (x) = Ax + b, x \in \mathbb{R}^{n_2}$, with $A \in \mathbb{R}^{n_1 \times n_2}$, $b \in \mathbb{R}^{n_1}$. For $f_1: \mathbb{R}^{d_0} \mapsto \mathbb{R}^{d_1}$ and $ f_2: \mathbb{R}^{d_0} \mapsto \mathbb{R}^{d_2}$, we define $( f_1, f_2 ): \mathbb{R
			}^{d_0} \mapsto \mathbb{R}^{d_1 + d_2}$ according to $ ( f_1,f_2 ) ( x ) = ( f_1 ( x ), f_2 ( x ) )$, $x \in \mathbb{R}^{d_0}$. 
   If $\mathcal{F}$ is a family of functions and $a \in \mathbb{R}$, we write $a\cdot \mathcal{F}:= \{ a f: f \in \mathcal{F} \}$.
			For $\mathbb{X} \subseteq \mathbb{R}^d$ and $f: \mathbb{X} \mapsto \mathbb{R}$, we define the  $L^\infty(\mathbb{X}) $-norm of $f$ as $\nleft\| f \nright\|_{L^\infty ( \mathbb{X} )} := \sup_{x \in \mathbb{X}} \nleft| f(x) \nright|$.  A constant will be called absolute if it does not depend on any variables or parameters. We may use the same letter for different absolute constants at different places in the paper.

		\subsection{Definition of key concepts and organization of the paper} % (fold)
		\label{sub:relu_networks_and_main_result}

			The purpose of this section is to introduce the key concepts needed to formalize the main results of the paper. We start by defining the family of $1$-Lipschitz functions on $[0,1]$ according to
			\begin{equation*}
				\begin{aligned}
					\lip  ( [0,1] ):= \{& f \in C ( [0,1] ): | f(x) | \leq 1, | f(x) - f(y) |
					\leq | x - y |, \,\forall x,y \in [0,1]     \}.
				\end{aligned}
			\end{equation*}
			Next,  we provide the  definition of neural networks. 
			\begin{definition}
			\label{def:ReLU_networks}
				Let $L,N_0,N_1,\dots, N_L \in \mathbb{N}$. A neural network configuration $\Phi$ is a sequence of matrix-vector tuples 
				\begin{equation*}
					\Phi = (( A_i,\nmathbf{b}_i )  )_{i = 1}^{L},
				\end{equation*}
				where $A_i \in \mathbb{R}^{N_i \times N_{i - 1}}$, $b_i \in \mathbb{R}^{N_i}$, $i = 1,\dots, L$. We refer to $N_i$ as the width of the $i$-th layer, $i = 0,\dots, L$, and call the tuple $( N_0,\dots,N_L )$ the architecture of the network configuration. $\mathcal{N} ( ( d,d' ) )$ refers to the set of all neural network configurations with input dimension $N_0 = d$ and output dimension $N_L = d'$. The depth of the configuration $\Phi$ is $\mathcal{L} ( \Phi )  := L$, its width $\mathcal{W}  ( \Phi ) := \max_{i = 0,\dots,L} N_i$, the  weight set $\coef ( \Phi ) := \bigcup_{i = 1,\dots,L} ( \coef ( A_i ) \bigcup \coef ( \nmathbf{b}_i )  )  $, where $\coef ( A )$ and $\coef ( \nmathbf{b} )$ denote the value set of the entries of $A$ and $\nmathbf{b}$, respectively, and the weight magnitude $\mathcal{B} ( \Phi )  := \max_{i = 1,\dots, L} \max \{ \| A_i \|_\infty, \| \nmathbf{b}_i \|_\infty \}$.

				We define, recursively, the neural network realization $R (\Phi ): \mathbb{R}^{N_0} \mapsto \mathbb{R}^{N_L}$,  associated with the neural network configuration $\Phi$, and the activation function $\rho$, according to

				\begin{equation}
				\label{eq:realization_computation}
				R ( \Phi ) = \left\{
				\begin{aligned}
					& \affine ( A_L, b_L ), &&\text{if } L = 1,\\
					& \affine ( A_L, b_L ) \circ \rho \circ R ( (( A_i,\nmathbf{b}_i )  )_{i = 1}^{L - 1} ), &&\text{if } L \geq 2.
				\end{aligned}
				\right.
				\end{equation}
				The family of network configurations with depth at most  $L$, width at most $W$, weight magnitude at most $B$, where $B \in \mathbb{R}_+ \cup \{ \infty \}$, weights taking values in $\mathbb{A} \subseteq \mathbb{R}$, $d$-dimensional input, and $d'$-dimensional output, for $d,d'\in \mathbb{N}$, $W, L \in \mathbb{N} \cup \{ \infty \}$, with\footnote{The condition $W \geq \max \{ d,d' \}$ is formally stated here so as to prevent the trivial case of $\mathcal{N}_\mathbb{A} ( ( d,d' ), W,L,B )$ being an empty set.
    It will be a standing assumption throughout the paper.}
        $W \geq \max \{ d,d' \}$, is defined as
				\begin{equation}
				\label{eqline:continuous_nn}
					\mathcal{N}_\mathbb{A} ( ( d,d' ), W,L,B )  :=  \{  \Phi \in \mathcal{N} ( ( d,d' ) ) |\mathcal{W} (  \Phi ) \leq W, \ \mathcal{L} ( \Phi ) \leq L,\ \mathcal{B} ( \Phi ) \leq B, \coef ( \Phi ) \subseteq \mathbb{A} \},
				\end{equation}
				with the family of associated network realizations
				\begin{equation}
					\mathcal{R}_\mathbb{A} ( ( d,d' ), W,L,B ):= \{ R ( \Phi )  | \Phi \in \mathcal{N}_\mathbb{A} ( ( d,d' ), W,L,B ) \}. \label{eqline:continuous_nn_realization}
				\end{equation}
				To simplify notation, we allow the omission of the argument $( d,d' )$ in $\mathcal{N}_\mathbb{A} ( ( d,d' ), W,L,B )$ and $\mathcal{R}_\mathbb{A} ( ( d,d' ), W,L,B )$ when $( d,d' ) = ( 1,1 )$. When $B = \infty$, we omit the argument $B$ in $\mathcal{N}_\mathbb{A} ( ( d,d' ), W,L,B )$ and $\mathcal{R}_\mathbb{A} ( ( d,d' ), W,L,B )$. Furthermore, for $\mathbb{A} = \mathbb{R}$, we allow omission of the argument $\mathbb{A}$ in $\mathcal{N}_\mathbb{A} ( ( d,d' ), W,L,B )$ and $\mathcal{R}_\mathbb{A} ( ( d,d' ), W,L,B )$.  One specific incarnation of this policy that will be used frequently is  $\mathcal{N} ( W,L ) = \mathcal{N}_\mathbb{R} ( (1,1), W,L, \infty )$ and $\mathcal{R} ( W,L ) = \mathcal{R}_\mathbb{R} ( (1,1), W,L, \infty )$. 
		\end{definition}

		To clarify and prevent confusion, we note that configurations in $\mathcal{N}_\mathbb{A} ( ( d,d' ), W,L,B )$ can have depth $\ell \leq L$ and will correspondingly be designated by $( A_i, b_i )_{i = 1}^{\ell} $.
  %, with their depth $\ell \leq L$. This convention will be used as $\mathcal{N}_\mathbb{A} ( ( d,d' ), W,L,B )$ contains network configurations with depth %less than or equal to $L$, instead of necessarily exactly $L$. 
  We also emphasize the importance of differentiating between network configurations and network realizations. As we shall see later in the paper, network configurations with different architectures and weights taking values in different sets may realize the same function. Nevertheless, whenever there is no potential for confusion, we will use the term network to collectively refer to both configurations and realizations.

		Regarding the value sets of the network weights, we will typically consider sets of the form 
		\begin{equation}
		\label{def:Q}
		\begin{aligned}
			\mathbb{Q}^{a}_b := &\, (-2^{a+1}, 2^{a+1} ) \cap 2^{-b}\mathbb{Z}\\
			=&\, \biggl\{ \pm \sum_{i = -b}^a \theta_i 2^{i}: \theta_i \in \{ 0,1 \} \biggr\},
		\end{aligned}
		\end{equation}
		for some $a,b \in \mathbb{N}$, which is the set of all base-$2$ quantized numbers with $a+1$ digits before and $b$ digits after the binary point. Each element in $\mathbb{Q}^{a}_b$ can hence be described by $a+b+2$ bits, taking into account that we need one bit to encode the sign, and we have $\nleft| \mathbb{Q}^{a}_b  \nright|  \leq 2^{a+b+2}$. To simplify notation, we shall write
		\begin{align}
			\mathcal{N}_b^a ( ( d,d' ), W,L ) :=&\,  \mathcal{N}_{\mathbb{Q}_b^a} ( ( d,d' ), W,L),\\
			\mathcal{R}_b^a ( ( d,d' ), W,L ) :=&\,  \mathcal{R}_{\mathbb{Q}_b^a} ( ( d,d' ), W,L),
		\end{align}
		and will, as before, allow omission of the argument $( d,d' )$ whenever $( d, d' ) = ( 1,1)$.  We shall frequently use the shorthands
  %A specific example we shall use frequently is, for $a,b,W, L \in \mathbb{N}$, 
  $\mathcal{N}_b^a ( W,L ) = \mathcal{N}_b^a ( (1,1), W,L )$ and $\mathcal{R}_b^a ( W,L ) = \mathcal{R}_b^a ( (1,1), W,L )$.

		Throughout the paper, approximation errors will be quantified in terms of the following concept.

		\begin{definition}[Minimax (approximation) error]
			\label{def:minimax_approximation_error} Let $( \mathcal{X}, \delta )$ be a metric space and $\mathcal{F}, \mathcal{G} \subseteq \mathcal{X}$. We define the minimax  error in the approximation of elements of $\mathcal{F}$ through elements of $\mathcal{G}$ according to 
			\begin{equation*}
				\mathcal{A} ( \mathcal{F}, \mathcal{G}, \delta ) :=  \sup_{f \in \mathcal{F}} \inf_{g \in \mathcal{G}} \delta(f,g).
			\end{equation*}
			When $\delta = \nleft\| \cdot \nright\|_{L^\infty ( [0,1] )}$, we shall write $\mathcal{A}_\infty ( \mathcal{F}, \mathcal{G})$ instead of $\mathcal{A} ( \mathcal{F}, \mathcal{G}, \delta)$.
		\end{definition}

		The main goal of this paper is to characterize the behavior of
		\begin{equation*}
			\mathcal{A}_\infty ( \lip ( [0,1] ), \mathcal{R}_b^1 ( W, L  ) ), 
		\end{equation*}
		for independent choices of $W,L,b \in \mathbb{N}$. Motivated by the fact that numbers in $\mathbb{Q}_b^1$ are specified by $b+3$ bits, we shall
  refer to $b$ as the precision of $\mathbb{Q}^1_b$, $\mathcal{N}_b^1(W,L)$, and $\mathcal{R}_b^1(W,L) $.
  %have precision\footnote{Conventionally, precision refers to the number of digits in a number's binary representation.} of $b+2$, we can use $b$ to describe %the precision of $\mathbb{Q}_b^1$, and will, for simplicity, refer to $b$ as the precision of $\mathbb{Q}^1_b$, $\mathcal{R}_b^1 ( W, L  ) $, and %$\mathcal{N}_b^1 ( W, L  ) $. }The characterization will be accomplished by establishing lower bounds on the minimax error based on fundamental properties %of ReLU networks and upper bounds through explicit network constructions. 

		The remainder of the paper is organized as follows. In Section~\ref{sec:memory_requirement_as_the_fundamental_limit}, we introduce the concepts of network memory consumption, memory optimality, and memory redundancy. Three lower bounds on $\mathcal{A}_\infty ( \lip ( [0,1] ), \mathcal{R}_b^1 ( W, L  ) )$ are then presented, the first one incurred by minimum memory requirements, the second one based on VC dimension arguments, and the third one resulting from numerical precision limitations inherent to ReLU networks with quantized weights. These three bounds combine to a minimax error lower bound whose constituents are active in different regimes with regards to the choice of $b$.

%		{ \color{blue} Upper bounds on $\mathcal{A}_\infty ( \lip ( [0,1] ), \mathcal{R}_b^1 ( W, L  ) )$ are derived in %Sections~\ref{sec:bounds_on_quantization_error}, \ref{sec:embedding_properties_of_relu_networks_with_quantized_weights}, and %\ref{sec:the_three_quantization_regime}. 
  %Section~\ref{sec:bounds_on_quantization_error}.

%%%HB: go back and rewrite once main part has been revised.
An upper bound on $\mathcal{A}_\infty ( \lip ( [0,1] ), \mathcal{R}_b^1 ( W, L  ) )$ is derived in Section~\ref{sec:bounds_on_quantization_error} by 
  constructing an approximating network whose precision $b$ is carefully chosen to depend on network width $W$ and depth $L$. 
  %\ref{sec:embedding_properties_of_relu_networks_with_quantized_weights}, and %\ref{sec:the_three_quantization_regime}. 
  %Section~\ref{sec:bounds_on_quantization_error}.
  %through explicit network constructions are provided in Sections~\ref{sec:bounds_on_quantization_error}, %\ref{sec:embedding_properties_of_relu_networks_with_quantized_weights}, and \ref{sec:the_three_quantization_regime}. %Section~\ref{sec:bounds_on_quantization_error} constructs approximating networks whose precision $b$ depends on network width $W$ and depth $L$. 
  In Section~\ref{sec:embedding_properties_of_relu_networks_with_quantized_weights}, we show how this dependency can be relaxed through what we call 
    the depth-precision tradeoff establishing---in a constructive manner---that network depth can be traded for network weight precision.
    %can be traded for network depth without sacrificing memory optimality. {\color{blue} RMK2: issued related to RMK1; whether precision is traded or traded %for. } 
    Finally, Section~\ref{sec:the_three_quantization_regime} combines the minimax error lower and upper bounds to identify the three different quantization regimes and to prove memory optimality in the proper-quantization regime.
    %presents the derived minimax error upper bounds for independent choice of $W$, $L$ and $b$, and analyzes three distinct quantization regimes. 

	% section introduction (end)

	%!TEX root = ../draft_quantized_weight_networks.tex

\section{Minimax Error Lower Bounds} % (fold)
	\label{sec:memory_requirement_as_the_fundamental_limit}

	We first introduce and explore the concept of minimum memory requirement and then derive an associated minimax error lower bound. To set the stage, we commence with a brief review of the Kolmogorov-Donoho rate-distortion theory for neural network approximation as developed in \cite{deep-it-2019}, and present a non-asymptotic version thereof. The theory in \cite{deep-it-2019} considers a metric space $( \mathcal{X}, \delta )$ along with a set $\mathcal{Y} \subseteq \mathcal{X}$. For each $\ell \in \mathbb{N} \cup \{ 0 \}$, the set of binary length-$\ell$ encoders $E$ of $\mathcal{Y}$ of length $\ell$ is defined as
	\begin{equation*}
		\mathfrak{E}^\ell (\mathcal{Y} ) := \{ E: \mathcal{Y} \mapsto \{ 0,1 \}^\ell \}
	\end{equation*}
	along with the set of binary decoders
	\begin{equation*}
		\mathfrak{D}^\ell ( \mathcal{X}) := \{ D: \{ 0,1 \}^\ell \mapsto \mathcal{X} \}.
	\end{equation*}
	We denote the empty string by $\phi$ and use the convention $\{ 0,1 \}^0 = \{ \phi \}$.  

	A quantity of central interest is the minimal length $\ell \in \mathbb{N} \cup \{ 0 \}$ for which there exists an encoder-decoder pair $( E,D ) \in \mathfrak{E}^\ell(\mathcal{Y} ) \times \mathfrak{D}^\ell ( \mathcal{X} ) $ such that $\sup_{y \in \mathcal{Y}} \delta ( y,  D ( E ( y ) ) ) \leq \varepsilon$; we refer to $\sup_{y \in \mathcal{Y}} \delta ( y,  D ( E ( y ) ) )$ as the uniform error over the set $\mathcal{Y}$. In plain language, $\ell$ is the minimum number of bits needed to encode the elements in $\mathcal{Y}$ while guaranteeing that the corresponding decoding error does not exceed $\varepsilon$.

	\begin{definition}
		Consider the metric space $( \mathcal{X}, \delta )$ and the set $\mathcal{Y} \subseteq \mathcal{X}$. For $\varepsilon > 0 $, the minimax code length needed to achieve uniform error $\varepsilon$ over the set $\mathcal{Y}$, is
		\begin{equation*}
			\ell ( \varepsilon, \mathcal{Y}, ( \mathcal{X}, \delta ) ) := \min \{ \ell \in \mathbb{N} \cup \{ 0 \}: \exists ( E, D ) \in \mathfrak{E}^\ell ( \mathcal{Y} ) \times \mathfrak{D}^\ell ( \mathcal{X}): \sup_{y \in \mathcal{Y}} \delta ( y, D ( E ( y ) ) ) \leq \varepsilon \}.
		\end{equation*}
		We omit the argument $( \mathcal{X} ,\delta )$ in $\ell$ whenever it is clear from the context.
	\end{definition}
	The theory developed in \cite{deep-it-2019} is built on the asymptotic behavior of the minimax code length, as characterized by the optimal exponent $\sup \{ \gamma \in \mathbb{R}: \ell ( \varepsilon, \mathcal{Y}, ( \mathcal{X} ,\delta ) ) \in \mathcal{O} ( \varepsilon^{- 1\slash \gamma} ), \varepsilon \to 0 \}$. Here, we shall instead work directly with the minimax code length $\ell ( \varepsilon, \mathcal{Y}, ( \mathcal{X}, \delta ) )$, $\varepsilon \in \mathbb{R}_+$, which yields a more refined and, in particular, non-asymptotic picture. 

	The minimax code length $\ell$ does not only measure the minimum memory, i.e., the minimum number of bits, required to encode elements in $\mathcal{Y}$ at an error of no more than $\varepsilon$, but also quantifies the minimum memory needed to store a set of approximants of $\mathcal{Y}$. This insight follows from the observation that a finite set $\mathcal{G}$ of approximants for $\mathcal{Y}$ induces specific encoder-decoder pairs for $\mathcal{Y}$, as follows.

	\begin{proposition}
		\label{thm:memory_requirement}
		Let $(\mathcal{X}, \delta)$ be a metric space, $\mathcal{Y} \subseteq \mathcal{X}$, and $\varepsilon \in \mathbb{R}_+$.  Every finite subset $\mathcal{G} \subseteq \mathcal{X}$ such that $\mathcal{A} ( \mathcal{Y} , \mathcal{G}, \delta ) \leq \varepsilon$, induces an encoder-decoder pair $(  E: \mathcal{Y} \mapsto \{ 0,1 \}^{ \lceil \log ( | \mathcal{G} |  ) \rceil }, D : \{ 0,1 \}^{\lceil \log ( | \mathcal{G} |  ) \rceil } \mapsto \mathcal{G} )$ satisfying $\sup_{y \in \mathcal{Y}} \delta ( y, D ( E ( y ) ) ) \leq \varepsilon$ and 
			\begin{equation}
			\label{eq:fundamental_limit}
				 \lceil \log ( | \mathcal{G} |  ) \rceil \geq \ell ( \varepsilon, \mathcal{Y}, ( \mathcal{X}, \delta ) ).
			\end{equation}

		\begin{proof}
			
			We first note that 
			\begin{align}
				\varepsilon \geq&\, \mathcal{A} ( \mathcal{Y} , \mathcal{G}, \delta )  \label{eq:251}\\
				=&\, \sup_{y \in \mathcal{Y}} \inf_{g \in \mathcal{G}} \delta(y,g)\\
				=&\, \sup_{y \in \mathcal{Y}} \min_{g \in \mathcal{G}} \delta(y,g), \label{eq:252}
			\end{align} 
			where in \eqref{eq:251} we used the assumption $\mathcal{A} ( \mathcal{Y} , \mathcal{G}, \delta ) \leq \varepsilon$, and the equivalence of $\inf$ and $\min$ in \eqref{eq:252} follows from the fact that $\mathcal{G}$ is finite by assumption. The inequality \eqref{eq:251}-\eqref{eq:252} implies that, for every $y \in \mathcal{Y}$, there exists an element in $\mathcal{G}$, which we denote by $A ( y )$, such that $\delta ( y, A ( y ) ) \leq \varepsilon$. This, in turn, induces a mapping $A: \mathcal{Y} \mapsto \mathcal{G}$ satisfying
			\begin{equation}
			\label{eq:projection}
				\delta ( y, A ( y ) ) \leq \varepsilon, \quad \text{for all }y \in \mathcal{Y}.
			\end{equation}

			We proceed to construct the desired encoder-decoder pair $( E, D )$ by building on \eqref{eq:projection}. First, define an auxiliary function $\tilde{E}: \mathcal{G} \mapsto \{ 0,1 \}^{ \lceil \log ( | \mathcal{G} |  ) \rceil }$, which maps every element in $\mathcal{G}$ to a unique bitstring of length $\lceil \log ( | \mathcal{G} |  ) \rceil$. For  $\nleft| \mathcal{G} \nright|  =1 $, set $\tilde{E}: \mathcal{G} \mapsto \{ 0,1 \}^{ \lceil \log ( | \mathcal{G} |  ) \rceil } = \{ \phi \}$ to be the mapping that takes the single element of $\mathcal{G}$ into the empty string. For $\nleft| \mathcal{G} \nright| \geq 2$, we first label the elements in $\mathcal{G}$ as $\mathcal{G} = ( x_i )_{i =1}^{| \mathcal{G} |}$, in an arbitrary manner.  Then, we take $\tilde{E}: \mathcal{G} \mapsto \{ 0,1 \}^{ \lceil \log ( | \mathcal{G} |  ) \rceil }$ such that, for $i = 1,\dots, | \mathcal{G} |  $, $\tilde{E} ( x_i )$ is the bitstring of the binary representation of the integer $(i-1)$ with $0$'s added at the beginning so that the overall bitstring has length $ \lceil \log ( | \mathcal{G} |  ) \rceil$. 
   %This encoding is valid as the binary representation of $| \mathcal{G} | - 1 $, which is $(i-1)$ for $i = |\mathcal{G}|$, has length at most $\lceil \log ( %| \mathcal{G} |  ) \rceil $.  
   For all $\nleft| \mathcal{G} \nright| $, $\tilde{E}$ is an injection, ensuring the existence of a decoder $D: \{ 0,1 \}^{ \lceil \log ( | \mathcal{G} |  ) \rceil }  \mapsto \mathcal{G}$ such that 
			\begin{equation}
			\label{eq:lossless}
				D ( \tilde{E} ( x ) ) = x, \quad \text{for all } x \in \mathcal{G}.
			\end{equation}
			With the mapping $A: \mathcal{Y} \mapsto \mathcal{G}$ defined above, we now set $E = \tilde{E} \circ A: \mathcal{Y} \mapsto \{ 0,1 \}^{ \lceil \log ( | \mathcal{G} |  ) \rceil}$, and note that, for all $y \in \mathcal{Y}$, 
			\begin{align}
				\delta ( y, D ( E ( y ) )  ) =&\, \delta ( y, D ( \tilde{E}( A ( y ) )) ) \\
				=&\, \delta ( y, A ( y )  ) \label{eq:253}\\
				\leq &\, \varepsilon, \label{eq:254}
			\end{align} 
			where in \eqref{eq:253} we used \eqref{eq:lossless} with $A ( y ) \in \mathcal{G}$, and \eqref{eq:254}  follows from \eqref{eq:projection}. In summary, $( E,D )$ constitutes an encoder-decoder pair of length $\lceil \log ( \nleft| \mathcal{G} \nright|  )\rceil$ achieving uniform error $\varepsilon$ over the set $\mathcal{Y}$. By the minimality of $\ell ( \varepsilon, \mathcal{Y}, ( \mathcal{X}, \delta ) )$, we deduce that  $ \lceil \log ( | \mathcal{G} |  ) \rceil  \geq \ell ( \varepsilon, \mathcal{Y}, ( \mathcal{X}, \delta ) )$. \qedhere
		\end{proof}
	\end{proposition}

	Proposition~\ref{thm:memory_requirement} states that a finite set  of approximants $\mathcal{G}$ achieving minimax error $\varepsilon$ in the approximation of $\mathcal{Y}$, under the metric $\delta$, requires at least $\ell ( \varepsilon, \mathcal{Y}, ( \mathcal{X}, \delta ) )$ bits to encode.

	This insight allows us to quantify the redundancy of a set of approximants. Specifically, consider the approximation of the set $\mathcal{Y}$ by  the set $\mathcal{G}$ with minimax error $\varepsilon := \mathcal{A} ( \mathcal{Y} , \mathcal{G}, \delta )$.   In the case $\ell ( \varepsilon, \mathcal{Y}, ( \mathcal{X}, \delta ) ) \geq 1$, we quantify redundancy in a multiplicative manner by defining it according to $ \frac{\lceil \log ( | \mathcal{G} |  )\rceil  }{\ell ( \varepsilon, \mathcal{Y}, ( \mathcal{X}, \delta ))}$. When $\ell ( \varepsilon, \mathcal{Y}, ( \mathcal{X}, \delta ) ) =  0$, we have to work with an additive redundancy measure, which we take to be $\lceil \log ( | \mathcal{G} |  )\rceil   - \ell ( \varepsilon, \mathcal{Y}, ( \mathcal{X}, \delta ) = \lceil \log ( | \mathcal{G} |  )\rceil  $. Instead of carrying along two separate redundancy measures, we will simply use $\frac{  \lceil \log ( | \mathcal{G} |  ) \rceil  }{1  +\ell ( \varepsilon, \mathcal{Y}, ( \mathcal{X}, \delta ))}$ to quantify redundancy. To see that this makes sense, we note that both multiplicative and additive redundancy, within their corresponding applicability regimes, are sandwiched between\footnote{When $\ell ( \varepsilon, \mathcal{Y}, ( \mathcal{X}, \delta ) ) \geq 1$, we have $\frac{\lceil \log ( | \mathcal{G} |  )\rceil }{1  +\ell ( \varepsilon, \mathcal{Y}, ( \mathcal{X}, \delta ))} \leq \frac{\lceil \log ( | \mathcal{G} |  )\rceil  }{\ell ( \varepsilon, \mathcal{Y}, ( \mathcal{X}, \delta ))} \leq \frac{2 \lceil \log ( | \mathcal{G} |  )\rceil }{1  +\ell ( \varepsilon, \mathcal{Y}, ( \mathcal{X}, \delta ))}  $. For $\ell ( \varepsilon, \mathcal{Y}, ( \mathcal{X}, \delta ) ) = 0$, it follows that $\frac{ \lceil \log ( | \mathcal{G} |  )\rceil }{1  +\ell ( \varepsilon, \mathcal{Y}, ( \mathcal{X}, \delta ))} = \lceil \log ( | \mathcal{G} |  )\rceil  < 2 \lceil \log ( | \mathcal{G} |  )\rceil = \frac{2 \lceil \log ( | \mathcal{G} |  )\rceil }{1  +\ell ( \varepsilon, \mathcal{Y}, ( \mathcal{X}, \delta ))}. $} $\frac{ \lceil \log ( | \mathcal{G} |  )\rceil }{1  +\ell ( \varepsilon, \mathcal{Y}, ( \mathcal{X}, \delta ))}$ and  $\frac{ 2 \lceil \log ( | \mathcal{G} |  ) \rceil }{1  +\ell ( \varepsilon, \mathcal{Y}, ( \mathcal{X}, \delta ))}$.

	Often we shall be dealing with families of approximants $\{ \mathcal{G}_i \}_{i \in \mathcal{I}}$ parametrized by a, not necessarily ordered, index set $\mathcal{I}$. This concept will allow us to consider neural network families indexed by their architectures and weight sets, aiming for different levels of approximation error. Specifically, we shall frequently take $\mathcal{I}  \subseteq \mathbb{N}^3$, and, for $(W,L,b) \in \mathcal{I}$, set $i = (W,L,b) $ and $\mathcal{G}_{i} = \mathcal{G}_{(W,L,b)} = \mathcal{R}_b^1 ( W,L )$.

	We will say that $\{ \mathcal{G}_i \}_{i \in \mathcal{I}}$ approximates $\mathcal{X}$ in a memory-optimal fashion if the approximation error can be made arbitrarily small while ensuring that the memory redundancy remains bounded, as formalized next.

	\begin{definition}
		[Memory redundancy and memory optimality] 
		\label{def:memory_optimality}
		Let $(\mathcal{X}, \delta)$ be a metric space and $\mathcal{Y} \subseteq \mathcal{X}$. We define the memory redundancy in  the approximation of $\mathcal{Y}$ by a subset $\mathcal{G} \subseteq \mathcal{X}$  as 
		\begin{equation*}
			r ( \mathcal{Y}, \mathcal{G}, \rho  ) :=  \frac{\lceil \log ( | \mathcal{G} |  )\rceil }{1 + \ell ( \mathcal{A} ( \mathcal{Y}, \mathcal{G}, \rho), \mathcal{Y}, ( \mathcal{X}, \delta ))}.
		\end{equation*}
		A family of finite subsets $\{ \mathcal{G}_i \}_{i \in \mathcal{I}} \subseteq \mathcal{X}$ is said to achieve memory optimality in the approximation of $\mathcal{Y}$ if 
		\begin{align}
			\inf_{i \in \mathcal{I}} \mathcal{A} ( \mathcal{Y}, \mathcal{G}_i, \rho) =&\, 0, \text{ and } \label{eq:arbitrary_error}\\
			\sup_{i \in \mathcal{I}}\, r ( \mathcal{Y}, \mathcal{G}_i, \rho  ) <&\, \infty. \label{eq:bounded_redundancy}
		\end{align}
	\end{definition}

	Recall that our main focus is the minimax error $\mathcal{A}_\infty( \mathcal{F}, \mathcal{G})$ with $\mathcal{F} = \lip ( [0,1] )$ and $\mathcal{G} = \mathcal{R}_b^1(W,L)$, for $W,L,b \in \mathbb{N}$. To analyze the associated memory redundancy, we hence need to characterize $\nleft| \mathcal{R}_b^1 ( W,L ) \nright| $ and the minimax code length $ \ell ( \varepsilon, \lip ( [0,1] ), ( L^\infty ( [0,1] ), \nleft\|  \cdot \nright\|_{L^\infty ( [0,1] )} ) ) $, $\varepsilon \in \mathbb{R}_+$, short-handed as $\ell ( \varepsilon, \lip ( [0,1] ))$.

	\subsection{Upper-bounding the cardinality of $\mathcal{R}_b^1 (W,L)$}

	We shall first establish an upper bound on the cardinality of $\mathcal{R}_\mathbb{A} ( ( d,d' ),W,L )$ for general $\mathbb{A}$ and then 
  particularize this bound for $\mathbb{A} = \mathbb{Q}_b^1$. The more general result does not demand any extra technical effort and makes for a more accessible exposition.

	Let us start with some heuristic reasoning. We can store a network realization in $\mathcal{R}_\mathbb{A} ( ( d,d' ),\allowbreak W,L)$ by storing its corresponding network configuration in $\mathcal{N}_\mathbb{A} ( ( d,d' ),W, L)$.  A given network configuration in $\mathcal{N}_\mathbb{A} ( ( d,d' ),W, L)$ has at most $W( W+1 )L $ weights, and each weight needs $\lceil \log ( \nleft| \mathbb{A} \nright|  )\rceil$ bits to represent it. It therefore takes at most $W( W+1 )L\lceil \log ( \nleft| \mathbb{A} \nright|  )\rceil$ bits to store all the weights in the network configuration. Storing the network depth and the widths of the individual layers, requires an extra $\lceil \log(L) \rceil$ bits and $L \lceil \log(W) \rceil$ bits, respectively, which can be absorbed by a constant multiplying $W( W+1 )L\lceil \log ( \nleft| \mathbb{A} \nright|  )\rceil$. In total we hence need at most  $C W( W+1 )L\lceil \log ( \nleft| \mathbb{A} \nright|  )\rceil$ bits, with $C$ an absolute constant. This intuitive reasoning is formalized in the following result.

	\begin{proposition}
		\label{prop:fundamental_limit_ReLU_networks_quantized_weights}
		For  $d,d',W,L \in \mathbb{N}$ and a finite subset $\mathbb{A} \subseteq \mathbb{R}$ with $| \mathbb{A} | \geq 2$,  we have
		\begin{equation}
			\label{eq:cardinality_realization}
			\log  (| \mathcal{R}_\mathbb{A} ( ( d,d' ),W,L ) |) \leq \log  (| \mathcal{N}_\mathbb{A} ( ( d,d' ),W,L ) |)   \leq 5 W^2 L \log  (| \mathbb{A} |).
		\end{equation}
		In particular, for $\mathbb{A} = \mathbb{Q}_b^a$, $a,b \in \mathbb{N}$, 
		\begin{equation}
			\label{eq:cardinality_realization_a_b}
			\log  (| \mathcal{R}_b^a ( ( d,d' ),W,L ) |) \leq \log  (| \mathcal{N}_b^a ( ( d,d' ),W,L ) |)   \leq 10 W^2 L (a+b).
		\end{equation}
	\end{proposition}

            \begin{proof}
			By definition,  
			\begin{align*}
				&\,\mathcal{N}_\mathbb{A} ( ( d,d' ),W,L) \\
                    &\subseteq \, \{ ( A_i, b_i)_{i = 1}^{\ell} \in \mathcal{N} ( ( d,d' ) ): \mathcal{W} (( A_i, b_i)_{i = 1}^{\ell}) \leq W, \ell \leq L, \coef (( A_i, b_i)_{i = 1}^{\ell} ) \subseteq \mathbb{A}  \}.
			\end{align*}
			Recall that, for a given network configuration $( A_i, b_i)_{i = 1}^{\ell} \in \mathcal{N}_\mathbb{A} ( ( d,d' ),W,L)$ with $A_i \in \mathbb{R}^{N_{i} \times N_{i-1}}$, $ N_{i - 1},N_i \in \mathbb{N}$, $i = 1,\dots, \ell$,  we call the tuple $( N_0,\dots, N_{\ell} )$ the architecture of $( A_i, b_i)_{i = 1}^{\ell}$. For given $\ell$, there are at most $W^{\ell + 1}$ different architectures. As $\ell \in \{ 1,\dots, L \}$, the total number of possible architectures $( N_0,\dots, N_{\ell} )$ for network configurations in $\mathcal{N}_\mathbb{A} ( ( d,d' ),W,L)$ is hence upper-bounded by $\sum_{\ell = 1}^{L} W^{\ell + 1} \leq L W^{L+1}$. For a given network configuration $\Phi$ with architecture $( N_0,\dots, N_{\ell} )$, the number of weights satisfies $\sum_{i = 1}^{\ell} ( N_{i} \, N_{i - 1} + N_i) \leq L W ( W+1 ) $; therefore, the number of possible $\Phi$ of a given architecture is no more than  $\nleft| \mathbb{A} \nright|^{L W ( W+1 )} $, as each weight can take $\nleft| \mathbb{A} \nright|$ different values. Putting everything together, we obtain
			\begin{align*}
				| \mathcal{N}_\mathbb{A} ( ( d,d' ),W,L )   | \leq L W^{L+1}\nleft| \mathbb{A} \nright|^{L W ( W+1 )},
			\end{align*}
			which, in turn, implies
			\begin{equation}
			\label{eq:counting_new}
			\begin{aligned}
				\log (| \mathcal{N}_\mathbb{A} ( ( d,d' ),W,L )  |) \leq&\, \log (L W^{L+1}) +  (L(W^2 + W)) \log (| \mathbb{A} |) \\
				\leq &\,  \log(L) + (L+1) \log(W) + 2 W^2 L \log( | \mathbb{A} |)\\
				\leq&\, 5 W^2 L \log( | \mathbb{A} |).
			\end{aligned}
			\end{equation} 
			Noting that $| \mathcal{R}_\mathbb{A} (  ( d,d' ),W,L ) | \leq | \mathcal{N}_\mathbb{A} ( ( d,d' ),W,L ) |$, yields \eqref{eq:cardinality_realization}. Finally, \eqref{eq:cardinality_realization_a_b} follows by using $\log (| \mathbb{Q}_b^a |) = \log (| \{ \pm \sum_{i = -b}^a \theta_i 2^{i}: \theta_i \in \{ 0,1 \} \} |)  \leq a+b+2 \leq 2 ( a+b )$. 
		\end{proof}
		Proposition~\ref{prop:fundamental_limit_ReLU_networks_quantized_weights} provides an upper bound on the memory required to store the network realizations in $\mathcal{R}_\mathbb{A} ( ( d,d' ),W,L )$. As this storage method mirrors how neural networks are stored on a computer,  we term it the natural encoding and refer to \eqref{eq:cardinality_realization} as the memory consumption upper bound under natural encoding. In contrast, \cite{deep-it-2019}  considers networks which are sparse in the sense of having a small number of nonzero weights, 
        and stores only the nonzero weights and their respective locations as uniquely decodable bitstrings.

	\subsection{Lower-bounding the minimax code length $\ell ( \varepsilon, \lip ( [0,1] ))$ } % (fold)
	\label{sub:characterization_of_the_minimax_code_length}
		We next lower-bound the minimax code length by relating it to the covering number and the packing number defined next. 
		\begin{definition}
		[Covering number and packing number] \cite[Definitions 5.1 and 5.4]{wainwright2019high}
		Let $(\mathcal{X}, \delta)$ be a metric space. An $\varepsilon$-covering of $\mathcal{X}$ is a finite set $\{ x_1, \dots,x_n \}$ of $\mathcal{X}$ such that for all $x \in \mathcal{X}$, there exists an $i \in \{ 1,\dots,n \}$ so that $\delta ( x,x_i )\leq \varepsilon$.  The $\varepsilon$-covering number $N ( \varepsilon,\mathcal{X}, \delta )$ is the cardinality of a smallest $\varepsilon$-covering of $\mathcal{X}$. An $\varepsilon$-packing of $\mathcal{X}$ is a finite subset $\{ x_1, \dots,x_n \}$ of $\mathcal{X}$ such that $\delta ( x_i,x_j ) > \varepsilon$, for all $i,j \in \{ 1,\dots, n \}$ with $i\neq j$. The $\varepsilon$-packing number $M ( \varepsilon, \mathcal{X}, \delta ) $ is the cardinality of a largest $\varepsilon$-packing of $\mathcal{X}$.
	\end{definition}
	An important relation between the covering number and the packing number is the following.
	\begin{lemma}
		\label{lem:equivalence_covering_packing}
		\cite[Lemma 5.5]{wainwright2019high} For a metric space $( \mathcal{X}, \delta )$ and $\varepsilon \in \mathbb{R}_+$, it holds that
		\begin{equation*}
			M ( 2 \varepsilon, \mathcal{X}, \delta ) \leq N ( \varepsilon, \mathcal{X}, \delta ) \leq M ( \varepsilon, \mathcal{X}, \delta ).
		\end{equation*}
	\end{lemma}

	The minimax code length can be related to the covering and the packing numbers as follows.

	\begin{lemma}
		\label{lem:equivalence_covering_number_code_length}
		Let $(\mathcal{X}, \delta)$ be a metric space, $\mathcal{Y} \subseteq \mathcal{X}$, and $\varepsilon \in \mathbb{R}_+$. We have
		\begin{equation}
			\label{eq:relation_between_covering_number_and_minimum_code_length}
			\log (M ( 2 \varepsilon, \mathcal{Y}, \delta )) \leq \ell ( \varepsilon, \mathcal{Y}, ( \mathcal{X}, \delta ) ) \leq  \lceil \log (N ( \varepsilon, \mathcal{Y}, \delta )) \rceil.
		\end{equation}

		\begin{proof}
			We first prove the inequality ${\ell ( \varepsilon, \mathcal{Y}, ( \mathcal{X}, \delta ) )} \geq \log (M ( 2 \varepsilon, \mathcal{Y}, \delta )) $. To this end, let $(E:\mathcal{Y} \mapsto \{ 0,1 \}^{\ell ( \varepsilon, \mathcal{Y}, ( \mathcal{X}, \delta ) )} ,D:\{ 0,1 \}^{\ell ( \varepsilon, \mathcal{Y}, ( \mathcal{X}, \delta ) )} \mapsto \mathcal{X})$ be an encoder-decoder pair achieving uniform error $\varepsilon$ over the set $\mathcal{Y}$, i.e., %with minimax code length $\ell ( \varepsilon, \mathcal{Y}, ( \mathcal{X}, \delta ) )$, i.e.,
			\begin{equation}
			\label{eqline:240}
				\delta( y, D (E (y))) \leq \varepsilon, \quad \text{for all } y \in \mathcal{Y},
			\end{equation}
			and let $\mathcal{P}$ be a largest $(2 \varepsilon)$-packing of $\mathcal{Y}$, i.e., $| \mathcal{P} | =  M ( 2 \varepsilon, \mathcal{Y}, \delta ) $. For $| \mathcal{P} |  = 1 $, the first inequality in \eqref{eq:relation_between_covering_number_and_minimum_code_length} is trivially satisfied as $\ell ( \varepsilon, \mathcal{Y}, ( \mathcal{X}, \delta ) ) \geq 0$ by definition.
   %= \log ( | \mathcal{P} |  ) = \log ( M ( 2 \varepsilon, \mathcal{Y}, \delta ) )$. 
   In the case $| \mathcal{P} | \geq 2$,  we have, for distinct $p_1, p_2 \in \mathcal{P}$,
			\begin{align}
				\delta ( D (E (p_1)), D (E (p_2)) ) \geq&\, \delta ( D (E (p_1)), p_2) - \delta ( p_2, D (E (p_2)) )  \label{eqline:241}\\
				\geq&\, \delta ( p_1, p_2) - \delta ( D (E (p_1)), p_1 ) - \delta ( p_2, D (E (p_2)) ) \label{eqline:242}\\
				>&\, 2 \varepsilon - \varepsilon - \varepsilon \label{eqline:2421} \\
				=&\, 0,
			\end{align}
			where \eqref{eqline:241} and \eqref{eqline:242} follow from the triangle inequality, and in \eqref{eqline:2421} we used \eqref{eqline:240} and $\delta ( p_1, p_2) > 2 \varepsilon$ owing to $\mathcal{P}$ being a $(2 \varepsilon)$-packing. We can hence conclude that $D (E (p_1)) \neq D (E (p_2))$ and have thereby established the injectivity of $D \circ E $ on $\mathcal{P}$. Consequentially, $E$ must also be injective on $\mathcal{P}$. This, in turn, implies that the cardinality of the range of $E$ is no less than the cardinality of $\mathcal{P}$, namely, $2^{\ell ( \varepsilon, \mathcal{Y}, ( \mathcal{X}, \delta ) )} \geq \nleft| \mathcal{P} \nright| = M ( 2 \varepsilon, \mathcal{Y}, \delta )$, and therefore $\log (M ( 2 \varepsilon, \mathcal{Y}, \delta )) \leq \ell ( \varepsilon, \mathcal{Y}, ( \mathcal{X}, \delta ) )$.

			It remains to show that $\ell ( \varepsilon, \mathcal{Y}, ( \mathcal{X}, \delta ) ) \leq  \lceil \log (N ( \varepsilon, \mathcal{Y}, \delta )) \rceil$. Let $\mathcal{C} = \{ c_i \}_{i = 1}^{N ( \varepsilon, \mathcal{Y}, \delta )} $ be a minimal $\varepsilon$-covering of $\mathcal{Y}$. Hence, $\mathcal{A} ( \mathcal{Y}, \mathcal{C}, \delta ) = \sup_{y \in \mathcal{Y}} \inf_{c \in \mathcal{C}}  \delta ( y,c ) \leq \varepsilon$. Application of Proposition~\ref{thm:memory_requirement} with $\mathcal{G} = \mathcal{C}$ yields $\lceil \log ( | \mathcal{C} | ) \rceil \geq \ell ( \varepsilon, \mathcal{Y}, ( \mathcal{X}, \delta ) )$, which together with $N ( \varepsilon, \mathcal{Y}, \delta )  = | \mathcal{C} |  $  establishes $\ell ( \varepsilon, \mathcal{Y}, ( \mathcal{X}, \delta ) ) \leq \lceil \log (N ( \varepsilon, \mathcal{Y}, \delta )) \rceil$.
		\end{proof}
	\end{lemma}

	\begin{remark}
		With the insights provided by Lemma~\ref{lem:equivalence_covering_number_code_length}, Proposition~\ref{thm:memory_requirement} and Definition~\ref{def:memory_optimality} could equivalently have been formulated in terms of packing and covering number. We decided, however, to work with the minimax code length so as to emphasize the implications of our results in terms of memory consumption.
	\end{remark}

	We can now lower-bound the minimax code length of $\lip ( [0,1] )$ using a lower bound on the covering number of $\lip ( [0,1] )$.

		\begin{lemma}
		\label{lem:minimax_code_length_lower_bound}
		For $\varepsilon > 0$, there exist absolute constants $C,\varepsilon_0 > 0$, such that 
		\begin{equation}
		\label{eq:minimax_code_length_H1}
		 	\ell ( \varepsilon, \lip ( [0,1] ) ) \geq C \,\varepsilon ^{-1}, \quad  \forall \varepsilon \in (0,  \varepsilon_0].
		 \end{equation} 
		\begin{proof}
			By \cite[Example 5.10]{wainwright2019high}, we have 
			\begin{equation}
			\label{eq:covering_packing_H1}
				\log (N ( \varepsilon, \lip  ( [0,1] ), \|\cdot \|_{L^\infty ( [0,1] ) }  )) \geq c\,\varepsilon^{-1}, \quad  \forall \varepsilon \in (0,  \varepsilon_1],
			\end{equation} 
			for absolute constants $c, \varepsilon_1 \in \mathbb{R}_+$. Set $\varepsilon_0 = \frac{1}{2} \varepsilon_1$ and $C = \frac{1}{2} c$. Then, for $\varepsilon \leq \varepsilon_0$, we have 
			\begin{align}
				\ell ( \varepsilon, \lip ( [0,1] ) ) \geq&\,  \log (M ( 2\varepsilon, \lip  ( [0,1] ), \|\cdot\|_{L^\infty ( [0,1] ) }  ))\label{eq:291}\\ 
				\geq&\, \log (N ( 2\varepsilon, \lip  ( [0,1] ), \|\cdot \|_{L^\infty ( [0,1] ) }  )) \label{eq:292}\\
				\geq&\,  c \,( 2 \varepsilon )^{-1} \label{eq:293}\\
				\geq&\, C \varepsilon^{-1},\label{eq:294}
			\end{align}
			where in \eqref{eq:291} we used Lemma~\ref{lem:equivalence_covering_number_code_length}, \eqref{eq:292} follows from Lemma~\ref{lem:equivalence_covering_packing}, and in \eqref{eq:293} we applied \eqref{eq:covering_packing_H1}.
		\end{proof}
	\end{lemma}

	% subsection characterization_of_the_minimax_code_length (end)

	\subsection{Lower bound incurred by minimum memory requirement} % (fold)
	\label{sub:approximation_error_lower_bounds_due_to_underquantization}
		We now  have all the ingredients to derive a lower bound on the minimax error incurred by neural network approximation of $1$-Lipschitz functions. The specific result is as follows.
		\begin{proposition}
			\label{proposition:lower_bound_approximation}
			There exists an absolute constant $c_m$
			% $C := \min \{ c_1, C_2 \}$ 
			such that for all $W,L,b \in \mathbb{N}$, it holds that
			\begin{equation}
			\label{eq:lower_bound_a_b}
				\mathcal{A}_\infty	( \lip   ( [0,1] ), \mathcal{R}_b^1 (W, L  )  ) \geq c_m ( W^2 L b  )^{-1}.
			\end{equation}

			\begin{proof}
				Let $C,\varepsilon_0$ be the absolute constants in Lemma~\ref{lem:minimax_code_length_lower_bound} and set $c_m = \min \{ \frac{C}{30}, \varepsilon_0 \}$. Suppose, for the sake of contradiction, that 
				\begin{equation}
				\label{eq:contradiction_assumption_28}
					\mathcal{A}_\infty	( \lip   ( [0,1] ), \mathcal{R}_b^1 (W, L  )  ) < c_m (  W^2 L b   )^{-1}.
				\end{equation}
				Set $\varepsilon = c_m( W^2 L b )^{- 1}$.
				Applying Proposition~\ref{thm:memory_requirement} with $\mathcal{Y} = \lip   ( [0,1] ) $, $\mathcal{G} = \mathcal{R}_b^1 (W, L  )$, $\delta = \nleft\| \cdot \nright\|_{L^\infty ( [0,1] )}$, upon noting that the prerequisites in Proposition~\ref{thm:memory_requirement} are satisfied as $\mathcal{G} = \mathcal{R}_b^1 (W, L  )$ and $\eqref{eq:contradiction_assumption_28}$ implies $\mathcal{A}_\infty	( \lip   ( [0,1] ), \mathcal{R}_b^1 (W, L  )  ) < c_m (  W^2 L b   )^{-1} = \varepsilon$, we obtain
				\begin{equation*}
					\lceil \log ( \nleft| \mathcal{R}_b^1 (W, L  ) \nright|  )\rceil \geq \ell ( \varepsilon, \lip ( [0,1] ) ),
				\end{equation*}
				which together with $\lceil\log ( \nleft| \mathcal{R}_b^1 (W, L  ) \nright|  )\rceil \leq 10 W^2 L  ( 1+b ) \leq 20 W^2 Lb$, thanks to Proposition~\ref{prop:fundamental_limit_ReLU_networks_quantized_weights}, establishes
				\begin{equation}
				\label{eq:binary_h1_2000}
					20 W^2 Lb \geq \ell ( \varepsilon, \lip ( [0,1] ) ).
				\end{equation}
				On the other hand,
				\begin{align}
					20 W^2 Lb = &\, 20\, c_m\, \varepsilon^{-1} \label{eq:binary_h1_20}\\
					< &\, C  \varepsilon^{-1}\label{eq:binary_h1_3}\\
					\leq&\,\ell ( \varepsilon, \lip ( [0,1] ) ) ,\label{eq:binary_h1_4}
				\end{align}
				where  \eqref{eq:binary_h1_3} follows from $20 c_m = 20\min \{ \frac{C}{30}, \varepsilon_0 \} < C$, and in \eqref{eq:binary_h1_4} we applied  Lemma~\ref{lem:minimax_code_length_lower_bound} with the prerequisite satisfied as $\varepsilon = c_m( W^2 L b  )^{- 1} \leq c_m \leq \varepsilon_0$. Since \eqref{eq:binary_h1_2000} contradicts the strict inequality \eqref{eq:binary_h1_20}-\eqref{eq:binary_h1_4}, we must have
				\begin{equation*}
					\mathcal{A}_\infty	( \lip   ( [0,1] ), \mathcal{R}_b^1 (W, L )  ) \geq  c_m ( W^2 L b )^{-1}. \qedhere
				\end{equation*}
			\end{proof}
		\end{proposition}

		We shall refer to \eqref{eq:lower_bound_a_b} as the minimax error lower bound incurred by the minimum memory requirement. Attaining this lower bound to within a multiplicative constant, implies memory optimality, as demonstrated next.

		\begin{proposition}
		\label{prop:converse_lower_bound}
			Let $\mathcal{I} \subseteq \mathbb{N}^3$ be an infinite set. Suppose that
			\begin{equation}
			\label{eq:memory_bound_attained}
				\mathcal{A}_\infty	( \lip   ( [0,1] ), \mathcal{R}_b^1 (W, L  )  ) \leq \cupper ( W^2 L b )^{-1},\quad \forall ( W,L, b) \in \mathcal{I}, 
			\end{equation}
			for some $\cupper \in \mathbb{R}_+$ independent of $W,L,b$. Then,
			\begin{align}
				\inf_{( W,L, b) \in \mathcal{I}} \mathcal{A}_\infty	( \lip   ( [0,1] ), \mathcal{R}_b^1 (W, L  )  ) =&\, 0, \label{eq:arbitrary_error_neural}\\
				\sup_{( W,L, b) \in \mathcal{I}}\, r ( \lip   ( [0,1] ), \mathcal{R}_b^1 (W, L  ), \nleft\| \cdot \nright\|_{L^\infty ( [0,1] )}   ) <&\, \infty, \label{eq:bounded_redundancy_neural}
			\end{align}
			and hence $\{R_b^1 (W,L):(W,L,b) \in \mathcal{I}\}$ achieves memory optimality---in the sense of Definition~\ref{def:memory_optimality}---in the approximation of $\lip([0,1])$.

		\begin{proof}
			We first note that
			\begin{align}
				\inf_{( W,L, b) \in \mathcal{I}} \mathcal{A}_\infty ( \lip ( [0,1] ),  \mathcal{R}_b^1 (W, L  )  ) \leq&\, \inf_{( W,L, b) \in \mathcal{I}} \cupper ( W^2 L b )^{-1} \label{eq:first_condition_1}\\
				=& \, 0,\label{eq:first_condition_2}
			\end{align}
			where \eqref{eq:first_condition_2} follows as $\mathcal{I}$ is an infinite set\footnote{For $\mathcal{I}$ infinite, we have $\sup_{( W,L, b) \in \mathcal{I}} \max \{ W,L,b \} = \infty$, which implies $\inf_{( W,L, b) \in \mathcal{I}} \cupper ( W^2 L b )^{-1} \leq \inf_{( W,L, b) \in \mathcal{I}} \cupper ( \max \{ W,L,b \} )^{-1} = 0$.}. This establishes \eqref{eq:arbitrary_error_neural}. To prove \eqref{eq:bounded_redundancy_neural}, we first fix a tuple $( W,L,b ) \in \mathcal{I}$ and consider the memory redundancy 
			\begin{align}
				r ( \lip ( [0,1] ),  \mathcal{R}_b^1 (W, L  ) ,\nleft\| \cdot \nright\|_{L^\infty ( [0,1] )}  ) =&\, \frac{\lceil \log ( \nleft| \mathcal{R}_b^1 (W, L  ) \nright|  )\rceil }{1 + \ell ( \mathcal{A}_\infty	( \lip   ( [0,1] ), \mathcal{R}_b^1 (W, L  )  ), \lip ( [0,1] ))} \label{eq:show_bounded_redundancy_1}\\
				\leq&\, \frac{ 10 W^2 L ( 1+b )}{1 + \ell ( \mathcal{A}_\infty	( \lip   ( [0,1] ), \mathcal{R}_b^1 (W, L  )  ), \lip ( [0,1] ))},\label{eq:show_bounded_redundancy_2}
			\end{align}
			where \eqref{eq:show_bounded_redundancy_2} follows from Proposition~\ref{prop:fundamental_limit_ReLU_networks_quantized_weights}.
			We now distinguish two cases. First, for $\cupper ( W^2 L b )^{-1} > \varepsilon_0$, where $\varepsilon_0$ is the absolute constant in Lemma~\ref{lem:minimax_code_length_lower_bound},  we have $W^2 L (1+b) < 2 W^2 L b< 2 D \varepsilon_0^{-1}$, which together with \eqref{eq:show_bounded_redundancy_1}-\eqref{eq:show_bounded_redundancy_2} leads to the memory redundancy upper bound $r ( \lip ( [0,1] ), \allowbreak  \mathcal{R}_b^1 (W, L  ) ,\nleft\| \cdot \nright\|_{L^\infty ( [0,1] )}  ) \leq 10 W^2 L ( 1+b ) \leq 20 D \varepsilon_0^{-1}$. Second, for $\cupper ( W^2 L b)^{-1} \leq \varepsilon_0$,  we have  
			\begin{align}
				&\ell ( \mathcal{A}_\infty	( \lip   ( [0,1] ), \mathcal{R}_b^1 (W, L  )  ), \lip ( [0,1] ))\\
				& \geq \, \ell ( \cupper ( W^2 L b )^{-1}, \lip ( [0,1] )) \label{eq:monotonicity_minimum_code}\\
				& \geq \, C \cupper^{-1}  W^2 L b, \label{eq:monotonicity_minimum_code_2}
			\end{align}
			where in \eqref{eq:monotonicity_minimum_code} we used  \eqref{eq:memory_bound_attained} together with the fact that $\varepsilon \mapsto \ell ( \varepsilon, \lip ( [0,1] ))$ is a nonincreasing function, and \eqref{eq:monotonicity_minimum_code_2} follows from Lemma~\ref{lem:minimax_code_length_lower_bound}, with $C$ being the  absolute constant from Lemma~\ref{lem:minimax_code_length_lower_bound}. Then, the memory redundancy can be upper-bounded according to 
			\begin{equation*}
				r ( \lip ( [0,1] ),  \mathcal{R}_b^1 (W, L  ) ,\nleft\| \cdot \nright\|_{L^\infty ( [0,1] )}  )\leq \frac{10 W^2 L ( 1+b )}{1 +C \cupper^{-1}  W^2 L b} \leq 20 D C^{-1}.
			\end{equation*}
			Combining the two cases, we get
			\begin{equation*}
				r ( \lip ( [0,1] ),  \mathcal{R}_b^1 (W, L  ) ,\nleft\| \cdot \nright\|_{L^\infty ( [0,1] )}  ) \leq \max \{ 20 D \varepsilon_0^{-1}, 20 D C^{-1} \} = 20 D \max \{ \varepsilon_0^{-1}, C^{-1} \}.
			\end{equation*}
			Recalling that $\varepsilon_0$ and $C$ are absolute constants, and the tuple $( W,L,b ) \in \mathcal{I}$ is fixed but arbitrary, it follows that
			\begin{equation*}
				\sup_{( W,L, b) \in \mathcal{I}} r ( \lip ( [0,1] ),  \mathcal{R}_b^1 (W, L  ) ,\nleft\| \cdot \nright\|_{L^\infty ( [0,1] )}  ) \leq \sup_{( W,L, b) \in \mathcal{I}} 20 D \max \{ \varepsilon_0^{-1}, C^{-1} \}  < \infty,
			\end{equation*}
			where we used $D < \infty$ and $\varepsilon_0, C > 0$.
			This validates \eqref{eq:bounded_redundancy_neural} and thereby finalizes the proof.
   %the second condition in Definition~\ref{def:memory_optimality} and thereby establishes the stated memory optimality. 
		\end{proof}

		\end{proposition}

		\subsection{Two additional lower bounds} % (fold)

		We proceed to establish two additional minimax error lower bounds. 
  %The first one is based on a combination of the technique in \cite{shen2021optimal} for lower-bounding the minimax error for unquantized networks, i.e., %networks with real-valued weights, with an upper bound on the VC dimension of ReLU networks reported in \cite{bartlett2019nearly}. {\color{blue} Remark: %\cite{shen2021optimal} used already \cite{bartlett2019nearly} for its low bound on the minimax error for unquantized networks. We didn't provide anything %new, not even combining results. Here is a possible new sentence: 
  The first one adapts the technique in \cite{shen2021optimal} for lower-bounding the minimax error for unquantized networks, i.e., networks with real-valued weights. Notably, \cite{shen2021optimal} uses an upper bound on the VC dimension of ReLU networks reported in \cite{bartlett2019nearly}. While this adaptation in itself is not substantial, we still feel that the underlying idea is worthy of recording, also in the sense of clarity and completeness of exposition.

		\begin{proposition} 
			\label{prop:lower_bound_approximation_VC_dimension}
			There exists an absolute constant $c_v$ such that for all $W \in \mathbb{N}$ and $L \in \mathbb{N}$ with $L \geq 2$, it holds that
			\begin{equation}
			\label{eq:lower_bound_VC_dimension}
				\mathcal{A}_\infty	( \lip  ( [0,1] ), \mathcal{R} (W, L  )  )  \geq  c_v ( W^2 L^2 ( \log (W) + \log (L) ) )^{-1},
			\end{equation}
			and, hence, for  nonempty  $\mathbb{A} \subseteq \mathbb{R}$, 
			\begin{equation}
				\label{eq:lower_bound_VC_dimension_general_weights}
				\mathcal{A}_\infty	( \lip  ( [0,1] ), \mathcal{R}_\mathbb{A} (W, L)  )  \geq  c_v ( W^2 L^2 ( \log (W) + \log (L) ) )^{-1}.
			\end{equation}
			\begin{proof}
				See Appendix~\ref{sub:proof_of_proposition_prop:lower_bound_approximation_vc_dimension}.
			\end{proof}
		\end{proposition}
		We will refer to  \eqref{eq:lower_bound_VC_dimension_general_weights} as the  minimax error lower bound incurred by the VC-dimension limit. Additionally, we point out that this bound indicates an advantage of deep networks over shallow networks. Specifically, fixing the number of network weights, which is on the order of $n:= W^2L$, an increase in depth $L$ leads to a decrease in the minimax error lower bound according to $c_v ( W^2 L^2 ( \log (W) + \log (L) ) )^{-1} = c_v ( nL ( \log (\sqrt{\frac{n}{L}}) + \log (L) ) )^{-1} = c_v \bigl(  nL \log\bigl(\sqrt{nL}\bigr)\bigr)^{-1}$. In contrast, fixing $n$ and increasing $W$ leads to an increase in the minimax error lower bound according to
        $c_v ( \frac{n^{2}}{W^{2}} \log (\frac{n}{W}) )^{-1}$. This advantage of deep over shallow networks will manifest itself in our final characterization of the three quantization regimes.

		The second bound we present is based on the observation that ReLU networks with quantized weights face inherent limitations in their approximation capability. The nature of these limitations is such that deep networks exhibit a fundamental advantage over shallow networks. To illustrate this aspect, consider ReLU networks with quantized weights of a fixed number of fractional bits\footnote{The number of fractional bits of $x \in \mathbb{R}$ refers to the number of digits after the binary point in the binary representation of $x$.} and inputs $x\in[0,1]$ also of a fixed number of fractional bits. The corresponding network outputs will also exhibit a fixed number of fractional bits. More importantly, this number increases with increasing $L$, but remains constant as a function of $W$. This is a consequence of the multiplication of two real numbers corresponding to the convolution of their binary expansions and the length of the convolution of two sequences being given by the sum of their lengths. More informally, multiplying small numbers in $[0,1]$ (deep network case) leads to even smaller numbers whereas adding them (shallow network case) can make them only larger and therefore does not result in an increase in numerical resolution. Finally, realizing that $f \in \lip ( [0,1] )$ can take arbitrary values, in particular, values with an infinite number of fractional bits, it follows that the approximation error in $L^\infty ( [0,1] )$-norm for shallow networks will suffer from an inherent numerical precision limitation, an effect not shared by deep networks. We proceed to formalize 
        these back-of-the-envelope arguments by first establishing the statement on the numerical precision of the outputs of quantized networks.
          
  \iffalse
  This discrepancy leads to the subsequent minimax error lower bound. 
  the $L^\infty ( [0,1] )$-norm measures the worst-case error, it follows that shallow networks will be inherently limited
  first note that if a network can not approximate a function well on a subdomain of $[0,1]$, in particular, subdomain of $[0,1]$ with base-2 quantized numbers, it can not approximate the function well on $[0,1]$ in terms of $L^\infty ( [0,1] )$-norm. Therefore, we, for now, consider the properties of networks with fixed depth, whose weights and inputs are base-2 quantized with a fixed number of fractional bits\footnote{The number of fractional bits of $x \in \mathbb{R}$ refers to the number of digits after the binary point in the binary representation of $x$.}. This will allow us to use the arithmetic of base-2 quantized numbers.  Specifically, we can show that the network output will also be base-2 quantized with a fixed number of fractional bits, regardless of network width, even if the width goes to infinity.
  \fi

		\begin{lemma}
			\label{lem:precision}
			Let $a,b, W,L\in \mathbb{N}$ and $c \in \mathbb{N} \cup \{ 0 \}$. For $f \in \mathcal{R}_b^a ( W,L )$ and $x \in 2^{-c} \mathbb{Z}$, it holds that

			\begin{equation}
                \label{eq:output_binary_weight_network_precision}
				f ( x )  \in  2^{-Lb - c} \mathbb{Z}.
			\end{equation}

		\begin{proof} 
			Fix $x = 2^{-c} K \in 2^{-c} \mathbb{Z} $ with $K \in \mathbb{Z}$ and an $f \in \mathcal{R}_b^a ( W,L )$. By definition, there exists  $\Phi_0 = ( ( A_i, \nmathbf{b}_i ) )_{i =1}^{\mathcal{L}( \Phi_0 )} \in \mathcal{N}_b^a ( W,L )$ with $\mathcal{L} ( \Phi_0 ) \leq L$ and $\coef ( \Phi_0 ) \subseteq \mathbb{Q}_b^a$ such that $R ( \Phi ) = f$. Now, consider the scaled network $\widetilde{\Phi}_0  = ( ( 2^b A_i, 2^b\nmathbf{b}_i ) )_{i =1}^{\mathcal{L}( \Phi_0 )}$. It follows from $\coef ( \Phi_0 ) \subseteq  \mathbb{Q}_{b}^a = ( (-2^{a+1}, 2^{a+1} ) \cap 2^{-b}\mathbb{Z} ) \subseteq 2^{-b} \mathbb{Z}$ that $\coef ( \widetilde{\Phi}_0 ) = 2^b \coef ( \Phi_0 ) \subseteq \mathbb{Z}$. Hence, $R ( \widetilde{\Phi}_0 ) ( K )$, as the output of an integer-weight ReLU network with integer-valued input, must be integer-valued. Thanks to the positive homogeneity of the ReLU function, i.e., $\rho ( \lambda x ) = \lambda \rho(x) $, $ \forall \lambda \geq 0 $, $ \forall x \in \mathbb{R}$, we have $2^{c}  ( 2^{b} )^{\mathcal{L} ( \Phi_0 )} R ( \Phi_0 ) ( x ) = 2^{c}  R ( \widetilde{\Phi}_0 ) ( x ) =  R ( \widetilde{\Phi}_0 ) ( 2^c x ) = R ( \widetilde{\Phi}_0 ) ( K ) \in \mathbb{Z}$, and therefore, $f(x) = R ( \Phi_0 ) ( x) \in ( 2^{-b} )^{\mathcal{L} ( \Phi_0 )} 2^{-c} \mathbb{Z} \subseteq 2^{-Lb - c} \mathbb{Z}$, which concludes the proof.
		\end{proof}
		\end{lemma}
		The next result quantifies the numerical precision advantage of deep networks over shallow networks announced above.
  %inherent limitation of shallow networks
  %In contrast, for base-2 quantized  $x \in [0,1]$, $f(x) $ with $f \in \lip ( [0,1] )$ can take arbitrary values in $[0,1]$, and, in particular, values with %an infinite number of fractional bits. This discrepancy leads to the subsequent minimax error lower bound. 

		\begin{proposition}
			\label{prop:underquantization_shallow}
			Let $a,b,W,L\in \mathbb{N}$.  It holds that
			\begin{equation}
				\label{eq:minimax_approximation_error_lower_bound_precision}
				\mathcal{A}_\infty	( \lip  ( [0,1] ), \mathcal{R}_{b}^{a} (  W, L ) ) \geq \,\frac{1}{2}\, 2^{-Lb}.
			\end{equation}

			\begin{proof}
				Fix an $f \in \mathcal{R}_b^a ( W,L )$. It follows from Lemma~\ref{lem:precision}, with $c = 0$ and $x= 1$, that $f(1) \in 2^{-Lb} \mathbb{Z}$.
				Let $g\in \lip ( [0,1] )$ be given by $g(x) = \frac{1}{2}\,2^{-Lb}$, $x \in [0,1]$. We then have 
				\begin{align*}
					\| f - g \|_{[0,1] } \geq | f(1) - g(1) | = \biggl| f(1) - \frac{1}{2}\,2^{-Lb} \biggr| \geq \frac{1}{2}\,2^{-Lb}. 
				\end{align*}
				As the choice of $f \in \mathcal{R}_b^a ( W,L )$ is arbitrary, we have $\| f - g \|_{[0,1] } \geq \frac{1}{2}\,2^{-Lb}$, for all $f \in \mathcal{R}_b^a ( W,L )$, which together with $g\in \lip ( [0,1] )$ implies $\mathcal{A}_\infty	( \lip  ( [0,1] ), \mathcal{R}_{b}^{a} ( W, L ) ) \geq \frac{1}{2}\, 2^{-Lb}$.
				\end{proof}
		\end{proposition}

		The bound in  \eqref{eq:minimax_approximation_error_lower_bound_precision} will henceforth be referred to as the minimax error lower bound incurred by the numerical-precision limit. As announced above, this lower bound does not depend on network width $W$ and decreases (exponentially) in network depth $L$, indicating an advantage of deep networks over shallow networks, which, again, will manifest itself when we characterize the three quantization regimes. We hasten to add that the exponential behavior of the lower bound \eqref{eq:minimax_approximation_error_lower_bound_precision} is a consequence of the length of the convolution product of two (binary) sequences being given by the sum of the lengths of the individual sequences.

		We finally put together the individual minimax error lower bounds incurred by the minimum memory requirement, Proposition~\ref{proposition:lower_bound_approximation}, the VC-dimension limit, Proposition~\ref{prop:lower_bound_approximation_VC_dimension}, and the numerical-precision limit, Proposition~\ref{prop:underquantization_shallow}, to obtain a combined lower bound as follows.

		\begin{corollary}
			\label{thm:approximation_error_lower_bound_overall}
			There exists an absolute constant $c_\ell$ such that for all $W,L, b\in \mathbb{N}$ with $L \geq 2$, it holds that
			\begin{equation}
			\label{eq:sumarized_lower_bound}
				\begin{aligned}
					&\,\mathcal{A}_\infty	( \lip  ( [0,1] ), \mathcal{R}_b^1 ( W, L  )   )\\
					&\geq\, c_\ell \max \bigl\{( W^2 L \,b )^{-1},( W^2 L^2 ( \log (W) + \log (L) ) )^{-1}, 2^{-Lb} \bigr \}.
				\end{aligned}
			\end{equation}
			\begin{proof}
				Set $c_\ell = \min \{c_m, c_v, \frac{1}{2} \}$, where $c_m$ and $c_v$ are the absolute constants in Proposition~\ref{proposition:lower_bound_approximation} and Proposition~\ref{prop:lower_bound_approximation_VC_dimension}, respectively. Then, \eqref{eq:sumarized_lower_bound} follows from \eqref{eq:lower_bound_a_b} in Proposition~\ref{proposition:lower_bound_approximation}, \eqref{eq:lower_bound_VC_dimension_general_weights} in Proposition~\ref{prop:lower_bound_approximation_VC_dimension} with $\mathbb{A} = \mathbb{Q}_b^1$, and \eqref{eq:minimax_approximation_error_lower_bound_precision} in \eqref{prop:underquantization_shallow} with $a = 1$. 
			\end{proof}
		\end{corollary}

		We now characterize the behavior of the combined lower bound \eqref{eq:sumarized_lower_bound} by identifying the neural network configurations, specifically the tuples $( W,L, b )$, that make any given individual component in the lower bound dominate the other two components. To simplify the discussion, we fix sufficiently large values for $W$ and $L$, and let $b$ range from $1$ to infinity.

		\begin{enumerate}
   \item The term $2^{-Lb}$ can dominate the other two terms only for very small values of $b$, as it decreases exponentially in $b$, in comparison to polynomial decrease and constant behavior.  Consider the extreme case $b = 1$. Then, the individual terms in \eqref{eq:sumarized_lower_bound} become $( W^2 L )^{-1}$, $ ( W^2 L^2 ( \log (W) + \log (L) ) )^{-1}$, and $ 2^{-L}$, respectively, and $2^{-L}$ can dominate the other two terms only if $L$ is logarithmically smaller than $W$, i.e., $L \leq C \log(W)$ for some $C \in \mathbb{R}_+$. We call the regime where $2^{-Lb}$ dominates, the \textit{under-quantization regime}, and note that it might be empty, concretely when $L$ is not logarithmically smaller than $W$.
			
	        \item The term $( W^2 L \,b )^{-1}$ dominates for medium values of $b$ and $L$ logarithmically smaller than $W$, and for small to medium values of $b$ when $L$ is not logarithmically smaller than $W$. We refer to this regime as the \textit{proper-quantization regime}.		

           \item For large $b$, the term $(W^2 L^2 ( \log (W) + \log (L) ) )^{-1}$ will dominate, owing to the other two terms going to $0$ when $b \to \infty$. As the quantization resolution increases when $b$ grows, we call this regime  the \textit{over-quantization regime}.
           
			\end{enumerate}

		In summary, as $b$ increases from $1$ to infinity, the lower bound \eqref{eq:sumarized_lower_bound} transitions from exponential decay in the \textit{under-quantization regime} (which might be empty) to polynomial decay in the \textit{proper-quantization regime}, and finally levels out at a constant value that is independent of $b$ in the \textit{over-quantization regime}.

		In the following two sections, we derive minimax error upper bounds that, when combined, exhibit the same three-regime behavior as \eqref{eq:sumarized_lower_bound} and allow for a precise characterization of the boundaries between the regimes.
% section bounds_on_quantization_error (end)

	%!TEX root = ../draft_quantized_weight_networks.tex

\section{A Constructive Minimax Error Upper Bound} % (fold)
\label{sec:bounds_on_quantization_error}

	%We first establish a minimax error upper bound in the approximation of functions in $\lip ( [0,1] )$ using ReLU networks with weights in $\mathbb{Q}_b^1$ %depending on network depth and width in a specific manner. 

	The first bound we establish is inspired by the two-step approach employed in \cite{deep-it-2019}. In the first step, we approximate functions in $\lip ( [0,1] )$ by networks contained in $\mathcal{R} ( W,L, 1 )$, resulting in an upper bound on  $\mathcal{A}_\infty ( \lip  ( [0,1] ), \mathcal{R} (W, L, 1  ) )$. Here we only demand that the weights of the approximating networks have absolute values no greater than $1$, but they need not be quantized. In the second step, we then quantize the weights of the networks in $\mathcal{R} (W,L,1)$ chosen in the first step, by rounding to the nearest neighbor in the set $\mathbb{Q}_b^1$. We then bound the error 
 %difference between the original network and the quantized network in terms of the \textit{quantization error} 
 $\mathcal{A}_\infty ( \mathcal{R} (W, L, 1  ), \mathcal{R}_b^1 (W, L  )  )$ incurred by quantization. Using the triangle inequality formalized in Lemma~\ref{lem:triangle_inequality}, we finally obtain an upper bound on the minimax error $\mathcal{A}_\infty ( \lip  ( [0,1] ), \mathcal{R}_b^1 (W, L  )  )$ according to
	\begin{equation}
	\label{eq:error_decomposition}
        \begin{aligned}
		&\,\mathcal{A}_\infty ( \lip  ( [0,1] ), \mathcal{R}_b^1 (W, L  )  )\\
            &\leq\,  \mathcal{A}_\infty ( \lip  ( [0,1] ), \mathcal{R} (W, L, 1  ) )  + \mathcal{A}_\infty ( \mathcal{R} (W, L, 1  ), \mathcal{R}_b^1 (W, L  )  ).
  \end{aligned}
	\end{equation}

	We proceed to detail the first step announced above, namely approximating functions in $\lip ([0,1])$ by ReLU networks with weight-magnitude bounded by $1$. 

	\begin{theorem}
		\label{thm:approximation_lip}
		There exist absolute constants $C,D \in \mathbb{R}_+$ such that, for all $W,L \in \mathbb{N}$ with $W,L \geq D$,
		\begin{equation}
		\label{eq:approximation_lip_001}
			\begin{aligned}
				 \mathcal{A}_\infty	( \lip  ( [0,1] ), \mathcal{R} ( W, L, 1 )  )\leq   C  ( W^2 L^2 \log (W)  )^{-1}.
			\end{aligned}
		\end{equation}
		\begin{proof}
			See Appendix~\ref{sec:approximate_Lip1}.
		\end{proof}
	\end{theorem}

	The proof of Theorem~\ref{thm:approximation_lip} is constructive in the sense that, for every function $f \in \lip  ( [0,1] )$, it specifies a network $g \in \mathcal{R} ( W, L, 1 )$ such that $\nleft\| f -g  \nright\|_{L^\infty ( [0,1] )} \leq C ( W^2 L ^2 \log(W) )^{-1} $.

	Theorem~\ref{thm:approximation_lip} is in the spirit of a line of papers on the approximation of smooth functions by ReLU networks, notably \cite{YAROTSKY2017103, schmidt2019deep, yarotsky2018optimal, yarotsky2019phase, Kohler2019, shen2021optimal}. Smoothness, in these references, is quantified by parameters $n \in \mathbb{N} \cup \{ 0 \}$ and $\alpha \in (0,1]$, with the functions bounded, up to order $n$ bounded differentiable, and 
 %such that those functions are differentiable up to the $n$-th order, with functions themselves and their derivatives up to order $n$ being bounded and 
 the $n$-th order derivative H\"older continuous\footnote{A function $f:[0,1] \mapsto \mathbb{R}$ is said to be H\"older continuous with exponent $\alpha$, if there exists $C \in \mathbb{R}_+$ such that $| f(x) - f(y) | \leq C | x - y |^\alpha$.} of exponent $\alpha$. Here, we are concerned with the special case $n = 0$ and $\alpha = 1$. The most recent development in this line of work \cite[Corollary 1.3]{shen2021optimal} deals with the case $n \in \mathbb{N} \cup \{ 0 \}$ and $\alpha = 1$ and, when particularized to $\lip ([0,1])$-functions, yields
 	\begin{equation}
	\label{eq:varying_width_and_depth}
		\mathcal{A}_\infty	( \lip  ( [0,1] ), \mathcal{R} ( W, L, f(W,L) )  )\leq  \widetilde{C}  ( W^2 L^2 \log (W)  )^{-1},
	\end{equation}
	for sufficiently large $W, L \in \mathbb{N}$, an absolute constant $\widetilde{C} \in \mathbb{R}_+$, and a function $f: \mathbb{N}^2 \mapsto \mathbb{R}$. Notably, upon examination of the proof of \cite[Corollary 1.3]{shen2021optimal}, one sees that $f(W,L) \geq W^{kL}$. This is markedly different from the constant weight-magnitude $1$ in our Theorem \ref{thm:approximation_lip}, which turns out to be crucial for achieving memory optimality. %of the approximating networks we construct. 
 The improvement in the weight-magnitude behavior we obtain is predominantly owed to the novel bit extraction technique developed in the proof. This technique is interesting in its own right and can readily be applied for general $n \in \mathbb{N} \cup \{ 0 \}$ and $\alpha \in (0,1]$. For conciseness of exposition, however, we limit ourselves to the case $n=0, \alpha=1$.
 %improving the dependency of network weight magnitude on their width and depth. %We decided to limit ourselves to $1$-Lipschitz functions for clarity of exposition.

	We proceed to quantize the weights of the approximating networks in Theorem~\ref{thm:approximation_lip} and bound the resulting quantization error. To this end, we first state an upper bound on the distance between two general ReLU network realizations, expressed in terms of $W,L$, and the distance between their associated configurations.

	\begin{lemma}
		\label{lem:quantization_error}
		Let $W,L, \ell \in \mathbb{N}$ with $\ell \leq L$, and let 
		\begin{equation*}
			\Phi^i = ( ( A_j^i,\nmathbf{b}_j^i ) )_{j =1 }^{\ell} \in \mathcal{N} ( W,L,1 ), \quad i = 1,2,
		\end{equation*} have the same architecture.
		Then,
		\begin{equation}
			\label{eq:quantization_error_bound}
			\|  R ( \Phi^1 )  -  R ( \Phi^2 )  \|_{L^\infty ( [0,1] )} \leq L (W+1)^L \| \Phi^1 - \Phi^2 \|,
		\end{equation}
		where 
		\begin{equation}
		\label{eq:weightwise_difference}
			\| \Phi^1 - \Phi^2 \| := \max_{j = 1,\dots, \ell} \max \bigl\{ \| A_j^1 - A_j^2\|_\infty,  \| \nmathbf{b}_j^1 - \nmathbf{b}_j^2\|_\infty  \bigr\} .
		\end{equation}

		\begin{proof}
			See Appendix~\ref{sub:quantization_error}.
		\end{proof}
	\end{lemma}

	This result now allows us to bound the error incurred by replacing real-valued network weights with values in a finite set $\mathbb{A} \subseteq \mathbb{R}$.

	\begin{proposition}
		\label{prop:quantization_error_bound_and_decomposition}
		Let $W,L \in \mathbb{N}$ and consider the finite set $\mathbb{A} \subseteq \mathbb{R}$ with $\mathbb{A} \cap [-1,1] \neq \emptyset$. It holds that
		\begin{align}
			\mathcal{A}_\infty ( \mathcal{R} ( W,L,1  ), \mathcal{R}_\mathbb{A} ( W,L  ) ) \leq&\, L (W+1)^L \mathcal{A} ( [-1,1], \mathbb{A}\cap [-1,1], | \cdot |  ) \label{eq:quantization_general_A}\\
			\leq &\,  2 L (W+1)^L \mathcal{A} ( [-1,1], \mathbb{A}, | \cdot |  ).\label{eq:quantization_general_A_2}
		\end{align}
		In particular, for all $a,b \in \mathbb{N}$, we have $\mathcal{A} ( [-1,1], \mathbb{Q}_b^a\cap [-1,1], | \cdot | ) \leq 2^{-b}$, and 
		\begin{equation}
			\label{eq:quantization_binary_A}
			\mathcal{A}_\infty ( \mathcal{R} ( W,L,1  ), \mathcal{R}_b^a ( W,L  ) ) \leq L (W+1)^L 2^{-b}.
		\end{equation}
	\end{proposition}
	\begin{proof}
		We start by proving \eqref{eq:quantization_general_A}-\eqref{eq:quantization_general_A_2} and note that
  %Let $\mathbb{A} \subseteq \mathbb{R}$ be a finite set such that $\mathbb{A} \cap [-1,1] \neq \emptyset$.  We note that 
		\begin{align}
			\mathcal{A} ( [-1,1], \mathbb{A}\cap [-1,1], | \cdot |  ) =& \sup_{x \in [-1,1]} \inf_{y \in \mathbb{A}\cap [-1,1]} | x - y | \label{eqline:being_finite_0} \\
			=& \sup_{x \in [-1,1]} \min_{y \in \mathbb{A}\cap [-1,1]} | x - y |, \label{eqline:being_finite}
		\end{align}
		where in \eqref{eqline:being_finite} we used that $\mathbb{A}\cap [-1,1]$ is finite. Operationally, \eqref{eqline:being_finite_0}-\eqref{eqline:being_finite} says that every $x \in [-1,1]$ can be quantized into an element $q(x) \in \mathbb{A}\cap [-1,1]$, depending on $x$, such that $| x - q(x) | \leq \mathcal{A} ( [-1,1], \mathbb{A}\cap [-1,1], | \cdot |  )$. This induces a mapping $q: [-1,1] \mapsto \mathbb{A}\cap [-1,1]$.  We can now conclude that, for every $\Phi \in \mathcal{N}(W,L,1)$, application of the mapping $q$ to each entry in $\Phi$ yields a corresponding quantized version  $Q ( \Phi )  \in \mathcal{N}_{\mathbb{A} \,\cap\, [-1,1]} ( W,L )$ such that the weights  of $Q ( \Phi )$
  %  \in \mathcal{N}_{\mathbb{A} \cap [-1,1]} ( W,L )$ 
  differ from those in $\Phi$ by no more than $\mathcal{A} ( [-1,1], \mathbb{A}\cap [-1,1], | \cdot | )$. This induces a mapping $Q: \mathcal{N} ( W,L, 1 ) \mapsto \mathcal{N}_{\mathbb{A} \cap [-1,1]} ( W,L )$, satisfying
		\begin{equation}
		\label{eqline:quantization_minimax_0}
			\| \Phi  - Q ( \Phi ) \| \leq  \mathcal{A} ( [-1,1], \mathbb{A}\cap [-1,1], | \cdot | ),
		\end{equation} with $\| \Phi  - Q ( \Phi ) \|$ defined according to \eqref{eq:weightwise_difference}. We next establish \eqref{eq:quantization_general_A} through the following chain of arguments
		\begin{align}
			\mathcal{A}_\infty ( \mathcal{R} ( W,L,1  ), \mathcal{R}_\mathbb{A} ( W,L  ) )  = &\, \sup_{f \in \mathcal{R} ( W,L,1  )} \inf_{\tilde{f} \in \mathcal{R}_\mathbb{A} ( W,L  ) } \| f - \tilde{f} \|_{L^\infty ( [0,1] )} \label{eqline:quantization_minimax_10} \\
			= &\, \sup_{\Phi \in \mathcal{N} ( W,L,1  )} \inf_{\widetilde{\Phi} \in \mathcal{N}_\mathbb{A} ( W,L  ) } \| R ( \Phi ) - R ( \widetilde{\Phi} ) \|_{L^\infty ( [0,1] )} \\
			\leq &\, \sup_{\Phi \in \mathcal{N} ( W,L,1  )}  \| R ( \Phi ) - R ( Q ( \Phi )) \|_{L^\infty ( [0,1] )} \label{eqline:quantization_minimax_1}\\
			\leq &\, \sup_{\Phi \in \mathcal{N} ( W,L,1  )} L (W+1)^L  \| \Phi - Q ( \Phi ) \|\label{eqline:quantization_minimax_2}\\ 
			\leq &\,  L (W+1)^L  \mathcal{A} ( [-1,1], \mathbb{A}\cap [-1,1], | \cdot | ),\label{eqline:quantization_minimax_3}
		\end{align}
		where in \eqref{eqline:quantization_minimax_1} we used $Q ( \Phi ) \in \mathcal{R}_{\mathbb{A} \cap [-1,1]} ( W,L ) \subseteq \mathcal{R}_{\mathbb{A}} ( W,L )$, for all $\Phi \in \mathcal{N} ( W,L,1  )$, \eqref{eqline:quantization_minimax_2} follows from Lemma~\ref{lem:quantization_error} with 	$\Phi^1 = \Phi$ and $\Phi^2 = Q ( \Phi ) \in  \mathcal{R}_{\mathbb{A} \cap [-1,1]} ( W,L ) \subseteq \mathcal{R} ( W,L, 1 )$, and in \eqref{eqline:quantization_minimax_3} we employed \eqref{eqline:quantization_minimax_0}. Then, \eqref{eq:quantization_general_A_2} follows by combining \eqref{eq:quantization_general_A}  with  the relation
		\begin{equation*}
			\mathcal{A} ( [-1,1], \mathbb{A} \cap [-1,1], | \cdot |  ) \leq 2  \mathcal{A} ( [-1,1], \mathbb{A}, | \cdot |  ),
		\end{equation*}
		established in Lemma~\ref{lem:auxiliary_lemma_minimax_r}.

		We proceed to the derivation of \eqref{eq:quantization_binary_A}. Fix $a,b \in \mathbb{N}$. Let $\tilde{q}: [-1,1] \mapsto \mathbb{Q}_b^a\cap [-1,1]$ be given by 
		\begin{equation*}
			\tilde{q}(x) = \begin{cases}
				2^{-b} \lceil 2^b x \rceil,  & \text{ if } x \leq 0,\\
				2^{-b} \lfloor 2^b x \rfloor,  & \text{ if } x > 0,
			\end{cases}
		\end{equation*}
		and note that
		\begin{equation}
			\label{eq:quantizing_into_binary}
			| x - \tilde{q}(x)  | \leq 2^{-b}, \text{ for all } x \in [-1,1].
		\end{equation}
		We have 
		\begin{equation}
		\label{eq:quantizing_into_binary_2}
			\mathcal{A} ( [-1,1], \mathbb{Q}_b^a\cap [-1,1], | \cdot | ) =   \sup_{x \in [-1,1]} \inf_{y \in \mathbb{Q}_b^a\cap [-1,1]} | x - y |  \leq  \sup_{x \in [-1,1]}  | x - \tilde{q}(x) | \leq  2^{-b},
		\end{equation}
		where in the first inequality we used that $\tilde{q}(x) \in \mathbb{Q}_b^a$, and the second inequality follows from \eqref{eq:quantizing_into_binary}. Evaluating \eqref{eqline:quantization_minimax_10}-\eqref{eqline:quantization_minimax_3} with $\mathbb{A} =  \mathbb{Q}_b^a$, upon noting that the prerequisite $\mathbb{Q}_b^a \cap [-1,1]  \neq \emptyset$ is satisfied as $ 0 \in \mathbb{Q}_b^a \cap [-1,1]$, and using \eqref{eq:quantizing_into_binary_2},  finally yields \eqref{eq:quantization_binary_A}.
	\end{proof}
	The quantization error upper bound \eqref{eq:quantization_binary_A} does not depend on $a$ as the weight magnitude of the networks we consider is bounded by $1$.

	We are now ready to characterize the minimax error upper bound in the approximation of $\lip ([0,1])$ through ReLU networks with quantized weights. 
	% \begin{equation}
	% \label{eq:error_decomposition_binary_quantized_weights}
	% 	\mathcal{A}_\infty ( \lip  ( [0,1] ), \mathcal{R}_b^1 (W, L  )  )  \leq   C  ( W^2 L^2 \log (W)  )^{-1}  + L (W+1)^L 2^{-b}.
	% \end{equation}
	% Choosing $b $ such that $C  ( W^2 L^2 \log (W)  )^{-1}$ and $L (W+1)^L 2^{-b}$ are of the same order yields the following minimax error upper bound.

	\begin{proposition}
		\label{prop:sufficient_rate}
		There exist absolute constants $C_1, D_1,E_1 \in \mathbb{R}_+$, with $D_1 \geq 2$, such that for $ W,L \in \mathbb{N}$, with $W,L \geq D_1$, 
		\begin{equation}
		\label{eq:error_upper_bound_quantization_error_bound}
			\mathcal{A}_\infty	( \lip  ( [0,1] ), \mathcal{R}_{\lceil E_1 L \log (W) \rceil}^1 (W, L ) )  \leq C_1 ( W^2 L^2 \log (W) )^{-1}.
		\end{equation}
	\end{proposition}
	
	\begin{proof}
		Let $C,D$ be the absolute constants specified in Theorem~\ref{thm:approximation_lip}.
		Set  $C_1 = C + 1 $, $D_1 = \max \{ 2, D \}$, and $E_1 = 12$. Then, for $W,L \geq D_1$, 
		\begin{align}
			&\,\mathcal{A}_\infty	( \lip  ( [0,1] ), \mathcal{R}_{\lceil E_1 L \log (W) \rceil}^{1} (  W, L ) ) \label{eqline:proof_fixed_quantization_100}\\
			&\leq  \, \mathcal{A}_\infty	( \lip  ( [0,1] ), \mathcal{R} (  W, L, 1 ) ) + \mathcal{A}_\infty	( \mathcal{R} (  W, L, 1 ), \mathcal{R}_{\lceil E_1 L \log (W) \rceil}^1 (  W, L) ) \label{eqline:proof_fixed_quantization_1}\\
			&\leq  \,C ( W^2 L^2 \log (W) )^{-1}  + L (W+1)^L 2^{-\lceil E_1 L \log (W) \rceil} \label{eqline:proof_fixed_quantization_2}\\
			&\leq  \, (C+1) ( W^2 L^2 \log (W) )^{-1}, \label{eqline:proof_fixed_quantization_3}
		\end{align}
		where \eqref{eqline:proof_fixed_quantization_1} follows from the triangle inequality \eqref{eq:error_decomposition} with $b = \lceil E_1 L \log (W) \rceil$,  in \eqref{eqline:proof_fixed_quantization_2} we used Theorem~\ref{thm:approximation_lip} with  $W,L \geq D_1 \geq D$ and Proposition~\ref{prop:quantization_error_bound_and_decomposition}, and \eqref{eqline:proof_fixed_quantization_3} is by
		\begin{equation}
		\label{eqline:choice_of_C_2}
			2^{-\lceil E_1 L \log (W) \rceil}  = 2^{-\lceil 12 L \log (W) \rceil} \leq ((W^{2L})^6)^{-1} \leq (W^2 L^3  \log (W) ( W+1 )^L)^{-1},
		\end{equation}
		as $W^{2L} \geq W^2$, $(W^{2L})^3 \geq L^3$, $W^{2L} \geq \log ( W )$, and $W^{2L} \geq ( W+1)^{L}$, all owing to $W,L \geq D_1 \geq 2$. 
	\end{proof}
 
	As $\mathcal{R}_{\lceil E_1 L \log (W) \rceil}^1 (W, L )$ is finite, we have
	\begin{align*}
		\mathcal{A}_\infty	( \lip  ( [0,1] ), \mathcal{R}_{\lceil E_1 L \log (W) \rceil}^1 (W, L ) ) =&\, \sup_{g \in \lip  ( [0,1] )} \inf_{f \in \mathcal{R}_{\lceil E_1 L \log (W) \rceil}^1 (W, L )} \nleft\| g - f \nright\|_{L^\infty ( [0,1] )} \\
		=&\, \sup_{g \in \lip  ( [0,1] )} \min_{f \in \mathcal{R}_{\lceil E_1 L \log (W) \rceil}^1 (W, L )} \nleft\| g - f \nright\|_{L^\infty ( [0,1] )},
	\end{align*}
	and, therefore, for every $g \in \lip ( [0,1] )$, there exists a network realization $f \in \mathcal{R}_{\lceil E_1 L \log (W) \rceil}^1 (W, L )$ such that $ \nleft\|  g - f \nright\|_{L^\infty ( [0,1] )} \leq C_1 ( W^2 L^2 \log (W) )^{-1}$. Indeed, the explicit construction of this $f$ can be inferred from the proofs of Theorem~\ref{thm:approximation_lip} and Proposition~\ref{prop:quantization_error_bound_and_decomposition}.

    We conclude the discussion by arguing that the construction underlying Proposition~\ref{prop:sufficient_rate} achieves memory optimality. To see this, note that starting from~\eqref{eq:error_upper_bound_quantization_error_bound}, we get
	\begin{align}
			\mathcal{A}_\infty	( \lip  ( [0,1] ), \mathcal{R}_{\lceil E_1 L \log (W) \rceil}^1 (W, L ) )  \leq &\, C_1 ( W^2 L^2 \log (W) )^{-1}\\
			= &\, C_1 \frac{\lceil E_1 L \log (W) \rceil}{L \log(W)} ( W^2 L \lceil E_1 L \log (W) \rceil )^{-1}\\
			\leq &\, C_1 ( E_1 + 1 ) ( W^2 L \lceil E_1 L \log (W) \rceil )^{-1},
	\end{align}
	where in the last inequality we used $L \log(W) \geq 1$ and $\lceil E_1 L \log (W) \rceil \leq  E_1 L \log (W)  + 1$. Hence, for $b=\lceil E_1 L \log (W) \rceil$, the minimax error lower bound in Proposition~\ref{proposition:lower_bound_approximation} incurred by the minimum memory requirement is attained to within a multiplicative factor. More formally, it follows from Proposition~\ref{prop:converse_lower_bound} with $D = C_1 ( E_1 + 1 )$ and $\mathcal{I} = \{ ( W,L,\lceil E_1 L \log (W) \rceil ) \in \mathbb{N}^3: W, L \geq D_1 \}$  that $\{\mathcal{R}_{\lceil E_1 L \log (W) \rceil}^1 (W, L ) ): W,L \geq D_1 \}$ achieves memory optimality in the approximation of functions in $\lip ( [0,1] )$. 
 %To illustrate its operational meaning, consider the approximation of functions in $\lip ( [0,1] )$ using networks with fixed width $W$ and depth $L$ of %sufficiently large values, while having the freedom to adjust the precision of network weights. We can ensure that the memory redundancy remain bounded by a %constant independent of $W,L$, by setting the weight set of the networks to be $\mathbb{Q}_{\lceil E_1 L \log (W) \rceil}^1$. 

	%!TEX root = ../draft_quantized_weight_networks.tex

\section{Depth-Precision Tradeoff} % (fold)
	\label{sec:embedding_properties_of_relu_networks_with_quantized_weights}

The bound in Proposition~\ref{prop:sufficient_rate} applies to a fixed choice of $b$ as a function of network width $W$ and depth $L$ according to $b=\lceil E_{1} L \log(W)  \rceil$. We now relax this dependency with the aim of obtaining a more flexible, and for certain parameter choices tighter, upper bound. In the process, we address a question of practical interest, namely
%which leads to a more flexible upper bound and in the process also yields an answer to a question of conceptual relevance, namely 
``Can we  
%		We now introduce the depth-precision tradeoff to relax the dependency of $b$ on network width $W$ and depth $L$. This tradeoff will also answer the %following question:  can 
realize neural networks with high precision weights by equivalent deeper networks of lower precision weights? If so, what would the impact of such a transformation on a potential memory optimality of the initial network be?''
This question is of significant interest as high-resolution quantization is difficult to realize in electronic circuits. The idea of sigma-delta conversion \cite{Tewksbury1978} exemplifies this principle, in the context of analog-to-digital (A/D) conversion, by trading oversampling rate for quantization resolution. Here, we can trade network depth for quantization resolution. The next result expresses the corresponding depth-precision tradeoff in a formal manner.
		% , while retaining memory optimality.
		\begin{proposition}
			\label{prop:tradeoff_binary}
  			Let $W,L,k\in \mathbb{N}$. For all $a,b \in \mathbb{N}$, we have
			\begin{equation}
				\label{eq:embedding_binary_weights}
				\mathcal{R}^{ka}_{kb} (W,  L   ) \subseteq  \mathcal{R}^a_b ( 16W,(k+2) L ).
			\end{equation}
			
			\begin{proof}
				See Appendix~\ref{sec:proof_of_proposition_prop:tradeoff_binary}.
			\end{proof}
		\end{proposition}

		Proposition~\ref{prop:tradeoff_binary}  states that (high-precision) networks in  $\mathcal{R}^{ka}_{kb} (W,  L   )$ can equivalently be realized by lower-precision networks, specifically by networks in $\mathcal{R}^a_b ( 16W,(k+2) L )$, at the expense of an increase in network width and depth. We emphasize that while the increase in width is by a constant factor, a reduction in network weight precision by a factor of $k$ leads to a $k$-fold increase in depth.
  %Specifically, each weight in $\mathcal{N}^{ka}_{kb} (W,L)$ requires $( ka + kb +2 )$ bits, while the weights in $\mathcal{N}^a_b ( 16W,(k+2) L )$ need $a + %b +2$ bits each. 
    The proof of Proposition~\ref{prop:tradeoff_binary} is constructive, meaning that for a given network configuration $\Phi_1 \in \mathcal{N}^{ka}_{kb} (W,  L   )$, we explicitly specify a  $\Phi_2 \in \mathcal{N}^a_b ( 16W,(k+2) L )$ such that $R ( \Phi_1 ) = R ( \Phi_2)$. In addition, the proof applies to
more general weight sets, input and output dimensions than those assumed in the statement of Proposition~\ref{prop:tradeoff_binary}.
  
\iffalse
\nc{In particular, letting the depth go to infinity will equivalently improve the precision indefinitely. We now argue that increasing network width, instead of depth, can not effectively improve the precision in a similar way. Indeed, it follows from Lemma~\ref{prop:underquantization_shallow}, with $c = 0$, that, for every $f \in R_b^a ( W, L )$, its value at $0$ satisfies $f(1) \in 2^{-Lb} \mathbb{Z}$ and has a fixed precision not depending on $W$, which, therefore, can not be improved even by letting $W$ go to infinity.}
\fi

		An important property of the depth-precision tradeoff just identified resides in the fact that it essentially preserves memory consumption behavior. Concretely, it follows from Proposition~\ref{prop:fundamental_limit_ReLU_networks_quantized_weights} that network realizations $f \in \mathcal{R}^{ka}_{kb} ( W,  L)$ are uniquely specified by no more than $10 W^2 L k(a+b)$ bits while the corresponding equivalent realizations in $\mathcal{R}^a_b ( 16W,(k+2) L )$ require at most $10 (16W)^2 (k+2) L (a + b)$ bits. Replacing a given (high-precision) network by a deep lower-precision network hence comes at the cost, in the number of bits needed, of at most a multiplicative constant factor of $3\cdot 16^2$, while the scaling behavior in $W,L,k,a,b$ is preserved.

		We now show that this insight allows us to conclude that the (constructive) transformation from high-precision to deeper low-precision networks effected by Proposition~\ref{prop:tradeoff_binary} preserves memory optimality. Specifically, we recall that, by Proposition~\ref{prop:sufficient_rate},  $\{\mathcal{R}_{\lceil E_1 L \log (W) \rceil}^1 (W, L ): W,L \geq D_1 \}$ achieves memory optimality in the approximation of functions in $\lip ( [0,1] )$ with 
		\begin{equation}
		\label{eq:error_upper_bound_quantization_error_bound_100}
		% \begin{aligned}
			\mathcal{A}_\infty	( \lip  ( [0,1] ), \mathcal{R}_{\lceil E_1 L \log (W) \rceil}^1 (W, L ) )  \leq C_1 ( W^2 L^2 \log (W) )^{-1},
		% \end{aligned}
		\end{equation}
		where $C_1, D_1$, and $E_1$ are the absolute constants specified in Proposition~\ref{prop:sufficient_rate}. Suppose now that we want to replace the (high-precision) weights of network configurations in $\mathcal{N}_{\lceil E_1 L \log (W) \rceil}^1 (W, L ) $, which realize functions in $\mathcal{R}_{\lceil E_1 L \log (W) \rceil}^1 (W, L ) $, by lower-precision weights, say in $\mathbb{Q}_{b}^1$, with $b < \lceil E_1 L \log (W) \rceil$.  Specifically, with $k = \bigl\lceil \frac{\lceil E_1 L \log (W) \rceil}{b} \bigr\rceil \geq 2$, we get $k b \geq \lceil E_1 L \log (W) \rceil$, and therefore
		\begin{equation}
			\mathcal{R}^1_{b} ( 16W,(k+2) L ) \supseteq \mathcal{R}^k_{kb} ( W, L ) \supseteq  \mathcal{R}^{1}_{\lceil E_1 L \log (W) \rceil} (W,  L   ),
		\end{equation}
		where (reading from the left) the first inclusion is a consequence of  Proposition~\ref{prop:tradeoff_binary} and the second follows directly from $kb \geq \lceil E_1 L \log (W) \rceil$ and $k \geq 2$. 
  In summary, we obtain
		\begin{align}
			&\,\,\mathcal{A}_\infty	( \lip  ( [0,1] ), \mathcal{R}^1_{b} ( 16W,(k+2) L ) ) \\
            &\leq\, \mathcal{A}_\infty	( \lip  ( [0,1] ), \mathcal{R}_{\lceil E_1 L \log (W) \rceil}^1 (W, L ) ) \label{eqline:apply_upper_bound_special_case}\\
			&\leq\, C_1 ( W^2 L^2 \log (W) )^{-1}\label{eqline:apply_upper_bound_special_case_1} \\
			&=\, \frac{256 C_1 ( k+2 )b}{L \log(W)} ( (16W)^2 (k+2) L \,b )^{-1}\label{eqline:apply_upper_bound_special_case_2} \\
			&\leq \,  \frac{ 2\cdot 256 C_1  kb}{L \log ( W )} ( (16W)^2 (k+2) L \,b )^{-1} \label{eq:compute_ratio_1}\\
			% \leq &\, {c_\ell} \frac{ 256\, C_16 ( E_1 L \log (W) + 1)}{L \log ( W )}  ( (16W)^2 (k+2) L \,b )^{-1} \label{eq:compute_ratio_2} \\
			&\leq \, 4\cdot 256 C_1  \frac{\lceil E_1 L \log (W) \rceil}{L \log ( W )} ( (16W)^2 (k+2) L \,b )^{-1} \label{eq:compute_ratio_2} \\
			&\leq \, 4\cdot 256 C_1  ( E_1 + 1 ) ( (16W)^2 (k+2) L \,b )^{-1} \label{eq:compute_ratio_3} 
		% \end{aligned}
		\end{align}
		where \eqref{eqline:apply_upper_bound_special_case_1} follows from \eqref{eq:error_upper_bound_quantization_error_bound_100}, in \eqref{eq:compute_ratio_1} we used $k \geq 2$, \eqref{eq:compute_ratio_2} is a consequence of $ k b =\bigl\lceil \frac{\lceil E_1 L \log (W) \rceil}{b} \bigr\rceil b \leq 2 \,\frac{\lceil E_1 L \log (W) \rceil}{b} b = 2 \lceil E_1 L \log (W) \rceil$, as  $\lceil x \rceil \leq 2x $ for $x \geq 1$, and \eqref{eq:compute_ratio_3} is by $\lceil E_1 L \log (W) \rceil \leq  E_1 L \log (W)  + 1 \leq ( E_1 + 1  ) L \log (W)$ as $W,L \geq D_1 \geq 2$. It then follows from Proposition~\ref{prop:converse_lower_bound} with $D = 4 \cdot 256 C_1  ( E_1 + 1 )$ and $\mathcal{I} = \{ ( 16W,(\bigl\lceil \frac{\lceil E_1 L \log (W) \rceil}{b} \bigr\rceil +2) L,b ) \in \mathbb{N}^3:  W,L \geq D_1 \text{ and } b < \lceil E_1 L \log (W) \rceil \}$  that $\{\mathcal{R}_{b}^1 ( 16W,(\bigl\lceil \frac{\lceil E_1 L \log (W) \rceil}{b} \bigr\rceil +2) L ) ): W,L \geq D_1 \text{ and } b < \lceil E_1 L \log (W) \rceil \}$ achieves memory optimality in the approximation of $\lip ( [0,1] )$. 

		This shows, as announced, that the transformation from high-to-low precision networks effected by the construction in the proof of Proposition~\ref{prop:tradeoff_binary}, indeed, preserves memory optimality.

		\section{The Three Quantization Regimes} % (fold)
		\label{sec:the_three_quantization_regime}
		
		% \nc{Application of the depth-precision tradeoff in a more refined manner to the minimax error upper bound in Proposition~\ref{prop:sufficient_rate} %yields the following upper bound with three different quantization regimes.}
%Putting We now combine the upper bounds we have obtained, with some refinement for simplicity in presentation, to obtain a 
            Putting the upper bounds we have obtained together, with some refinement for simplicity in presentation, yields the following 
            combined minimax error upper bound exhibiting three different quantization regimes.

		\begin{theorem}
			\label{thm:three_phase_achievability}
			There exist absolute constants $D_2,C_2, E_{2,1}, E_{2,2},\alpha \in \mathbb{R}_+$, with $\alpha > 1$, such that, for $W,L,b \in \mathbb{N}$, with $W,L \geq D_2$, we have $E_{2,2} \frac{\log (W)}{L} < E_{2,1} L  \log (W) $, and the following statements hold:
			\begin{enumerate}
				% \item We have $E_{2,2} \frac{\log (W)}{L} < E_{2,1} L  \log (W)$.
	\item In the under-quantization regime\footnote{We use the convention $\Bigl[1,  E_{2,2} \frac{\log (W)}{L}\Bigr) = \emptyset$, if $E_{2,2} \frac{\log (W)}{L} \leq 1$, that is the \textit{under-quantization regime} can be empty.}, i.e., $b \in [1,  E_{2,2} \frac{\log (W)}{L})$, we have 
				\begin{equation*}
					\mathcal{A}_\infty	( \lip  ( [0,1] ), \mathcal{R}_{b}^{1} ( W, L ) ) \leq C_2 \alpha^{- L b}.
				\end{equation*}

    	\item In the proper-quantization regime, i.e., $ b \in [E_{2,2} \frac{\log (W)}{L}, E_{2,1} L  \log (W) )$, we have 
				\begin{equation*}
						\mathcal{A}_\infty	( \lip  ( [0,1] ), \mathcal{R}_{b}^1 ( W, L ) ) \leq  C_2 ( W^2L b )^{-1}.
				\end{equation*}
    
    \item In the over-quantization regime, i.e.,  $b \in  [E_{2,1} L  \log (W), \infty) $, we have
				\begin{equation*}
						\mathcal{A}_\infty	( \lip  ( [0,1] ), \mathcal{R}_{b}^1 (  W, L ) ) \leq C_2 (W^2 L^2 \log (W))^{- 1}.
				\end{equation*}

			\end{enumerate}

			\begin{proof}
				See Appendix~\ref{sub:approximating_lipschitz_continuous_functions_under_assumption_basic_assumption}.
			\end{proof}
		\end{theorem}
		%{\color{blue}

        Combining Theorem~\ref{thm:three_phase_achievability} with Corollary~\ref{thm:approximation_error_lower_bound_overall} now yields three different quantization regimes in terms of $b$ as a function of network width $W$ and depth $L$.
\iffalse 
exhausts the parameter range for $b$ and yields three different quantization regimes identified by network width $W$ and depth $L$.  %We note that the \textit{under-quantization regime} may be empty.
  %if $ L < E_{2,2} \log (W) $
We now put it together with 
%	We now put the minimax error upper bound in Theorem~\ref{thm:three_phase_achievability} together with 
  the combined minimax error lower bound from Corollary~\ref{thm:approximation_error_lower_bound_overall} given by
  %given in Corollary~\ref{thm:approximation_error_lower_bound_overall}, according to, 
		\begin{equation}
		\label{eq:combined_lower_bound_123}
            \begin{aligned}
			&\,\mathcal{A}_\infty	( \lip  ( [0,1] ), \mathcal{R}_b^1 ( W, L  )   ) \\
            &\geq\, c_\ell \max \bigl\{ ( W^2 L \,b )^{-1}, ( W^2 L^2 ( \log (W) + \log (L) ) )^{-1}, 2^{-Lb} \bigr \},
   \end{aligned}
		\end{equation}
		for all $W,L, b\in \mathbb{N}$, with $L \geq 2$, and $c_\ell$ an absolute constant. 
  \fi 
  These regimes exhibit markedly different minimax error behavior, namely, exponential decrease, polynomial decrease, and constant, mirroring what was already indicated by the combined minimax error lower bound in Corollary~\ref{thm:approximation_error_lower_bound_overall}. While the delineation of the three regimes was left vague in the context of the lower bound, the fact that the upper bound in Theorem~\ref{thm:three_phase_achievability} exhausts the parameter range for $b$ allows us to make the transition boundaries more precise.

	\begin{enumerate}
		\item The \textit{under-quantization regime}: For $b \in [1,  E_{2,2} \frac{\log (W)}{L})$, the minimax error satisfies
		\begin{equation*}
			c_\ell  2^{-Lb} \leq \mathcal{A}_\infty	( \lip  ( [0,1] ), \mathcal{R}_{b}^{1} ( W, L ) ) \leq C_2 \alpha^{- L b},
		\end{equation*}
        and hence falls into a band that decreases exponentially in $b$. 
		%with both bounds decreasing exponentially in $b$. Visually, the minimax error curve with respect to $b$ falls within an exponentially decreasing band, %bounded by the minimax error upper and lower bounds. 
        This behavior emerges as a consequence of the main limiting factor in the neural network approximation of $\lip ( [0,1] )$ being given by the numerical precision of the quantized network weights.

		\item The \textit{proper-quantization regime}: For $ b \in [E_{2,2} \frac{\log (W)}{L}, E_{2,1} L  \log (W) )$, the minimax error satisfies
		\begin{equation*}
		% \label{eq:bound_proper_quantization}
			c_\ell ( W^2 L b )^{-1} \leq \mathcal{A}_\infty	( \lip  ( [0,1] ), \mathcal{R}_{b}^{1} ( W, L ) ) \leq C_2 ( W^2L b )^{-1},
		\end{equation*}
		and is therefore contained in a polynomially decreasing band. Notably, by Proposition~\ref{prop:converse_lower_bound}, in this regime 
  $\{\mathcal{R}_{b}^1 (W, L ) ): W,L \geq D_2, b \in  [E_{2,2} \frac{\log (W)}{L},  E_{2,1} L  \log (W))  \}$ achieves memory optimality---in the sense of Definition~\ref{def:memory_optimality}---in the approximation of $\lip ([0,1])$.
  
%  we get memory optimality in the sense of Definition~\ref{def:memory_optimality}. {\color{blue}rmk: i.e. the class of families $\{\mathcal{R}_{b}^1 (W, L ) %): W,L \geq D_2, b \in  [E_{2,2} \frac{\log (W)}{L},  E_{2,1} L  \log (W))  \}$ achieves the memory optimality in the approximation of $\lip ([0,1])$.}

\iffalse 
		Notably, as the upper bound and the lower bound incurred by the minimum memory requirement  differ by  an absolute constant only, it then follows from  Proposition~\ref{prop:converse_lower_bound}, that the class of families
		\begin{equation*}
		% \label{eq:form_of_optimality}
			\biggl\{\mathcal{R}_{b}^1 (W, L ) ): W,L \geq D_2, b \in  \biggl[E_{2,2} \frac{\log (W)}{L},  E_{2,1} L  \log (W)\biggr)  \biggr\}
		\end{equation*}
		achieve the memory optimality in the sense of Definition~\ref{def:memory_optimality}.  
\fi

		\item The \textit{over-quantization regime}: For  $b \in  [E_{2,1} L  \log (W), \infty) $, the minimax error satisfies
		\begin{equation*}
			c_\ell ( W^2 L^2 ( \log (W) + \log (L) ) )^{-1}\leq \mathcal{A}_\infty	( \lip  ( [0,1] ), \mathcal{R}_{b}^{1} ( W, L ) ) \leq C_2 (W^2 L^2 \log (W))^{- 1},
		\end{equation*}
		and hence resides in a band between two constants (w.r.t. $b$).
  %bands. {\color{blue} within a constant band? }
  %the minimax error curve is bounded by a constant band.
	\end{enumerate}

	%}

%	{\color{blue}
	These results also provide guidance on the choice of network architectures in practical applications. 
 %, specifically in terms of wide vs. deep networks.
 %Our analysis of the minimax error also provides guidance for choosing between wide and deep networks. 
 Specifically, assume that one operates under a total memory budget for the storage of approximating neural networks in $\mathcal{R}_{b}^1 (W,L)$.
Recalling that $b+3$ bits are required to store an individual network weight, the overall fixed bit budget is given by  $n:= W^2 L (b+3)$.  
% consider a hardware system which exhibits a fixed weight resolution $b$ and has a total memory budget for the storage of the approximating neural network. %The question we ask is which network architecture (specifically wide vs. deep) minimizes the approximation error.
 %scenario with a fixed precision hardware system and a limited memory capacity, where we want to find approximating networks, stored in the hardware system, %with the best network architecture that minimizes the approximation error. 
 Seeking memory optimality, we now ask which choices of $W$ and $L$ maximize the size of the proper-quantization regime. With $L = \frac{n}{W^2 (b+3)}$, it follows that
 %Concretely, we want to identify the 
 %For our setting of approximating functions in $\lip ( [0,1] )$ by networks in  $\mathcal{R}_b^1 ( W,L )$, the goal is to find the 
 %triplet $( W,L,b )$ minimizing $\mathcal{A}_\infty	( \lip  ( [0,1] ), \mathcal{R}_{b}^{1} ( W, L ) )$ under the constraints of fixed $b$ and a fixed %number of bits $n  = W^2 L b$. As we are seeking memory optimality, we want to operate in the proper-quantization regime, which upon noting that
 %$L = \frac{n}{W^2 b}$ 
 %We note that \eqref{eq:combined_lower_bound_123} necessitates $\mathcal{A}_\infty	( \lip  ( [0,1] ), \mathcal{R}_{b}^{1} ( W, L ) )  \geq c_\ell ( W^2 L b %)^{-1} = c_\ell n^{-1}$, achievable within a multiplicative absolute constant only when $b$ is in the \textit{proper-quantization regime}, for which %$\mathcal{A}_\infty	( \lip  ( [0,1] ), \mathcal{R}_{b}^{1} ( W, L ) ) \leq C_2 ( W^2L b )^{-1} = C_2 n^{-1}$. Noting that $L = \frac{n}{W^2 b}$,  the %\textit{proper-quantization regime} 
$\bigl[E_{2,2} \frac{\log (W)}{L}, E_{2,1} L  \log (W) \bigr) = \bigl[\frac{E_{2,2} \log (W) W^2 (b+3)}{n}, \frac{E_{2,1} \log (W) n}{W^2 (b+3)} \bigr) $, which shows that for $n$ and $b$, and hence $W^2 L$, fixed, 
%the regime   
%  and hence expands as network width $W$ decreases, i.e., as network depth $L$ increases. Therefore, in this setting, 
deep networks will result in larger proper-quantization regimes than wide networks.
%maximize 
%increasing the network depth $L$ rather than the width $W$ will expand 
%the \textit{proper-quantization regime}.
%for $b$ to fall within so that the corresponding network achieves a smaller approximation error. 
This insight adds to the existing literature on depth-width tradeoffs in neural network approximation, see, e.g., \cite{deep-it-2019, Vardi2022, telgarsky2016benefits}.	
 %}

	% section universal_and_optimal_approximation_by_deep_relu_networks (end)
	%!TEX root = ../draft_quantized_weight_networks.tex

% \section{Discussion} % (fold)
% \label{sec:future_work}

	% The memory optimality is achieved only in the second region. where $a=b$ belongs to $ \left[ C_4 \frac{\log W}{L}, C_3 L \left( \log L + \log W \right) \right] $. The length of the interval depends positively on the depth $L$. This suggests deeper networks have greater tolerance to weight precision, while keeping the memory optimality. 

	% \todo{Comments on the example}

	% In the Section~\ref{sec:embedding_properties_of_relu_networks_with_quantized_weights}, we will give an approximation error upper bound of a similar form under some mild conditions and an assumption. This suggests the approximation error lower bound is tight and therefore the approximation error itself $\mathcal{A}_\infty	\left( \lip   \left( [0,1]^d \right), \mathcal{R}_b^a \left( d, W, L  \right)  \right)$ also have two regions with different rate under some assupmtion.

% section future_work (end)

	\appendix

		\section{Proof of Proposition~\ref{prop:lower_bound_approximation_VC_dimension}} % (fold)
		\label{sub:proof_of_proposition_prop:lower_bound_approximation_vc_dimension}

		We first describe the main ingredients of the proof, starting with the definition of VC dimension.
		\begin{definition}
			\cite[Definition 1]{bartlett2019nearly} Let $H$ denote a class of functions mapping from $\mathcal{X}$ to $\{ 0,1 \}$. Define the growth function as
			\begin{equation*}
				\Pi_H ( m ) := \max_{x_1,\dots,x_m \in \mathcal{X}} | \{ (h(x_1),\dots, h ( x_m )): h \in H\}|, \quad \text{for } m \in \mathbb{N}. 
			\end{equation*}
			For a given set $\{ x_1,\dots, x_m \} \subseteq \mathcal{X}$, if $| \{ (h(x_1),\dots, h ( x_m )): h \in H\}| = 2^m$, we say that $H$ shatters $\{ x_1, \dots,x_m \}$. The \textit{Vapnik-Chervonenkis (VC) dimension} of $H$, denoted by $\text{VCdim} ( H )$, is the largest $m$ such that $\Pi_{H} ( m ) = 2^m$. If there is no such largest $m$, we set $\text{VCdim} ( H ) = \infty$.
		\end{definition}
		VC dimension upper bounds for certain families of ReLU networks are provided in \cite{bartlett2019nearly}. The results in \cite{bartlett2019nearly} apply, however, only to families of network realizations whose associated configurations have a fixed architecture, whereas $\mathcal{N}(W,L)$, the object of interest here, consists of network configurations with different architectures. The following result shows how \cite[Eq. (2)]{bartlett2019nearly} can be adapted to this setting.

		\begin{lemma}
			For all $W \in \mathbb{N}$ and $L \in \mathbb{N}$, with $L \geq 2$, we have 
			\begin{equation}
				\label{eq:vc_dimension_upper_bounds}
				\text{VCdim}\,( \text{sgn} \circ \mathcal{R} ( W,L )  ) \leq C_h W^2L^2 (\log ( W ) + \log ( L )),\quad 
			\end{equation}
			where $\text{sgn} \circ \mathcal{R} ( W,L )  := \{ \text{sgn}\circ f: f \in \mathcal{R} ( W,L ) \}$ and $C_h$ is an absolute constant. 
			\begin{proof}
				Fix $W\in \mathbb{N}$ and $L \in \mathbb{N}$,  with $L \geq 2$, throughout the proof. Consider the set $\mathcal{N}^* ( 2W,L ) = \{ ( A_\ell,b_\ell )_{\ell = 1}^{L}: A_1 \in \mathbb{R}^{2W \times 1}, b_1 \in \mathbb{R}^{2W}, A_L \in \mathbb{R}^{1 \times 2W}, b_L \in \mathbb{R}^1, A_\ell \in \mathbb{R}^{2W \times 2W}, b_\ell \in \mathbb{R}^{2W}, \text{ for } \ell \in \{ 2,\dots, L-1 \} 
    %\backslash \{ 1,L \}   
    \}$ consisting of all network configurations with the  fixed architecture
				\begin{equation}
				\label{eq:defining_architecture}
					( N_\ell )_{\ell = 0}^L =  (1, \underbrace{2W, \dots,}_{\text{ repeats } (L-1) \text{ times}} 1).
				\end{equation}
				The associated family of network realizations is $\mathcal{R}^* ( 2W,L ) = \{  R (\Phi ): \Phi \in  \mathcal{N}^* ( 2W,L )\}$. We note that the network configurations in  $\mathcal{N}^* ( 2W,L )$ have $ n(2W,L) := 6W + 1 + ( L-2 ) ( (2W)^2 + 2W )$ weights.   As $\mathcal{R}^* ( 2W,L )$ consists of realizations of network configurations with the fixed architecture \eqref{eq:defining_architecture}, we can apply the  results in \cite{bartlett2019nearly}. Specifically, it follows that 
				\begin{align}
					\text{VCdim} ( \text{sgn} \circ\mathcal{R}^* (2W,L )  ) \leq&\, C  n ( 2W,L ) L \log ( n ( 2W,L ) ) \label{eq:vc_dimension_upper_bounds_vanilla_1} \\
					\leq &\, C ( 13W^2 L ) L \log ( 13 W^2 L ) \label{eq:vc_dimension_upper_bounds_vanilla_2} \\
					\leq &\, 104\, C W^2 L^2 ( \log(W) + \log(L) ), \label{eq:vc_dimension_upper_bounds_vanilla_3}
				\end{align}
				where in \eqref{eq:vc_dimension_upper_bounds_vanilla_1} we used \cite[Eq.(2)]{bartlett2019nearly} with $C \in \mathbb{R}_+$ an absolute constant, \eqref{eq:vc_dimension_upper_bounds_vanilla_2} follows from $n ( 2W,L ) \leq 13 W^2 L $, and  \eqref{eq:vc_dimension_upper_bounds_vanilla_3} is owing to $\log(13 W^2 L) \leq \log((WL)^8) = 8 ( \log(W) + \log(L) )$.

				We continue by showing that $R ( W,L ) \subseteq R^* ( 2W, L )$, which will then allow us to conclude that \eqref{eq:vc_dimension_upper_bounds_vanilla_3} also upper-bounds $\text{VCdim} ( \text{sgn} \circ \mathcal{R} ( W,L )  )$. To this end, fix  an $f \in R ( W,L ) $. It follows from Lemma~\ref{lem:extension} that there exists a network configuration $\overline{\Phi} = ( \overline{A}_\ell, \overline{b}_\ell )_{\ell = 1}^L \in \mathcal{N} (\max \{ W, 2 \},L) \subseteq \mathcal{N} (2W,L)$  such that  $R ( \overline{\Phi} ) = f$. Next, we enlarge the layers of $\overline{\Phi}$ such that the resulting configuration has the architecture \eqref{eq:defining_architecture} while realizing the same function $f$. To this end, denote the architecture of $\overline{\Phi}$ by $( \overline{N}_\ell )_{\ell = 0}^L$, and note that $\overline{N}_0 = 1 = N_0$, $\overline{N}_L = 1 = N_L$, and $\overline{N}_{\ell} \leq 2W = N_\ell$, for $\ell \in \{ 2,\dots, L-1 \} 
    %\backslash \{ 1,L \}
    $. Now augment the configuration $\Phi$ to the architecture in \eqref{eq:defining_architecture} according to  $\widetilde{\Phi} = ( \tilde{A}_\ell, \tilde{b}_\ell )_{\ell = 1}^L \in \mathcal{N}^* ( 2W,L )$, with 
				\begin{equation*}
					\tilde{A}_\ell = \begin{pmatrix}
						\overline{A}_\ell & 0_{ \overline{N}_\ell \times ( N_{\ell-1}  - \overline{N}_{\ell-1} ) }\\
						0_{(N_{\ell}  - \overline{N}_{\ell}) \times \overline{N}_{\ell-1}} & 0_{(N_{\ell}  - \overline{N}_{\ell}) \times   (N_{\ell-1}  - \overline{N}_{\ell-1} )}
					\end{pmatrix}, \quad \tilde{b}_\ell = \begin{pmatrix}
						\overline{b}_\ell\\
						0_{N_\ell - \overline{N}_{\ell}}
					\end{pmatrix}, \quad \ell = 1,\dots, L.
				\end{equation*}
				We then have $R ( \widetilde{\Phi} ) = R ( \overline{\Phi})$ and thereby $f = R ( \widetilde{\Phi} ) = R ( \overline{\Phi}) \in R^* ( 2W, L )$. As the choice of $f \in R ( W,L ) $ was arbitrary, we have established that $R ( W,L ) \subseteq R^* ( 2W, L )$ as announced. We can finally conclude that
				\begin{align*}
					\text{VCdim} ( \text{sgn} \circ \mathcal{R} ( W,L )  ) \leq \text{VCdim} ( \text{sgn} \circ\mathcal{R}^* ( 2W,L )  ) \leq 104 C W^2 L^2 ( \log(W) + \log(L) ),
				\end{align*}
				where in the last inequality we used \eqref{eq:vc_dimension_upper_bounds_vanilla_1}-\eqref{eq:vc_dimension_upper_bounds_vanilla_3}. The proof is concluded by setting $C_h := 104 C$.
			\end{proof}
		\end{lemma}

		With the VC dimension upper bound \eqref{eq:vc_dimension_upper_bounds}, we are now ready to proceed to the proof of Proposition~\ref{prop:lower_bound_approximation_VC_dimension}.
		\begin{proof}
			[Proof of Proposition~\ref{prop:lower_bound_approximation_VC_dimension}] 
			Set $c_v = ( 4 ( C_h + 1 ) )^{-1}$, where $C_h$ is the constant in \eqref{eq:vc_dimension_upper_bounds}. Suppose, for the sake of contradiction, that the approximation error lower bound \eqref{eq:lower_bound_VC_dimension} does not hold with this $c_v$. This would then imply the existence of a $W \in \mathbb{N}$ and an $L \in \mathbb{N}$, with $L\geq 2$, such that 
			\begin{align}
				\mathcal{A}_\infty	( \lip  ( [0,1] ), \mathcal{R} (W, L )  ) <&\, c_v ( W^2 L^2 ( \log (W) + \log (L) ) )^{-1}. \label{eqline:vc_contradiction_1}
				% \\
				% < & ??   ( W^2 L^2 ( \log W + \log L ) )^{-1}.\label{eqline:vc_contradiction_2}
			\end{align}
			Let $N = \lceil C_h W^2L^2 (\log ( W ) + \log ( L ))\rceil$.  We shall show that \eqref{eqline:vc_contradiction_1} implies that the family $\text{sgn} \circ \mathcal{R} ( W,L ) $ shatters the set $\{ \frac{0}{N},\dots, \frac{N}{N} \}$, which will then lead to a contradiction to \eqref{eq:vc_dimension_upper_bounds}. To this end, fix $( \theta_0,\dots, \theta_N) \in \{ 0,1 \}^{N+1}$, and let  $f \in H^1 ( [0,1] )$ be given by 
			\begin{equation*}
				f(x) = (\theta_i - \theta_{i-1}) \biggl( x - \frac{i -1}{N} \biggr) + \frac{2\theta_{i-1} -1}{2N},\quad  x \in \biggl[\frac{i -1}{N}, \frac{i}{N}\biggr],\,\,  i =1,\dots, N,
			\end{equation*}
			such that, for $i = 0,\dots, N$,
			\begin{equation}
				\label{eq:shattering_f}
				f \biggl( \frac{i}{N} \biggr) =\frac{2 \theta_i - 1}{2N} = 
				\begin{cases}
					\frac{1}{2N},& \text{if } \theta_i = 1,\\
					-\frac{1}{2N}, & \text{if } \theta_i = 0.
				\end{cases}
			\end{equation}
			Then, according to \eqref{eqline:vc_contradiction_1}, there exists a ReLU network realization $g \in \mathcal{R} (W, L)$ such that $\| g - f \|_{L^\infty ( [0,1] )} < c_v ( W^2 L^2 ( \log (W) + \log (L) ) )^{-1}  = \frac{1}{4} ( ( C_h + 1 )W^2 L^2 ( \log (W) + \log (L) )  )^{-1} \leq \frac{1}{4} (  \lceil C_h W^2 L^2 ( \log (W) + \log (L) )  \rceil)^{-1}= \frac{1}{4N}$, which in combination with \eqref{eq:shattering_f}  implies that,  for $i = 0,\dots, N$, 
			\begin{equation*}
				g \biggl( \frac{i}{N} \biggr) 
				\begin{cases}
					> 0,& \text{if } \theta_i = 1,\\
					<0, & \text{if } \theta_i = 0,
				\end{cases}
			\end{equation*}
			and therefore $( \text{sgn}\circ g )  ( \frac{i}{N} )  = \theta_i$, for $i = 0,\dots, N$. Upon noting that $( \text{sgn}\circ g ) \in \text{sgn} \circ  \mathcal{R} ( W,L ) $, we have 
			\begin{equation*}
				 (\theta_0,\dots, \theta_N)  \in  \biggl\{ \biggl(h\biggl(\frac{0}{N}\biggr),\dots, h \biggl( \frac{N}{N} \biggr)\biggr) : h \in \text{sgn} \circ  \mathcal{R} ( W,L ) \biggl\}. 
			\end{equation*}
			Since the choice of  $\{ \theta_0,\dots, \theta_N \} \in \{ 0,1 \}^{N+1}$ was arbitrary, we have, indeed, established that
			\begin{equation*}
				\{ 0,1 \}^{N+1}  \subseteq  \biggl\{ \biggl(h\biggl(\frac{0}{N}\biggr),\dots, h \biggl( \frac{N}{N} \biggr)\biggr) : h \in \text{sgn} \circ  \mathcal{R} ( W,L )  \biggl\},
			\end{equation*}
			and therefore $| \{ (h(\frac{0}{N}),\dots, h ( \frac{N}{N} )) : h \in \text{sgn} \circ \mathcal{R} ( W,L )  \} | = | \{ 0,1 \}^{N+1} |  = 2^{N+1}$. This proves that $\text{sgn} \circ \mathcal{R} ( W,L ) $ shatters the set $\{ \frac{0}{N},\dots, \frac{N}{N} \}$, which in turn leads to 
			\begin{align*}
				\text{VCdim}( \text{sgn} \circ \mathcal{R} ( W,L )  ) \geq \biggl| \biggl\{ \frac{0}{N},\dots, \frac{N}{N} \biggr\} \biggr| =   N + 1 >  C_h W^2L^2 (\log ( W ) + \log ( L )),
			\end{align*}
			and thus stands in contradiction to the VC dimension upper bound \eqref{eq:vc_dimension_upper_bounds}. Therefore, \eqref{eq:lower_bound_VC_dimension} must hold. 

			Upon noting that $\mathcal{R}_\mathbb{A} (W, L)   \subseteq \mathcal{R} (W, L   )$, which follows from $\mathcal{N}_\mathbb{A} (W, L)   \subseteq \mathcal{N} (W, L  )  $, we obtain \eqref{eq:lower_bound_VC_dimension_general_weights} from \eqref{eq:lower_bound_VC_dimension}. This concludes the proof.
 		\end{proof}

\section{Proof of Theorem~\ref{thm:approximation_lip}} % (fold)
\label{sec:approximate_Lip1}

	We prove Theorem~\ref{thm:approximation_lip} with $\mathcal{R} ( W,L,1 )$ replaced by $\mathcal{R} ( W,L,W^K )$, where $K \in \mathbb{N}$ is an absolute constant, i.e., for networks with weight magnitude growing polynomially in $W$, and then relax this polynomial dependency using Proposition~\ref{prop:depth_weight_magnitude_tradeoff}. 

	\begin{proposition}
	\label{prop:approximation_lip_increasing_weights}
		There exist absolute constants $C_a,D_a  \in \mathbb{R}_+$ and an absolute constant $K \in \mathbb{N}$, such that for all $W,L \in \mathbb{N}$ with $W,L \geq D_a$,
		\begin{equation}
			\begin{aligned}
				\mathcal{A}_\infty	( \lip  ( [0,1] ), \mathcal{R} (W, L, W^K  )  ) \leq   C_a  ( W^2 L^2 \log (W)  )^{-1}.
			\end{aligned}
		\end{equation}
		\begin{proof}
			See Appendix~\ref{sub:preparation_to_prove_proposition_prop:approximation_lip_increasing_weights} for preparatory material and then Appendix~\ref{sub:proof_of_proposition_sub:preparation_to_prove_proposition_prop:approximation_lip_increasing_weights} for the actual proof.
		\end{proof}
	\end{proposition}
	% The weight-magtitude depends polynomially on the width instead of being the constant one.  

	The proof of Theorem~\ref{thm:approximation_lip} is  now effected by applying Proposition~\ref{prop:depth_weight_magnitude_tradeoff} along with Proposition~\ref{prop:approximation_lip_increasing_weights} as follows.
	% Given Proposition~\ref{prop:approximation_lip_increasing_weights}, we can now apply  to prove Theorem~\ref{thm:approximation_lip} with the help
	% \begin{proof}
	% 	[Proof of Theorem~\ref{thm:approximation_lip}]
		% Let $P$ be the polynomial as given by Proposition~\ref{prop:approximation_lip_increasing_weights}. There exists $k \in \mathbb{N}$ depending on $f$ only such that 
		% \begin{equation}
		% 	\label{eq:choice_of_k_for_polynomial}
		%  	P(x) \leq x^k, \,\text{for all } x \geq 2.
		% \end{equation} 
		% Since $P$ does not depend on anything, $k$ is an absolute constant. 
		Set $D = \max \{ ( 4K+2 )\lceil D_a \rceil, 10 \}$, where $D_a$ is the constant specified in Proposition~\ref{prop:approximation_lip_increasing_weights}. Fix $W,L \in \mathbb{N}$ with $W,L \geq D$.  Let $U = \lfloor \frac{L}{2K+1} \rfloor$, ensuring that $( 2K+1 )U \leq L$ and  $U \geq \lfloor \frac{D}{2K+1} \rfloor \geq \bigl\lfloor \frac{(4K +2) \lceil D_a \rceil}{2K+1} \bigr\rfloor \geq D_a$. We have
		\begin{align}
			\mathcal{R} (W, U, W^K )  \subseteq& \, R ( W, (2K+1)U, 1 )  \label{eqline:proof_main_1} \\
			\subseteq&\, \mathcal{R} (W,L,1), \label{eqline:proof_main_2}
		\end{align} 
		where \eqref{eqline:proof_main_1} is a consequence of Proposition~\ref{prop:depth_weight_magnitude_tradeoff} with $(W,L,L',B,B') $ replaced by $( W, U, 2KU, \allowbreak W^K,  1   )$ and the prerequisite satisfied as $\frac{(\lfloor W\slash 2 \rfloor)^{2KU}}{(W^K)^{U}} \geq \frac{ W^{KU}}{W^{KU}} = 1$ thanks to $( \lfloor x \slash 2  \rfloor )^2 \geq x $, for  $ x = W \geq D  =\max \{ ( 4K+2 )\lceil D_a \rceil, 10 \} \geq  10$, and  in \eqref{eqline:proof_main_2} we used $(2K+1)U \leq L$.  We now get 
		\begin{align}
			\mathcal{A}_\infty	( \lip  ( [0,1] ), \mathcal{R} (W, L, 1  )  )  \leq &\: \mathcal{A}_\infty	( \lip  ( [0,1] ), \mathcal{R} (W, U, W^K  ) ) \label{eqline:W_to_U}\\
			\leq &\: C_a  ( W^2 U^2 \log (W)  )^{-1} \label{eqline:apply_main_proposition_U} \\
			\leq &\: C_a \Bigl( W^2 \Bigl( \frac{1}{2}\cdot \frac{L}{2K+1} \Bigr)^2 \log (W)  \Bigr)^{-1} \label{eqline:U_greater_than_L} \\
			=&\: ( 4K + 2 )^2 C_a ( W^2 L^2 \log (W) )^{-1},
		\end{align}
		where in \eqref{eqline:W_to_U} we used the inclusion \eqref{eqline:proof_main_1}-\eqref{eqline:proof_main_2}, in \eqref{eqline:apply_main_proposition_U} we applied Proposition~\ref{prop:approximation_lip_increasing_weights} with $C_a$ as specified in Proposition~\ref{prop:approximation_lip_increasing_weights}, and \eqref{eqline:U_greater_than_L} follows from $U = \bigl\lfloor \frac{L}{2K+1} \bigr\rfloor \geq \frac{1}{2} \cdot \frac{L}{2K+1} $, upon noting that $U \geq D_a > 0$. The proof is finalized by taking $C = ( 4K +2 )^2 C_a$.
	% \end{proof}

	\subsection{Preparation for the Proof of Proposition~\ref{prop:approximation_lip_increasing_weights}} % (fold)
	\label{sub:preparation_to_prove_proposition_prop:approximation_lip_increasing_weights}
	% \todo{Change the set notation to the set.}

	The proof of Proposition~\ref{prop:approximation_lip_increasing_weights} is based on two specific ingredients, namely the realization of one-dimensional bounded piecewise linear functions by ReLU networks and the bit extraction technique.
	
	% subsection preparation_to_prove_proposition_prop:approximation_lip_increasing_weights (end)
	% It remains to prove Proposition~\ref{prop:approximation_lip_increasing_weights} in the rest of Appendix~\ref{sec:approximate_Lip1} and we start by preparing necessary tools for the proof. 

	\subsubsection{Realizing One-Dimensional Bounded Piecewise Linear Functions by ReLU Networks} % (fold)
	\label{ssub:realizing_piece_wise_linear_function_by_relu_networks}
	
	% subsubsection realizing_piece_wise_linear_function_by_relu_networks (end)

	% As mentioned in the introduction, the first ingredient is ReLU networks realizing bounded piecewise-linear functions, per the following definition.
	% \begin{figure}[H]
	% 	\centering
	% 	\includegraphics[width=0.5\textwidth]{images/approximation.jpg}
	% 	\caption{Approximating $g$}
	% 	\label{fig:figure1}
	% \end{figure}
	% The main ingredients used by those literature and our proof are similar, which are piecewise-linear representation/approximation by ReLU networks and the bit extraction technique. Our contribution is to provide a weight-magnitude-controlled version of these two ingredients as follows. We start with the piecewise-linear representation by ReLU networks.
	% \todo{bounded continuous piecewise linear function}
	We start with the definition of one-dimensional bounded piecewise linear functions.
	\begin{definition}
		[One-dimensional bounded piecewise linear functions]
		Let $M \in \mathbb{N}$, with $M \geq 3$, $E \in \mathbb{R}_+ \cup \{ \infty \}$, and let $X = (x_i)_{i= 0}^{M-1}$ be a strictly increasing sequence taking values in $\mathbb{R}$. Define the set of functions
		\begin{alignat*}{2}
			\Sigma ( X, E) = \bigl\{f \in&&\, C ( \mathbb{R} ): \| f \|_{L^\infty ( \mathbb{R} )}  \leq E,  f \text{ is constant on } (-\infty, x_0] \text{ and } [x_{M-1}, \infty),\, \\
			&&\,f \text{ is affine on } [x_i, x_{i +1}], \, i = 0,\dots, M -2 \bigr\}.
		\end{alignat*}
	\end{definition}
	For a function $f \in \Sigma ( X, E)$, we call $X$ the set of breakpoints of $f$, as the slope of $f$ can change only at these points. We  refer to the intervals $(-\infty, x_0], [x_i, x_{i +1}], \, i = 0,\dots, M -2, [x_{M-1}, \infty)$ as the piecewise linear regions of $f$.

	% Next, we note that, for every finite set of data points $\{ ( x_i,y_i ) \}_{i=0}^{M-1} \in (\mathbb{R}^2)^{M}$ there exists a function $f \in \Sigma ( \{ x_i \}_{i = 0}^{M-1}, \max_{i=0, \dots, M-1}| y_i |  )$ such that $f(x_i) = y_i$, $i = 0,\dots, M-1$. \cite{shen2019nonlinear} \cite{shen2019deep} provides a way to approximate an $f \in \Sigma ( X, E)$ by a ReLU network $g$, in the sense that $f = g$ except for on the interior of some piecewise linear regions of $f$. Meanwhile, \cite{daubechies2022nonlinear} constructs ReLU networks exactly realizing continuous piecewise linear functions with finite distinct breakpoints in $(0,1)$, where the functions are not necessary bounded as opposed to our setting. These works in literature either don't discuss the weight magnitude of networks they used, as in \cite{daubechies2022nonlinear}, or give only control over the weight magnitude not strong enough to prove Theorem~\ref{thm:approximation_lip}, as in \cite{shen2019nonlinear} \cite{shen2019deep}.
	% where, for $M \geq 3$, the set of continuous piecewise linear functions with exactly $M$ distinct breakpoints can be formally defined as
	% \begin{alignat*}{2}
	% 	&&\Sigma_M = \bigl\{f \in\, C ( \mathbb{R} ): \text{ there exists } X  \subseteq ( 0,1 ) \text{ with } | X | = M \text{ such that } f \text{ is affine on } \\
	% 	&& (-\infty, x_0], [x_i, x_{i +1}], \, i = 0,\dots, M -2, \text{and } [x_{M-1}, \infty)\bigr\}.
	% \end{alignat*}
	% \todo{a function in $\Sigma ( X, \infty )$ is unique determined by its value at its breakpoints $X$.} 
	We now show how one-dimensional bounded piecewise linear functions can be realized through  ReLU networks while retaining control over the networks' weight magnitude. To avoid dealing with tedious corner cases, we restrict ourselves to  $| X |  \geq 3$ and $X \subseteq [0,1]$, which, as seen later, suffices to cover what is needed in the proof of Proposition~\ref{prop:approximation_lip_increasing_weights}.  
	% \begin{proposition}
	% 	% \label{prop:piecewise_representation}
	% 	Let $M \in \mathbb{N}$ with $M \geq 2$, $E \in \mathbb{R}_+$, $X = \{ x_i \}_{i = 0}^{M-1}$ be a strictly increasing sequence taking values in $[0,1]$. Then for all $u,v \in \mathbb{N}$ such that $u^2 v \geq M$, we have
	% 	\begin{align*}
	% 		\Sigma ( X, E) \subseteq&\, \mathcal{R} ( 1,16u + 4, 24v +5 , \max \{ 1, C_k N^7R_m R_c^3 E \}  )\\
	% 		\subseteq&\, \mathcal{R} ( 1,16u + 4, 24v +5 , \max \{ 1,C_k N^7 R_m^4 E \}  ),
	% 	\end{align*}
	% 	for some absolute constant $C_k \in \mathbb{R}^+$, where $R_m := R_m(X) := \sup_{i =1,\dots, N} ( x_{i} - x_{i -1} )^{-1}$, and $R_c := R_c(X) := \frac{\sup_{i =1,\dots, N}  (x_{i} - x_{i -1}) }{\inf__{i =1,\dots, N}  (x_{i} - x_{i -1}) }$. \todo{change $\sup_{i=1}^n$ to $\sup_{i = 1,\dots, N}$, and change $\sup$, $\inf$, to $\max, \min$}
	% 	\begin{proof}
	% 		See Sec~\ref{sub:piecewise_linear_representation}.
	% 	\end{proof}
	% \end{proposition}

	\begin{proposition}
		\label{prop:piecewise_representation}
		Let $M \in \mathbb{N}$, with $M \geq 3$, $E \in \mathbb{R}_+$, and let $X = (x_i)_{i = 0}^{M-1}$ be a strictly increasing sequence taking values in $[0,1]$. For all $u,v \in \mathbb{N}$, and $w \in \mathbb{R}$ with $w \geq 1$, such that 
			\begin{align}
				u^2 v \geq&\, M, \label{eqline:cpwl_realization_architecture_condition}\\
				w^{30v} \geq&\, {M^6 ( R_m(X) )^4 E}, \label{eqline:cpwl_realization_weight_condition}
			\end{align}
			with $R_m(X) := \max_{i =1,\dots, M -1 } ( x_{i} - x_{i -1} )^{-1}$, we have
			\begin{align*}
				\Sigma ( X, E) \subseteq&\, \mathcal{R} ( 20u, 30v, 2w).
			\end{align*}
		\begin{proof}
			The proof, detailed in Appendix~\ref{sub:piecewise_linear_representation}, is constructive, in the sense of explicitly specifying a network realizing a given $f \in \Sigma ( X, E)$.
		\end{proof}
	\end{proposition}

\iffalse
	The functions in $\Sigma ( X, E )$ realized by ReLU networks according to Proposition~\ref{prop:piecewise_representation} determine the depth, width, and weight magnitude of these networks in two different aspects. First, by  \eqref{eqline:cpwl_realization_architecture_condition} the order of the number of the network weights, given by $W^2 L$ with $W = 20u$ and $L = 30v$, is at least of the order of the number of breakpoints of $f$.
	% the size of the realizing network $g$, measured by the number of weights\footnote{We note that a network with width $W$ and depth $L$ has at most $W ( W+1 ) L$ weights.}, is . 
	Second, it follows from \eqref{eqline:cpwl_realization_weight_condition} that the weight magnitude $2w$ increases polynomially in $M$ (i.e., the number of breakpoints  of $f$), $R_m(X)$, and $E$, while decreasing with respect to network depth, $30v$, in an inverse exponential manner. While the first dependence identified here is already visible in \cite{daubechies2022nonlinear, shen2019deep, shen2019nonlinear}, the second
 %, according to \eqref{eqline:cpwl_realization_weight_condition}, 
 appears to be novel. Notably, this second dependence is crucial for controlling the weight magnitude of the networks constructed in the proof of Theorem~\ref{prop:approximation_lip_increasing_weights}.
\fi

Proposition~\ref{prop:piecewise_representation} makes the requirements on the width, depth, and weight magnitude of networks realizing functions $ f \in \Sigma ( X, E )$ explicit, in particular in two aspects. First, by  \eqref{eqline:cpwl_realization_architecture_condition} it suffices to choose the number of network weights, given by $W^2 L$ with $W = 20u$ and $L = 30v$, to be on the order of $M$ (the number of breakpoints of $f$). Second, 
\eqref{eqline:cpwl_realization_weight_condition} rewritten as $w \geq  (M^6 ( R_m(X) )^4 E)^{1\slash (30v)}$ shows that 
%can be reformulated as
%$w \geq  (M^6 ( R_m(X) )^4 E)^{1\slash 30v}$. Therefore, the minimum requirement  for 
the weight magnitude $2w$ can be taken to grow no faster than polynomial in $M$, $R_m(X)$, and $E$, and is allowed to decrease with respect to network depth, $30v$, in an inverse exponential manner. While the first requirement is identical to those reported in \cite{daubechies2022nonlinear, shen2019deep, shen2019nonlinear}, the second one is novel and constitutes a relaxation relative to the constructions available in the literature.
%nt is  explicit and potentially more relaxed than those in the literature. 
Specifically, \cite{daubechies2022nonlinear} shows that $\Sigma ( X, \infty) \subseteq \mathcal{R} (W,L, \infty)$, with $W,L \in \mathbb{N}$ depending on $X$, in a way that does not allow for general conclusions on how the weight magnitude of the network depends on the function $f \in \Sigma ( X, \infty)$ to be realized. %the corresponding piecewise linear functions. 
Scrutinizing the proofs in \cite{shen2019deep, shen2019nonlinear}, one finds that the weight magnitude of the networks constructed therein has to be at 
least exponential in the number of breakpoints of the piecewise linear function realized; this stands in stark contrast to our construction which requires 
polynomial growth only.

	\subsubsection{Bit Extraction} % (fold)
	\label{ssub:bit_extraction_realization}

	Another important ingredient for the proof of Proposition~\ref{prop:approximation_lip_increasing_weights} is the bit extraction technique, as first introduced in  \cite{bartlett1998almost} to derive a lower bound on the VC dimension of ReLU networks. A refinement of this technique improving the lower bound in  \cite{bartlett1998almost} was reported in \cite{bartlett2019nearly}. Further variants developed in the context of function approximation through ReLU networks can be found in \cite{yarotsky2018optimal, yarotsky2019phase, shen2021optimal, Kohler2019, Vardi2022}.

	There are two fundamental components constituting the bit extraction technique in all its variants. The first one encodes a string of elements from a finite alphabet into a real number. For example, \cite{bartlett1998almost} encodes the string $( \theta_i )_{i = 1}^s \in \{ 0,1 \}^s$ into the number $E ( ( \theta_i )_{i = 1}^s  ) = \sum_{i = 1}^s 2^{-i} \theta_i$. The references \cite{bartlett2019nearly, shen2021optimal, Vardi2022} also work with $\{ 0,1 \}$-alphabets,  while  \cite{yarotsky2018optimal, yarotsky2019phase, Kohler2019}  employ $k$-ary alphabets, $k \in \mathbb{N}$, with $k > 2$, but otherwise follow the philosophy of \cite{bartlett1998almost}. Specifically, all these approaches encode strings of $k$-ary digits into $k$-ary numbers. The second component consists of a decoder $D$, realized by a ReLU network, which extracts either individual elements of the string encoded into the real number or functions thereof. For example, the decoder $D: \mathbb{R}^2 \mapsto \mathbb{R}$ in \cite{bartlett1998almost} extracts the individual $\theta_\ell$, $\ell = 1, \dots, s,$ from $E ( ( \theta_i )_{i = 1}^s) $ according to
% and the corresponding decoder $D: \mathbb{R}^2 \mapsto \mathbb{R}$ satisfies  
	\begin{equation*}
		D ( E ( ( \theta_i )_{i = 1}^s  ), \ell  ) = \theta_\ell, \quad \ell = 1,\dots, s,
	\end{equation*}
	whereas in \cite{shen2021optimal} the sum of leading elements of the string $( \theta_i )_{i = 1}^s$ is recovered through\footnote{We use the convention $\sum_{i=1}^{0} \theta_{i} = 0$.}
	\begin{equation}
	\label{eq:sum_extraction}
		D ( E ( ( \theta_i )_{i = 1}^s  ), k  ) = \sum_{i = 1}^\bit \theta_i, \quad k =0,\dots, s.
	\end{equation}
	%where we use the convention $\sum_{i=1}^{0} \theta_{i} = 0$.

	% Then we need a decoder $D$ realized by a ReLU network which ``extracts''  from $E ( ( \theta_i )_{i = 1}^s )$ some specific information of the bit sequence $( \theta_i )_{i = 1}^s \in \{ 0,1 \}^s$. Two common choice of the information are consider. The first is to extract the $k$-th bit, $ k \in \{ 1,\dots, s \}$, whose corresponding $D$ satisfies 
	% \begin{equation*}
	% 	D ( E ( ( \theta_i )_{i = 1}^s  ), k  ) = \theta_i, \text{for } k = 1,\dots, s.
	% \end{equation*}

	The bit extraction technique we develop here is also based on the two constituents just described, but provides improvements and refinements---to be discussed as we go along---in ways that are fundamental to our purposes. We start with definitions and notation regarding the choice of alphabet and the encoding procedure.

	% We formulate our specific choice of alphabet, way of encoding, and decoding networks in the following, and later discuss our improvement over other choice in the literature. For the first step, we start with definitions and notations regarding the choice of alphabet and the way of encoding.
	% As opposed to existed literature, and can extract a sum of digit rather than a single digit. Different from    .  , specifically, ternary number with alphabet $\{ 0,1 \}$, and this leads to better control on the weights of the bit extraction ReLU networks.
	% \todo{formal definition and notation for compact expression.}
	\begin{definition}
	\label{def:our_encoder}
			Let $\mathbb{T}$ be the set of ternary numbers with reduced alphabet $\{ 0,1 \}$ and possibly infinitely many digits, formally, 
			\begin{equation*}
				\mathbb{T} := \biggl\{ \sum_{i = 1}^\infty \theta_i 3^{-i}: \theta_i \in \{ 0,1 \}, i \in \mathbb{N}  \biggr\}.
			\end{equation*}
			Set $T ( ( \theta_1,\dots, \theta_s ) ) := \sum_{i = 1}^s \theta_i 3^{-i}$, for $s \in \mathbb{N}$, $\theta_i \in \{ 0,1,2 \}$, $i = 1,\dots, s$.
	\end{definition}

The basic idea underlying the first component of our variation of the bit extraction technique is to encode 
 %work with $\{ 0,1 \}$-alphabets. Different from \cite{bartlett1998almost},  we encode a $\{ 0,1 \}$-string $( \theta_i )_{i = 1}^s$, $s \in \mathbb{N}$, as 
 $\{ 0,1 \}$-strings $( \theta_i )_{i = 1}^s$, $s \in \mathbb{N}$, into ternary numbers according to $T (( \theta_i )_{i = 1}^s) = \sum_{i = 1}^s \theta_i 3^{-i} \in \mathbb{T}$. 
	% Regarding the encoder, we denote the encoding function by $T$.
	The ReLU network realizing the decoder we employ is specified in the following result, whose proof is constructive.

	\begin{proposition}
		\label{prop:bit_extraction}
		Let $N,L  \in \mathbb{N}$. There exists a function 
		\begin{equation}
		\label{eq:complexity_decoding_networks}
			F_{N,L} \in \mathcal{R} ( ( 2,1 ),2^{N+4}, 5L,  3^{N+2}) 
		\end{equation}
		such that for all $\sum_{i = 1}^{\infty} \theta_i 3^{-i} \in \mathbb{T}$ and $\bit \in \mathbb{N} \cup \{ 0 \}$,
		\begin{equation}
		\label{eq:computational_property_f_NL}
			F_{N,L}\biggl(\sum_{i = 1}^{\infty} \theta_i 3^{-i} , \bit\biggr) = \sum_{i=1}^{\min \{ NL,\,\bit \}} \theta_{i},
		\end{equation}
		where we use the convention $\sum_{i=1}^{0} \theta_{i} = 0$. In particular, for $s \in \mathbb{N}$ such that $s \leq NL $, $\theta_i \in \{ 0,1 \}$, $i = 1,\dots, s$, we have 
		 \begin{equation}
		 \label{eq:bit_extraction}
		 	F_{N,L}(T ( ( \theta_1,\dots, \theta_s ) ) , \bit) = \sum_{i = 1}^\bit \theta_i, \,\, k =0,\dots, s.
		 \end{equation}
		\begin{proof}
			See Appendix~\ref{sub:bit_extraction}.
		\end{proof}
	\end{proposition}

	It is worth emphasizing two aspects of the decoding ReLU network $F_{N,L}$ in Proposition~\ref{prop:bit_extraction}. First, the width of the network is determined solely by the parameter $N$, while its depth depends only on the parameter $L$, and both of these parameters can be chosen freely, and, in particular, independently of each other. 
  Moreover, the weight magnitude of the network is polynomial in network width and does not depend on network depth.
	This control over the weight magnitude sets our result apart from the existing literature, and is essential for the proof of Theorem~\ref{prop:approximation_lip_increasing_weights}. Specifically, the decoding networks in \cite{bartlett1998almost, yarotsky2018optimal, yarotsky2019phase, Kohler2019, Vardi2022} all impose constant network width and have weight magnitude that is exponential in network depth, those in
  \cite{bartlett2019nearly, shen2021optimal} allow for decoupled width and depth behavior, 
  but exhibit weight magnitudes that are exponential in network depth and polynomial in network width. We note that the aspect of weight magnitude behavior is not explicitly discussed in these references, but can be uncovered by scrutinizing the proofs of the corresponding results.

	The flexibility afforded by our construction can be attributed to the use of a reduced alphabet, i.e., $\{ 0,1 \}$ in a ternary expansion. While this idea is novel in the context of bit extraction, its roots can be traced back to the study of the computational power of neural networks \cite{Siegelmann1992}. Specifically, \cite{Siegelmann1992} encodes  $\{ 0,1 \}$-strings as quaternary numbers. We next illustrate the philosophy underlying this idea by way of an example.

	% , considering the Lipschitz constant of decoders under different ways of encoding.

	% \todo{on the computational power of neural network \cite{Siegelmann1992} }

	% This difference follows from the difference degree of evenness of encoded numbers with different encoding method,
	% from encoding with full alphabets as oppposed to reduced alphabets, 

	\begin{example}
	\label{example:single_digit_extraction}
	 	Let $n \in \mathbb{N}$ with $n \geq 4$.  We want to encode a $\{ 0,1\}$-string of length $n$ as a real number and extract the first element of the string from this number. For concreteness, we consider the $\{ 0,1 \}$-strings of length $n$ given by $s_1 = ( 1,0,\dots, 0 )$ and $s_2 = (0,1,\dots, 1)$, and compare the Lipschitz constants of the decoders associated with base-$2$ and base-$3$ encoding, respectively. The reason for studying the Lipschitz  constant of the decoders resides in the fact that it determines the behavior of the weight magnitude of corresponding ReLU network realizations.

		We first perform base-$2$ encoding according to
  %$\{ 0,1\}$-strings of length $n$ as binary numbers, i.e. with full alphabet.  Specifically, let the encoding procedure be given by 
  $E ( ( \theta_i )_{i = 1}^n  ) = \sum_{i = 1}^n 2^{-i} \theta_i$ with $( \theta_i )_{i = 1}^n \in \{ 0,1 \}^n$. This yields $b_1 := E ( s_1 ) = 2^{-1}$ and $b_2 := E ( s_2 ) = 2^{-1} - 2^{-n}$. Let $D$ be any decoder\footnote{There are infinitely many such decoders.} that extracts the first bit, i.e., $D ( E ( ( \theta_i )_{i = 1}^n  ) ) = \theta_1$, for $( \theta_i )_{i = 1}^n \in \{ 0,1 \}^n$. The Lipschitz constant of this decoder is at least exponential in $n$, which can be seen by evaluating
		\begin{equation}
		\label{eq:lip_lower_binary}
			\frac{\nleft| D ( E ( s_1 ) )  - D ( E ( s_2 ) )\nright| }{\nleft| E ( s_1 ) -  E ( s_2 ) \nright| } = 2^n.
		\end{equation}
%		and is hence at least exponential in $n$.

		Alternatively, we can employ base-$3$ encoding 
  %$\{ 0,1\}$-strings of length $n$, as real numbers in a ternary manner, 
  according to Definition~\ref{def:our_encoder} and take the corresponding decoder $\tilde{D}$ to be given by
  %be the decoder given by 
		\begin{equation*}
			\tilde{D}(x) = 9 \rho(x - T(( 0,2))) - 9\rho(x - T((1,0))), \quad x \in \mathbb{R}.
		\end{equation*}
		We note that $\tilde{D}(x) = 0$, for $x \in [0,T(( 0,2))]$, $\tilde{D}(x) = 1$, for $x \in [T((1,0)), \infty)$, and therefore $\tilde{D} ( T ( ( \theta_i )_{i = 1}^n  ) ) =\theta_1$, for $( \theta_i )_{i = 1}^n \in \{ 0,1 \}^n$.  The Lipschitz constant of this decoder is given by $9$, and is hence independent of $n$. 
  %In particular, the lower bound on the Lipschitz constant of $\tilde{D}$, derived in a similar manner as \eqref{eq:lip_lower_binary}, 
%		\begin{equation*}
%			\frac{\nleft| \tilde{D} ( T ( s_1 ) )  - \tilde{D} ( T ( s_2 ) )\nright| }{\nleft| T ( s_1 ) -  T ( s_2 ) \nright| } = \frac{1}{3^{-1} - \sum_{i = %2}^n 3^{-i}},
%		\end{equation*}
%		does not grow exponentially in $n$.

		% \begin{equation}
		% \label{eq:lip_lower_binary}

		% % the Lipschitz constant lower  of $\tilde{D}$ can only be lower bounded by
		
		% which is no longer exponential in $n$. 

		% In particular, we can set the decoder $\tilde{D}: \mathbb{R} \mapsto \mathbb{R} $ to be 
		% \begin{equation*}
		% 	\tilde{D}(x) = 9 \rho(x - T(( 0,2))) - 9\rho(x - T((1,0))), \quad x \in \mathbb{R},
		% \end{equation*}
		% which takes on value $0$ on $[0,T(( 0,2)))]$ and takes on value $1$ on $[T((1,0)), \infty)$.
		% The corresponding lower bound of the Lipschitz constant of the decoder $\tilde{D}$ is mitigated, according to
		% \begin{equation*}
		% 	\sup_{x \neq  y} \frac{ \nleft| \tilde{D} ( x ) - \tilde{D}(y) \nright| }{\nleft| x - y \nright| } \geq &\, \frac{\nleft| \tilde{D} ( T ( s_1 ) )  - \tilde{D} ( T ( s_2 ) )\nright| }{\nleft| T ( s_1 ) -  T ( s_2 ) \nright| } = \frac{\nleft| k - 2 \nright| }{3^{-1} - \sum_{i = 2}^{k} 3^{-i}}  = \frac{\nleft| k - 2 \nright| }{3^{-2}   + 3^{-n}},
		% \end{equation*}
		% which doesn't suffer from the exponential growth in $n$ in this case.
	\end{example} 

	Example~\ref{example:single_digit_extraction} provides a heuristic argument for reduced alphabets leading to better weight magnitude behavior of  decoding ReLU networks. The procedure described in Example~\ref{example:single_digit_extraction} extracts the first element of the string under consideration; it was used in \cite{bartlett1998almost} together with bitshift operations to extract multiple elements, one by one in multiple rounds.
  On the other hand, the bit extraction technique in \cite{bartlett2019nearly, shen2021optimal} and our Proposition~\ref{prop:bit_extraction}, while also employing a multi-round approach, deliver multiple elements of the string in each round. Specifically, in Proposition~\ref{prop:bit_extraction} $N$ elements are extracted in each of the $L$ rounds. Notably, the weight magnitude of our decoding network $F_{N,L}$ is independent of the number of rounds $L$ and depends on $N$ only. In contrast, the weight magnitudes of the decoding networks reported in the literature all depend exponentially on the number of rounds, a consequence of what was illustrated in Example~\ref{example:single_digit_extraction}. As the depth of the extracting networks is proportional to the number of rounds they are to carry out, the constructions reported previously in the literature all exhibit weight magnitude growth that is exponential in network depth.

	\subsection{Proof of Proposition~\ref{sub:preparation_to_prove_proposition_prop:approximation_lip_increasing_weights}} % (fold)
	\label{sub:proof_of_proposition_sub:preparation_to_prove_proposition_prop:approximation_lip_increasing_weights}
	
	% subsection proof_of_proposition_sub:preparation_to_prove_proposition_prop:approximation_lip_increasing_weights (end)

		The proof is constructive in the sense that, for each $g \in \lip ([0,1]) $, and given $W,L \geq D_a$, we explicitly specify a function 
		\begin{equation}
			\label{eq:form_f}
			f \in \mathcal{R} (W, L, W^K  )
		\end{equation}
		such that   
		\begin{equation}
			\label{eq:form_approximation}
			\| g - f \|_{L^\infty ( [0,1])} \leq C_a  ( W^2 L^2 \log (W)  )^{-1},
		\end{equation}
		with  absolute constants $C_a,D_a \in \mathbb{R}$ and $K \in \mathbb{N}$ to be specified later. In the following, let $m,n,\elll \in \mathbb{N}$, with $\ell \geq 2$, to be determined later, and set
		\begin{equation}
		\label{eq:choice_of_delta}
			\Delta = \frac1{10m^2\elll^2 n}.
		\end{equation}
		Throughout the proof, we will frequently consider the grid points $\{ \frac{i}{m^2\elll^2 n}: i =0,\dots, m^2\elll^2 n - 1 \}$, which we sometimes rewrite as 
		\begin{equation}
		\label{eq:rewrite_grid_points}
			\biggl\{ \frac{i}{m^2\elll^2 n}: i =0,\dots, m^2\elll^2 n - 1 \biggr\} =\biggl \{ \frac{j}{m^2\elll} + \frac{k}{m^2\elll^2 n}: ( j,k ) \in \mathcal{I} \biggr\},
		\end{equation}
		where
		\begin{equation}
			\label{eq:index_set}
			\mathcal{I} = \{ ( j,k ): j \in \{ 0,\dots, m^2 \elll - 1 \}, k \in \{ 0,\dots, n\elll -1 \}  \}.
		\end{equation}
		% we will sometime rewrite the grid points $\{ \frac{i}{m^2n^2\elll}: i =0,\dots, m^2\elll^2 n - 1 \}$, according to  
		
		For fixed $g \in \lip ( [0,1] )$, the construction of the corresponding $f$ proceeds in four steps as follows. 

		\begin{enumerate}
			\item We specify a function $f_1 \in C ( \mathbb{R})$ realized by a ReLU network and approximating $g$ ``well enough" on  the grid points $\{ \frac{i}{m^2 \elll^2 n} \}_{i=0}^{m^2 \elll^2 n-1}$.

			\item Then, based on $f_1$, we construct a function $f_2 \in C ( \mathbb{R} )$ realized by a ReLU network and approximating $g$ ``well enough" on the subdomain $\bigcup_{i = 0}^{m^2 \elll^2 n - 1} \bigl[ \frac{i}{m^2 \ell^2 n},   \frac{i + 1}{m^2 \ell^2 n} - \Delta \bigl]$.
			% such that  for $i = 0,\dots,  m^2 \elll^2 n - 1 $,
			% \begin{equation}
			%  	f_2 (x) = f_1 \Bigl( \frac{i}{m^2 \elll^2 n} \Bigr), \text{ for all } x \in \Bigl[\frac{i}{m^2 \elll^2 n}, \frac{i + 1}{m^2 \elll^2 n}- \Delta\Bigr].
			% \end{equation} 

			% The function $f_2$ will hence approximates $g$ "well enough" on the set $[0,1]\backslash ( \cup_{i = 1}^{N} (\frac{i}{N} - \Delta, \frac{i}{N}) )$, which is the whole domain $[0,1]$ excluding union of some small interval $\cup_{i = 1}^{N} (\frac{i}{N} - \Delta, \frac{i}{N})$.
			\item Starting from $f_2$, we determine a function $f \in C ( \mathbb{R})$ realized by a ReLU network and approximating $g$ ``well enough''  on the entire domain $[0,1]$.

			\item The depth and width of the ReLU network in Step 3 as well as the corresponding approximation error $\| f - g \|_{L^\infty ( [0,1] )}$ depend explicitly on $m,n,\elll$. Then, values for $m,n,\elll$ are chosen to make \eqref{eq:form_f} and \eqref{eq:form_approximation} hold for absolute constants $C_a, D_a, K$. We do announce that   $m$ will be linear in $W$, $\elll$ linear in $L$, and $n$ logarithmic in $W$. 
		\end{enumerate}

		\noindent \textbf{Step 1.} This step is summarized in the form of the following result.

		\begin{lemma}
			\label{lem:lemma_step_1}
			For $g \in \lip ( [0,1] )$, $m,n,\elll \in \mathbb{N}$, with $\elll \geq 2$, there exists a function
			\begin{equation}
				\label{eq:complexity_requirement_for_f_1}
				\begin{aligned}
				f_1 \in \mathcal{R} ( 200m+ 2^{n+5}, 37\elll, \max \{8mn, 3^{n+2}\} ),
				\end{aligned}
			\end{equation}
			such that for $i = 0,\dots, m^2\elll^2 n - 1$,
			\begin{equation}
				\label{eq:error_bound_000}
				\Bigl| g\Bigl(\frac{i}{m^2\elll^2n}\Bigr) - f_1\Bigl(\frac{i}{m^2\elll^2n}\Bigr) \Bigr| \leq \frac{1}{m^2\elll^2n}. 
			\end{equation}
		\end{lemma}

		\begin{proof}
			We use a two-stage approach to approximate $g$, starting by fitting a function to $g$ on the grid points $\{ \frac{j}{m^2\elll} \}_{j = 0}^{m^2\elll - 1}$, and then lifting this function to obtain a ReLU network that approximates $g$ on the target grid points $\{ \frac{i}{m^2 \elll^2 n} \}_{i=0}^{m^2 \elll^2 n-1}$.

			In the first stage, we identify a bounded piecewise linear function $h$ that is fit to $g$ at the grid points $\{ \frac{j}{m^2\elll} \}_{j = 0}^{m^2\elll - 1}$. This function will later be  realized by a ReLU network using Proposition~\ref{prop:piecewise_representation}. Specifically, we let $h: \mathbb{R} \mapsto \mathbb{R}$ be given by
			\begin{align*}
                    &h (x)\\
				% \label{eq:definition_h}
				    &:= \left\{
				\begin{aligned}
					&g\Bigl(\frac{j}{m^2\elll}\Bigr), \quad &x \in \Bigl[ \frac{j}{m^2\elll}, \frac{j+1}{m^2\elll} - \Delta \Bigr], j \in \Bigl\{ 0,\dots, m^2\elll -1 \Bigr\}, \\[5pt]
					&\frac{g\Bigl(\frac{j+1}{m^2\elll}\Bigr) -  g\Bigl(\frac{j}{m^2\elll}\Bigr)}{\Delta} \bigl(x -  \frac{j+1}{m^2\elll} + \Delta   \bigr)\\
					& + g\Bigl(\frac{j}{m^2\elll}\Bigr)  , & x \in \Bigl[ \frac{j+1}{m^2\elll} - \Delta, \frac{j+1}{m^2\elll} \Bigr], j \in \bigl\{ 0,\dots, m^2\elll - 2 \bigr\}, \\[5pt]
					&g\Bigl(\frac{m^2\elll - 1}{m^2 \elll}\Bigr), & x \in [1 - \Delta, \infty),\\[5pt]
					&g(0), & x \in (\infty, 0],
				\end{aligned}
				\right.
			\end{align*}
			with $\Delta = \frac{1}{10m^2 \elll^2 n}$ as specified in \eqref{eq:choice_of_delta}. We note that
			\begin{equation}
			\label{eq:approximation_property_h}
				h \biggl( \frac{j}{m^2\elll} +  \frac{k}{m^2\elll^2n}   \biggr) = g\biggl(\frac{j}{m^2\elll}\biggr), \quad \text{for }(j,k) \in \mathcal{I},
			\end{equation}
			where the index set $\mathcal{I}$ was defined in \eqref{eq:index_set}.
			% , and, in particular, $h ( \frac{j}{m^2\elll} ) = g(\frac{j}{m^2\elll})$, for $j = 0,\dots, m^2 l - 1$.

			For the second stage, let $t = g - h$ be the remainder of the approximation of $g$ by $h$. We want to approximate $t$ on the target grid points $\{ \frac{i}{m^2 \elll^2 n}: i =0,\dots, m^2 \elll^2 n - 1 \} = \{ \frac{j}{m^2\elll} +  \frac{k}{m^2\elll^2n}: ( j,k ) \in \mathcal{I} \}$. This will be done by rounding $t \bigl( \frac{j}{m^2\elll} +  \frac{k}{m^2\elll^2n} \bigr)$, for all $( j,k ) \in \mathcal{I}$, down to the nearest number in $\frac{ 1}{m^2\elll^2n} \mathbb{Z}$.  To this end, we define
			\begin{equation}
				t_j (k) = \biggl\lfloor  m^2\elll^2n \cdot t\Bigl(\frac{j}{m^2\elll} +  \frac{k}{m^2\elll^2n}  \Bigr)  \biggr\rfloor, 
			\end{equation}
			and note that 
			\begin{equation}
			\label{eq:t_i_error}
				\biggl| \frac{ t_j (k)}{m^2\elll^2n} - t\biggl(\frac{j}{m^2\elll} +  \frac{k}{m^2\elll^2n}  \biggr) \biggr|  \leq\frac1{m^2\elll^2n}.
			\end{equation}
			Formally, the rounding operation will be effected by application of a function $p: \mathbb{R} \mapsto \mathbb{R}$ satisfying
			% We will construct a function $p:\mathbb{R} \mapsto \mathbb{R}$ such that $p$ is the rounded version of the remainder $t$ on $\{ \frac{i}{m^2 \elll^2 n}: i =0,\dots, m^2 \elll^2 n - 1 \} = \{ \frac{j}{m^2\elll} +  \frac{k}{m^2\elll^2n}: ( j,k ) \in \mathcal{I} \}$, i.e.,
			\begin{equation}
				\label{eq:key_properties}
				p\Bigl(\frac{j}{m^2\elll} +  \frac{k}{m^2\elll^2n}\Bigr) =\frac{t_j (k)}{m^2\elll^2n}, \quad \text{ for } ( j,k ) \in \mathcal{I}.
			\end{equation}
			%relevant to our approximation. 
   We note that \eqref{eq:key_properties} determines the values of $p$ on $\{ \frac{j}{m^2\elll} +  \frac{k}{m^2\elll^2n}: ( j,k ) \in \mathcal{I} \}$ only, and we still have to specify $p$ on $\mathbb{R}\backslash \{ \frac{j}{m^2\elll} +  \frac{k}{m^2\elll^2n}: ( j,k ) \in \mathcal{I} \}$. This will be done such that $p$ can be realized by a ReLU network. Once $p$ has been determined, we let $f_1 = h + p$ and note that by \eqref{eq:key_properties} and \eqref{eq:t_i_error} the desired approximation error bound holds according to
			\begin{align}
				&\,\Biggl| g\Bigl(\frac{j}{m^2\elll} +  \frac{k}{m^2\elll^2n}\Bigr) - f_1\Bigl(\frac{j}{m^2\elll} +  \frac{k}{m^2\elll^2n}\Bigr) \Biggr| \label{eqline:construction_error_bound_1} \\
				&= \, \Biggl| t\Bigl(\frac{j}{m^2\elll} +  \frac{k}{m^2\elll^2n}\Bigr) - \frac{t_j (k)}{m^2\elll^2n} \Biggr|\label{eqline:construction_error_bound_11} \\
				&\leq \, \frac1{m^2\elll^2n}, \quad \text{for } ( j,k ) \in \mathcal{I}. \label{eqline:construction_error_bound_2}
			\end{align}

			We proceed with the construction of $p$, which will be based on the bit extraction technique described in Appendix~\ref{ssub:bit_extraction_realization}.  To this end, we first represent $t_j (k)$, $(j,k) \in \mathcal{I}$, in the form $\sum_{i =1}^k \theta^{+}_{j,i} - \sum_{i =1}^k \theta^{-}_{j,i}$, for some $( \theta^+_{j,i})_{i = 1}^{n\ell -1}, ( \theta^-_{j,i} )_{i = 1}^{n\ell -1} \in  \{0,1\}^{n\ell -1}$. Specifically, we note that, for $( j,k ) \in \mathcal{I}$, with $k \geq 1$,
			\begin{align}
				&\,\Bigl| t\Bigl(\frac{j}{m^2\elll} +  \frac{k}{m^2\elll^2n}  \Bigr)  - t\Bigl(\frac{j}{m^2\elll} +  \frac{k -1}{m^2\elll^2n}  \Bigr)  \Bigr| \label{eqline:tg0}\\
				&=\, \Bigl| g\Bigl(\frac{j}{m^2\elll} +  \frac{k}{m^2\elll^2n}  \Bigr) - h\Bigl(\frac{j}{m^2\elll} +  \frac{k}{m^2\elll^2n}  \Bigr) \nonumber\\
                &\quad \quad - g\Bigl(\frac{j}{m^2\elll} +  \frac{k -1}{m^2\elll^2n}  \Bigr) + h\Bigl(\frac{j}{m^2\elll} +  \frac{k-1}{m^2\elll^2n}  \Bigr)  \Bigr| \label{eqline:tg1} \\
				&=\, \Bigl| g\Bigl(\frac{j}{m^2\elll} +  \frac{k}{m^2\elll^2n}  \Bigr)  - g\Bigl(\frac{j}{m^2\elll} +  \frac{k -1}{m^2\elll^2n}  \Bigr)   \Bigr|\label{eqline:tg2}\\
				&\leq\, \frac{1}{{m^2\elll^2n}},  \label{eqline:tg3}
			\end{align}
			where  \eqref{eqline:tg2} follows from $h\Bigl(\frac{j}{m^2\elll} +  \frac{k}{m^2\elll^2n}  \Bigr) = h\Bigl(\frac{j}{m^2\elll} +  \frac{k -1}{m^2\elll^2n}  \Bigr) = g(\frac{j}{m^2\elll})$ as a consequence of \eqref{eq:approximation_property_h}, and in \eqref{eqline:tg3} we used the $1$-Lipschitz-continuity of $g$.
			% \begin{equation*}
			% 	% \label{eq:tij_difference}
			% 	\Bigl| t\Bigl(\frac{i}{m^2\elll} +  \frac{j}{m^2\elll^2n}  \Bigr)  - t\Bigl(\frac{i}{m^2\elll} +  \frac{j -1}{m^2\elll^2n}  \Bigr)  \Bigr| \leq \frac{1}{{m^2\elll^2n}}.
			% \end{equation*}
			Multiplication of \eqref{eqline:tg0}-\eqref{eqline:tg3} by  $m^2 \elll^2 n$ then yields
			\begin{equation}
				\label{eq:tij_difference}
				\biggl|  {m^2\elll^2n} \cdot t\Bigl(\frac{j}{m^2\elll} +  \frac{k}{m^2\elll^2n}  \Bigr)  -  {m^2\elll^2n} \cdot t\Bigl(\frac{j}{m^2\elll} +  \frac{k -1}{m^2\elll^2n}  \Bigr)  \biggr| \leq 1.
			\end{equation}
			As $| \lfloor x\rfloor - \lfloor y \rfloor |  \leq 1 $, for all $x,y \in \mathbb{R}$ such that $| x - y | \leq 1 $, it follows from \eqref{eq:tij_difference} that $|  \lfloor{m^2\elll^2n} \cdot t(\frac{j}{m^2\elll} +  \frac{k}{m^2\elll^2n}  )\rfloor  -  \lfloor{m^2\elll^2n} \cdot t(\frac{j}{m^2\elll} +  \frac{k -1}{m^2\elll^2n}  )\rfloor  | \leq 1$, which is $| t_j (k) - t_j (k-1) | \leq 1$. Moreover, as $t_j (k)$ and $t_j (k -1)$ are integers, we must have $t_j (k) - t_j (k-1) \in \{ -1,0,1 \}$. Next, we define, for $( j,k ) \in \mathcal{I}$ with $k \geq 1$,
			\begin{align*}
				\theta^+_{j,k} =& \max \{ t_j (k) - t_j (k-1), 0 \} \in \{ 0,1 \},\\
				\theta^-_{j,k} =& \max \{- (  t_j (k ) - t_j (k - 1) ), 0 \} \in \{ 0,1 \}.
			\end{align*}
			Then, for $(j,k) \in \mathcal{I}$, we can write 
			\begin{equation}
				\label{eq:decomposition_tij}
				t_j (k ) - t_j (k - 1) = \theta^+_{j,k} - \theta^-_{j,k},
			\end{equation}
			and get
			\begin{align}
				t_j (k) =&\, t_j(0) + \sum_{i = 1}^k (t_j (i ) - t_j (i - 1))\label{eqline:represent_t_j_0}\\
				=&\, \sum_{i = 1}^k \theta^{+}_{j,i} - \sum_{i = 1}^k  \theta^{-}_{j,i}, \label{eqline:represent_t_j_1}
			\end{align}
			where in \eqref{eqline:represent_t_j_1} we used \eqref{eq:decomposition_tij} along with $t_j(0) =  \bigl\lfloor m^2\elll^2n \cdot t(\frac{j}{m^2\elll})   \bigr\rfloor = \bigl\lfloor m^2\elll^2n \cdot ( g(\frac{j}{m^2\elll}) - h(\frac{j}{m^2\elll}) )   \bigr\rfloor = 0$  thanks to  \eqref{eq:approximation_property_h}. 

			We are now ready to detail the construction of $p$. For $j = 0,\dots, m^2\elll - 1$, we encode the $\{ 0,1 \}$-strings $(\theta^+_{j,1},\dots, \theta^+_{j,n\elll - 1})$ and $(\theta^-_{j,1}, \dots, \theta^-_{j,n\elll - 1} )$ into ternary numbers with reduced alphabet according to
			% \begin{align*}
			% 	b^+_j =& ter\, 0.\theta^+_{j,1}\dots \theta^+_{j,n\elll - 1},\\
			% 	b^-_j =& ter\, 0.\theta^-_{j,1}\dots \theta^-_{j,n\elll - 1}.
			% \end{align*}
			\begin{align*}
				b^+_j :=&\, T((\theta^+_{j,1},\dots, \theta^+_{j,n\elll - 1})),\\
				b^-_j :=&\, T((\theta^-_{j,1}, \dots, \theta^-_{j,n\elll - 1} )).
			\end{align*}
			By Proposition~\ref{prop:bit_extraction}, it follows that there exists a decoder $F_{n,\elll} \in \mathcal{R} ( ( 2,1 ),2^{n+4}, 5\elll,  3^{n+2}) $ such that, for $( j,k ) \in \mathcal{I}$, we have $F_{n,\elll} ( b^+_j, k ) = \sum_{i = 1}^k \theta^{+}_{j,i} $ and $F_{n,\elll} ( b^-_j, k ) = \sum_{i = 1}^k \theta^{-}_{j,i} $, which, combined with \eqref{eqline:represent_t_j_0}-\eqref{eqline:represent_t_j_1}, implies
			\begin{equation}
				\label{eq:represent_t_i_2}
				t_j (k) = F_{n,\elll} ( b^+_j, k ) - F_{n,\elll} ( b^-_j, k ).
			\end{equation}
			Next, let $b^+,b^-,s\in C ( \mathbb{R} )$  be defined according to
			\begin{equation*}
				\label{eq:definition_b_plus}
				b^+ (x) = \left\{
				\begin{aligned}
					&b_j^+, &x \in \Bigl[ \frac{j}{m^2\elll}, \frac{j+1}{m^2\elll} - \Delta \Bigr],\,\,\,  j \in \{ 0,\dots, m^2\elll -1 \}, \\[5pt]
					&b^+_j + \frac{b^+_{j+1} - b^+_j}{\Delta} \biggl( x - \frac{j+1}{m^2\elll} + \Delta \biggr), &x \in \Bigl[ \frac{j+1}{m^2\elll} - \Delta, \frac{j+1}{m^2\elll}\Bigr],\,\,\, j\in \{ 0,\dots, m^2\elll -2 \}, \\[5pt]
					&b^+_0, & x \in (-\infty,0], \\[5pt]
					&b^+_{m^2\elll -1}, & x \in [1 - \Delta, \infty),
				\end{aligned}
				\right.
			\end{equation*}
			\begin{equation*}
				\label{eq:definition_b_minus}
				b^- (x) = \left\{
				\begin{aligned}
					&b_j^-, &x \in \Bigl[ \frac{j}{m^2\elll}, \frac{j+1}{m^2\elll} - \Delta \Bigr],\,\,\,  j \in \{ 0,\dots, m^2\elll -1 \}, \\[5pt]
					&b_j^- + \frac{b^-_{j+1} - b^-_j}{\Delta} \biggl( x - \frac{j+1}{m^2\elll} + \Delta \biggr), &x \in \Bigl[ \frac{j+1}{m^2\elll} - \Delta, \frac{j+1}{m^2\elll}\Bigr],\,\,\,  j\in \{ 0,\dots, m^2\elll -2 \}, \\[5pt]
					&b^-_0, & x \in (-\infty,0], \\[5pt]
					&b^-_{m^2\elll -1}, & x \in [1 - \Delta, \infty),
				\end{aligned}
				\right.
			\end{equation*}
			and 
			\begin{equation*}
				\label{eq:definition_s}
				s (x) = \left\{
				\begin{aligned}
					&m^2\elll^2n \Bigl( x -   \frac{j}{m^2\elll} \Bigr), &x \in \Bigl[ \frac{j}{m^2\elll}, \frac{j+1}{m^2\elll} - \Delta \Bigr],\,\,\,  j \in \{ 0,\dots, m^2\elll -1 \}, \\[5pt]
					&n\elll\,\frac{m^2\elll \Delta - 1}{ \Delta} \Bigl( x - \frac{j+1}{m^2\elll} \Bigr)  , &x \in \Bigl[ \frac{j+1}{m^2\elll} - \Delta, \frac{j+1}{m^2\elll}\Bigr],\,\,\,  j\in \{ 0,\dots, m^2\elll -2 \}, \\[5pt]
					&0, & x \in (-\infty,0], \\[5pt]
					&n\elll - m^2\elll^2n\Delta, & x \in [1 - \Delta, \infty).
				\end{aligned}
				\right.
			\end{equation*}
			% \begin{equation}
			% 	% \label{eq:definition_s}
			% 	s (x) = \left\{
			% 	\begin{aligned}
			% 		& m^2\elll^2n ( x -   \frac{i}{m^2\elll} )  &&x \in \Bigl[ \frac{i}{m^2\elll}, \frac{i+1}{m^2\elll} - \Delta \Bigr], \\
			% 		&&& \quad i \in \{ 0,\dots, m -1 \}, \\
			% 		&0, && x = 1,\\
			% 		&\text{linearly interpolated},&&\text{otherwise},
			% 	\end{aligned}
			% 	\right.
			% \end{equation}
			These choices of $b^+$, $b^{-}$, and $s$ guarantee that, for $( j,k ) \in \mathcal{I}$,
			\begin{align}
				b^+ \Bigl( \frac{j}{m^2\elll} +  \frac{k}{m^2\elll^2n} \Bigr) =&\, b_j^+, \label{eqline:properties_b_plus} \\
				b^- \Bigl( \frac{j}{m^2\elll} +  \frac{k}{m^2\elll^2n} \Bigr) =&\,  b_j^-,\label{eqline:properties_b_minus}\\
				s\Bigl( \frac{j}{m^2\elll} +  \frac{k}{m^2\elll^2n} \Bigr) =&\, k, \label{eqline:properties_s}
				% s\Bigl( \frac{j}{m^2\elll} +  \frac{k}{m^2\elll^2n} \Bigr) =&\, m^2\elll^2n \Bigl( \frac{j}{m^2\elll} +  \frac{k}{m^2\elll^2n}  -   \frac{j}{m^2\elll} \Bigr) \label{eqline:properties_s_0} \\
				% =&\, k. \label{eqline:properties_s}
			\end{align}
			as a consequence of $\frac{j}{m^2\elll} +  \frac{k}{m^2\elll^2n}  \in\, [ \frac{j}{m^2\elll}, \frac{j+1}{m^2\elll} - \Delta ]$.
			We specify $p$ according to
			\begin{equation}
				\label{eq:definition_p}
			 	p := \frac{1}{m^2\elll^2n} ( F_{n,\elll} \circ ( b^+, s ) - F_{n,\elll} \circ ( b^-, s ) ).
			\end{equation} 
			It hence follows that $p$ has the desired property \eqref{eq:key_properties}, as, for $( j,k ) \in \mathcal{I}$,
			\begin{align}
				&\, p \Bigl( \frac{j}{m^2\elll} +  \frac{k}{m^2\elll^2n} \Bigr)\\
				&= \,\frac{1}{m^2\elll^2n} \Bigl(F_{n,\elll} \Bigl( b^+ \Bigl( \frac{j}{m^2\elll} +  \frac{k}{m^2\elll^2n} \Bigr), s\Bigl( \frac{j}{m^2\elll} +  \frac{k}{m^2\elll^2n} \Bigr) \Bigr)\nonumber \\
				&\, \quad \quad \quad \,- F_{n,\elll} \Bigl( b^- \Bigl( \frac{j}{m^2\elll} +  \frac{k}{m^2\elll^2n} \Bigr), s\Bigl( \frac{j}{m^2\elll} +  \frac{k}{m^2\elll^2n} \Bigr) \Bigr)\Bigr)\label{eqline:claim_property_p_1}\\
				&=\, \frac{1}{m^2\elll^2n} \Bigl(F_{n,\elll} ( b^+_j, k ) - F_{n,\elll} ( b^-_j, k ) \Bigr)\label{eqline:claim_property_p_2}\\
				&=\, \frac{1}{m^2\elll^2n}\, t_j (k),\label{eqline:claim_property_p_3}
			\end{align}
			where \eqref{eqline:claim_property_p_2} follows from \eqref{eqline:properties_b_plus}, \eqref{eqline:properties_b_minus}, and \eqref{eqline:properties_s}, and \eqref{eqline:claim_property_p_3} is by \eqref{eq:represent_t_i_2}.
			% Application of $???$ then 
			% \begin{align*}
			% 	&\left| g(x) - h(x) - \frac{F_{n,\elll} \left( b^+ \left( x \right), s(x) \right) - F_{n,\elll} \left(b^- \left( x \right) , s(x) \right) }{m^2\elll^2 n} \right|  \\
			% 	\leq &\frac{1}{m^2\elll^2 n}
			% \end{align*}
			% for x in 
			% \begin{align*}
			% 	&\left\{ \frac{i}{m^2\elll} +  \frac{j}{m^2\elll^2n}: i \in \left\{ 0,\dots, m -1 \right\}, j \in \left\{ 0,\dots, n\elll-1 \right\} \right\}\\
			% 	&\cup \left\{ 1 \right\}, 
			% \end{align*}which is $\left\{ \frac{i}{N}: i =0,\dots, N \right\}$ . 

			% We have constructed $f_1 = h + p = h + \frac{1}{m^2\elll^2n} \cdot ( F_{n,\elll} \circ ( b^+, s ) - F_{n,\elll} \circ ( b^-, s ) )$ that satisfies the desired approximation property \eqref{eq:error_bound_000}. 

			It remains to show that $f_1 = h + p = h + \frac{1}{m^2\elll^2n} \cdot ( F_{n,\elll} \circ ( b^+, s ) - F_{n,\elll} \circ ( b^-, s ) ) \in \mathcal{R} ( 200m+ 2^{n+5}, 37\elll, \max \{8mn, 3^{n+2}\} )$. To this end, we first consider ReLU network realizations of $h, F_{n,\elll},\allowbreak b^+, b^-,$ and $s$, and then put them together according to Lemma~\ref{lem:algebra_on_ReLU_networks}. We start by noting that $h, b^+,b^-,$ and $s$ are all bounded piecewise linear functions with  breakpoints 
			%\begin{equation}
			%\label{eqline:def_X_1}
			%	X_1 = \biggl(0, \frac{1}{m^2\elll} - \Delta , \frac{1}{m^2\elll}, \dots, 1 - \Delta \biggl),
			%\end{equation}
			%formally, 
            $X_1 = ( x_{i} )_{i = 0}^{2 m^2 \elll  -1 }$, where $x_{2k} = \frac{k}{m^2\elll}$ and $x_{2k + 1} = \frac{k + 1}{m^2\elll} - \Delta$, for $k = 0,\dots, m^2 \elll -1$. In addition, the $L^\infty ( \mathbb{R} )$-norm of $h, b^+,b^-,$ and $s$ is upper-bounded by $n\elll$, which can be verified by checking the values of these functions at their breakpoints and noting that bounded piecewise linear functions take on their maximum absolute values at breakpoints\footnote{To be more precise, the $L^\infty ( \mathbb{R} )$-norms of $h$, $b^+$, and $b^-$ are upper-bounded by $1$ and that of $s$ by $n\elll - m^2 \elll^2 n \Delta$. }.  We therefore have
			\begin{equation}
			\label{eq:embedding_hbbs}
				h, b^+,b^-,s \in \Sigma ( X_1, n\elll ).
			\end{equation}
			Upon noting that $| X_1 | = 2 m^2\elll $ and $R_m (X_1) = \frac{1}{\Delta} = 10 m^2 \elll^2 n$, application of Proposition~\ref{prop:piecewise_representation} to $\Sigma ( X_1, n\elll )$ with $M = \nleft| X_1 \nright| = 2 m^2\elll$, $E = n \elll$, $u = 2m$, $v = \elll$, and $w = 4mn$ so that $u^2 v = 4m^2 \elll \geq | X_1 |$ and $w^{30v} = (4mn)^{30\elll} \geq 2^{30\elll} 2^{30} ( mn )^{30} \geq \elll^{30} 2^{30} ( mn )^{30} \geq   (2 m^2\elll  )^6  ( 10 m^2\elll^2 n )^4 n\elll = M^6 (R_m ( X_1) )^4 E $, yields 
			\begin{equation}
			\label{eq:embedding_sigma_hbbs}
				\Sigma ( X_1, n\elll ) \subseteq \mathcal{R} ( 40m, 30\elll, 8mn ).
			\end{equation}
			Putting \eqref{eq:embedding_hbbs} and \eqref{eq:embedding_sigma_hbbs} together shows that $h, b^+,b^-,$ and $s$ can be realized by ReLU networks such that
			\begin{equation}
			\label{eq:embedding_hbbs_2}
				h, b^+,b^-,s \in \mathcal{R} ( 40m, 30\elll, 8mn ).
			\end{equation}
			% We note that, according to the definition of $h$, $h$ is a continuous piecewise-linear functions with supremum norm less than $1$ and breakpoints in $X = ( \frac{0}{m^2\elll}, \frac{1}{m^2\elll} - \Delta , \frac{1}{m^2\elll}, \dots, 1)$, formally,
			% \begin{equation*}
			% 	h \in \Sigma ( X, 1 ).
			% \end{equation*}
			% Application of Proposition~\ref{prop:piecewise_representation} to $h$ with $\text{card} ( X ) = 2 m^2\elll + 1 \leq (2m)^2 \times l$ and $R_m ( X_h ) = \Delta^{-1} = 10 m^2 l^2 n $, we have \todo{replace $m$ and $l$}
			% \begin{equation}
			% \label{eq:expression_h}
			% \begin{aligned}
			% 	h \in \mathcal{R} \bigl(1, 32m + 4, 24l + 5, \max \{ 1, C_k  (2 m^2\elll + 1)^7 (  10 m^2 l^2 n  )^4\}  \bigr).
			% \end{aligned} 
			% \end{equation}
			% It follows from the definition of $b^+,b^-,s$ that
			% \begin{equation*}
			% 	b^+, b^-, s \in  \Sigma (X, m^2 l^2 n  ),
			% \end{equation*}
			% with $X = ( \frac{0}{m^2\elll}, \frac{1}{m^2\elll} - \Delta , \frac{1}{m^2\elll}, \dots, 1 - \Delta)$. Noting that $| X |  = 2 m^2\elll \leq (2m)^2 \times l$ and $R_m ( X ) = \Delta^{-1} = 10 m^2 l^2 n $.  Application of  Proposition~\ref{prop:piecewise_representation} to $b^+,b^-,s$ with $M = 2 m^2\elll - 1 \leq 2 m^2\elll$, $u = 2m$ and $v = l$ yields 
			% \begin{equation}
			% \label{eq:complexity_bbs}
			% \begin{aligned}
			% 	&b^+, b^-, s \\
			% 	\in &\mathcal{R} \bigl(1, 32m + 4, 24l + 5, \max \{ 1, C_k  (2 m^2\elll)^7 (  10 m^2 l^2 n  )^4 m^2 l^2 n \}  \bigr).
			% \end{aligned} 
			% \end{equation}
			In addition, we recall that, according to Proposition~\ref{prop:bit_extraction}, the decoder $F_{n,\elll}$ can be realized by a ReLU network so that
			\begin{equation}
				\label{eq:complexity_of_f_nl}
				F_{n,\elll} \in \mathcal{R} \left( ( 2,1 ),2^{n+4},5\elll,3^{n+2} \right). 
			\end{equation}
			Application of Lemma~\ref{lem:algebra_on_ReLU_networks}
			 % together with the realizability of $h, b^+,b^-,s$ and $F_{n,\elll}$ given in \eqref{eq:embedding_hbbs_2} and \eqref{eq:complexity_of_f_nl}, respectively, yields the following sequence of arguments
			now yields
			\begin{align*}
				( b^+, s ),( b^-, s ) \in&\, \mathcal{R} ( ( 1,2 ), 80m, 30\elll, 8mn ), \\
				F_{n,\elll} \circ ( b^+, s ), F_{n,\elll} \circ ( b^-, s ) \in&\, \mathcal{R} ( \max \{80m,  2^{n+4}\}, 35\elll, \max \{8mn, 3^{n+2}\} ),\\
				F_{n,\elll} \circ ( b^+, s ) + (-1)\cdot F_{n,\elll} \circ ( b^-, s )  \in&\, \mathcal{R} ( 2 \max \{80m,  2^{n+4}\}, 35\elll  +1, \max \{8mn, 3^{n+2}\} )\\
				\subseteq& \, \mathcal{R} ( 160m + 2^{n+5}, 36\elll, \max \{8mn, 3^{n+2}\} ),\\
				f_1 = & \, h + \frac{1}{m^2\elll^2n} \cdot ( F_{n,\elll} \circ ( b^+, s ) - F_{n,\elll} \circ ( b^-, s ) ) \\
				\in&\, \mathcal{R} ( 200m+ 2^{n+5}, \max \{30 \elll, 36 \elll\} +1 , \max \{8mn, 3^{n+2}\} )\\
				\subseteq&\, \mathcal{R} ( 200m+ 2^{n+5}, 37\elll , \max \{8mn, 3^{n+2}\} ). \qedhere
			\end{align*}

		\end{proof}
		% \begin{figure}[H]
		% 	\centering
		% 	\includegraphics[width=0.5\textwidth]{images/compute_f_1.jpg}
		% 	\caption{Computing $f_1$}
		% 	\label{fig:compute_f_1}
		% \end{figure}

		\noindent \textbf{Step 2.} We summarize this step in the following result. 

		\begin{lemma}
			\label{lem:sec_B_2}
			For $m,n,\elll \in \mathbb{N}$, with $\elll \geq 2$, there exists a function 
			\begin{equation}
				\label{eq:property_u_1}
				u \in \mathcal{R} ( \max \{ 40m, 40n\}, 61 \elll, 8mn ),
			\end{equation}
			% \begin{equation}
			% \label{eq:complexity_u}
			% 	u \in \mathcal{R} ( \max \{ 40m, 40n\}, 61 \elll, 8mn ).
			% \end{equation}
			such that 
			\begin{equation}
				\label{eq:property_u_2}
				u (x) =  \frac{i}{m^2\elll^2 n}, \quad   x \in \biggl[\frac{i}{m^2 \elll^2 n}, \frac{i + 1}{m^2 \elll^2 n}- \Delta\biggr], \text{ for }i = 0,\dots,  m^2\elll^2n - 1.
			\end{equation}
			% \begin{equation}
			% 	\label{eq:definition_u}
			% 	u (x) = \Bigl( \frac{i}{m^2n^2\elll}\Bigr), \text{ for }i = 0,\dots,  m^2\elll^2n - 1,  x \in \Bigl[\frac{i}{m^2 \elll^2 n}, \frac{i + 1}{m^2 \elll^2 n}- \Delta\Bigr].
			% \end{equation}
			For $g \in \lip ( [0,1] )$, let $f_1$ be the function given by Lemma~\ref{lem:lemma_step_1}, and let $f_2 = f_1 \circ u$.
			% \begin{equation}
			% \label{eqline:def_f_2}
			% \end{equation}
			We have
			\begin{equation}
			\label{eq:complexity_f_2}
				f_2 \in \mathcal{R} ( 200m + 2^{n+5}, 98\elll, \max \{8mn, 3^{n+2}\} )
			\end{equation}
			with
			\begin{equation}
				\label{eqline:property_f_2}
			 	\nleft| f_2 (x)  - g(x) \nright| \leq \frac{2}{m^2\elll^2n}, \quad\text{for } x \in \bigcup_{i = 0}^{m^2 \elll^2 n - 1} \biggl[ \frac{i}{m^2 \ell^2 n},   \frac{i + 1}{m^2 \ell^2 n} - \Delta \biggr].
			\end{equation} 
		\end{lemma}

		\begin{proof}
		% [Proof of Lemma~\ref{lem:sec_B_2}]
			We start by constructing $u$. Let $u_1,u_2 \in C ( \mathbb{R} )$ be given by
			\begin{equation}
				\label{eq:definition_u_1}
				u_1 (x) = \left\{
				\begin{aligned}
					&  x -   \frac{j}{m^2\elll},  &x \in \Bigl[ \frac{j}{m^2\elll}, \frac{j+1}{m^2\elll} - \Delta \Bigr],  \, j \in \{ 0,\dots, m^2\elll -1 \}, \\[5pt]
					& \frac{m^2\elll \Delta - 1}{m^2\elll \Delta} \Bigl( x - \frac{j+1}{m^2\elll} \Bigr),     &x \in \Bigl[ \frac{j+1}{m^2\elll} - \Delta, \frac{j+1}{m^2\elll} \Bigr],  \, k \in \{ 0,\dots, m^2\elll -2 \}, \\[5pt]
					&0, & x \in (-\infty, 0], \\[5pt]
					& \frac{1}{m^2 \elll } - \Delta, & x \in  [1 - \Delta, \infty),
				\end{aligned}
				\right.
			\end{equation}
			and 
			\begin{equation}
				\label{eq:definition_u_2}
				u_2 (x) = \left\{
				\begin{aligned}
					&x -   \frac{k}{m^2\elll^2 n}, &x \in \Bigl[ \frac{k}{m^2\elll^2 n}, \frac{k+1}{m^2\elll^2 n} - \Delta \Bigr],\, k \in \{ 0,\dots, n\elll -1 \}, \\[5pt]
					&\frac{m^2 \elll^2 n\Delta -1 }{m^2 \elll^2 n\Delta} \Bigl( x - \frac{k+1}{m^2 \elll^2 n} \Bigr),     &x \in \Bigl[ \frac{k+1}{m^2 \elll^2 n} - \Delta, \frac{k+1}{m^2 \elll^2 n} \Bigr],\, j \in \{ 0,\dots, n\elll -1 \}, \\[5pt]
					&0, & x \in (-\infty, 0], \\[5pt]
					&\frac{1}{m^2\elll^2n} - \Delta, & x \in \Bigl[\frac{1}{m^2\ell} - \Delta, \infty),
				\end{aligned}
				\right.
			\end{equation}
			and let 
			\begin{equation}
				u= \text{Id} \,- \, u_2\circ u_1,
			\end{equation}
			where %$\text{Id}: \mathbb{R} \mapsto \mathbb{R}$ is 
   $\text{Id} ( x ) = x$, $x \in \mathbb{R}$.

			We first verify \eqref{eq:property_u_2}. To this end, we start by noting that, for every $i \in \{ 0,\dots, m^2 \elll^2 n-1 \} $, thanks to \eqref{eqline:vc_contradiction_1}, there exists $( j(i),k(i) ) \in \mathcal{I}$ with $\mathcal{I}$ as defined in \eqref{eq:index_set}, such that $\frac{i}{m^2 \elll^2 n} = \frac{j(i)}{m^2\elll} + \frac{k(i)}{m^2 \elll^2 n} $. Then, for $i \in \{ 0,\dots, m^2 \elll^2 n-1 \} $ and $x \in \Bigl[\frac{i}{m^2 \elll^2 n}, \frac{i+1}{m^2 \elll^2 n} - \Delta\Bigr]$, we have
			\begin{equation*}
				x \in \Bigl[\frac{j(i)}{m^2\elll} + \frac{k(i)}{m^2 \elll^2 n}, \frac{j(i)}{m^2\elll} + \frac{k(i)+1}{m^2 \elll^2 n} - \Delta\Bigr],
			\end{equation*}
			and \eqref{eq:property_u_2}, follows upon noting that
			\begin{align}
				u(x) =&\, x - u_2 ( u_1 ( x ) ) \label{eq:property_of_u_0} \\
				=&\, x - u_2 \Bigl( x -  \frac{j(i)}{m^2\elll}\Bigr) \label{eqline:apply_u1_definition}\\
				=&\, x - \Bigl( x -  \frac{j(i)}{m^2\elll} - \frac{k(i)}{m^2 \elll^2 n} \Bigr) \label{eqline:apply_u2_definition} \\
				=&\, \frac{j(i)}{m^2\elll} + \frac{k(i)}{m^2 \elll^2 n}\\
				= &\, \frac{i}{m^2 \elll^2 n}.\label{eq:property_of_u_1}
			\end{align}
			% where in \eqref{eqline:apply_u1_definition} we used the definition of $u_1$ and $x \in \Bigl[\frac{j}{m^2\elll}, \frac{j+1}{m^2\elll}  - \Delta\Bigr]$, and \eqref{eqline:apply_u2_definition} follows from the definition of $u_2$ and $x -  \frac{j}{m^2\elll} \in \Bigl[ \frac{k}{m^2\elll^2 n}, \frac{k+1}{m^2\elll^2 n} - \Delta \Bigr]$. We establish the claimed property of $u$, we proceed to showing that $u$ can be realized by a ReLU network. 

			We proceed to realize $u$ by a ReLU network with the goal of establishing \eqref{eq:property_u_1}. This will be accomplished by realizing the constituents $u_1,u_2,\text{Id}$ of $u$ by suitable ReLU networks and combining them using Lemma~\ref{lem:algebra_on_ReLU_networks}. 
			% According to the definition of $u_1$, the function $u_1$ is a bounded piecewise linear function with $L^\infty$ norm less than $1$ and breakpoints in $X_1  = \bigl\{ \frac{0}{m^2\elll}, \frac{1}{m^2\elll} - \Delta , \frac{1}{m^2\elll}, \dots, 1 - \Delta \bigr\} $ as given in \eqref{eqline:def_X_1}, formally, 
			It follows by inspection that $u_1 \in \Sigma ( X_1, 1 ) $, with $X_1$ as defined in the paragraph after \eqref{eqline:claim_property_p_3}. Together with $\Sigma ( X_1, 1 ) \subseteq \Sigma ( X_1, n\elll ) $ and $ \Sigma ( X_1, n\elll ) \subseteq \mathcal{R} ( 40m, 30\elll, 8mn )$,  thanks to \eqref{eq:embedding_sigma_hbbs}, this then implies 
			\begin{equation}
			\label{eq:complexity_u_1}
				u_1 \in \mathcal{R} ( 40m, 30\elll, 8mn ).
			\end{equation}
			Again, by inspection, 
			\begin{equation}
			\label{eq:eq:embedding_sigma_u2_0}
				u_2 \in \Sigma ( X_2, 1 ),
			\end{equation}
			with 
			%\begin{equation*}
			%	X_{2} = \left( 0,\frac{1}{m^2\elll^2n}- \Delta , \frac{1}{m^2\elll^2n},  \dots, \frac{1}{m^2\elll} - \Delta \right),
			%\end{equation*}
			%or, formally, 
            $X_2 = ( x_i )_{i = 0}^{2n \elll -1 }$, where $x_{2k} = \frac{k}{m^2\elll^2n},\, x_{2k+1} = \frac{k+1}{m^2\elll^2n} - \Delta$, for $k = 0,\dots, n\elll -1$.
			% According to the definition $u_2$ given by \eqref{eq:definition_u_2}, $u_2$ is a bounded continuous piecewise linear function with breakpoints 
			% and $L^\infty ( \mathbb{R} ) $-norm upper-bounded by $1$, formally,
			Upon noting that $| X_2 | = 2 n\elll $ and $R_m (X_2) = \frac{1}{\Delta} = 10 m^2 \elll^2 n$, application of Proposition~\ref{prop:piecewise_representation} to $\Sigma ( X_2, n\elll )$ with $M = \nleft| X_2 \nright| =  2n \elll$, $E = 1$, $u = 2n$, $v = \elll$, and $w = 4mn$, ensuring that $u^2 v = 4n^2\elll \geq M$ and $w^{30v} = (4mn)^{30\elll} \geq 2^{30\elll} 2^{30} ( mn )^{30} \geq \elll^{30} 2^{30} ( mn )^{30} \geq   (2 n\elll  )^6  ( 10 m^2 \elll^2 n )^4 = M^6 (R_m ( X_2))^4 E $, yields
			\begin{equation}
			\label{eq:embedding_sigma_u2}
				\Sigma ( X_2, 1 ) \subseteq \mathcal{R} ( 40n, 30\elll, 8mn ).
			\end{equation}
			With \eqref{eq:eq:embedding_sigma_u2_0} this then implies
			\begin{equation}
			\label{eq:complexity_u_2}
			 	u_2 \in \mathcal{R} ( 40n, 30\elll, 8mn ).
			\end{equation} 
			Next, trivially,
			\begin{equation}
			\label{eq:complexity_Id}
				\text{Id} \in \mathcal{R} ( 1,1,1 ).
			\end{equation}
			Application of Lemma~\ref{lem:algebra_on_ReLU_networks} together with \eqref{eq:complexity_u_1}, \eqref{eq:complexity_u_2}, and \eqref{eq:complexity_Id} leads to %the following sequence of arguments
			\begin{align}
				u_2\circ u_1 \in&\, \mathcal{R} ( \max \{ 40m, 40 n \}, 60\elll, 8mn ),\\
				u =\text{Id} \, +  ( -1 ) \cdot \, u_2\circ u_1 \in&\, \mathcal{R} ( \max \{ 40m, 40n\} + 2 , 61 \elll, 8mn ). \label{eqline:complexity_u}
			\end{align}

			Regarding $f_2$, 
			% $u \in \mathcal{R} ( \max \{ 40m, 40n\}, 61 l, 8mn )$
			% $f_1 \in \mathcal{R} ( 200m+ 2^{n+5}, 37l, \max \{8mn, 3^{n+2}\} ).$
			application of Lemma~\ref{lem:algebra_on_ReLU_networks} together with \eqref{eq:complexity_requirement_for_f_1}  and \eqref{eqline:complexity_u}, yields
			\begin{align}
			 	f_2 =&\, f_1 \circ u \label{eqline:complexity_f2_0}\\
			 	\in&\, \mathcal{R} ( \max \{ 200m+ 2^{n+5}, \max \{ 40m, 40n\} + 2\}  , 61 \elll + 37\elll,  \max \{8mn, 3^{n+2}\}) \label{eqline:complexity_f2_1}\\
			 	\subseteq&\, \mathcal{R} ( 200m + 2^{n+5}, 98\elll, \max \{8mn, 3^{n+2}\} ),\label{eqline:complexity_f2_2} 
			\end{align}
			which establishes \eqref{eq:complexity_f_2}. Moreover, for $i \in \{ 0,\dots, m^2 \elll^2 n -1 \}$ and $x \in  \bigl[ \frac{i}{m^2 \ell^2 n},   \frac{i + 1}{m^2 \ell^2 n} - \Delta \bigr]$, we have 
			\begin{align}
				\nleft| f_2 (x) - g(x) \nright| \leq&\, \biggl| f_1 (u(x)) - g\biggl(\frac{i}{m^2 \ell^2 n}\biggr) \biggr|  + \biggl| g\biggl(\frac{i}{m^2 \ell^2 n}\biggr) - g(x) \biggr| \label{eq:asdfew_1}\\
				\leq&\, \biggl| f_1 \biggl(\frac{i}{m^2 \ell^2 n}\biggr) - g\biggl(\frac{i}{m^2 \ell^2 n}\biggr) \biggr|  + \biggl| \frac{i}{m^2 \ell^2 n} - x \biggr| \label{eq:asdfew_2}\\
				\leq&\, \frac{2}{m^2 \elll^2 n},\label{eq:asdfew_3}
			\end{align}
			where \eqref{eq:asdfew_2} follows from \eqref{eq:property_of_u_0}-\eqref{eq:property_of_u_1} and the $1$-Lipschitz continuity of $g$, and in \eqref{eq:asdfew_3} we used \eqref{eq:error_bound_000}. This concludes the proof. \qedhere
		\end{proof}
		
		\noindent \textbf{Step 3.} In this step, we construct $f: \mathbb{R} \mapsto \mathbb{R}$ such that $\| f - g \|_{L^\infty ( [0,1] )} \leq \frac{3}{m^2 \elll^2 n}$, which will be effected by application of the median kernel smoothing technique, introduced in \cite{lu2020deep}, to the function $f_2$ built in Step 2. The construction is formalized as follows.

		\begin{lemma}
			\label{lem:sec_B_3}
			For $g \in \lip ( [0,1] )$, $m,n,\elll \in \mathbb{N}$, with $\elll \geq 2$, define  $f: \mathbb{R} \mapsto \mathbb{R} $ according to
			\begin{equation}
			\label{eq:constituting_f}
				f(x) := \text{median} ( f_2(\rho(x - 2 \Delta)),f_2 (\rho(x - 4 \Delta)) ,f_2 (\rho(x - 6 \Delta))), \quad  x \in \mathbb{R},
			\end{equation}
			with $\Delta = \frac{1}{10m^2 \elll^2 n}$, and $f_2: \mathbb{R} \mapsto \mathbb{R}$ as in Lemma~\ref{lem:sec_B_2}. Here, for $ x_1,x_2,x_3  \in \mathbb{R}$ with  reordering from smallest to largest denoted by $ x_{(1)},x_{(2)},x_{(3)}$, 
			\begin{equation*}
				\text{median} ( x_1, x_2, x_3 ) := x_{(2)}.
			\end{equation*}
			We have 
			\begin{equation}
			\label{eq:claimed_property_f_g}
				\| f - g \|_{L^\infty ( [0,1] )} \leq \frac{3}{m^2 \elll^2 n}
			\end{equation}
			and
			\begin{equation*}
			\label{eq:complexity_f}
				f \in \mathcal{R} ( 600m + 2^{n+7} , 101\elll, \max \{ 8mn, 3^{n+2} \} ).
			\end{equation*}

			\begin{proof}
				% \todo{formalize the function by ReLU networks }
				% Using the following computational flowchart given in Fig~\ref{fig:computing_f_2},
				% \begin{figure}[H]
				% 	\centering
				% 	\includegraphics[width=0.4\textwidth]{images/f_2.jpg}
				% 	\caption{Computing $f_2$}
				% 	\label{fig:computing_f_2}
				% \end{figure}
			We start by upper-bounding $\| f - g \|_{L^\infty ( [0,1] )}$. Fix $x \in [0,1]$ and note that at least two elements\footnote{It is possible that the set $\{\rho(x - 2\Delta), \rho(x - 4\Delta), \rho(x - 6\Delta)\}$ contains duplicates, e.g., when $x = 0$, we have $\rho(x - 2\Delta) = \rho(x - 4\Delta) = \rho(x - 6\Delta)= 0$. We shall not account for such cases explicitly, but simply note that our exposition incorporates them.} of $\{\rho ( x - 2\Delta ), \rho ( x - 4\Delta ), \rho ( x - 6\Delta )\}$ are contained in the set $ \bigcup_{i = 0}^{m^2 \elll^2 n - 1} [ \frac{i}{m^2 \ell^2 n},   \frac{i + 1}{m^2 \ell^2 n} - \Delta ]$. Specifically, with $\frac{1}{m^2\elll^2 n} = 10 \Delta$, we have $ \rho ( x - 4\Delta ), \rho ( x - 6\Delta ) \in \bigcup_{i = 0}^{m^2 \elll^2 n - 1} [ \frac{i}{m^2 \ell^2 n},   \frac{i + 1}{m^2 \ell^2 n} - \Delta ]$ if  $x \in \bigcup_{i = 0}^{m^2 \elll^2 n - 1} \bigl[ \frac{i}{m^2 \ell^2 n},   \frac{i}{m^2 \ell^2 n} + 3 \Delta \bigr]$, $ \rho ( x - 2\Delta ), \rho ( x - 6\Delta ) \in \bigcup_{i = 0}^{m^2 \elll^2 n - 1} [ \frac{i}{m^2 \ell^2 n},   \frac{i + 1}{m^2 \ell^2 n} - \Delta ]$ if $x \in \bigcup_{i = 0}^{m^2 \elll^2 n - 1} \bigl[ \frac{i}{m^2 \ell^2 n} + 3 \Delta,   \frac{i}{m^2 \ell^2 n} + 5 \Delta \bigr]$, and $ \rho ( x - 2\Delta ), \rho ( x - 4\Delta ) \in \bigcup_{i = 0}^{m^2 \elll^2 n - 1} [ \frac{i}{m^2 \ell^2 n},   \frac{i + 1}{m^2 \ell^2 n} - \Delta ]$ if  $x \in \bigcup_{i = 0}^{m^2 \elll^2 n - 1} \bigl[ \frac{i}{m^2 \ell^2 n}  +5 \Delta,   \frac{i}{m^2 \ell^2 n} + 10 \Delta \bigr]$. 
			Therefore, there exist distinct numbers $a_1 ( x ), a_2(x) \in \{ 2,4,6 \}$, depending on $x$, such that  
			\begin{equation*}
				\rho(x - a_1 ( x ) \Delta), \rho(x - a_2 ( x ) \Delta) \in \bigcup_{i = 0}^{m^2 \elll^2 n - 1} \biggl[ \frac{i}{m^2 \ell^2 n},   \frac{i + 1}{m^2 \ell^2 n} - \Delta \biggr],
			\end{equation*}
			and we define $a_3 ( x )$ to be the unique element given by $\{ 2,4,6 \}\backslash \{ a_1 ( x ), a_2 ( x ) \}$. If all elements of $\{\rho(x - 2\Delta), \rho(x - 4\Delta), \rho(x - 6\Delta)\}$ are contained in $\bigcup_{i = 0}^{m^2 \elll^2 n - 1} [ \frac{i}{m^2 \ell^2 n},   \frac{i + 1}{m^2 \ell^2 n} - \Delta ]$, we take $a_1(x) = 2$, $a_2(x) = 4$, and $a_3(x) = 6$.  For $i = 1,2$, we have 
			\begin{align}
				&\,\nleft| f_2 ( \rho( x - a_i ( x ) \Delta) ) - g(x) \nright|\label{eq:iuwqf_00} \\
				\leq&\, \nleft| f_2 ( \rho( x - a_i ( x ) \Delta) ) - g( \rho( x - a_i ( x ) \Delta) ) \nright| + \nleft| g( \rho( x - a_i ( x ) \Delta) ) - g(x) \nright| \label{eq:iuwqf_0} \\
				\leq&\, \frac{2}{m^2 \elll^2 n} + \nleft| g( \rho( x - a_i ( x ) \Delta) ) - g(x) \nright| \label{eq:iuwqf_1} \\
				\leq &\, \frac{3}{m^2 \elll^2 n}, \label{eq:iuwqf_4}
			\end{align}
			where in \eqref{eq:iuwqf_1} we used \eqref{eqline:property_f_2}, and \eqref{eq:iuwqf_4} follows from the $1$-Lipschitz continuity of $g$, combined with $6\Delta \leq \frac{1}{m^2 \ell^2 n}$. To simplify notation, we set $y_i ( x ) =  f_2(\rho(x - a_i(x) \Delta))$, for $i =1,2,3$. Then, we have 
			\begin{equation}
			\label{eq:asdfdsafe_0}
				f(x) = \, \text{median} (y_1 ( x ), y_2 ( x ), y_3(x))  \in [\min ( y_1 ( x ), y_2 ( x )  ), \max ( y_1 ( x ), y_2 ( x ) ) ],
			\end{equation}
			and
			\begin{align}
				&\,\nleft| f (x) - g(x) \nright| \label{eq:asdfdsafe_1}\\
				&\leq \, \max ( \nleft|\min ( y_1 ( x ), y_2 ( x )  ) - g(x) \nright|, \nleft|\max ( y_1 ( x ), y_2 ( x )  ) - g(x) \nright|  ) \label{eq:asdfdsafe_2}\\
				&= \, \max ( \nleft| y_1 ( x ) - g(x) \nright|,  \nleft| y_2 ( x ) - g(x) \nright|  ) \label{eq:asdfdsafe_3}\\
				&=\, \max ( \nleft| f_2 ( \rho( x - a_1 ( x ) \Delta) ) - g(x)\nright|,  \nleft|f_2 ( \rho( x - a_2 ( x ) \Delta) ) - g(x)  \nright|  )\label{eq:asdfdsafe_4} \\
				&\leq \, \frac{3}{m^2 \elll^2 n}, \label{eq:asdfdsafe_5}
			\end{align}
			where in \eqref{eq:asdfdsafe_2} we used \eqref{eq:asdfdsafe_0},  and \eqref{eq:asdfdsafe_5} follows from \eqref{eq:iuwqf_00}-\eqref{eq:iuwqf_4}. As the choice of $x \in [0,1]$ was arbitrary, we have established that $\| f - g \|_{L^\infty ([0,1] )} \leq \frac{3}{m^2 \elll^2 n}$.

			It remains to show that $f$ can be realized by a ReLU network such that $f \in \mathcal{R} ( 600m + 2^{n+7} , 101\elll, \max \{ 8mn, 3^{n+2} \} )$. This will be accomplished by realizing the individual components of $f$ by suitable ReLU networks and then combining them according to Lemma~\ref{lem:algebra_on_ReLU_networks}. For $z \in \mathbb{R}$, define $r_{z}:\mathbb{R} \mapsto \mathbb{R}$, $r_z ( x ) = \rho ( x - z )$ and note that
   			% \footnote{Recall that for  $f_1:\mathbb{R}^d \mapsto \mathbb{R}^{d'}$ and $f_2: \mathbb{R}^d \mapsto \mathbb{R}^{d''}$, $d,d',d'' \in \mathbb{N}$, $( f_1, f_2 ): \mathbb{R}^{d} \mapsto \mathbb{R}^{d' + d''}$ is defined according to $( f_1, f_2 ) ( x ) = ( f_1 ( x ), f_2 ( x ) )$, for $x \in \mathbb{R}^d$. Here, we write $( ( f_2\circ  r_{2\Delta}, f_2 \circ  r_{4\Delta}  ), f_2  \circ  r_{6\Delta} )$, instead of $ (  f_2\circ  r_{2\Delta}, f_2 \circ  r_{4\Delta}  , f_2  \circ  r_{6\Delta} )$, in order to apply Lemma~\ref{lem:algebra_on_ReLU_networks}, which considers binary operations over two network realization. In addition $( ( f_2\circ  r_{2\Delta}, f_2 \circ  r_{4\Delta}  ), f_2  \circ  r_{6\Delta} )$ also indicates the application order of the operations. }
			\begin{align*}
				f =&\, \text{median} \circ (  f_2\circ  r_{2\Delta}, f_2 \circ  r_{4\Delta}  , f_2  \circ  r_{6\Delta} ).
			\end{align*}
			As $2\Delta, 4\Delta, 6\Delta \leq 1$, we get
			\begin{equation}
			\label{eq:complexity_r}
				r_{2\Delta},r_{4\Delta},r_{6\Delta} \in \mathcal{R} ( 1, 2, 1).
			\end{equation} 
			Moreover, thanks to Lemma~\ref{lem:complexity_median}, we have 
			\begin{equation}
			\label{eq:complexity_median}
			 	\text{median} \in \mathcal{R} ( ( 3,1 ), 16, 3, 1 ).
			\end{equation}
			Application of Lemma~\ref{lem:algebra_on_ReLU_networks} together with \eqref{eq:complexity_r},\eqref{eq:complexity_median}, and \eqref{eq:complexity_f_2} then yields %the following sequence of arguments
			% of $u_1,u_2,\text{Id}$ given by \eqref{eq:complexity_of_f_n\elll} given by \eqref{eq:complexity_u_1}, \eqref{eq:complexity_u_2}, and \eqref{eq:complexity_Id}, respectively, yields the following sequence of arguments
			% \begin{align*}
			% 	u_2\circ u_1 \in& \mathcal{R} ( \max \{ 40m, 40 n \}, 60l, 8mn ),\\
			% 	u =\text{Id} \, +  ( -1 ) \cdot \, u_2\circ u_1 \in& \mathcal{R} ( \max \{ 40m, 40n\}, 61 l, 8mn ).
			% \end{align*}
			% It then follows from \eqref{eq:complexity_f_2} that 
			% \begin{align*}
			% 	&f_1 \circ u \circ  r_{2\Delta}, f_1 \circ u \circ  r_{4\Delta},f_1 \circ u \circ  r_{6\Delta}\\
			% 	\in &\, \mathcal{R} ( \max \{ 200m + 2^{n+5}, 40m, 40 n \}, 98l, \max \{ 8mn, 3^{n+2} \} )\\
			% 	\subseteq&\, \mathcal{R} (  200m + 2^{n+5} , 98l, \max \{ 8mn, 3^{n+2} \} ) 
			% \end{align*}
			% and 
			\begin{align}
				f_2 \circ r_{2\Delta}, f_2 \circ r_{4\Delta},f_2 \circ r_{6\Delta} \in &\, \mathcal{R} ( 200m + 2^{n+5} , 98\elll + 2, \max \{ 8mn, 3^{n+2} \} )  \\
				\subseteq&\, \mathcal{R} ( 200m + 2^{n+5} , 99\elll, \max \{ 8mn, 3^{n+2} \} ), \label{eqline:l_geq_2_1} \\
				(  f_2 \circ r_{2\Delta}, f_2 \circ r_{4\Delta}  , f_2 \circ r_{6\Delta} ) \in&\, \mathcal{R} ( ( 1,3 ),  600m + 3\cdot 2^{n+5} , 99\elll, \max \{ 8mn, 3^{n+2} \} ),
			\end{align}
			and 
			\begin{align}
				f = &\, median \circ (  f_2 \circ r_{2\Delta}, f_2 \circ r_{4\Delta}  , f_2 \circ r_{6\Delta} )\\
				\in &\, \mathcal{R} ( 600m + 3\cdot 2^{n+5} , 99\elll + 3, \max \{ 8mn, 3^{n+2} \} ) \\
				\subseteq&\, \mathcal{R} ( 600m + 2^{n+7} , 101\elll, \max \{ 8mn, 3^{n+2} \} ) \label{eqline:l_geq_2_2},
			\end{align}
			where in \eqref{eqline:l_geq_2_1} and \eqref{eqline:l_geq_2_2} we used the assumption $\elll\geq 2$.
		\end{proof}
		\end{lemma}
		% In the end, we show $f_2$ can be represented by a ReLU network.
		% Now we compute $f_2$. Note that $v(x) = \rho(x) - \rho(x - 1)$, and therefore $v \in \mathcal{R} \left( 1, 2, 1, 1 \right)$. 
		\noindent \textbf{Step 4.} We are now ready to prove Proposition~\ref{prop:approximation_lip_increasing_weights}.
		\begin{proof}
			[Proof of Proposition~\ref{prop:approximation_lip_increasing_weights}] Set $D_a  = 2000$. For $W,L \geq D_a = 2000$, set $m = \bigl\lfloor \frac{W}{1000} \bigr \rfloor > 1$, $n = \bigl\lfloor  \log (\frac{2W}{5}) \bigr \rfloor - 7 > 1$, and $\elll = \lfloor\frac{L}{101}\rfloor> 10$. Fix $g \in \lip ( [0,1] ) $. Application of Lemma~\ref{lem:sec_B_3} to $g$ yields the existence of an 
			\begin{equation}
			\label{eq:complexity_constructed_f}
				f \in \mathcal{R} ( 600m + 2^{n+7} , 101\elll, \max \{ 8mn, 3^{n+2} \} )
			\end{equation}
			such that 
			\begin{equation}
			\label{eq:error_bound_f_minus_g}
				\| f - g \|_{L^\infty ( [0,1] )} \leq \frac{3}{m^2 \elll^2 n}.
			\end{equation}
			Owing to
			\begin{align}
				600m + 2^{n+7} \leq &\; 600\cdot \frac{W}{1000} + \frac{2}{5} W \leq W \\ 
				101 \elll \leq&\; L,\\
				\max \{ 8mn, 3^{n+2} \} \leq & \; W^2, \label{eqline:bound_weight_magnitude_finally}
			\end{align}
			where in \eqref{eqline:bound_weight_magnitude_finally} we used $m \leq \frac{W}{1000}$, $n \leq \log (\frac{2W}{5}) \leq \frac{2W}{5}$, and $3^{n+2} \leq (2^{ n+7 })^2 \leq W^2$, it follows from \eqref{eq:complexity_constructed_f} that 
			\begin{equation}
				f \in \mathcal{R} ( W,L,W^2 ).
			\end{equation}
			Next, note that there exists an absolute constant $c \in \mathbb{R}_+$ such that $m \geq c W$, $n \geq c \log (W) $, and $\elll \geq c L$. Hence, \eqref{eq:error_bound_f_minus_g} implies
			\begin{equation}
				\| f- g \|_{L^\infty ( [0,1] )}  \leq \frac{3}{m^2\elll^2 n} \leq \frac{3}{c^5} \left( W^2 L^2 \log (W) \right)^{-1}.
			\end{equation}
			Since the choice of $g \in \lip \left( [0,1] \right)$ was arbitrary, we have established
			\begin{equation}
				\begin{aligned}
				&\,\mathcal{A}_\infty	( \lip  \left( [0,1] \right), \mathcal{R} ( W, L, W^2  )  ) \leq   \frac{3}{c^5} ( W^2 L^2 \log (W) )^{-1},
				\end{aligned}
			\end{equation}
			which, upon setting $C_a = \frac{3}{c^5}$ and $ K = 2$, concludes the proof.
		\end{proof}

	\subsection{Realization of the Median Function by ReLU Networks} % (fold)
	\label{sub:proof_for_median}
		A ReLU network realization of the median function was reported in \cite{lu2020deep}. For completeness, we provide a formal statement thereof here, but note that its proof follows exactly the construction in \cite{lu2020deep}.
  		\begin{lemma}
		\label{lem:complexity_median}

			Let $\text{median}: \mathbb{R}^3 \mapsto \mathbb{R}$ be given by   
			\begin{equation*}
				\text{median} ( x_1, x_2, x_3 ) := x_{(2)},
			\end{equation*}
			where for $ x_1,x_2,x_3  \in \mathbb{R}$, the reordering of $x_1,x_2,x_3 $ from smallest to largest is denoted as $ x_{(1)},x_{(2)},x_{(3)} \in \mathbb{R}$. It holds that 
			\begin{equation*}
				\text{median} \in \mathcal{R} ( ( 3,1 ), 16, 3, 1 ).
			\end{equation*}
			\begin{proof}
				Let $x_1,x_2,x_3 \in \mathbb{R}$. First, note that%, for $x_1,x_2,x_3 \in \mathbb{R}$,
				\begin{equation}
					\label{eq:decompose_median}
					\text{median} ( x_1, x_2, x_3 ) = x_1 + x_2 +x_3 - \max ( x_1,x_2,x_3 ) - \min ( x_1,x_2,x_3 ).
				\end{equation}
				We have
				\begin{align}
					&\,\max ( x_1,x_2,x_3 )\label{eqline:median_1}\\
					&=\, \max (x_1, \max (x_2,x_3 ) )\\
					&=\, x_1 + \rho ( \max ( x_2,x_3 ) - x_1 )\label{eqline:median_11}\\
					&=\, x_1 + \rho ( x_2 + \rho (x_3 - x_2) - x_1 ) \label{eqline:median_12}\\
					&=\, \rho(x_1) - \rho(-x_1) + \rho (\rho(x_2) - \rho(-x_2) - \rho(x_1) + \rho(-x_1) + \rho (x_3 - x_2)  ),\label{eqline:median_2}
				\end{align}
				where \eqref{eqline:median_11} and \eqref{eqline:median_12} follow from $\max ( a,b ) = a + \rho( b -a)$, for $a,b \in \mathbb{R}$,  and in \eqref{eqline:median_2} we used $x = \rho(x) - \rho(-x)$, for $x \in \mathbb{R}$. Inserting \eqref{eqline:median_1}-\eqref{eqline:median_2} with $( x_1,x_2,x_3 )$ replaced by $( -x_1,-x_2,-x_3 )$ into the relation $\min ( x_1,x_2,x_3 ) = - \max ( -x_1,-x_2,-x_3 )$, $x_1,x_2,x_3 \in \mathbb{R}$, yields%, for $x_1,x_2,x_3 \in \mathbb{R}$,
				\begin{align}
					& \min ( x_1,x_2,x_3 )\label{eqline:median_3}\\
					&=\, - ( \rho(-x_1) - \rho(x_1) + \rho (\rho(-x_2) - \rho(x_2) - \rho(-x_1) + \rho(x_1) + \rho (-x_3 + x_2)  )) \\
					&=\,  - \rho(-x_1) + \rho(x_1) - \rho (\rho(-x_2) - \rho(x_2) - \rho(-x_1) + \rho(x_1) + \rho ( x_2-  x_3)  ). \label{eqline:median_4} 
				\end{align}
				Moreover, %we have%, for $x_1,x_2,x_3 \in \mathbb{R}$,
				\begin{equation}
					x_1 + x_2 +x_3 = \rho ( x_1 + x_2 +x_3 ) - \rho ( - x_1 - x_2 - x_3 ).\label{eqline:median_5}
				\end{equation}
				Substituting \eqref{eqline:median_1}-\eqref{eqline:median_5} into \eqref{eq:decompose_median} and using $\rho\circ \rho = \rho$ yields
    %,  for $x_1,x_2,x_3 \in \mathbb{R}$,
				\begin{align*}
					&\text{median} ( x_1, x_2, x_3 ) \\
					&=\, \rho ( \rho ( x_1 + x_2 +x_3 ) ) - \rho  ( \rho ( - x_1 - x_2 - x_3 ) )\\
					&\,\,- \rho ( \rho(x_1) ) + \rho(\rho(-x_1)) - \rho (\rho(x_2) - \rho(-x_2) - \rho(x_1) + \rho(-x_1) + \rho (x_3 - x_2)  )   \\
					&\,\, + \rho ( \rho(- x_1) ) - \rho(\rho(x_1)) + \rho (\rho(- x_2) - \rho(x_2) - \rho(-x_1) + \rho(x_1) + \rho (x_2 - x_3)  ),
				\end{align*}
				which allows us to conclude that
    %The expression for ReLU network realization of the median function given above now allows us to conclude that 
				% which is linear combinition of eight terms with coefficient in $\{ -1,1 \}$, and each of the eight terms can be realized by a ReLU network of depth 3, width upper-bounded by  5, and weight magnitute 1. Using the Lemma from ???, median functions can be realized by a ReLU network of depth $3$, width no greater than $40$ and weight magnitute upper-bounded by $1$. Formally, we have 
				\begin{equation*}
				 	\text{median} \in \mathcal{R} ( (3,1), 16, 3, 1 ). \qedhere
				\end{equation*} 
			\end{proof}
		\end{lemma}

		% Let 
		% \begin{equation*}
		% 	A_1 = 
		% 	\begin{pmatrix}
		% 		1&-1&0&0&0&0&1&-1\\
		% 		0&0&1&-1&1&-1&1&-1\\
		% 		0&0&0&0&-1&1&-1&1\\
		% 	\end{pmatrix},
		% 	b_1 = 0
		% \end{equation*}
		% and ??? 
		% We have 
		% \begin{equation*}
		% 	W ( A_1, b_1 ) ( x_1, x_2, x_3) = ( x_1, -x_1, x_2, -x_2, x_1 +x_2+x_3, -x_1- x_2 -x_3 ).
		% \end{equation*}
		% Let 
		% \begin{equation*}
		% 	A_2 = \begin{pmatrix}
		% 		I_6\\
		% 		1&-1& 1&-1&\\
		% 	\end{pmatrix}
		% \end{equation*}

	% subsection proof_for_median (end)

% section approximate_ (end)
	%!TEX root = ../draft_quantized_weight_networks.tex

	\section{Proof of Proposition~\ref{prop:piecewise_representation}} % (fold)
	\label{sub:piecewise_linear_representation}
		% This part will prove Lemma~\ref{prop:piecewise_representation}. Suppose $N \in \mathbb{N}$ and $X = \left( x_i \right)_{i = 0 }^N$ is a strictly increasing sequence in $[0,1]$. 

		% We shall establish Proposition~\ref{prop:piecewise_representation} through proving an intermediate result, and then applying a technical lemma provided in Appendix~\ref{sub:trade_depth_for_weight_magnitute}, that trades depth for weight-magnitude, to the intermediate result, in a similar way as we used Proposition~\ref{prop:approximation_lip_increasing_weights} to prove Theorem~\ref{thm:approximation_lip}. The intermediate result is the ensuing proposition.		
		We start with an intermediate result.
		\begin{proposition}
			\label{prop:piecewise_representation_constructive}
			Let $M \in \mathbb{N}$ with $M \geq 3$, $E \in \mathbb{R}_+$, and let $X = (x_i)_{i = 0}^{M-1}$ be a strictly increasing sequence taking values in $[0,1]$. Then, for all $u,v \in \mathbb{N}$ such that $u^2 v \geq M$, we have
			\begin{align*}
				\Sigma ( X, E) \subseteq&\, \mathcal{R} ( 20u, 30v, \max \{ 1, C_k M^6 R_m ( X ) (R_c( X ))^3 E \} ) \\
				\subseteq &\,\mathcal{R} ( 20u, 30v, \max \{ 1, C_k M^6 (R_m ( X ))^4 E \} ),
			\end{align*}
			for an absolute constant $C_k \in \mathbb{R}$ satisfying $2\leq C_k \leq 10^5$, and where $R_m(X) := \max_{i =1,\dots, M} ( x_{i} - x_{i -1} )^{-1}$ and $R_c(X) := \frac{\max_{i =1,\dots, M-1}  (x_{i} - x_{i -1}) }{\min_{i =1,\dots, M-1}  (x_{i} - x_{i -1}) }$.
		\end{proposition}
		We now show how Proposition~\ref{prop:piecewise_representation_constructive} leads to the proof of Proposition~\ref{prop:piecewise_representation} and provide the proof of Proposition~\ref{prop:piecewise_representation_constructive} thereafter.
		\begin{proof}
		[Proof of Proposition~\ref{prop:piecewise_representation}]
			Application of Proposition~\ref{prop:piecewise_representation_constructive} to $\Sigma \bigl( X, \frac{1}{C_k M^6 ( R_m(X) )^4} \bigr)$, with $C_k$ the constant in the statement of Proposition~\ref{prop:piecewise_representation_constructive}, yields

			\begin{equation}
			\label{eq:normalize_the_weight}
				\Sigma \biggl( X, \frac{1}{C_k M^6 ( R_m(X) )^4}\biggr) \subseteq \mathcal{R} ( 20u, 30v, 1 ).
			\end{equation}
			We hence get
			\begin{align}
				\mathcal{R} ( 20u, 30v, 2w) \supseteq&\, ( 2w )^{30v} \cdot \mathcal{R} ( 20u, 30v, 1) \label{eqline:relu_set_weight_1}\\
				\supseteq &\,  ( 2w )^{30v} \cdot \Sigma \biggl( X, \frac{1}{C_k M^6 ( R_m(X) )^4}\biggr)\label{eqline:embedding_cpwl_in_relu}\\
				\supseteq &\, ( 2w )^{30v} \cdot \frac{1}{( 2w )^{30v}}\cdot\Sigma \biggl( X, \frac{( 2w )^{30v}}{C_k M^6 ( R_m(X) )^4}\biggr) \label{eqline:scaling_relation_cpwl} \\
				\supseteq &\, \Sigma ( X, E ), \label{eqline:embedding_cpwl_in_cpwl}
			\end{align}
			where in \eqref{eqline:relu_set_weight_1} we applied Proposition~\ref{prop:depth_weight_magnitude_tradeoff} with $( W,L,L',B,B' ) = ( 20u, 30v, 0, 1, 2w )$, \eqref{eqline:embedding_cpwl_in_relu} follows from \eqref{eq:normalize_the_weight},  and \eqref{eqline:embedding_cpwl_in_cpwl} is a consequence of \eqref{eqline:cpwl_realization_weight_condition} and $2^{30v} \geq 10^5 \geq C_k$. Further, in \eqref{eqline:scaling_relation_cpwl} we used that\footnote{While the reverse inclusion is also valid, it will not be needed here.}
   %To avoid providing proofs of results not used, we will restrict our presentation to the one-sided inclusion.} 
   $a\cdot \Sigma ( X, b ) \subseteq \Sigma ( X, ab )$, for all $a,b \in \mathbb{R}_+$, which follows from the fact that for every $f \in \Sigma ( X, b ) $, $a\cdot f$ is a bounded piecewise linear function with breakpoints in $X$ and $L^\infty ( \mathbb{R} )$-norm no greater than $ab$.
		\end{proof}

		It remains to prove Proposition~\ref{prop:piecewise_representation_constructive}. %We start by describing the ideas and ingredients underlying the %proof. Specifically, t
  The proof will be effected by representing the functions in $\Sigma ( X, B )$ in terms of a specific basis for the linear space $\Sigma ( X, \infty)$. Crucially, the elements of this basis will be realized by ReLU networks with suitable properties. Concretely, we shall work with the basis $\{ \gamma_i: \mathbb{R} \mapsto \mathbb{R} \}_{i =0 }^{M- 1} $ given by 
		\begin{equation}
		\label{eq:basis_1}
			\gamma_0 (x) = \left\{
			\begin{aligned}
				&1, && \quad x \in (-\infty, x_0 ],\\
				&1 - \frac{x - x_0}{x_1 - x_0}, && \quad x \in (x_0, x_1],\\
				&0, && \quad x \in (x_1, \infty),
			\end{aligned}
			\right.
		\end{equation}
		for $i = 1, \dots, M - 2$,
		\begin{equation}
		\label{eq:basis_2}
			\gamma_i (x) = \left\{
			\begin{aligned}
				&0, && \quad x \in (-\infty, x_{ i -1} ] \cup (x_{i+1}, \infty),\\
				&\frac{x - x_{i-1}}{x_i - x_{i-1}}, && \quad x \in (x_{i-1}, x_i],\\
				&1 - \frac{x - x_{i}}{x_{i+1} - x_{i}}, && \quad x \in (x_{i}, x_{i+1}],
			\end{aligned}
			\right.
		\end{equation}
		and
		\begin{equation}
		\label{eq:basis_3}
			\gamma_{M-1} (x) = \left\{
			\begin{aligned}
				&0, && \quad x \in (-\infty, x_{M-2} ],\\
				&\frac{x - x_{M-2}}{x_{M-1} - x_{M-2}}, && \quad x \in (x_{M -2}, x_{M-1}],\\
				&1, && \quad x \in (x_{M - 1}, \infty).
			\end{aligned}
			\right.
		\end{equation}
		We note that  $\gamma_0,\dots, \gamma_{M-1} \in \Sigma ( X, 1)$, and for $i = 0,\dots, M-1$, $j  = 0,\dots, M-1$, we have the interpolation property
		\begin{equation*}
			\gamma_i(x_j) = \left\{ 
			\begin{aligned}
				&1, && \text{if } i = j,\\
				&0, && \text{if } i \neq j.
			\end{aligned}
			\right.
		\end{equation*}
		An illustration of the basis $\{ \gamma_i: \mathbb{R} \mapsto \mathbb{R}\}_{i =0 }^{M-1}$ is provided in Figure~\ref{fig:basis}.
		% \begin{figure}[H]
		% 	\centering
		% 	\includegraphics[width=0.4\textwidth]{images/basis.jpg}
		% 	\caption{Basis of $\Sigma(X,\infty)$}
		% 	\label{fig:basis}
		% \end{figure}

		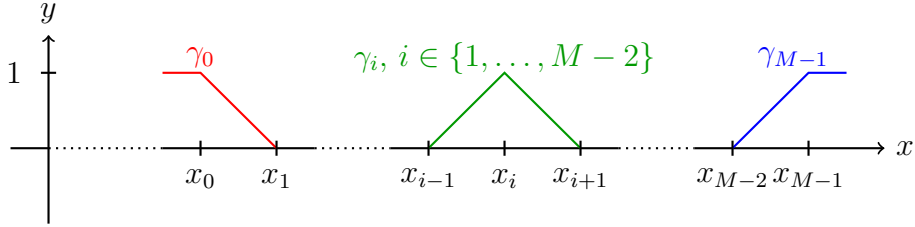
\begin{figure}[H]
			\centering
			\begin{tikzpicture}
			    % Axes
			    \draw[thick, -] (-0.5,0) -- (0,0);
			    \draw[thick, dotted] (0,0) -- (1.5,0) ;
			    \draw[thick, -] (1.5,0) -- (3.5,0);
			    \draw[thick, dotted] (3.5,0) -- (4.5,0) ;
			    \draw[thick, -] (4.5,0) -- (7.5,0);
			    \draw[thick, dotted] (7.5,0) -- (8.5,0) ;
			    \draw[thick, ->] (8.5,0) -- (11,0) node[right] {$x$};
			    \draw[->, thick] (0,-1) -- (0,1.5) node[above] {$y$};

			    %tick
			    \draw[-, thick] (-0.1,1) -- (0.1,1) node[left=3mm] {$1$};
			    
			    % Red Line
			    \draw[red, thick] (1.5,1) -- (2,1) -- (3,0);
			    
			    % Green Line
			    \draw[green!60!black, thick] (5,0) -- (6,1) -- (7,0);
			    
			    % Blue Line
			    \draw[blue, thick] (9,0) -- (10,1) -- (10.5,1);
			    
			    % Ticks
			    \foreach \x in {2,3,5,6,7,9,10} {
			    	\draw[thick] (\x,-0.1) -- (\x,0.1);
			    }
			    
			    % % Dots
			    % \draw[dotted] (3.5,0) -- (4,0);
			    % \draw[dotted] (6.5,0) -- (7,0);
			    
			     % Labels 
			    % \node at (1,-0.4) {$x_0$};
			    \node at (2,-0.4) {$x_0$};
			    \node at (3,-0.4) {$x_1$};
			    \node at (5,-0.4) {$x_{i-1}$};
			    \node at (6,-0.4) {$x_{i}$};
			    \node at (7,-0.4) {$x_{i+1}$};
			    \node at (9,-0.4) {$x_{M-2}$};
			    \node at (10,-0.4) {$x_{M-1}$};
			    
			    \node[red] at (2,1.2) {$\gamma_0$};
			    \node[green!60!black] at (6,1.2) {$\gamma_i,\, i \in \{ 1,\dots, M-2 \}$};
			    \node[blue] at (9.8,1.2) {$\gamma_{M-1}$};
			\end{tikzpicture}
			\caption{The basis $\{ \gamma_i\}_{i =0 }^{M - 1}$ for $\Sigma(X, \infty)$.}
			\label{fig:basis}
		\end{figure}
		% Here $\gamma_0$ is a piecewise takes 1 on $(-\infty, x_0)$, takes 0 on $[x_2, \infty)$, and is linear between $x_0$ and $x_1$. And Here $\gamma_n$ takes 0 on $(-\infty, x_{n-1}]$, takes 1 on $[x_2, \infty)$, and is linear between $x_{n-1}$ and $x_n$. And for $i \in \left\{ 1,\dots, n - 1 \right\}$, $\gamma_i$ takes $0$ on $(-\infty, x_{i -1}]$ and $[x_{i + 1}, \infty) $, takes 1 at $x_{i}$, and is linear on $[x_{i-1}, x_i]$ and $[x_{i}, x_{i+1}]$. 
		% It is easy to see $\left\{ \gamma_i \right\}_{i =1}^N$ is a basis of $\Sigma \left( X \right)$, and 
		For every $f \in \Sigma ( X, \infty )$, we have 
		\begin{equation}
			\label{eq:basis_representation}
			f(x) = \sum_{i = 0}^{M-1} f(x_i) \gamma_i(x), \quad \text{ for all } x \in \mathbb{R},
		\end{equation}
		which is a consequence of the fact that $f$ and $\sum_{i = 0}^{M-1} f(x_i) \gamma_i$, by virtue of both being bounded piecewise linear functions with the same breakpoints and the same function values on these breakpoints, must be identical. This shows that $\{ \gamma_i\}_{i =0 }^{M - 1}$ is, indeed, a basis for $\Sigma ( X, \infty )$.
		% Therefore, we can realize functions in $\Sigma \left( X \right)$ by ReLU networks if we can realize the basis $\left( \gamma_i \right)_{i =0}^N$ by ReLU networks. We start by the realization through 2-layer ReLU networks.
		Based on \eqref{eq:basis_representation}, we proceed to the next building block of our proof.
  %next show how $f \in \Sigma ( X, E )$  can be realized by a ReLU network of depth $L = 2$. This construction will constitute a building block in the %realization of $\Sigma ( X, E )$ by ReLU networks with general depth $L$.
		\begin{lemma}
			\label{lem:representation_single_hidden_layer}
			Let $M \in \mathbb{N}$ with $M \geq 3$, $E \in \mathbb{R}_+$, and let $X = ( x_i )_{i = 0}^{M-1}$ be a strictly increasing sequence taking values in $[0,1]$. Then, every function $f \in \Sigma ( X, E )$ can be represented as
			\begin{equation*}
				f(x) = b + \sum_{i = 0}^{M-1} a_i\, \rho(x - x_i), \quad x \in \mathbb{R},
			\end{equation*}
			for some $b,a_0,\dots, a_{M-1} \in \mathbb{R} $ with  $\max \{ | b |,| a_0 |, \dots, | a_{M-1} |    \}  \leq 4 R_m(X) E$, where $R_m(X) := \max_{i =1}^{M} ( x_{i} - x_{i -1} )^{-1}$. 
			In particular, $f \in \mathcal{R} ( M, 2, 4 R_m(X) E )$.
			% , and in particular, $b = 0$ if $f(x_0) = 0$.
			% \begin{equation*}
			% 	\Sigma \left( X, E \right) \subseteq \mathbb{R} \left( 1, N + 1, 2, B \right),
			% \end{equation*}
			% with $B = \max \left\{ 1, 6 \sup_{i = 0,\dots, n-1} \frac1{x_i - x_{i-1}}E \right\}$.
		\end{lemma}

		\begin{proof}

			% We shall use the representation \eqref{eq:basis_representation}.
			% In order to represent $f$ by a ReLU network, it suffices to represent $\gamma_i$, $i = 0, \dots, n$, by ReLU networks. 
			% We have, according to expression of $\{ \gamma_0,\dots, \gamma_{M-1} \}$ given by \eqref{eq:basis_1}-\eqref{eq:basis_3}, for $x \in \mathbb{R}$,

			The following representations of the functions $\{\gamma_0,\dots, \gamma_{M-1}\}$ can be read off directly from their definition 
			\begin{equation}
			\label{eq:represented}
			\begin{aligned}
				\gamma_0 (x) =&\, - \frac{1}{x_1 - x_0}\,\rho(x - x_0) + \frac{1}{x_1 - x_0}\,\rho(x - x_1) + 1,\\
				\gamma_i(x) = &\, \frac{1}{x_{i} - x_{i-1}} \,\rho(x - x_{i-1})\\
				& - \biggl( \frac{1}{x_{i} - x_{i-1}} + \frac{1}{x_{i+1} - x_{i}} \biggr) \,\rho(x - x_i)\\
				&+ \frac{1}{x_{i+1} - x_{i}} \,\rho(x - x_{i+1}), \quad i = 1,\dots, M-2,\\
				\gamma_{M-1} (x) =&\, \frac{1}{x_{M-1} - x_{M-2}}\,\rho(x - x_{M-2}) - \frac{1}{x_{M - 1} - x_{M-2}}\,\rho(x - x_{M-1}),
			\end{aligned}
			\end{equation}
			for $x \in \mathbb{R}$ in all cases.  Inserting \eqref{eq:represented} into \eqref{eq:basis_representation} yields
			\begin{align*}
				f(x) =& \, f(x_0) +  \Bigl( - \frac{f(x_0)}{x_1 - x_0} + \frac{f(x_{1})}{x_{1} - x_{0}} \Bigr) \rho(x - x_{0}) \\
				& +  \sum_{i = 1}^{M-2} \biggl( \frac{f(x_{i+1})}{x_{i+1} - x_{i}}  -  \frac{f(x_i)}{x_{i} - x_{i-1}} - \frac{f(x_i)}{x_{i+1} - x_{i}}   +  \frac{f(x_{i -1})}{x_{i} - x_{i -1 }} \biggr) \rho(x - x_{i}) \\
				& + \Bigl( \frac{f(x_{M -2})}{x_{M- 1} - x_{M -2 }}  - \frac{f(x_{M-1})}{x_{M-1} - x_{M-2}} \Bigr)\rho(x - x_{M-1}), \quad \text{for } x \in \mathbb{R}.
			\end{align*}
			% $f$ can be written as linear combination of $\rho(x - x_i)$, $i = 0\dots, n$,  and $1$, with coefficients less than $\max \left\{1,  6 \sup_{i = 0,\dots, n-1} (x_i - x_{i-1})^{-1} \right\}$.
			We conclude the proof by setting $b = f(x_0)$, $a_0 =   - \frac{f(x_0)}{x_1 - x_0} + \frac{f(x_{1})}{x_{1} - x_{0}} $, $a_i = \frac{f(x_{i+1})}{x_{i+1} - x_{i}}  -  \frac{f(x_i)}{x_{i} - x_{i-1}} - \frac{f(x_i)}{x_{i+1} - x_{i}}   +  \frac{f(x_{i -1})}{x_{i} - x_{i -1 }}$, for $i = 1,\dots, M-2$, and $a_{M-1} = \frac{f(x_{M -2})}{x_{M-1} - x_{M -2 }}  - \frac{f(x_{M-1})}{x_{M-1} - x_{M-2}} $, so that, indeed, $\max \{ | b |,| a_0 |, \dots, | a_{M-1} |    \}  \leq 4 \| f \|_{L^\infty ( [0,1] )} \max_{i =1,\dots, M-1} ( x_{i} - x_{i -1} )^{-1} \leq  4 R_m(X) E. $
		\end{proof}
		% Proposition~\ref{lem:representation_single_hidden_layer} shows that every function in $\Sigma ( X, E )$ can be realized by a 2-layer ReLU network with width proportional to the number of breakpoints and weight magnitude upper-bounded by four times the supremum norm of the function to be realized multiplied by the inverse minimum gap in X as quantified by $R_m(X) = \max_{i =1}^{M-1} ( x_{i} - x_{i -1} )^{-1}$. 
		% \todo{The number of parameter is proportional to the breakpoints, and it will be used as a building block for a deep construction.} 

		%Recalling that we want to realize elements in $\Sigma ( X, E )$ through deep networks, 
  We next provide a lemma describing ReLU network realizations of the basis functions $\{ \gamma_i \}_{i = 0}^{M-1}$.
  %in the special case $M = tu$ for some $t,u \in \mathbb{N}$ with $t \geq 8$. 
  These constructions are inspired by \cite{daubechies2022nonlinear} and \cite{shen2019nonlinear}.

		\begin{lemma}
			\label{lem:representation_two_hidden_layer}
			Let $t,u \in \mathbb{N}$ with $t \geq 8$, set $M = tu$ and 
			\begin{equation}
				\label{eq:index_set_cpwl}
				\mathcal{I}:= \{ ( k,\ell ): k \in \{ 0,\dots, u-1 \},\ell \in \{ 0,\dots, t -1 \}\}.
			\end{equation}
			We have 
			\begin{equation}
				\{ 0,\dots, M -1 \} = \{ kt + \ell: ( k,\ell ) \in \mathcal{I} \}.
			\end{equation}
			Let $X = (x_i)_{i = 0}^{M-1}$ be a strictly increasing sequence taking values in $[0, 1]$ and let $\{\gamma_i \}_{i = 0}^{M-1}$ be the basis for $\Sigma ( X, \infty )$ defined in \eqref{eq:basis_1}-\eqref{eq:basis_3}. Then, there exist $f_{k,\ell}^1, f_{k,\ell}^2, f_{k,\ell}^3 \in \Sigma ( X, \infty ) $, for all $( k,\ell ) \in \mathcal{I}$, such that the following statements hold.
			\begin{itemize}
				\item (Property 1) For all $( k,\ell ) \in \mathcal{I}$, the basis function $\gamma_{kt +\ell}$ can be realized according to
				\begin{equation*}
					\gamma_{kt +\ell} = \rho \circ f_{k,\ell}^1 - \rho \circ f_{k,\ell}^2 + \rho \circ f_{k,\ell}^3.
				\end{equation*}
				% \item (Property 2) For $( k,i ) \in \mathcal{I}$, the function $f^j_{k,i}$ can be written as
				% \begin{equation*}
				% 	 f^j_{k,i} = b + \sum_{\ell \in \mathcal{A}} a_\ell \rho(x - x_\ell),
				% \end{equation*}
				% \todo{adding $b$}
				% for the index set $\mathcal{A} = \bigl\{ kn +i : k \in \{ 0,\dots, m \}, i \in \{ 0,1,2,3,n-3,n-2,n-1 \}, kn \leq N \bigr\}$ not depending on $j,k,i$, and some $c$, $a_\ell \in \mathbb{R}$ such that $ | c |, \max_{\ell \in \mathcal{A}} | a_\ell | \leq  12 n^3 R_m  R_c^3 $, where $R_m := R_m(X) := \sup_{i =1}^{N} ( x_{i} - x_{i -1} )^{-1}$, and $R_c := R_c(X) := \frac{\sup_{i =1}^{N}  x_{i} - x_{i -1} }{\inf_{i =1}^{N}  x_{i} - x_{i -1} }$. \todo{change the definition of $R_c$ and $R_m$.}
				\item (Property 2) Let $( z_i )_{i = 0}^{8u -1}$ be the strictly increasing sequence obtained by sorting the elements in 
				\begin{equation}
					\label{eq:definitio_z_sequence}
					\{ x_{kt + \elll}: k \in \{ 0,\dots, u-1 \}, \elll \in \{ 0,1,2,3,t - 4, t-3,t-2,t-1 \} \}.
				\end{equation}
				For all $( k,\ell ) \in \mathcal{I}$, and for $j \in \{ 1,2,3 \}$, the function $f^j_{k,\ell}$ can be written as
				\begin{equation}
					\label{eq:expression_of_fjki}
					 f^j_{k,\ell} (x) = b + \sum_{i = 0}^{8u - 1} a_i \rho(x - z_i), \quad x \in \mathbb{R},
				\end{equation}
				for $b, a_0,\dots, a_{8u - 1} \in \mathbb{R}$ depending on $k,\elll,j$ and such that $$ | b |, | a_0 |,\dots, | a_{8u - 1} | \leq  12 t^2 R_m(X)  (R_c(X))^3 ,$$ where $R_m(X) := \max_{i =1, \dots, M-1} ( x_{i} - x_{i -1} )^{-1}$ and $R_c(X) := \frac{\max_{i =1, \dots, M-1}  ( x_{i} - x_{i -1} ) }{\min_{i =1, \dots, M-1}  ( x_{i} - x_{i -1} ) }$.

				\item (Property 3) For every $( \ell, j ) \in \{ 0,1\dots, t-1 \} \times  \{ 1,2,3 \}$, the functions in $\{ f^j_{k,\ell} \}_{k =0}^{u-1}$ have pairwise disjoint supports

				% \footnote{We use the convention that zero function has empty support}. \todo{check whethe I should use product or not, see the comments }
			\end{itemize}

			\begin{proof}
				See Appendix~\ref{sub:proof_of_lemma_lem:representation_two_hidden_layer}.
			\end{proof}
		\end{lemma}

		% Property 1 and Property 2 give the possibility of representing each basis function $\gamma_i$ together with the price of the representation in terms of weight magnitude. And Property 3 allows us to use the following lemma to reduce the width of the network.
		% \begin{lemma}
		% 	\label{lem:reduce_width}
		% 	Let $d \in \mathbb{N}$, $\mathbb{O} \subseteq \mathbb{R}^d$. $f_i$, $i =1, \dots, N$, are functions on $\mathbb{O}$ with disjoint supports. Then 
		% 	\begin{equation*}
		% 		\sum_{i=1}^N \rho\left( f_i(x)\right) = \rho\left( \sum_{i =1}^N f_i (x)\right), \quad x \in \mathbb{O}.
		% 	\end{equation*}
		% \end{lemma}
		% Given that $M = tu$ for some $t,u \in \mathbb{N}$ with $t \geq 8$,  Property~$1$ of Lemma~\ref{lem:representation_two_hidden_layer} gives the representation of $\{ \gamma_i \}_{i = 0}^{M-1}$ in terms of some auxiliary functions $f_{k,\ell}^1$, $f_{k,\ell}^2$, $f_{k,\ell}^3$, $( k,\ell ) \in \mathcal{I}$, satisfying Properties 2 and 3. We will seen in the proof of Proposition~\ref{prop:piecewise_representation_constructive} that Properties 2 and 3 help reduce the complexity of ReLU networks realizing elements in $\Sigma ( X, B )$. If we assume further that $t = uv$ for some $v \in \mathbb{N}$, with the help of the Lemma~\ref{lem:representation_two_hidden_layer} and the basis decomposition of function $g \in \Sigma ( X, B )$, we are able to write $g$ as $H^+ - H^-$, for some $H^+$ and $H^-$ satisfying the conditions of $H$ in the following lemma and being able to be realized by ReLU networks.

		% \newpage

		Further, we need the following technical lemma, which realizes a 3-layer ReLU network of a specific form by an equivalent deeper ReLU network. The result is inspired by the construction reported in \cite[Lemma 4.2]{shen2021optimal}.

		\begin{lemma}
		\label{lem:realization_piecewise_linear_functions}
			Let $u,s,r \in \mathbb{N}$, $( z_i )_{i = 0}^{r-1} \subseteq [0,1]$, and $T \in \mathbb{R}_+$. For $\ell = 0,\dots, us -1$, let $h_\ell: \mathbb{R} \mapsto \mathbb{R} $ be  given by
			\begin{equation}
				\label{eq:definition_h}
				h_\ell ( x ) = d_\ell + \sum_{i = 0}^{r-1} c_{\ell,i} \rho ( x - z_i ),
			\end{equation}
			with $d_\ell, c_{\ell,0}, \dots, c_{\ell, r-1} \in \{ x \in \mathbb{R}: | x | \leq T  \}$, and let 
			\begin{equation}
			\label{eq:form_of_H}
				H = \sum_{\ell = 0}^{us - 1} (\rho \circ h_\ell).
			\end{equation}
			Then,
			\begin{equation*}
				H \in \mathcal{R} ( r + u + 1, s + 2, \max \{ 1, T  \} ).
			\end{equation*}
			\begin{proof}
				See Appendix~\ref{sub:proof_of_lemma_lem:realization_piecewise_linear_functions}.
			\end{proof}
		\end{lemma}

		We proceed to the proof of Proposition~\ref{prop:piecewise_representation_constructive}, which will be effected through Lemmata~\ref{lem:representation_two_hidden_layer} and \ref{lem:realization_piecewise_linear_functions}.
		\begin{proof}
		[Proof of Proposition~\ref{prop:piecewise_representation_constructive}]
			Let $g \in \Sigma ( X, E)$ with $X= ( x_\ell)_{\ell = 0}^{M-1}$, $M\geq 3$, a strictly increasing sequence taking values in $[0,1]$ and $E \in \mathbb{R}_+$. We start with the special case $M = u^2v$ and $uv \geq 8$, and will later reduce the other cases to this one. Write $g$ as $H^+ - H^-$, for some $H^+$ and $H^-$, both of the form \eqref{eq:form_of_H}. Setting $t = uv \geq 8$, we can write $M = tu$. According to Lemma~\ref{lem:representation_two_hidden_layer}, there exist $f_{k,\ell}^j$, $k = 0,\dots, u- 1 $, $\ell = 0,\dots, t-1$, $j = 1,2,3$, such that Properties $1$-$3$ in the statement of Lemma~\ref{lem:representation_two_hidden_layer} hold. For $i = 0,\dots, M-1$, let $y_i = g(x_i)$, $y_i^+ = \max \{ y_i ,0 \}$, $y_i^- = \max \{ -y_i ,0 \}$, and note that 
			\begin{align}
				y_i =&\, y_i^+ - y_i^- \label{eqline:decompose_y_i}\\
				| y_i^+ |,| y_i^- | \leq&\, | y_i | \leq\, \| g \|_{L^\infty ( [0,1] )} \leq\, E . \label{eqline:bound_on_y_i} 
			\end{align}
			% We next note that starting from the basis representation \eqref{eq:basis_representation}, $g$ can be expressed as, $x \in \mathbb{R}$, 
			We then have, for $x \in \mathbb{R}$,
			% \begin{align}
			% 	&g \\
			% 	= &\sum_{i = 0}^n g(x_i) \gamma_i \\
			% 	= &\sum_{i = 0}^n y_i \gamma_i \\
			% 	=& \sum_{i = 0}^{n - 1} \sum_{k = 0}^{m - 1 } y_{kn + i} \gamma_{kn + i} + y_N \gamma_N \\
			% 	=& \sum_{i = 0}^{n - 1}  \left( y_{kn + i}^+  - y_{kn + i}^- \right) \Biggl( \rho\left( \sum_{k = 0}^{m - 1 } f_{k,i}^1\right) - \\
			% 	&\quad \rho\left( \sum_{k = 0}^{m - 1 } f_{k,i}^3\right) + \rho\left( \sum_{k = 0}^{m - 1 } f_{k,i}^3\right) \Biggr) \\
			% 	&+ \left( y_N^+ - y_N^- \right) \rho(\gamma_N) \nonumber\\
			% 	=& \sum_{i = 0}^{n - 1}  \Biggl( \rho\left(y^+_{kn + i}  \sum_{k = 0}^{m - 1 } f_{k,i}^1\right) - \\
			% 	&\quad \rho\left( y^+_{kn + i}  \sum_{k = 0}^{m - 1 } f_{k,i}^3\right) + \rho\left( y^+_{kn + i}  \sum_{k = 0}^{m - 1 } f_{k,i}^3\right) \Biggr) \label{eqline:positive_homogeneous}\\
			% 	&+ \rho(y^+_N  \gamma_N) \nonumber\\
			% 	& - \sum_{i = 0}^{n - 1}  \Biggl( \rho\left(y^-_{kn + i}  \sum_{k = 0}^{m - 1 } f_{k,i}^1\right) - \\
			% 	&\quad \rho\left( y^-_{kn + i}  \sum_{k = 0}^{m - 1 } f_{k,i}^3\right) + \rho\left( y^-_{kn + i}  \sum_{k = 0}^{m - 1 } f_{k,i}^3\right) \Biggr) \nonumber\\
			% 	&- \rho( y^-_N  \gamma_N).\nonumber
			% \end{align}
			\begin{align}
				g (x) = &\sum_{\ell = 0}^{M-1} g(x_i) \gamma_i (x) \label{eqline:decompose_g}\\
				= &\sum_{i = 0}^{M-1} y_i \gamma_i(x) \\
				=& \sum_{\ell = 0}^{t - 1} \sum_{k = 0}^{u - 1 } y_{kt + \ell} \gamma_{kt + \ell}(x)\label{eqline:decompose_M}\\
				=& \sum_{\ell = 0}^{t - 1} \sum_{k = 0}^{u - 1 }   ( y_{kt + \ell}^+  - y_{kt + \ell}^- )  \Biggl( \rho(  f_{k,\ell}^1(x)) -  \rho(  f_{k,\ell}^2(x) ) + \rho(  f_{k,\ell}^3(x)) \Biggr) \label{eqline:apply_property_one} \\
				=& \sum_{\ell = 0}^{t - 1} \sum_{k = 0}^{u - 1 } \Biggl( \rho(y^+_{kt + \ell}   f_{k,\ell}^1(x)) +  \rho( y^-_{kt + \ell}   f_{k,\ell}^2(x))  + \rho( y^+_{kt + \ell}   f_{k,\ell}^3 (x))  \nonumber\\
				&  - \rho(y^-_{kt + \ell}   f_{k,\ell}^1(x)) -  \rho( y^+_{kt + \ell}   f_{k,\ell}^2(x))  - \rho( y^-_{kt + \ell}   f_{k,\ell}^3(x)) \Biggr)\label{eqline:positive_homogeneous}  \\
				=& \sum_{\ell = 0}^{t - 1}  \Biggl( \rho\Biggl( \sum_{k = 0}^{u - 1 } y^+_{kt + \ell}   f_{k,\ell}^1(x)\Biggr) +  \rho\Biggl( \sum_{k = 0}^{u - 1 } y^-_{kt + \ell}   f_{k,\ell}^2(x)\Biggr)  + \rho\Biggl( \sum_{k = 0}^{u - 1 } y^+_{kt + \ell}   f_{k,\ell}^3(x)\Biggr)\nonumber \\
				&  - \rho\Biggl( \sum_{k = 0}^{u - 1 } y^-_{kt + \ell}   f_{k,\ell}^1(x)\Biggr) -   \rho\Biggl(\sum_{k = 0}^{u - 1 }  y^+_{kt + \ell}   f_{k,\ell}^2(x)\Biggr)  - \rho\Biggl( \sum_{k = 0}^{u - 1 } y^-_{kt + \ell}   f_{k,\ell}^3(x)\Biggr) \Biggr), \label{eqline:merging_lemma}
			\end{align}
			where \eqref{eqline:apply_property_one} follows from  Property 1 in Lemma~\ref{lem:representation_two_hidden_layer}, in \eqref{eqline:positive_homogeneous} we used the positive homogeneity of the ReLU function, i.e., $\rho(xy) = x\rho(y)$, for all $y \in \mathbb{R}$ and $x \in \mathbb{R}_+ \cup \{ 0 \}$, and \eqref{eqline:merging_lemma} is a consequence  
			% we use 
			% \begin{equation*}
			% 	\sum_{k = 0}^{u-1} \rho (  g_{k} ( x )) =   \rho \Biggl(  \sum_{k = 0}^{u-1} g_{k} ( x )\Biggr),\, x \in \mathbb{R},
			% \end{equation*} 
			% for disjointly supported $g_k$, $k = 0,\dots, u-1$, combining with Property 3 given by Lemma~\ref{lem:representation_two_hidden_layer}.  
			of the functions in $\{ f^j_{k,\ell} \}_{k =0}^{u-1}$ being of pairwise-disjoint support, for all $\ell \in \{ 0,1\dots, t-1 \}$, $j \in \{ 1,2,3 \}$, as guaranteed by Property 3 in Lemma~\ref{lem:representation_two_hidden_layer}.

			For $\ell = 0,\dots, t -1$, now define
			\begin{equation}
				\label{eqline:component_1}
				h^+_{\ell}= \sum_{k = 0}^{u - 1 } y^+_{kt + \ell}\,   f_{k,\ell}^1,
			\end{equation}
			$h^+_{t+\ell} =  \sum_{k = 0}^{u - 1 } y^-_{kt + \ell}   f_{k,\ell}^2$, $h^+_{2t+ \ell} =  \sum_{k = 0}^{u - 1 } y^+_{kt + \ell}   f_{k,\ell}^3$, $h^-_{\ell} =   \sum_{k = 0}^{u - 1 } y^-_{kt + \ell}   f_{k,\ell}^1$, $h^-_{t + \ell} = \sum_{k = 0}^{u - 1 }  y^+_{kt + \ell}   f_{k,\ell}^2$, $h^-_{2t+\ell} =   \sum_{k = 0}^{u - 1 } y^-_{kt + \ell}   f_{k,\ell}^3$, and let 
			\begin{equation*}
				H^+ = \sum_{\ell =0}^{3t -1 } \rho \circ h_\ell^+, \quad H^- = \sum_{\ell =0 }^{3t -1 } \rho \circ h_\ell^-.
			\end{equation*}
			By \eqref{eqline:decompose_g}-\eqref{eqline:merging_lemma}, we can therefore write
			\begin{align*}
				g = H^+ - H^-.
			\end{align*}
			We next show that $h^+_\ell$, $h^-_\ell$, $\ell = 0,\dots, 3t -1$ are all of the form \eqref{eq:definition_h} as required by Lemma~\ref{lem:realization_piecewise_linear_functions}. First, consider $h_0^+$. Thanks to Property 2 in Lemma~\ref{lem:representation_two_hidden_layer}, for all $( k,\ell )\in \mathcal{I}$ and $j \in  \{1,2,3\}$, one has
			\begin{equation}
				f^j_{k,\ell} (x) = b^{(j,k,\ell)} + \sum_{i = 0}^{8u - 1} a^{(j,k,\ell)}_i \rho(x - z_i), \label{eq:decompose_f_proof}
			\end{equation}
			with the strictly increasing sequence $( z_i )_{i = 0}^{8u -1 }$ obtained by sorting the elements in $\{ x_{kt + \ell}: k \in \{ 0,\dots, u-1 \}, \ell \in \{ 0,1,2,3,t-4, t-3,t-2,t-1 \}\}$,  and $b^{(j,k,\ell)}, a^{(j,k,\ell)}_0,\dots, a^{(j,k,\ell)}_{8u -1 } \in \mathbb{R}$ such that 
			\begin{equation}
				\label{eq:bound_on_c}
				| b^{(j,k,\ell)} |,  \max_{i = 0,\dots, 8u-1} | a^{(j,k,\ell)}_i | \leq  12\,t^2 R_m(X)  (R_c(X))^3,
			\end{equation}
			where $R_m(X)$ and  $R_c(X)$ are as defined in Lemma~\ref{lem:representation_two_hidden_layer}. We hence obtain
			\begin{align}
				h^+_{0}(x) =&\, \sum_{k = 0}^{u - 1 } y^+_{kt}   f_{k,0}^1(x) \label{eqline:application_component_1} \\
				=&\,  \sum_{k = 0}^{u - 1 } y^+_{kt}   \biggl(b^{(1,k,0)} + \sum_{i = 0}^{8u -1 } a^{(1,k,0)}_i \rho(x - z_i)\biggr) \label{eqline:application_decompose_f_proof}\\
				=&\, d_0^+ + \sum_{i=0}^{8u - 1} c_{0,i}^+ \rho(x - z_i), \, x \in \mathbb{R}, \label{eqline:define_d_1_c}
			\end{align}
			where in \eqref{eqline:define_d_1_c} we set $d_0^+ := \sum_{k = 0}^{u-1} y^+_{kt}\, b^{(1,k,0)}$ and $c_{0,i}^+ :=  \sum_{k =0}^{u-1} y^+_{kt}\, a^{(1,k,0)}_i$, for $ i = 0,\dots, 8u - 1$. Moreover, we note that $| d^+_0 | \leq \sum_{k = 0}^{u-1} |y^+_{kt}| | b^{(1,k,0)}| \leq 12 ut^2 R_m(X)  (R_c(X))^3 E $ and $| c^+_{0, i} |  =  \sum_{k = 0}^{u-1} | y^+_{kt}  | | a^{(1,k,0)}_i| \leq 12 ut^2 R_m(X)  (R_c(X))^3 E$, $i = 0, \dots, 8u -1$, where we used \eqref{eqline:bound_on_y_i} and \eqref{eq:bound_on_c}. Analogously, 
   %and together with \eqref{eqline:application_component_1}-\eqref{eqline:define_d_1_c}, 
   one can show that, for $\ell = 0,\dots, 3 t - 1$, 
			\begin{align}
				h^+_{\ell}(x) = d^+_\ell + \sum_{i=0}^{8u -1 } c^+_{\ell, i} \rho(x - z_i)\label{eqline:hplus}\\
				h^-_{\ell}(x) = d^-_\ell + \sum_{i=0}^{8u - 1} c^-_{\ell, i} \rho(x - z_i),\label{eqline:hminus}
			\end{align}
			with $d^+_\ell,d^-_\ell, c^+_{\ell, 0},\dots, c^+_{\ell, 8u-1}, c^-_{\ell, 0},\dots, c^-_{\ell, 8u-1} \in \mathbb{R}$  of absolute values not exceeding the term $12 ut^2 R_m(X)  (R_c(X))^3 E$. We recall that $H^+  = \sum_{\ell =0}^{3uv -1 } \rho \circ h_\ell^+$ and $ H^- =  \sum_{\ell =0 }^{3uv -1 } \rho \circ h_\ell^-$, and apply Lemma~\ref{lem:realization_piecewise_linear_functions} to $H^+$ and $H^-$, individually, with $(u,s, r)$ replaced by $( u,3v, 8u )$ and $T = 12 ut^2 R_m(X) (R_c(X))^3 E$, to get 
			% eq
			% with $m$ chosen as in this proof, $D = 3n + 1$, $F = 7m + 1$, $\lceil D \slash m\rceil = \lceil (3n + 1) \slash m\rceil = \lceil (3ml + 1) \slash m\rceil = 3l + 1$, yields
			\begin{align*}
				H^+ \in&\, \mathcal{R} (9u + 1, 3v  + 2, \max \{ 1, 12 ut^2 R_m(X) (R_c(X))^3 E \}    ),\\
				H^- \in&\, \mathcal{R} (9u + 1, 3v  + 2, \max \{ 1, 12 ut^2 R_m(X) (R_c(X))^3 E \}    ).
			\end{align*}
			We conclude the proof for the special case $M = u^2v$ and $uv \geq 8$ by noting that 
			\begin{align}
				g = H^+  - H^-  \in&\, \mathcal{R} \bigl( 18u + 2, 3v+3, \max \{ 1, 12 ut^2 R_m(X) (R_c(X))^3 E \} \bigr) \label{eq:g_h_h_1}\\
				\subseteq &\,\mathcal{R} \bigl( 18u + 2, 3v+3, \max \{ 1, 12 M^2 R_m(X) (R_c(X))^3 E \} \bigr),\label{eq:g_h_h_2}
			\end{align}
			where \eqref{eq:g_h_h_1} follows from Lemma~\ref{lem:algebra_on_ReLU_networks}, and \eqref{eq:g_h_h_2} is by $ut^2 = u^3 v^2 \leq M^2$.

			We proceed to the case of general $u,v,M \in \mathbb{N}$, with $M \geq 3$ and $u^2 v \geq M$. The proof will be effected by reducing to the special case just established. To this end, we shall find $\hat{u},\hat{v}$ and a strictly increasing sequence $( \hat{x}_{i} )_{i = 0}^{\hat{u}^2\hat{v}-1}$ taking values in $[0,1]$ such that $( x_i )_{i = 0}^{M-1} \subseteq ( \hat{x}_{i} )_{i = 0}^{\hat{u}^2\hat{v}-1}$ and $g \in \Sigma ( ( \hat{x}_{i} )_{i = 0}^{\hat{u}^2\hat{v}-1}, E )$.
			% Specifically, let 
			% Let $\tilde{v}$ be the smallest integer such that $M \leq u^2 \tilde{v}$. 
			% We have $234$
			% and let $\tilde{m}$ be the smallest natural number such that $ N \leq \tilde{m}^2 \tilde{l}$. Since $M \leq m^2\elll$, the minimality of $\tilde{l}$ and $\tilde{m}$ implies $\tilde{l} \leq l$ and $\tilde{m} \leq m$. 
			Concretely, let $(\hat{u},\hat{v})$ be a solution\footnote{If there are multiple solutions, we simply pick any one of them.} to the following constrained optimization problem
			\begin{equation}
			\label{eq:optimization_problem}
				(\hat{u},\hat{v}) = \argmin_{(e,f)\, \in\, \mathbb{G} \cap \mathbb{H}} (e + f  ),
			\end{equation}
			with $\mathbb{G} := \{ ( e,f) \in \mathbb{N}^2: e \leq u, 8 \leq f \leq 8v  \}$ and $\mathbb{H} = \{ ( e,f ) \in \mathbb{N}^{2}: e^2f \geq M  \}$, and set $\hat{M} = \hat{u}^2 \hat{v}$. Here, the choice of $\mathbb{G}$ ensures that the condition $\hat{u}\hat{v} \geq 8$ is met while $\hat{v}$ is not much greater than $v$. Note that we enforce $\hat{u} \hat{v} \geq 8$ so as to, indeed, be able to apply the result for the special case dealt with above. Since $( u,8v ) \in \mathbb{G} \cap \mathbb{H}$, the feasible region $\mathbb{G} \cap \mathbb{H}$ is non-empty, which together with $\mathbb{G} \cap \mathbb{H}$ being finite guarantees the existence of a minimizer to \eqref{eq:optimization_problem}. Now, let $\hat{X} = (\hat{x}_i )_{i = 0}^{\hat{M}-1}$ with $\hat{x}_i = x_i$, for $i = 0,\dots, M-2$, and $\hat{x}_{i} = x_{M-2} + ( i - M +2 )\cdot \frac{x_{M-1} - x_{M-2}}{\hat{M} - M + 1} $, for $i = M - 1, \dots, \hat{M}-1$. As $\hat{x}_{\hat{M} -1} = x_{M-1}$, it follows that $X \subseteq \hat{X}$. Intuitively, $\hat{X}$ is obtained by adding $\hat{M} - M$ points in the interval $[x_{M-2}, x_{M-1}]$ to $X$ chosen such that $[x_{M-2}, x_{M-1}]$ is partitioned into $\hat{M} - M + 1$ subintervals of equal length. It then follows that $g \in\, \Sigma ( X, E ) \subseteq \Sigma ( \hat{X}, E )$. Using \eqref{eq:g_h_h_1}-\eqref{eq:g_h_h_2} with $( X,u,v )$ replaced by $( \hat{X}, \hat{u},\hat{v} )$, we get 
			\begin{align}
				g \in&\,  \mathcal{R} \bigl( 18 \hat{u} + 2, 3\hat{v}+3, \max \{ 1, 12 \hat{M}^2 R_m ( \hat{X} ) (R_c( \hat{X} ))^3 E \} \bigr). \label{eqline:apply_special_case}
			\end{align}
			
			Next, we note that
			\begin{align*}
				R_m ( \hat{X} ) =&\, \max_{i =1,\dots, \hat{M}} ( \hat{x}_{i} - \hat{x}_{i -1} )^{-1} \\
				=&\, \max \biggl(\max_{i =1,\dots, M-2} ( \hat{x}_{i} - \hat{x}_{i -1} )^{-1}, \max_{i =M-1,\dots, \hat{M}-1} ( \hat{x}_{i} - \hat{x}_{i -1} )^{-1}  \biggr)\\
				=&\, \max \biggl(\max_{i =1,\dots, M-2} ( x_{i} - x_{i -1} )^{-1}, \biggl(\frac{x_{M-1} - x_{M-2}}{\hat{M} - M + 1}\biggr)^{-1}  \biggr) \\
				\leq &\, ( \hat{M} - M + 1 )\max_{i =1,\dots, M-1}  (x_{i} - x_{i -1})^{-1}\\
				= &\,  ( \hat{M} - M + 1 ) R_m (X)
			\end{align*}
			and 
			\begin{align*}
				R_c(\hat{X}) =&\, \frac{\max_{i =1,\dots, \hat{M}-1}  (\hat{x}_{i} - \hat{x}_{i -1}) }{\min_{i =1,\dots, \hat{M}-1}  (\hat{x}_{i} - \hat{x}_{i -1}) }\\
				=&\, \frac{(\min_{i =1,\dots, \hat{M}-1}  (\hat{x}_{i} - \hat{x}_{i -1})^{-1})^{-1} }{(\max_{i =1,\dots, \hat{M}-1}  (\hat{x}_{i} - \hat{x}_{i -1})^{-1} )^{-1}}\\
				=&\, \frac{\max_{i =1,\dots, \hat{M}-1}  (\hat{x}_{i} - \hat{x}_{i -1})^{-1} }{\min_{i =1,\dots, \hat{M}-1}  (\hat{x}_{i} - \hat{x}_{i -1})^{-1} }\\
				\leq&\, \frac{( \hat{M} - M + 1 ) \max_{i =1,\dots, M-1}  (x_{i} - x_{i -1})^{-1} }{\min_{i =1,\dots, M-1}  (x_{i} - x_{i -1})^{-1} }\\
				=&\, \frac{( \hat{M} - M + 1 ) \max_{i =1,\dots, M-1}  (x_{i} - x_{i -1}) }{\min_{i =1,\dots, M-1}  (x_{i} - x_{i -1})}\\
				=&\, ( \hat{M} - M + 1 ) R_c (X),
			\end{align*}
			which, when used in \eqref{eqline:apply_special_case}, yields
			\begin{equation}
				g \subseteq   \mathcal{R} ( 18\hat{u} + 2, 3\hat{v} + 3, \max \{ 1, 12 \hat{M}^2 (( \hat{M} - M + 1 ) R_m ( X ) ) (( \hat{M} - M + 1 ) R_c( X ) )^3 E \} ). \label{eqline:apply_relation_X_hatX}
			\end{equation}
			Finally, owing to $\hat{u} \leq u$, $\hat{v} \leq 8 v$, and $\hat{M} - M  + 1 \leq \hat{M}$, it follows from \eqref{eqline:apply_relation_X_hatX} that
			\begin{equation}
				g \subseteq   \mathcal{R} ( 18u + 2, 24v + 3, \max \{ 1, 12 \hat{M}^6 R_m ( X )  (R_c( X ))^3 E \} ).  \label{eqline:apply_relation_X_hatX_1}
			\end{equation}
			% \begin{align*}
			% 	g \in&\, \Sigma ( \hat{X}, E ), \label{eqline:connect_hat_0} \\
			% 	\hat{u} \leq&\, u,\\
			% 	\hat{v} \leq&\, 8v,
			% \end{align*} 
			% We then get 
			% \begin{align}
			% 	g \in&\,  \mathcal{R} \bigl( 18 \hat{u} + 2, 3\hat{v}+3, \max \{ 1, 12 \hat{M}^2 R_m ( \hat{X} ) R_c^3( \hat{X} ) E \} \bigr)\\
			% 	\subseteq&\,  \mathcal{R} ( 18u + 2, 24v + 3, \max \{ 1, 12 ( 4M )^2 (4 M R_m ( X ) ) (4 M R_c( X ) )^3 E \} ) \\
			% 	\subseteq &\, \mathcal{R} ( 20u, 30v, \max \{ 1, C_k M^6 R_m ( X ) (R_c ( X ))^3 E \} ) \label{eqline:set_C_k} \\
			% 	\subseteq &\,\mathcal{R} ( 20u, 30v, \max \{ 1, C_k M^6 (R_m ( X ))^4 E \} )
			% \end{align}
			We next want to get rid of $\hat{M}$ in \eqref{eqline:apply_relation_X_hatX_1}. This will be accomplished by establishing a quantitative relation between $\hat{M}$ and $M$ organized into  three different cases, which, taken together, are exhaustive in terms of the parameter range under consideration.
			\begin{enumerate}
				\item If $\hat{u} \geq 2$, it must hold that $(\hat{u}-1,\hat{v}) \not \in \mathbb{H}$, i.e., $( \hat{u} -1 )^2 \hat{v} < M$,  otherwise $(\hat{u},\hat{v})$ can not be a solution to \eqref{eq:optimization_problem} as the objective function takes on  a smaller value at $(\hat{u} - 1,\hat{v}) \in \mathbb{G} \cap \mathbb{H}$.  We therefore have $\hat{M} = \hat{u} ^2 \hat{v} \leq 4 ( \hat{u} -1 )^2 \hat{v} < 4M$. 
				\item If $\hat{v} \geq 9$, we must have $(\hat{u},\hat{v} - 1) \not \in \mathbb{H}$, i.e., $ \hat{u}^2 ( \hat{v} - 1 ) < M$, as, again, otherwise $(\hat{u},\hat{v})$ can not be a solution to \eqref{eq:optimization_problem}. We then have $\hat{M} = \hat{u} ^2 \hat{v} \leq 2  \hat{u}^2 ( \hat{v}  - 1 )< 2M$.
				\item If $\hat{u} = 1$ and $\hat{v} = 8$, we obtain $\hat{M} = \hat{u}^2 \hat{v} = 8 \leq 4 M$ owing to $M \geq 3$.
			\end{enumerate}
			In all cases, one gets $\hat{M} \leq 4M$, so that
			% \begin{equation}
			% 	\label{eq:M_and_hatM}
			% 	\hat{M} = \hat{u}^2 \hat{v} \leq  4M.
			% \end{equation}
			\eqref{eqline:apply_relation_X_hatX_1} yields 
			\begin{align*}
				g \in &\,  \mathcal{R} ( 18u + 2, 24v + 3, \max \{ 1, 12\cdot 4^6 M^6 R_m ( X )  (R_c( X )) ^3 E \} ).
			\end{align*}
			The proof is concluded by setting $C_k = 12\cdot 4^6 \leq 10^5$ and noting that  
			\begin{equation*}
				R_c(X) =\frac{\max_{i =1,\dots, M-1}  (x_{i} - x_{i -1}) }{\min_{i =1,\dots, M-1}  (x_{i} - x_{i -1}) } \leq \max_{i =1,\dots, M} ( x_{i} - x_{i -1} )^{-1}   =  R_m(X), 
			\end{equation*}
			where the inequality follows from $( x_i )_{i = 0}^{M-1} \in [0,1]$ and hence  $\min_{i =1,\dots, M-1}  (x_{i} - x_{i -1})  \leq 1$.
			% We have 
			% \begin{align}
			% 	R_m ( \hat{X} ) \leq&\, ( \hat{M} - M + 1 ) R_m (X) \leq 4 M R_m ( X ) \label{eqline:connect_hat_1}\\
			% 	R_c ( \hat{X} ) \leq&\, ( \hat{M} - M + 1 ) R_c (X) \leq 4M R_c ( X )\label{eqline:connect_hat_2}
			% \end{align}
			% where \eqref{eqline:connect_hat_1}-\eqref{eqline:connect_hat_2} follows from the definition of $\hat{X}$ and $\hat{M} \leq 4 M$. Application of \eqref{eq:g_h_h_1}-\eqref{eq:g_h_h_1} with $( X, u ,v  )$ replaced by $( \hat{X},\hat{u},\hat{v} )$ yields 
			% \begin{align}
			% 	g \in&\,  \mathcal{R} \bigl( 18 \hat{u} + 2, 3\hat{v}+3, \max \{ 1, 12 \hat{M}^2 R_m ( \hat{X} ) R_c^3( \hat{X} ) E \} \bigr)\\
			% 	\subseteq&\,  \mathcal{R} ( 18u + 2, 24v + 3, \max \{ 1, 12 ( 4M )^2 (4 M R_m ( X ) ) (4 M R_c( X ) )^3 E \} ) \\
			% 	\subseteq &\, \mathcal{R} ( 20u, 30v, \max \{ 1, C_k M^6 R_m ( X ) (R_c ( X ))^3 E \} ) \label{eqline:set_C_k} \\
			% 	\subseteq &\,\mathcal{R} ( 20u, 30v, \max \{ 1, C_k M^6 (R_m ( X ))^4 E \} )
			% \end{align}
			% where in \eqref{eqline:set_C_k} 
		\end{proof}

		\subsection{Proof of Lemma~\ref{lem:representation_two_hidden_layer}} % (fold)
		\label{sub:proof_of_lemma_lem:representation_two_hidden_layer}
			% Let $\left\{ \gamma_i \right\}_{i =1}^N$ be the basis of $\Sigma \left( X, \infty \right)$, defined as before. With a similar idea, we will first represent each $\gamma_i$ by ReLU networks and then represent every function in $\Sigma \left( X, E \right)$ according to its basis representation in \eqref{eq:basis_representation}.  \todo{We group the basis and represent each group.} 
			The proof is constructive in the sense of explicitly specifying $f^1_{k,\ell},f^2_{k,\ell},f^3_{k,\ell}$, for $ ( k,\ell ) \in \mathcal{I}$. We consider the cases $( k,\ell ) \in \mathcal{I}^1$ with  
			% \begin{align*}
			% 	\mathcal{I}^1 :=& \{ ( k,i ) \in \mathcal{I}: i \in \{0,1,2,v -2, v-1\}\},\\
			% 	\mathcal{I}^2 :=& \{ ( k,i ) \in \mathcal{I}: i \in \{3,\dots, v- 3\} \},
			% \end{align*}
			\begin{equation*}
				\mathcal{I}^1 := \{ ( k,\ell ) \in \mathcal{I}: \ell \in \{0,1,2,t-3, t -2, t-1\}\},
			\end{equation*}
			and $( k,\ell ) \in \mathcal{I}^2$ with 
			\begin{equation*}
				\mathcal{I}^2 := \{ ( k,\ell ) \in \mathcal{I}: \ell \in \{3,\dots, t- 4\} \},
			\end{equation*}
			separately as the corresponding constructions are different. 
            The sets $\mathcal{I}^1$ and $\mathcal{I}^2$ are nonempty and, owing to $t \geq 8$, disjoint. Further, $\mathcal{I}$ is the union of $\mathcal{I}^1$ and $\mathcal{I}^2$. %upon noting that $t \geq 8$. 
            The verification of the three properties to be established is conducted right after the respective constructions.% have been presented.

			Fix $( k,\ell ) \in \mathcal{I}^1$. Let $f_{k,\ell}^1 = \gamma_{kt + \ell}$, and let $f_{k,\ell}^2 (x) = f_{k,\ell}^3 (x) = 0$, $x \in \mathbb{R}$. Regarding Property 1, we have 
			\begin{align}
				\gamma_{kt+\ell}(x) =&\, \rho(\gamma_{kt+\ell}(x)) \label{eqline:positivity_of_gamma}\\
				=&\, \rho(f_{k,\ell}^1(x)) - \rho(f_{k,\ell}^2(x)) + \rho(f_{k,\ell}^3(x)), \quad  x \in \mathbb{R},
			\end{align}
			where in \eqref{eqline:positivity_of_gamma} we used $\gamma_{kt+\ell}(x) \geq 0$, for $x \in \mathbb{R}$.

			Property 2 is satisfied, for $j = 2,3$, as  $f_{k,\ell}^2 (x) = f_{k,\ell}^3(x) = 0 = \sum_{i = 0}^{8u - 1} 0\cdot \rho(x - z_i)$, for $x \in \mathbb{R}$. For $j = 1$, thanks to \eqref{eq:represented}, we note that
			\begin{enumerate}
				\item $f^1_{k,\ell} (x) = - \frac{1}{x_{1} - x_0}\rho(x - x_0) + \frac{1}{x_1 - x_0}\rho(x - x_1) + 1, \,\,x \in \mathbb{R}$,  if $k = \ell = 0$,
				\item $f^1_{k,\ell}(x) = \frac{1}{x_{ut-1} - x_{ut-2}}\rho(x - x_{ut-2}) - \frac{1}{x_{ut - 1} - x_{ut-2}}\rho(x - x_{ut-1}),\,\,x \in \mathbb{R}$, if $k = u- 1$, $\ell = t - 1$,
				\item $f^1_{k,\ell} ( x ) = \frac{1}{x_{kt +\ell} - x_{kt +\ell-1}} \rho(x - x_{kt +\ell-1}) - \bigl( \frac{1}{x_{kt +\ell} - x_{kt +\ell-1}} + \frac{1}{x_{kt +\ell+1} - x_{kt +\ell}} \bigr) \rho(x - x_{kt +\ell})+  \frac{1}{x_{kt +\ell+1} - x_{kt +\ell}} \cdot \allowbreak  \rho(x - x_{kt +\ell+1}), \, \,x \in \mathbb{R}$, if $( k,\ell ) \in \mathcal{I}^1 \backslash \{ (0,0), (u-1,t -1) \}$.
			\end{enumerate}
			In each of these three cases, upon recalling the definition of $( z_i )_{i =0}^{8u -1 }$ in \eqref{eq:definitio_z_sequence}, $f^1_{k,\ell}$ can be written as 
			\begin{equation*}
				f^1_{k,\ell}(x) = b + \sum_{\ell = 0}^{8u -1 } a_\ell \rho(x - z_i),\, x \in \mathbb{R},
			\end{equation*}
			% \begin{align*}
			% 	&- \frac{1}{x_1 - x_0}\rho(x - x_0) + \frac{1}{x_1 - x_0}\rho(x - x_1) + 1, \text{ if } k = n = 0,\\
			% 	& \frac{1}{x_{M-1} - x_{M-2}}\rho(x - x_{M-2}) - \frac{1}{x_{M} - x_{M-2}}\rho(x - x_{M-1})
			% \end{align*}
			% \begin{equation*}
			% 	f^1_{k,v}^1(x) = \sum_{\ell = \max \{ 0, kt+i -1 \}}^{\min \{ M-1, kt+i +1 \}} c_\ell \rho(x - x_\ell)
			% \end{equation*}
			for some $b, a_0,\dots, a_{8u -1 }  \in \mathbb{R}$ such that $| b | , |a_0|,\dots, |a_{8u -1}| \leq \max \{ 1, 2\max_{i =1, \dots, M-1} ( x_{i} - x_{i -1} )^{-1}\}\allowbreak  = \max \{ 1, 2R_m(X) \} < 12t^2 R_m(X)  (R_c(X))^3 $, where the inequality follows from $t \geq 1$, $ R_c(X) \geq 1$, and $ R_m(X) \geq 1$. Since the choice of $( k,\ell ) \in \mathcal{I}^1$ was arbitrary, we have given the construction of $f^1_{k,\ell},f^2_{k,\ell},f^3_{k,\ell}$ and established Properties 1 and 2 for all $( k,\ell ) \in \mathcal{I}^1$. 

			We proceed to verify Property 3 for $( k, \ell ) \in \mathcal{I}^1$.
   %, namely the elements of $\{ f^j_{k,\ell} \}_{k =0}^{u-1}$ having pairwise disjoint support for $(\ell, j) \in \{0,1,2,t - 3, t -2, t-1\} \times \{ 1, 2,3 %\}$.  
   To this end, we first note that for $( \ell, j ) \in \{0,1,2,t- 3,t -2, t-1\} \times \{ 2,3 \}$, the functions $\{ f^j_{k,\ell} \}_{k =0}^{u-1}$ are all identically equal to zero and hence trivially of pairwise disjoint support. %functions, by definition, empty, and hence pairwise disjoint supports. 
   For  $( \ell, j ) \in \{0,1,2,t-3, t -2, t-1\} \times \{1\}$, $\{ f^1_{k,\ell} \}_{k =0}^{u-1} = \{ \gamma_{kt +\ell} \}_{k =0}^{u-1}$ and the pairwise disjoint support is a consequence of %have pairwise disjoint support by definition of $\{ \gamma_i \}_{i = 0}^{M-1}$ as given in 
   \eqref{eq:basis_1}-\eqref{eq:basis_3}.

			% See Fig~\ref{fig:basis_0} for 
			% \begin{figure}[H]
			% 	\centering
			% 	\includegraphics[width=0.5\textwidth]{images/basis_1.jpg}
			% 	\caption{Illustration of $\gamma_{kn +i}$ for $i = 0,1,2,n -2, n-1$.}
			% 	\label{fig:basis_0}
			% \end{figure}

			% \begin{figure}[H]
			% 	\centering
			% 	\includegraphics[width=0.5\textwidth]{images/basis_2.jpg}
			% 	\caption{$f_{k,i}^j$, $j = 1,2,3$}
			% 	\label{fig:basis_1}
			% \end{figure} 	
			% \begin{figure}[H]
			%  	\centering
			%  	\includegraphics[width=0.5\textwidth]{images/basis_3.jpg}
			%  	\caption{Represent $\gamma_{kn +i}$ for $i = 3,\dots, n-2 $.}
			%  	\label{fig:represent_gamma_kn}
			% \end{figure} 

			% \vspace{20em}
			% \vspace{20em}

			% \begin{equation*}
			% 	\mathcal{I}^2 := \{ ( k,i ) \in \mathcal{I}: i \in \{3,\dots, v- 3\} \},
			% \end{equation*}

			We continue by constructing $f^1_{k,\ell},f^2_{k,\ell},f^3_{k,\ell}$ and verifying Properties $1$-$3$ for $( k,\ell ) \in \mathcal{I}^2$. Fix $( k,\ell ) \in \mathcal{I}^2$. For $j = 1,2,3$, let
			\begin{equation}
			\label{eq:piecewise_linearity_f}
			 	f_{k,\ell}^j \in \Sigma ( (x_{kt},x_{kt+1},x_{kt + t- 2},x_{kt + t- 1} ), \infty ) 
			 \end{equation} 
			be such that 
			\begin{align}
				f_{k,\ell}^{j} ( x_{kt} ) =&\, 0, \label{eqline:choice_y_i_0}  \\
				f_{k,\ell}^{j} ( x_{kt +\ell +j - 2 } ) =&\,  0, \label{eqline:choice_y_i_1}\\
				f_{k,\ell}^{j} ( x_{ kt + t - 2} ) = &\, y_{k,\ell}^{j}, \label{eqline:choice_y_i_2}\\
				f_{k,\ell}^{j} ( x_{kt + t -1} ) =&\, 0, \label{eqline:choice_y_i_3}
			\end{align}
			for a
			\begin{equation}
			\label{eq:assumption_y}
				y_{k,\ell}^{j} > 0
			\end{equation}
			to be specified later. We refer to Figure~\ref{fig:fjki} for an illustration of the so-defined $f_{k,\ell}^{j}$, $j \in \{ 1,2,3 \}$. Next, we remark that $f_{k,\ell}^{j}$, for $j=1,2,3$, is uniquely determined by $y_{k,\ell}^j$. To see this, first note that 
      $f_{k,\ell}^{j}$ is affine on the interval $[x_{kt+1},x_{kt + t - 2}]$, by definition, and therefore the values of $f_{k,\ell}^{j}$ at the points $ x_{kt +\ell +j - 2 } $ and $ x_{ kt + t - 2}$, both of which are contained in the interval, determine the values of $f_{k,\ell}^{j}$ at the end points $x_{kt+1}, x_{kt + t - 2}$ of the interval. 
            With \eqref{eqline:choice_y_i_0} and \eqref{eqline:choice_y_i_3}, owing to \eqref{eq:piecewise_linearity_f}, this determines 
            %the values of $f_{k,\ell}^{j}$ 
            %at the points $\{x_{kt}, x_{kt +\ell +j - 2 }, x_{ kt + t - 2}, x_{kt + t -1}\}$  will uniquely determine the values of $f_{k,\ell}^{j}$ 
            %at all its breakpoints $\{x_{kt}, x_{kt +1 }, x_{ kt + t - 2}, x_{kt + t -1}\}$, which, in turn,  uniquely determines 
            $f_{k,\ell}^{j}$ itself.  In preparation for the choice of $y_{k,\ell}^{j}$, we need to characterize $f_{k,\ell}^{j}$ and $\rho\circ f_{k,\ell}^{j}$ in more detail. Fix $j \in \{ 1,2,3 \}$.  We first investigate the sign of $f_{k,\ell}^{j}$ on $\mathbb{R}$. To this end, we use the fact that $f_{k,\ell}^j$ is affine on the interval $[x_{kt+1},x_{kt + t - 2}]$, by definition, and therefore satisfies the interpolation formula
			\begin{align}
				f^j_{k,\ell} (x) =&\, f^j_{k,\ell}( x_{kt +\ell +j - 2 } ) +  ( x - x_{kt +\ell +j - 2 } )\,\frac{f^j_{k,\ell} ( x_{kt + t -2 } ) - f^j_{k,\ell}( x_{kt +\ell +j - 2 } )}{x_{kt + t -2 } - x_{kt +\ell +j - 2 }} \label{eqline:expression_f_1}\\
				=&\, y^j_{k,\ell}\, \frac{ x - x_{kt +\ell +j - 2 }}{x_{kt + t - 2} - x_{kt +\ell +j - 2 }}, \quad  x \in [x_{kt+1},x_{kt + t - 2}] \label{eqline:expression_f_2},
			\end{align}
		where we used \eqref{eqline:choice_y_i_1} and \eqref{eqline:choice_y_i_2}. In particular, we have 
		\begin{equation}
			\label{eqline:choice_y_i_4}
			f_{k,\ell}^{j} ( x_{kt + 1} ) = \, y_{k,\ell}^{j} \frac{x_{kt + 1} - x_{kt +\ell+j - 2}}{x_{ kt +t  - 2} - x_{kt +\ell  +j - 2}} < 0,
		\end{equation}
		where we used that $\ell + j -2 \in \{ 2, \dots, t- 3 \}$ and therefore $kt + 1 < kt +\ell+j - 2$ and $kt +t  - 2 > kt +\ell  +j - 2$.
			% \vspace{4em}
			% $f_{k,\ell}^{j}$ is an affine mapping on $[x_{kt+1},x_{kt + t -2}]$, it follows from the interpolation formulat that
			% \begin{align}
			% 	f_{k,\ell}^{j} ( x_{kt + t - 2} ) >& 0 \label{eqline:choice_y_i_3}\\
			% 	f_{k,\ell}^{j} ( x_{kt + 1} ) <&  0\label{eqline:choice_y_i_4}
			% \end{align}
			% thanks to \eqref{eqline:choice_y_i_1}-\eqref{eqline:choice_y_i_2} and $y_{k,\ell}^{j} > 0$. 
			% \vspace{4em}
			% \begin{equation*}
			% 	f_{k,\ell}^{j} ( x_{kt + 1} ) =& \, y_{k,\ell}^{j} \frac{x_{kt + 1} - x_{kt +\ell+j - 2}}{x_{ kt +t  - 2} - x_{kt +\ell  +j - 2}},\label{eqline:choice_y_i_2}
			% \end{equation*}
			% \vspace{4em}
			% Moreover, since $f_{k,\ell}^{j}$ is an affine mapping on $[x_{kt+1},x_{kt + t -2}]$, it satisfies the interpolation formula that, for $x \in [x_{kt + 1},x_{kt + t -2}]$, 
			% \begin{equation*}
			% 	f_{k,\ell}^{j} ( x  ) =   f_{k,\ell}^{j} ( x_{kt + t - 2})  \frac{x - x_{kt + 1}}{x_{kt + t - 2} - x_{kt + 1} } +  f_{k,\ell}^{j} ( x_{kt+1}) \frac{x - x_{kt + t -2}}{ x_{kt+1} - x_{kt + t -2} },
			% \end{equation*}
			% which together with \eqref{eqline:choice_y_i_1}-\eqref{eqline:choice_y_i_2} implies 
			% \begin{equation}
			% % \label{eq:expression_f_0}
			% 	f_{k,\ell}^{j} ( x_{kt +\ell +j - 2 } ) =  0.
			% \end{equation}
			% % and $f_{k,i}^{j} (x ) \leq  0$ for $x \in [x_{kn+1}, x_{kn +i +j - 2 }]$. 
			The values of $f_{k,\ell}^{j}$ at the points $\{x_{kt}, x_{kt + 1}, x_{kt +\ell +j - 2 }, x_{ kt + t - 2}, x_{kt + t - 1}\}$, as specified by \eqref{eqline:choice_y_i_0}, \eqref{eqline:choice_y_i_4}, \eqref{eqline:choice_y_i_1}, \eqref{eqline:choice_y_i_2}, \eqref{eqline:choice_y_i_3}, combined with \eqref{eq:piecewise_linearity_f}, allow us to conclude that %now determine $f_{k,\ell}^{j}$ according to
			\begin{equation}
			\label{eq:sgn_of_f}
			f_{k,\ell}^{j} (x) \left\{
			\begin{aligned}
				& = 0, && x \in (-\infty, -x_{kt}] \cup \{  x_{kt +\ell +j - 2 }  \} \cup [x_{kt + t -1}, \infty),\\
				&> 0, && x \in (x_{kt +  \ell +j - 2 }, x_{kt + t - 1}),  \\
				&< 0, && x \in (x_{kt }, x_{kt +\ell +j - 2 }), 
			\end{aligned}
			\right.
			\end{equation}
			which, in turn, implies 
			\begin{equation}
			\label{eq:expression_of_f}
			\rho(f_{k,\ell}^{j} (x)) =  \left\{
			\begin{aligned}
				&f_{k,\ell}^{j} (x)   , && x \in (x_{kt +\ell +j - 2 }, x_{kt + t -1})  ,\\
				&0, && x \in (-\infty, x_{kt +\ell +j - 2 }] \cup [x_{kt + t -1}, \infty).
			\end{aligned}
			\right.
			\end{equation}

			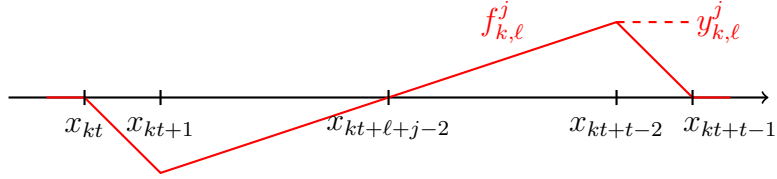
\begin{figure}[H]
			\centering
			\begin{tikzpicture}
			    \draw[thick, ->] (1,0) -- (11,0) node[right]{};

			    \draw[red, thick] (1.5,0)-- (2,0) -- (3,-1) -- (6,0) -- (9,1) -- (10,0) -- (10.5,0);
			    \draw[dashed, thick, red] (9,1) -- (10,1) node[right=-1mm] {$y_{k,\ell}^j$};
			    \node[red] at (7.5,1) {$f_{k,\ell}^j$};

			    \foreach \x in {2,3,6,9,10} {
			    	\draw[thick] (\x,-0.1) -- (\x,0.1);
			    }
			    
			    % % Dots

			     % Labels 
			    \node at (2,-0.4) {$x_{kt}$};
			    \node at (3,-0.4) {$x_{kt+1}$};
			    \node at (6,-0.4) {$x_{kt + \ell + j - 2}$};
			    % \node at (7,-0.4) {$x_{i+1
			    \node at (9,-0.4) {$x_{kt + t - 2}$};
			    \node at (10.5,-0.4) {$x_{kt + t - 1}$};
			    
			\end{tikzpicture}
			\caption{The function $f^j_{k,\ell}$.}
			\label{fig:fjki}	
		\end{figure}

		\begin{figure}[H]
			\centering
			\begin{tikzpicture}

			    \draw[thick, ->] (1,0) -- (11,0) node[right]{};

			    % Red Line
			    \draw[red, thick] (1.5,0) -- (2,0) -- (3,0) -- (6,0) -- (9,1) -- (10,0) -- (10.5,0);
			    \draw[dashed, thick, red] (9,1) -- (10,1) node[right=-1mm] {$y_{k,\ell}^j$};
			    \node[red] at (7.5,1) {$\rho \circ f_{k,\ell}^j$};
			    
			    \foreach \x in {2,3,6,9,10} {
			    	\draw[thick] (\x,-0.1) -- (\x,0.1);
			    }

			     % Labels 
			    \node at (2,-0.4) {$x_{kt}$};
			    \node at (3,-0.4) {$x_{kt+1}$};
			    \node at (6,-0.4) {$x_{kt +\ell + j - 2}$};
			    \node at (9,-0.4) {$x_{kt + t - 2}$};
			    \node at (10.5,-0.4) {$x_{kt + t - 1}$};
			\end{tikzpicture}
			\caption{The function $\rho\circ f^j_{k,\ell}$.}
			\label{fig:rhofjki}	
		\end{figure}
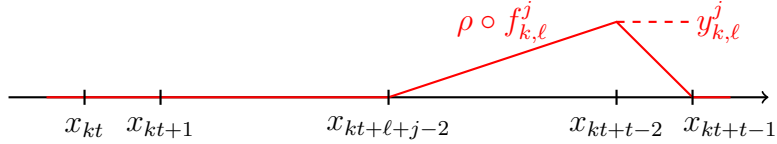

		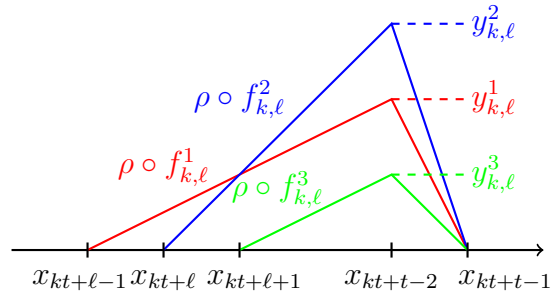
\begin{figure}[H]
			\centering
			\begin{tikzpicture}
			    \draw[thick, ->] (4,0) -- (11,0) node[right]{};
			    % Red Line
			    \draw[red, thick](5,0) -- (9,2) -- (10,0);
			    \draw[dashed, thick, red] (9,2) -- (10,2) node[right=-1mm] {$y_{k,\ell}^1$};
			    \node[red] at (6,1.1) {$\rho \circ f_{k,\ell}^1$};
			    %Blue 
			    \draw[blue, thick](6,0) -- (9,3) -- (10,0);
			    \draw[dashed, thick, blue] (9,3) -- (10,3) node[right=-1mm] {$y_{k,\ell}^2$};
			    \node[blue] at (7,2) {$\rho \circ f_{k,\ell}^2$};
			    %Green
			    \draw[green, thick](7,0) -- (9,1) -- (10,0);
			    \draw[dashed, thick, green] (9,1) -- (10,1) node[right=-1mm] {$y_{k,\ell}^3$};
			    \node[green] at (7.5,0.8) {$\rho \circ f_{k,\ell}^3$};

			    \foreach \x in {5,6,7,9,10} {
			    	\draw[thick] (\x,-0.1) -- (\x,0.1);
			    }

			    \node at (4.9,-0.4) {$x_{kt  + \ell -1}$};
			    \node at (6,-0.4) {$x_{kt + \ell }$};
			    \node at (7.2,-0.4) {$x_{kt + \ell + 1 }$};
			    \node at (9,-0.4) {$x_{kt + t - 2}$};
			    \node at (10.5,-0.4) {$x_{kt + t - 1}$};

			\end{tikzpicture}
			\caption{The functions $\rho \circ f^j_{k,\ell}, j =1,2,3$.}
			\label{fig:fjki123}
		\end{figure}

		\begin{figure}[H]
			\centering
			\begin{tikzpicture}
			    \draw[thick, ->] (4,0) -- (11,0) node[right]{};
			    % Red Line
			    \draw[red, thick](4.5,0)  -- (5,0) -- (6,1) -- (7,0) -- (10.5,0);
			    \draw[dashed, thick, red] (6,1) -- (7,1) node[right=-1mm] {$1$};
			    \node[red, right] at (6.5,0.5) {$\gamma_{kt + \ell} $};
			    \foreach \x in {5,6,7,9,10} {
			    	\draw[thick] (\x,-0.1) -- (\x,0.1);
			    }
			    \node at (4.9,-0.4) {$x_{kt  + \ell -1}$};
			    \node at (6,-0.4) {$x_{kt + \ell }$};
			    \node at (7.2,-0.4) {$x_{kt + \ell + 1 }$};
			    \node at (9,-0.4) {$x_{kt + t - 2}$};
			    \node at (10.5,-0.4) {$x_{kt + t - 1}$};
			\end{tikzpicture}
			\caption{$\gamma_{kt + \ell} = \rho\circ f_{k, \ell}^1 - \rho\circ f_{k, \ell}^2 + \rho\circ f_{k, \ell}^3$}
			\label{fig:sum_fjki123}
		\end{figure}
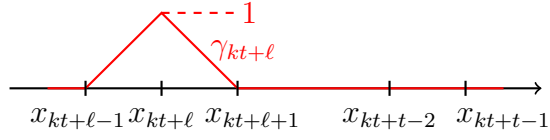

			We refer to Figure~\ref{fig:rhofjki} for an illustration of $\rho\circ f_{k,\ell}^{j}$. Owing to \eqref{eq:piecewise_linearity_f} and \eqref{eq:expression_of_f}, we have
			% According to \eqref{eq:expression_of_f},  $\rho\circ f_{k,i}^{j}$ equals $f_{k,i}^{j}$ and therefore an affine function on $[x_{kn +i +j - 2 }, x_{(k+1)n- 1}]$ thanks to  the affinity of $f_{k,i}^{j}$ on $[x_{kn+1},x_{(k+1)n- 1}]$, as assumed in \eqref{eq:piecewise_linearity_f},  with $[x_{kn +i +j - 2 }, x_{(k+1)n- 1}]$ being a subinterval of $[x_{kn+1},x_{(k+1)n- 1}]$. Similarly, $\rho\circ f_{k,i}^{j}$ is affine on $[ x_{(k+1)n- 1},x_{(k+1)n}]$. Moreover, $\rho\circ f_{k,i}^{j}(x) = 0$ for $x \in (\infty,x_{kn +i +j - 2 }] \cup  [x_{(k+1)n}, \infty)$. Taking together, 
			\begin{equation}
				\label{eq:rho_f}
				\rho \circ f_{k,\ell}^{j} \in \Sigma ( ( x_{kt +\ell +j - 2 }, x_{kt + t -2},x_{kt + t - 1} ), \infty ).
			\end{equation}
			As the choice of $j \in \{ 1,2,3 \}$ was arbitrary, we have constructed $f_{k,\ell}^j$, $ j = 1,2,3$, satisfying \eqref{eq:piecewise_linearity_f}-\eqref{eq:rho_f}.
			% \todo{?? implies ??? and  $\rho ( f_{k,i}^{j}) \in $}\todo{See Fig?? for the illustration of  }  \todo{another interpolation formula} We get 
			% \begin{equation*}
			% 	\rho \circ f_{k,i}^{j} \in \Sigma ( \{ x_{kn +i +j - 2 }, x_{(k+1)n- 1},x_{(k+1)n} \}, \infty )
			% \end{equation*}
			
   We proceed to determine $y_{k,\ell}^{j}$, $j = 1,2,3$, such that Property 1 is satisfied.
			As the functions $\gamma_{kt + \ell}$, $\rho\circ f_{k,\ell}^1$, $\rho \circ  f_{k,\ell}^2$, and $\rho \circ f_{k,\ell}^3$ are all bounded piecewise linear functions with breakpoints $\{ x_{kt+\ell -1}, x_{kt+\ell}, x_{kt+\ell+1}, x_{kt + t -2}, x_{kt + t -1}\}$,
			% $ \in \Sigma ( \{ x_{kn+i -1}, x_{kn+i}, x_{kn+i+1}, x_{(k+1)n -1}, x_{(k+1)n} \}, \infty )$ , 
			Property $1$ is equivalent to
			\begin{equation}
			\label{eq:equivalent_condition}
			\begin{aligned}
				\gamma_{kt + \ell} (x_i) =& \, \rho(f_{k,\ell}^1(x_i)) - \rho(f_{k,\ell}^2(x_i)) + \rho(f_{k,\ell}^3(x_i)),
			\end{aligned}
			\end{equation}
			for $i  =  kt + \ell -1, kt+\ell, kt+\ell+1, kt + t - 2, kt + t -1 $. As all of the functions $\gamma_{kt+\ell}$, $\rho\circ f_{k,\ell}^1$,  $\rho \circ  f_{k,\ell}^2$, $\rho \circ f_{k,\ell}^3$ take on the value 0 at $x_{kt+\ell -1}$ and $ x_{kt + t - 1}$,  
			the relation \eqref{eq:equivalent_condition} holds trivially for $i = kt+\ell -1, kt + t - 1$ regardless of the choice of $y_{k,\ell}^{j}$, $j = 1,2,3$. Substituting \eqref{eq:expression_of_f} into \eqref{eq:equivalent_condition}, for $i = kt+\ell$, $kt+\ell+1$, $kt + t -2$ yields 
			% \begin{equation}
			% \begin{aligned}
			% 	f_{k,i}^{j} ( x  ) =&\,  f(x_{(k+1)n - 1}) \frac{ x - x_{kn +i +j - 2 }}{ x_{(k+1)n - 1} - x_{kn +i +j - 2 }} + f ( x_{kn +i +j - 2 } )\\
			% 	=&\, y^j_{k,i} \frac{ x - x_{kn +i +j - 2 }}{ x_{(k+1)n - 1} - x_{kn +i +j - 2 }}.
			% \end{aligned}
			% \end{equation}
			\begin{align}
				1 =&\, f_{k,\ell}^1 (x_{kt+\ell}),\label{eqline:system_equation_f_1}\\
				0 =&\, f_{k,\ell}^1(x_{kt+\ell + 1}) - f_{k,\ell}^2 (x_{kt+\ell + 1}),\label{eqline:system_equation_f_2}\\
				0 =&\, f_{k,\ell}^1 (x_{kt + t - 2 }) - f_{k,\ell}^2(x_{kt + t - 2 }) + f_{k,\ell}^3(x_{kt + t - 2 }). \label{eqline:system_equation_f_3}
			\end{align}
			% For $j = 1,2,3$, $f_{k,\ell}^j$ is affine on the interval $[x_{kt+1},x_{kt + t - 2}]$ and therefore satisfies the interpolation formula, that, for $x \in [x_{kt+1},x_{kt + t - 2}]$, \z
			% \begin{align}
			% 	f^j_{k,\ell} (x) =&\, f^j_{k,\ell}( x_{kt +\ell +j - 2 } ) +  ( x - x_{kt +\ell +j - 2 } )\frac{f^j_{k,\ell} ( x_{kt + t -2 } ) - f^j_{k,\ell}( x_{kt +\ell +j - 2 } )}{x_{kt + t -2 } - x_{kt +\ell +j - 2 }} \\
			% 	=&\, y^j_{k,\ell} \frac{ x - x_{kt +\ell +j - 2 }}{x_{kt + t - 2} - x_{kt +\ell +j - 2 }},
			% \end{align}
			% where in \eqref{eqline:expression_f_2} we plugged in the values of $f^j_{k,\ell}( x_{kt +\ell +j - 2 } )$ and $f^j_{k,\ell} ( x_{kt + t -2 } )$ given by  \eqref{eqline:choice_y_i_1} and \eqref{eq:expression_f_0}. 
			Next, using the interpolation formula \eqref{eqline:expression_f_1}-\eqref{eqline:expression_f_2} in  \eqref{eqline:system_equation_f_1}-\eqref{eqline:system_equation_f_3}, we can rewrite \eqref{eqline:system_equation_f_1}-\eqref{eqline:system_equation_f_3} in terms of $y_{k,\ell}^1, y_{k,\ell}^2,y_{k,\ell}^3$ according to
			\begin{align}
				1 =&\,\, y^1_{k,\ell} \frac{ x_{kt+\ell} - x_{kt +\ell -1 }} {x_{kt + t -2 } - x_{kt +\ell -1 }},\\
				0 =&\,\, y^1_{k,\ell} \frac{ x_{kt+\ell+1} - x_{kt +\ell -1 }} {x_{kt + t -2 } - x_{kt +\ell -1 }} - y^2_{\ell,k} \frac{ x_{kt + \ell + 1} - x_{kt +\ell  }} {x_{kt + t -2 } - x_{kt +\ell }}   ,\\
				0 =&\,\, y_{k,\ell}^1 - y_{k,\ell}^2 + y_{k,\ell}^3, \label{eqline:b6_100}
			\end{align}
		which has the unique solution 
			\begin{align}
				y^1_{k,\ell} =&\,\, \frac{x_{kt + t -2 } - x_{kt +\ell -1 }}{ x_{kt+\ell} - x_{kt +\ell -1 }}, \label{eqline:solution_1}\\
				y^2_{k,\ell} =&\,\, y^1_{k,\ell} \frac{ x_{kt+\ell+1} - x_{kt +\ell -1 }} {x_{kt + t -2 } - x_{kt +\ell -1 }} \cdot \frac{x_{kt + t -2 } - x_{kt +\ell }}{ x_{kt + \ell + 1} - x_{kt +\ell  }}\label{eqline:solution_4}\\
				= &\, \frac{ x_{kt+\ell+1} - x_{kt +\ell -1 }}{ x_{kt+\ell} - x_{kt +\ell -1 }}   \frac {x_{kt + t -2} - x_{kt +\ell }} { x_{kt+\ell + 1} - x_{kt +\ell }},\label{eqline:solution_2}\\
				y^3_{k,\ell} =&\,\, \frac{ x_{kt+\ell+1} - x_{kt +\ell -1 }}{ x_{kt+\ell} - x_{kt +\ell -1 }}   \frac {x_{kt + t -2} - x_{kt +\ell }} { x_{kt+\ell + 1} - x_{kt +\ell }} - \frac{x_{kt + t -2 } - x_{kt +\ell -1 }}{ x_{kt+\ell} - x_{kt +\ell -1 }}.\label{eqline:solution_3}
			\end{align}
		It remains to verify that the specific choices for $y^1_{k,\ell}, y^2_{k,\ell}, y^3_{k,\ell}$ per \eqref{eqline:solution_1}-\eqref{eqline:solution_3}, indeed, satisfy the positivity assumption \eqref{eq:assumption_y}. %and then establishes Property~1. 
  To this end, first note that, owing to $\ell \in \{ 3, \dots, t -4 \}$, we have $kt + t -2 > kt +\ell >  kt +\ell  - 1 $, and hence  \eqref{eqline:solution_1}-\eqref{eqline:solution_2} implies $y^1_{k,\ell}, y^2_{k,\ell} > 0$. Then, it follows from  \eqref{eqline:solution_4}, exploiting the strictly increasing nature of  $( x_{i} )_{i = 0}^{M - 1}$, that $\frac{ x_{kt+\ell+1} - x_{kt +\ell -1 }} {x_{kt + t -2 } - x_{kt +\ell -1 }} \cdot \frac{x_{kt + t -2 } - x_{kt +\ell }}{ x_{kt + \ell + 1} - x_{kt +\ell  }} = \frac{x_{kt + t -2 } - x_{kt +\ell }}{x_{kt + t -2 } - x_{kt +\ell -1 }} \cdot \frac{ x_{kt+\ell+1} - x_{kt +\ell -1 }}{ x_{kt + \ell + 1} - x_{kt +\ell  }} > 1 $ and hence $y^2_{k, \ell} > y^1_{k,\ell}$. Together with \eqref{eqline:b6_100} this yields $y_{k,\ell}^3 = y^2_{k,\ell} - y^1_{k,\ell} > 0$. We refer to Figures~\ref{fig:fjki123} and \ref{fig:sum_fjki123} for an illustration of $\rho \circ f^j_{k,\ell}, j =1,2,3,$ and $\gamma_{kt + \ell} = \rho\circ f_{k, \ell}^1 - \rho\circ f_{k, \ell}^2 + \rho\circ f_{k, \ell}^3$, respectively.

		% Considering the supports of $\rho(f^j_{k,i})$,$j =1,2,3$.
		% which are shown in Fig~\ref{fig:represent_gamma_kn},  and the positive and negative of $f^j_{k,i}$,$j =1,2,3$, which is also shown in Fig~\ref{fig:basis_1}, \eqref{eq:equivalent_condition} is equivalent to 
		% It is easy to see there exists $f_{k,i}^1$ such that \eqref{eqline:f_1} holds, and given $f_{k,i}^1$, there exists $f_{k,i}^2$ such that $f_{k,i}^2(x_{kn+i + 1}) =  f_{k,i}^1 (x_{kn+i + 1})$ which fulfill \eqref{eqline:f_2}, and given $f_{k,i}^1$ and $f_{k,i}^2$, there exists $f_{k,i}^3$ such that $f_{k,i}^3(x_{(k+1)n -1 } ) = - f_{k,i}^1 (x_{(k+1)n -1 }) + f_{k,i}^2(x_{(k+1)n -1 })$ which fulfill \eqref{eqline:f_3}. Condition 1 comes valid under this choice of $f_{k,n}^j$, $j = 1,2,3$. 

			Property 2 will be validated by upper-bounding $\| f_{k,\ell}^j \|_{L^\infty ( \mathbb{R} )} $, $j = 1,2,3$, followed by application of
      Lemma~\ref{lem:representation_single_hidden_layer}. First, note that the functions $f_{k,\ell}^1, f_{k,\ell}^2, f_{k,\ell}^3 \in \Sigma ( ( x_{kt},x_{kt+1},x_{kt + t -2},x_{kt + t -1} ), \infty )$ are equal to zero on $(-\infty,x_{kt}] \cup [x_{kt + t - 1},\infty) $ and hence take on their maximum absolute values at the breakpoints  $ x_{kt + 1}, x_{kt + t - 2}$. Next, we provide a relation, which will be  used repeatedly later, namely for integers $m_1,m_2,m_3,m_4$ with $0 \leq m_1< m_2 < M-1, 0 \leq m_3 < m_4 \leq M - 1$,
			\begin{align}
				\Bigl |\frac{x_{m_2} - x_{m_1}}{ x_{m_4} - x_{m_3}} \Bigr | =&\, \Biggl |\frac{\sum_{i = m_1 + 1}^{m_2} ( x_{i} - x_{i - 1}  )}{\sum_{i = m_3 + 1}^{m_4} ( x_{i} - x_{i - 1}  )}\Biggr | \label{eq:b6_20}\\
				\leq &\, \Biggl | \frac{( m_2 - m_1 ) \max_{i = 1,\dots, M-1}  ( x_{i} - x_{i - 1}  )}{( m_4 - m_3 ) \min_{i = 1,\dots, M-1}  ( x_{i} - x_{i - 1}  )}\Biggr |\\
				= &\, \frac{m_2 - m_1}{m_4 - m_3} R_c(X).\label{eq:b6_21}
			\end{align}
			For $j =1$, we have 
			\begin{align}
				| f_{k,\ell}^1(x_{kt + t - 2}) |  =&\, |y^1_{k,\ell}|\label{eqline:b6_201}\\
				=&\, \Bigl |\frac{x_{kt + t - 2} - x_{kt +\ell -1 }}{ x_{kt+\ell} - x_{kt +\ell -1 }} \Bigr | \label{eqline:b6_1}\\
				\leq &\, ( t-2  - ( \ell - 1 )) R_c(X)\label{eqline:b6_101}\\ 
				\leq &\, t R_c(X),\label{eqline:b6_2}
			\end{align}
			and 
			\begin{align}
				| f_{k,\ell}^1(x_{kt +1}) | =&\, \Bigl| y_{k,\ell}^{1} \frac{x_{kt + 1} - x_{kt +\ell -1}}{x_{ kt + t -2} - x_{kt +\ell  -1}} \Bigr| \label{eqline:b6_3}\\
				=&\, \Bigl|\frac{x_{kt + 1} - x_{kt +\ell -1}} { x_{kt+\ell} - x_{kt +\ell -1 }} \Bigr| \label{eqline:b6_300}\\
				\leq &\, ( \ell - 2 ) R_c(X) \label{b6_7}\\
				\leq &\, t R_c(X), \label{eqline:b6_4}
			\end{align}
			where \eqref{eqline:b6_201} follows from \eqref{eqline:choice_y_i_2}, \eqref{eqline:b6_1} is by \eqref{eqline:solution_1}, in \eqref{eqline:b6_101} we used \eqref{eq:b6_20}-\eqref{eq:b6_21},   in \eqref{eqline:b6_3} we employed \eqref{eqline:choice_y_i_4}, \eqref{eqline:b6_300} is a consequence of \eqref{eqline:solution_1},  \eqref{b6_7} is by \eqref{eq:b6_20}-\eqref{eq:b6_21}, and for \eqref{eqline:b6_4} we used $\ell \in \{ 3,\dots, t-4 \}$ owing to $( k,\ell ) \in \mathcal{I}^2$. Since $f_{k,\ell}^1$ takes on its maximum absolute value at $x_{kt+1}$ or $x_{kt + t - 2}$, as noted above, we get $\| f_{k,\ell}^1  \|_{L^\infty ( \mathbb{R} )} \leq t R_c(X)$. Similarly, for $j =2$, it holds that
			\begin{align}
				| f_{k,\ell}^2(x_{kt + t - 2}) |  =&\, |y^2_{k,\ell}| \label{eqline:b6_80}\\
				=&\, \Bigl| \frac{ x_{kt+\ell+1} - x_{kt +\ell -1 }}{ x_{kt+\ell} - x_{kt +\ell -1 }}   \frac {x_{kt + t -2} - x_{kt +\ell }} { x_{kt+\ell + 1} - x_{kt +\ell }} \Bigr| \label{b6_822} \\
				=&\, \Bigl| \frac{ x_{kt+\ell+1} - x_{kt +\ell -1 }}{ x_{kt+\ell} - x_{kt +\ell -1 }} \Bigr| \Bigl| \frac {x_{kt + t -2} - x_{kt +\ell }} { x_{kt+\ell + 1} - x_{kt +\ell }} \Bigr|\label{b6_8} \\
				\leq&\, 2 R_c(X) ( t-2 - \ell )R_c(X) \label{eqline:b6_82}\\
				\leq &\, 2t (R_c(X))^2, \label{eqline:b6_9}
			\end{align}
			and 
			\begin{align}
				| f_{k,\ell}^2(x_{kt+1}) |  =&\, \Bigl|y^2_{k,\ell} \frac{x_{kt + 1} - x_{kt +\ell}}{x_{ kt + t -2} - x_{kt +\ell }} \Bigr|\label{eqline:b6_81}\\
				=&\, \Bigl| \frac{ x_{kt+\ell+1} - x_{kt +\ell -1 }}{ x_{kt+\ell} - x_{kt +\ell -1 }}   \frac {x_{kt + 1} - x_{kt +\ell}} { x_{kt+\ell + 1} - x_{kt +\ell }} \Bigr| \\
				\leq &\, 2 R_c(X)  ( \ell - 1)R_c(X) \label{eqline:b6_83}\\
				\leq &\, 2t (R_c(X))^2, \label{eqline:b6_8333}
			\end{align}
			where \eqref{eqline:b6_80} follows from \eqref{eqline:choice_y_i_2}, \eqref{b6_822} is by \eqref{eqline:solution_4}-\eqref{eqline:solution_2}, \eqref{eqline:b6_82} is a consequence of \eqref{eq:b6_20}-\eqref{eq:b6_21}, \eqref{eqline:b6_9} is thanks to $\ell \in \{ 3,\dots, t-4 \}$,  in \eqref{eqline:b6_81} we used \eqref{eqline:choice_y_i_4}, \eqref{eqline:b6_83} follows from \eqref{eq:b6_20}-\eqref{eq:b6_21}, and \eqref{eqline:b6_8333} is owing to $\ell \in \{ 3,\dots, t-4 \}$. Again, as $f_{k,\ell}^2$ takes on its maximum absolute value at $x_{kt+1}$ or $x_{kt + t - 2}$, we get
			\begin{equation*}
				\| f_{k,\ell}^2  \|_{L^\infty ( \mathbb{R} )} \leq 2t (R_c(X))^2.
			\end{equation*} 
			For $j = 3$, we have  
			\begin{align}
				| f_{k,\ell}^3(x_{kt + t -2}) |  =&\, |y^3_{k,\ell}|\\
				=&\, |y^2_{k,\ell} - y^1_{k,\ell}| \label{eqline:b6_5}\\
				\leq &\, |y^2_{k,\ell}| + |y^1_{k,\ell}|  \\
				\leq &\, 2t R_c(X)^2 + t R_c(X)\label{eqline:b6_7}\\
				\leq &\, 3t (R_c(X))^2,
			\end{align}
			where in \eqref{eqline:b6_5} we used \eqref{eqline:b6_100}, and \eqref{eqline:b6_7} follows from \eqref{eqline:b6_201}-\eqref{eqline:b6_2} and  \eqref{eqline:b6_80}-\eqref{eqline:b6_9}. Finally, 
			\begin{align}
				| f_{k,\ell}^3(x_{kt+1}) |  =&\, \Bigl|y^3_{k,\ell}\,  \frac{x_{kt + 1} - x_{kt +\ell + 1}}{x_{ kt + t - 2} - x_{kt +\ell + 1}}\Bigr|\label{eqline:b6_8}\\
				\leq &\, 3t (R_c(X))^2  \frac{\ell}{t- \ell -3} R_c(X) \label{eqline:b6_10}\\
				\leq &\, 3 t^2 (R_c(X))^3,
			\end{align}
			where \eqref{eqline:b6_8} is a consequence of \eqref{eqline:choice_y_i_4}, and \eqref{eqline:b6_10} follows from \eqref{eq:b6_20}-\eqref{eq:b6_21}. We therefore have 
			\begin{equation*}
				\| f_{k,\ell}^3  \|_{L^\infty ( \mathbb{R} )} \leq 3t^2 (R_c(X))^3.
			\end{equation*}
			In summary, we established that $\| f_{k,\ell}^1  \|_{L^\infty ( \mathbb{R} )},\| f_{k,\ell}^2  \|_{L^\infty ( \mathbb{R} )}, \| f_{k,\ell}^3  \|_{L^\infty ( \mathbb{R} )} \leq 3 t^2 (R_c(X))^3$, which together with $f_{k,\ell}^j \in \Sigma ( ( x_{kt},x_{kt+1},x_{kt + t -2},x_{kt + t -1} ), \infty )$, implies 
			\begin{equation}
				\label{eq:embedding}
				 f_{k,\ell}^1, f_{k,\ell}^2,  f_{k,\ell}^3 \in \Sigma ( ( x_{kt},x_{kt+1},x_{kt + t -2},x_{kt + t -1} ), 3 t^2 (R_c(X))^3).
			\end{equation}
			Application of Lemma~\ref{lem:representation_single_hidden_layer} to $f_{k,\ell}^j$, for $j \in \{ 1,2,3 \}$, with $X = ( x_{kt},x_{kt+1},x_{kt + t -2},x_{kt + t -1} )$ and $E = 3 t^2 (R_c(X))^3$, yields
			\begin{align*}
				f_{k,\ell}^j (x) =d\, +\,  &c_1 \rho(x - x_{kt}) + c_2 \rho(x - x_{kt+1}) \\
				&+ c_3 \rho( x - x_{kt + t - 2}) + c_4 \rho(x - x_{kt + t - 1}), \quad  x \in \mathbb{R},
			\end{align*}
			for some $d,c_1,c_2,c_3,c_4 \in \mathbb{R}$ such that $|d|,|c_1|,|c_2|,|c_3|,|c_4| \leq  12 t^2 R_m (X)  (R_c(X))^3$. 
   %4 R_m(X) \cdot 3 t^2 (R_c(X))^3 =
The validity of Property 2 now follows upon noting that
			\begin{align*}
				&( x_{kt}, x_{kt+1}, x_{kt + t - 2}, x_{kt + t - 1}) \subseteq ( z_i )_{i = 1}^{8u - 1},
			\end{align*}
			with $( z_i )_{i = 1}^{8u - 1}$ as defined in \eqref{eq:definitio_z_sequence}. Since the choice of $( k,\ell ) \in \mathcal{I}^2$ was arbitrary, this completes the construction of $f^1_{k,\ell},f^2_{k,\ell},f^3_{k,\ell}$ and establishes Properties 1 and 2 for all $( k,\ell ) \in \mathcal{I}^2$.

			We conclude the proof by verifying Property 3.
   %, namely $\{ f^j_{k,\ell} \}_{k =0}^{u-1}$ having pairwise disjoint supports, for  every $( \ell, j ) \in \{ 3,\dots, t -4  \} \times \{ 1,2,3 \}$. 
   The statement follows directly as, for every $( \ell, j ) \in \{ 3,\dots, t -4  \} \times \{ 1,2,3 \}$, the functions $f_{k,\ell}^j $, $k \in \{ 0,\dots,u-1 \}$, are supported in $(x_{kt}, x_{kt + t -1})$ according to \eqref{eq:sgn_of_f}.
   %which implies functions in  $\{ f_{k,\ell}^j \}_{k =0}^{u -1}$ are pairwise-disjointly supported. 
		% subsection proof_of_lemma_lem:representation_two_hidden_layer (end)

		\subsection{Proof of Lemma~\ref{lem:realization_piecewise_linear_functions}} % (fold)
		\label{sub:proof_of_lemma_lem:realization_piecewise_linear_functions}
			% We have 
			% \begin{equation*}
			% 	H = \sum_{i = 1}^{D} \rho \circ h_i = \sum_{i = 1}^{Km} \rho \circ h_i.
			% \end{equation*}
			We start by defining several auxiliary quantities. Let $G(x) := (  \rho(x - z_0), \dots, \rho (x - z_{r-1}) )^T$, $x \in \mathbb{R}$. 
   For $j = 0, \dots, s - 1, x \in \mathbb{R}$, define, 
			\begin{equation*}
				U_{j} := \begin{pmatrix}
					c_{ju + 0, 0} & \ldots & c_{ju + 0, r - 1}\\
					\vdots & \ddots & \vdots\\
					c_{ju + u -1 , 0}& \ldots & c_{ju + u - 1, r-1}
				\end{pmatrix}\!,\, v_{j} := \begin{pmatrix}
					d_{ju + 0} \\
					\vdots\\
					d_{ju + u -1 }
				\end{pmatrix}\!,\, w_j (x) := \begin{pmatrix}
					h_{ju + 0}(x)\\
					\vdots \\
					h_{ju + u - 1}(x)
				\end{pmatrix},
			\end{equation*}
			such that 
			\begin{equation}
				\label{eq:compute_w_j}
				U_j\, G(x) + v_j = w_j (x), \quad x \in \mathbb{R}.
			\end{equation}
			Based on \eqref{eq:compute_w_j}, we proceed to construct a ReLU network realization of $H$ according to
			\begin{equation}
				\label{eq:decomposition_H}
				H(x) = \sum_{\ell = 0}^{us - 1} \rho (h_\ell ( x )) = \sum_{j = 0}^{s- 1} 1_{u}^T \rho ( w_j ( x ) ),\, \quad x \in \mathbb{R},
			\end{equation}
			where we recall that the ReLU function $\rho$ acts componentwise. Specifically, let $\Phi = ( ( A_k, b_k ) )_{k =1}^{s+2}$ be given by
			\begin{alignat*}{2}
				A_1 = & \, 1_{r \times 1  }, & b_1 & = (- z_0,\cdots, -z_{r-1 } )^T,\\
				A_2 = & \begin{pmatrix}
					 I_{r } \\
					U_0\\
					0_{1\times r}
				\end{pmatrix}, & b_2 & = \begin{pmatrix}
					0_{r}\\
					v_0\\
					0_1
				\end{pmatrix},\\
				A_{s+2} = &  \begin{pmatrix}
				0_{1 \times r} & 1_{1 \times u} & 1_{1 \times 1}
				\end{pmatrix}, \quad &b_{s+2} & =  0,
			\end{alignat*}
			and for $k \in \mathbb{N}$ such that $3 \leq  k \leq s+1$,\footnote{For $s = 1$, there is no $k \in \mathbb{N}$ satisfying the constraint and the assignment is thus skipped.}
			\begin{equation*}
				A_{k } =  \begin{pmatrix}
				I_{r }& 0_{r \times u }& 0_{r \times 1}\\
				U_{k -2 }& 0_{u \times u}& 0_{u \times 1} \\
				0_{1 \times r} & 1_{1 \times u} & 1_{1 \times 1}
				\end{pmatrix},\quad  b_{k } =  \begin{pmatrix}
				0_{r}\\
				v_{k - 2}\\
				0_1
				\end{pmatrix}.
			\end{equation*}
			Note that $\mathcal{L} ( \Phi ) = s+2$, $\mathcal{W} (\Phi  ) = r + u + 1$, and $\mathcal{B} ( \Phi ) \leq \max \{ 1, \max_{\ell,i}| c_{\ell, i}|, \max_\ell | d_\ell | ,\max_{i} \{ z_i \}   \} \leq \max \{ 1, T \}$.

			We proceed to show that $H = R ( \Phi )$, which will then imply $H \in \mathcal{R} ( r + u + 1, s+2, \max \{ 1, T \} )$, as desired. For $k = 1,\dots, s+2$, set  $y_k := R ( ( A_\ell,b_\ell )_{\ell = 1}^{k} )$. We have, 
			\begin{equation}
				y_1 (x) = A_1 x  + b_1 = ( x - z_0, \dots, x - z_{r-1} )^T, \quad  x \in \mathbb{R}.
			\end{equation}
			Next, it is proved by induction that for $k = 2,\dots, s + 1$,
			\begin{equation}
				\label{eq:induction_goal_linear_sum}
				y_{k } (x) =\begin{pmatrix}
					G(x) \\
					w_{k - 2} (x) \\
					\sum_{j =0}^{k-3  } 1_{u}^T \rho ( w_j ( x ) )
				\end{pmatrix}, \quad  x \in \mathbb{R},
			\end{equation}
			where we use the convention $\sum_{j =0}^{-1  } 1_{u}^T \rho ( w_j ( x ) ) = 0$. The base case $k = 2$ follows from 
			\begin{align}
				y_2 ( x ) =\,  & A_2 \rho ( y_1 ( x ) ) + b_2\\
				=\,& \begin{pmatrix}
					\rho ( ( x - z_0, \dots, x - z_{r-1} )^T )\\ 
					U_0 \rho ( ( x - z_0, \dots, x - z_{r-1} )^T ) + v_0\\
					0_r
				\end{pmatrix}\\
				=\, & \begin{pmatrix}
					G(x) \\
					w_{0} ( x )\\
					0_r
				\end{pmatrix} \quad x \in \mathbb{R}, \label{eqline:y_2}
			\end{align}
			where in \eqref{eqline:y_2} we used $U_0 G(x) + v_0 = w_0 ( x )$, for $x \in \mathbb{R}$. 
			% \begin{align}
			% 	y_2 (x) =& A_2 \rho ( y_1 (x) ) + b_2 \\
			% 	=& \begin{pmatrix}
			% 		I_n G(x)\\
			% 		U_1 G(x) + v_1 \\
			% 		0 
			% 	\end{pmatrix} \\
			% 	= & \begin{pmatrix}
			% 		G(x) \\
			% 		w_1 ( x ) \\
			% 		0
			% 	\end{pmatrix}.
			% \end{align} 
			We proceed to prove the induction step $k - 1 \mapsto k$ with $3 \leq  k \leq s+1$\footnote{For $s = 1$, the induction step is not needed as the base case already complete the proof.}, noting that the induction assumption is given by
			\begin{equation}
			\label{eqline:induction_square_assumption}
				y_{k -1 } (x) =\begin{pmatrix}
					G(x) \\
					w_{k - 3} (x) \\
					\sum_{j =0}^{k-4  } 1_{u}^T \rho ( w_j ( x ) )
				\end{pmatrix},\, x \in \mathbb{R}.
			\end{equation}
			The result is immediate from
			\begin{align}
				y_{k}(x) =&\, A_{k } \rho ( y_{k-1} ( x ) ) + b_{k} \\
				=&\, \begin{pmatrix}
					I_r   G(x) \\
					U_{k - 2}  G(x)  +  v_{k - 2}\\
					1_u^T \rho ( w_{k-3} (x) ) + \sum_{j =0}^{k-4  } 1_{u}^T \rho ( w_j ( x ) )
				\end{pmatrix} \label{eqline:induction_square_1}\\
				=& \, \begin{pmatrix}
					G(x)\\
					w_{k-2} (x)\\
					\sum_{j =0}^{k-3  } 1_{u  }^T \rho ( w_j ( x ) )
				\end{pmatrix}, \quad x \in \mathbb{R},
			\end{align}
			where in \eqref{eqline:induction_square_1} we used the induction assumption~\eqref{eqline:induction_square_assumption} and the fact that $G(x),\sum_{j =0}^{k-4  } 1_{u}^T \rho ( w_j ( x ) ) \geq 0 $, for all $x \in \mathbb{R}$. We conclude the overall proof by noting that, for $x \in \mathbb{R}$, 
			\begin{align}
				y_{s+2} (x) =&\, A_{s+2} \rho ( y_{s+1} (x) ) + b_{s+2}\\
				=&\, 1_u^T \rho ( w_{s-1} (x) ) + \sum_{j =0}^{s-2  } 1_{u  }^T \rho ( w_j ( x ) )\\
				=&\, \sum_{j =0}^{s-1  } 1_{u  }^T \rho ( w_j ( x ) ) \\
				=&\, H(x), \label{eqline:induction_square_2}
			\end{align}
			where  \eqref{eqline:induction_square_2} follows from \eqref{eq:decomposition_H}.
			% \begin{figure}[H]
			% 	\centering
			% 	\includegraphics[width = 0.4\textwidth]{images/flowchart_1.jpg}
			% 	\caption{Computational flowchart}
			% 	\label{fig:computational_flowchart}
			% \end{figure}

		% subsection proof_of_lemma_lem:realization_piecewise_linear_functions (end)

		% \begin{proof}
		% 	[Proof of Lemma~\ref{lem:reduce_width}]

		% 	Suppose $x \in \mathbb{R}_+$. By the assumption that $(f_i)_{i=1}^N$ have disjoint support, there exists at most one $\hat{i} \in \left\{ 1,\dots, N \right\}$ such that $f_i$ is non zero. If such a $\hat{i}$ exists, then 
		% 	\begin{equation*}
		% 		\sum_{i=1}^N \rho\left( f_i(x)\right) = \rho\left( f_{\tilde{i}}(x)\right) = \rho\left( \sum_{i =1}^N f_i (x)\right).
		% 	\end{equation*}

		% 	Otherwise all $f_i(x)$, $i \in \left\{ 1,\dots, N \right\}$ are identically zero and therefore 
		% 	\begin{equation*}
		% 		\sum_{i=1}^N \rho\left( f_i(x)\right) = 0 = \rho\left( \sum_{i =1}^N f_i (x)\right).
		% 	\end{equation*}
		% \end{proof}

	% section piecewise_linear_representation (end)

	%!TEX root = ../draft_quantized_weight_networks.tex

\section{Proof of Proposition~\ref{prop:bit_extraction}} % (fold)
	\label{sub:bit_extraction}
		% The classical bit extraction technique is to store information in digits of a binary number and extract the information in the digits through ReLU networks. 
		% \begin{proof}
		% \todo{where do we borrow the ideas from shen? See comments 221}
		Let $N \in \mathbb{N}$ be fixed throughout. We start by providing intuition behind the construction of $F_{N,L}$, $L \in \mathbb{N}$, and then follow up with the proof. Consider the following decomposition,  for all $L \in \mathbb{N}$, $\sum_{i\, = \,1}^\infty \theta_{i} 3^{-i} \in \mathbb{T}$, and $\bit \in \mathbb{N} \cup \{ 0 \}$,
		\begin{align}
			\sum_{i\, = \,1}^{\min \{N(L+1),\, \bit \}} \theta_{i} = \sum_{i\, = \,1}^{\min \{ N, \, \bit \}} \theta_{i} + \sum_{i\, = \,1}^{\min \{ N L,\, \max  \{\bit - N,\, 0\} \} } \theta_{N + i},\label{eqline:decomposition_for_several_steps}
		\end{align}
		where we recall the convention  $\sum_{i\, = \,1}^{0} \theta_{N + i} = 0$. With $F_{N,\,L}: \mathbb{R} \mapsto \mathbb{R}$ according to \eqref{eq:computational_property_f_NL}, \eqref{eqline:decomposition_for_several_steps} implies the recursive relation 
		\begin{equation}
			\label{eq:intuitive_induction}
			F_{N,L+1} \biggl( \sum_{i =1}^\infty \theta_{i} 3^{-i}, \bit \biggr) = \sum_{i\, = \,1}^{\min \{ N, \bit \}} \theta_{i} + F_{N,L} \biggl( \sum_{i =1}^\infty \theta_{N+ i} 3^{-i}, \max \{ \bit -N , 0 \} \biggr),
		\end{equation}
        which will be seen below to inspire our proof.
		%for all $L \in \mathbb{N}$, $\sum_{i\, = \,1}^\infty \theta_{i} 3^{-i} \in \mathbb{T}$, and $\bit \in \mathbb{N} \cup \{ 0 \}$. The recursive relation %\eqref{eq:intuitive_induction} will be seen below to inspire our proof.

		% While our proof is inspired by the decomposition \eqref{eqline:decomposition_for_several_steps}, we need a more elaborate construction, in order to control the width of networks that realize $F_{N,L}$, $L \in \mathbb{N}$. 
		%For the formal proof, w
  We begin with two technical lemmata introducing the function $g$ and the family of functions $\{ G_{N,L} \}_{L \in \mathbb{N}}$ that serve as basic building blocks of our construction.

		% To this end, we need the following technical lemma that provides a function $g$ as the basic construction block.
		\begin{lemma}
			\label{lem:constructive_proof_of_g}
			% For $N \in \mathbb{N}$, t
			There exists a function 
			\begin{equation}
				\label{eq:g_complexity}
				g \in \mathcal{R} ( ( 3,3 ),2^{N+4}, 3,  3^{N+2})
			\end{equation}
			such that for all $\zero \in \mathbb{R}$, $\sum_{i = 1}^\infty \theta_{i} 3^{-i} \in \mathbb{T}$, $\bit \in \mathbb{N} \cup \{ 0 \}$,
			\begin{equation}
				\label{eq:expression_of_g_bit_extraction}
				g\biggl(\zero, \sum_{i = 1}^\infty  \theta_i\, 3^{-i}, \bit  \biggr) = \Bigl( \rho(\zero) +  \sum_{i = 1}^{\min \{ N,\,\bit \}} \theta_i, \sum_{i = 1}^\infty  \theta_{N+i}\, 3^{-i} , \max \{ \bit - N, 0 \}   \Bigr).
			\end{equation}
			\begin{proof}
				See Appendix~\ref{sub:proof_of_lemma_lem:constructive_proof_of_g}.
			\end{proof}
		\end{lemma}

		Define $G_{N,L }: \mathbb{R}^3 \mapsto \mathbb{R}^3$, $L \in \mathbb{N}$, recursively, according to 
			\begin{equation}
			\label{eq:complexity_G_N}
			G_{N,\,L } = \left\{
			\begin{aligned}
				& g, &&\text{for } L = 1,\\
				& g \circ G_{N,\, L - 1}, && \text{for } L \geq 2,
			\end{aligned}
			\right.
			\end{equation}
		where $g$ is the function specified by Lemma~\ref{lem:constructive_proof_of_g}. The properties of $G_{N,\,L }$, $L \in \mathbb{N}$, are summarized in the following result.
		\begin{lemma}
		\label{lem:C2}
			Let $L \in \mathbb{N}$. It holds that 
			\begin{equation}
			\label{eq:C2_1}
				G_{N,\,L} \in \, \mathcal{R} ( (3,3),2^{N+4},3 L,3^{N+2}),
			\end{equation}
			and, for $\sum_{i = 1}^\infty \theta_{i} 3^{-i} \in \mathbb{T}$, $\bit \in \mathbb{N} \cup \{ 0 \}$,
			\begin{equation}
			\label{eq:C2_2}
				G_{N,L} \biggl(0, \sum_{i = 1}^\infty \theta_i 3^{-i}, \bit  \biggr) = \biggl(\,\sum_{i = 1}^{\min \{ NL,\,\bit \}}\theta_i ,  \sum_{i =1}^\infty \theta_{NL + i}\, 3^{-i},  \max \{ \bit - NL, 0 \}   \biggr).
			\end{equation}
			\begin{proof}
				See Appendix~\ref{sub:proof_of_lemma_lem:c2}.
			\end{proof}
		\end{lemma}
	
		Fix $L \in \mathbb{N}$. We are now ready to define $F_{N,\,L}$ and verify the corresponding properties \eqref{eq:complexity_decoding_networks}-\eqref{eq:bit_extraction}. Recall that $S(A,b)$ refers to the affine mapping given by $S(A,b) (x) = Ax + b, x \in \mathbb{R}_{n_2}$, for $A \in \mathbb{R}_{n_1 \times n_2}$, $b \in \mathbb{R}_{n_1}$, $n_1,n_2 \in \mathbb{N}$. Set  
		\begin{equation}
		\label{eq:def_F_N_1}
		 	F_{N,\,L} = f_2 \circ G_{N,\,L} \circ f_1,
		\end{equation} 
		with 
		\begin{alignat}{3}
			f_1 =&\, \affine \biggl(\biggl(\,\begin{smallmatrix}
			0 & 0\\
			1 & 0\\
			0 & 1
		\end{smallmatrix} \,\biggr), 0_3\biggr) \,&\in \mathcal{R} ( ( 2,3 ), 3,1,1 ), \label{eq:c2_f1}\\
			f_2 =&\, \affine ( (1\,0\,0) , 0)  \,&\in \mathcal{R} ( ( 3,1 ), 3,1,1 ), \label{eq:c2_f2}
		\end{alignat}
		% The choice of $f_1$ and $f_2$ implies
		and note that
		\begin{alignat}{2}
			f_1 ( y,z ) =&\, ( 0,y,z), \quad &&\text{for } (y,z) \in \mathbb{R}^2, \label{eq:c2_f1_meaning}\\
			f_2 ( x,y,z)  =&\,\, x,  &&\text{for } (x,y,z) \in \mathbb{R}^3. \label{eq:c2_f2_meaning}
		\end{alignat}

		%We proceed to verify the properties \eqref{eq:complexity_decoding_networks}-\eqref{eq:bit_extraction} of $F_{N,\,L}$. 
		Noting that $F_{N,L}$ is the composition of functions which can be realized by ReLU networks according to \eqref{eq:C2_1}, \eqref{eq:c2_f1}, and \eqref{eq:c2_f2},  application of Lemma~\ref{lem:algebra_on_ReLU_networks} yields
		\begin{align*}
			F_{N,L} \in&\, \mathcal{R} ( (2,1),2^{N+4},3 L + 2,3^{N+2})\\
			\subseteq&\, \mathcal{R} ( (2,1),2^{N+4},5L,3^{N+2}),
		\end{align*}
		which establishes \eqref{eq:complexity_decoding_networks}.
			Furthermore, for  $\sum_{i = 1}^\infty \theta_{i} 3^{-i} \in \mathbb{T}$, and $\bit \in \mathbb{N} \cup \{ 0 \}$,  we have
			\begin{align}
				F_{N,\,L} \biggl( \sum_{i = 1}^\infty \theta_i 3^{-i}, \bit  \biggr) =&\, f_2\circ G_{N,\,L} \circ f_1 \biggl( \sum_{i = 1}^\infty \theta_i 3^{-i}, \bit \biggr) \label{eq:FNL_1}\\
				=&\, f_2\circ G_{N,L}\biggl(0, \sum_{i = 1}^\infty \theta_i 3^{-i}, \bit \biggr)\label{eq:FNL_2} \\
				=&\, f_2 \Biggl(\,\sum_{i = 1}^{\min \{ NL,\,\bit \}} \theta_i ,  \sum_{i =1}^\infty \theta_{NL + i}  3^{-i},  \max \{ \bit - NL, 0 \}   \Biggr)  \label{eq:FNL_3} \\
				=&\, \sum_{i = 1}^{\min \{ NL,\,\bit \}} \theta_i, \label{eq:FNL_4}
			\end{align}
			where \eqref{eq:FNL_2} follows from \eqref{eq:c2_f1_meaning}, in \eqref{eq:FNL_3} we used \eqref{eq:C2_2}, and
   %which, in summary, 
   \eqref{eq:FNL_4} is by \eqref{eq:c2_f2_meaning}, thereby verifying property \eqref{eq:computational_property_f_NL}.  We conclude the proof by noting that  \eqref{eq:bit_extraction} follows directly from \eqref{eq:computational_property_f_NL}. 

			\subsection{Proof of Lemma~\ref{lem:constructive_proof_of_g}} % (fold)
			\label{sub:proof_of_lemma_lem:constructive_proof_of_g}
			
			% subsection proof_of_lemma_lem:constructive_proof_of_g (end)
			
			% [Proof of Lemma~\ref{lem:constructive_proof_of_g}]
			We start by constructing functions $g^1,g^2,g^3 \in \mathcal{R} ( (3,1) )$ such that, for all $\zero \in \mathbb{R}$, $\sum_{i = 1}^\infty \theta_{i} 3^{-i} \in \mathbb{T}$, $\bit \in \mathbb{N} \cup \{ 0 \}$, 
			\begin{enumerate}[label=(\roman*)]
				\item $g^1(\zero, \sum_{i = 1}^\infty  \theta_i 3^{-i}, \bit  ) = \rho(\zero) +  \sum_{i = 1}^{\min \{ N,\bit \}} \theta_i$,
				\item $g^2(\zero, \sum_{i = 1}^\infty  \theta_i 3^{-i}, \bit  ) = \sum_{i = 1}^\infty  \theta_{N+i} 3^{-i}$,
				\item $g^3(\zero, \sum_{i = 1}^\infty  \theta_i 3^{-i}, \bit  ) = \max \{ \bit - N, 0 \}$,
			\end{enumerate}
			and then putting them together, through Lemma~\ref{lem:algebra_on_ReLU_networks}, as $g = ( (g_1, g_2), g_3 )$\footnote{Recall that for  $f_1:\mathbb{R}^d \mapsto \mathbb{R}^{d'}$ and $f_2: \mathbb{R}^d \mapsto \mathbb{R}^{d''}$, $d,d',d'' \in \mathbb{N}$, $( f_1, f_2 ): \mathbb{R}^{d} \mapsto \mathbb{R}^{d' + d''}$ is defined according to $( f_1, f_2 ) ( x ) = ( f_1 ( x ), f_2 ( x ) )$, for $x \in \mathbb{R}^d$.} to obtain a ReLU network.
   %constructed through application of Lemma~\ref{lem:algebra_on_ReLU_networks}. 
   We proceed to the construction of $g^1,g^2,g^3$.
			\begin{enumerate}[label=(\roman*)]
				\item \label{item:g1} Construction of $g^1$. Fix $\zero \in \mathbb{R}$, $\sum_{i = 1}^\infty \theta_{i} 3^{-i} \in \mathbb{T}$, $\bit \in \mathbb{N} \cup \{ 0 \}$ throughout part \ref{item:g1}.  We first note that 
				\begin{equation}
				\label{eq:inspiration_realizing_g_1}
				 	 \rho(\zero) +  \sum_{i = 1}^{\min \{ N,\bit \}} \theta_{i} = \rho(\zero) + \sum_{\ell = 1}^{N} 1_{\{\theta_\ell =1\}} 1_{\{\ell \leq \bit\}}.
				\end{equation} 
				% Here, for a proposition $P$, we define the quantity $1_P$ according to \todo{Define the indication function $1_p$. See comments here. And use $\{  \}$ to enclose the proposition.}
				% \begin{equation*}
				% 1_P = \left\{
				% \begin{aligned}
				% 	&1, &&\text{if $P$ is true},\\
				% 	&0, &&\text{if $P$ is not true}.  
				% \end{aligned}
				% \right.
				% \end{equation*}
				Next, we construct functions $h^\ell, m^\ell:\mathbb{R} \mapsto \mathbb{R}$ such that $h^\ell ( \sum_{i = 1}^\infty \theta_{i} 3^{-i}) = 1_{\{\theta_\ell = 1\}}$ and $m^\ell ( \bit ) = 1_{\{\ell\, \leq\, \bit\}} $, $\ell = 1,\dots, N$.
    %provide the definition of $g^1$ according to \eqref{eq:inspiration_realizing_g_1} in terms of $h^\ell, m^\ell$, $\ell = 1,\dots,N$, and then realize $g_1$ %by a ReLU network. We detail the corresponding constructions as follows.
				\begin{enumerate}
					\item \label{item:indication_function} Construction of $h^\ell$. We distinguish the cases $\ell = 1$ and $\ell \geq 2$. For $\ell = 1$, we let $h^1:\mathbb{R} \mapsto \mathbb{R}$ be given by 
					\begin{equation}
					\label{eq:h_0_expression}
					  	h^1 (y) =\, 9 \rho(y - T(( 0,2))) - 9\rho(y - T((1,0))),
					  \end{equation}  
					which takes on the value $0$ on $[0,T(( 0,2))]$ and equals $1$ on $[T((1,0)), \infty)$. Upon noting that $\sum_{i = 1}^\infty \theta_{i} 3^{-i} \in \mathbb{T} \subseteq [0,T((0,2))] \cup [T((1,0)), \infty) $, we therefore have $h^1 (\sum_{i = 1}^\infty \theta_{i} 3^{-i}) = 1_{\{\theta_1 = 1\}}$. For $\ell \geq 2$, we first decompose $1_{\{\theta_\ell = 1\}}$  according to 
					\begin{equation}
						\label{eq:furthur_decomposition}
						1_{\{\theta_\ell = 1\}} = \sum_{\substack{(a_1,\dots, a_{\ell-1}) \\ \in \{ 0,1 \}^{\ell - 1}}} 1_{\{ \theta_1 = a_1, \dots, \theta_{\ell -1} = a_{\ell -1}, \theta_\ell =1 \}}.
					\end{equation}

					Next, note that, for $(a_1,\dots, a_{\ell-1})  \in \{ 0,1 \}^{\ell - 1}$, the function $r^{( a_1,\dots, a_{\ell-1} )}: \mathbb{R} \mapsto \mathbb{R}$, given by
					% \begin{equation}
					% \label{eq:indication_functions_multiple_digits}
					% \begin{aligned}
					% 	y \in \mathbb{R} \mapsto & 3^{\ell+2}\rho(y - T ( \{ a_1,\dots, a_{\ell-1},0,2 \} )) \\
					% 	&- 3^{\ell+2} \rho(y - T ( \{a_1,\dots, a_{\ell-1},1\}) \\
					% 	&- 3^{\ell+2}\rho(y - ter 0.a_1\dots a_{\ell-1}2)\\
					% 	& + 3^{\ell+2} \rho(y - ter 0.a_1\dots a_{\ell-1}21),
					% \end{aligned}
					% \end{equation}
					\begin{equation}
					\label{eq:indication_functions_multiple_digits}
					\begin{aligned}
						r^{( a_1,\dots, a_{\ell-1} )} (y) & = 3^{\ell+1}\rho(y - T ( ( a_1,\dots, a_{\ell-1},0,2 ) ))\\
						&\,\,- 3^{\ell+1} \rho(y - T ( (a_1,\dots, a_{\ell-1},1) )) \\
						&\,\,- 3^{\ell+1}\rho(y - T((a_1,\dots,a_{\ell-1},2)))\\
						&\,\, + 3^{\ell+1} \rho(y - T((a_1,\dots,a_{\ell-1},2,1) )),
					\end{aligned}
					\end{equation}
					satisfies $r^{( a_1,\dots, a_{\ell-1} )} (\sum_{i=1}^\infty \theta_i 3^{-i})=  1_{\{ \theta_1 = a_1, \dots, \theta_{\ell -1} = a_{\ell -1}, \theta_\ell =1 \}}$. We refer to Figure~\ref{fig:indication_function} for an illustration of $r^{( a_1,\dots, a_{\ell-1} )}$. 
					% This is because all $b \in \mathbb{T}_{NL}$ with $\theta_1 = a_1, \dots, \theta_{i -1} = a_{i -1}, \theta_i =1$ lies in the interval $I_A = [ter\, 0.a_1\dots a_{i-1}1, ter\, 0.a_1\dots a_{i-1}2]$ and there are no points in $\mathbb{T}_{NL}$ that lies in $I_B = (ter\, 0.a_1\dots a_{i-1}02,ter\, 0.a_1\dots a_{i-1}1) $ or $I_C = (ter\, 0.a_1\dots a_{i-1}2,ter\, 0.a_1\dots a_{i-1}21 )$. 
					% \todo{Adding the explanation?}
					% Substituting \eqref{eq:indication_functions_multiple_digits} into \eqref{eq:furthur_decomposition}, $1_{\theta_i =1}$, we obtain a function that computes $\sum_{i=1}^\infty \theta_i 3^{-i} \mapsto 1_{\theta_\ell = 1}$ with $L =2$, according to 
					Summation over $r^{( a_1,\dots, a_{\ell-1} )}$ for $a_1,\dots, a_{\ell-1}  \in \{ 0,1 \}$ yields the desired $h^\ell: \mathbb{R} \mapsto \mathbb{R}$ according to
					\begin{alignat}{2}
						h^\ell ( y ) :=& \sum_{\substack{(a_1,\dots, a_{\ell-1}) \\ \in \{ 0,1 \}^{\ell - 1}}} r^{( a_1,\dots, a_{\ell-1} )}( y )\\
						=&\:\:  \sum_{\substack{(a_1,\dots, a_{\ell-1}) \\ \in \{ 0,1 \}^{\ell - 1}}} ( 3^{\ell+1}\rho(y - T ( ( a_1,\dots, a_{\ell-1},0,2 ) )) \label{eq:define_h_i_1}\\
						&\:\:- 3^{\ell+1} \rho(y -T ( (a_1,\dots, a_{\ell-1},1) )) \\
						&\:\:- 3^{\ell+1}\rho(y - T((a_1,\dots,a_{\ell-1},2) ))\\
						&\:\: + 3^{\ell+1} \rho(y - T((a_1,\dots,a_{\ell-1},2,1) )) ), \label{eq:define_h_i_2}
					\end{alignat}
					which satisfies 
					\begin{align}
						h^\ell \biggl( \sum_{i = 1}^\infty \theta_{i} 3^{-i}\biggr) =&\: \sum_{\substack{(a_1,\dots, a_{\ell-1}) \\ \in \{ 0,1 \}^{\ell - 1}}} r^{( a_1,\dots, a_{\ell-1} )}\biggl(  \sum_{i = 1}^\infty \theta_{i} 3^{-i} \biggr) \label{eqline:property_h_ell_1}\\
						=&\: \sum_{\substack{(a_1,\dots, a_{\ell-1}) \\ \in \{ 0,1 \}^{\ell - 1}}} 1_{\{ \theta_1 = a_1, \dots, \theta_{\ell -1} = a_{\ell -1}, \theta_\ell =1 \}}\\
						=&\:\: 1_{\{\theta_\ell = 1\}},\label{eqline:property_h_ell_2}
					\end{align}
					as desired. We finally note that $h^\ell$, $\ell \in \mathbb{N}$, can be written as  
					\begin{equation}
					\label{eq:another_expression_h_ell}
						h^\ell ( y ) = \sum_{j = 1}^{2^{\ell+1}} u_{\ell,j} \rho (y - v_{\ell,j}), \quad y \in \mathbb{R},
					\end{equation}
					for some $u_{\ell,j},v_{\ell,j} \in \mathbb{R}$ with $|u_{\ell,j}|,  |v_{\ell,j}| \leq 3^{\ell + 1}$, $j = 1,\dots, 2^{\ell + 1}$. 
					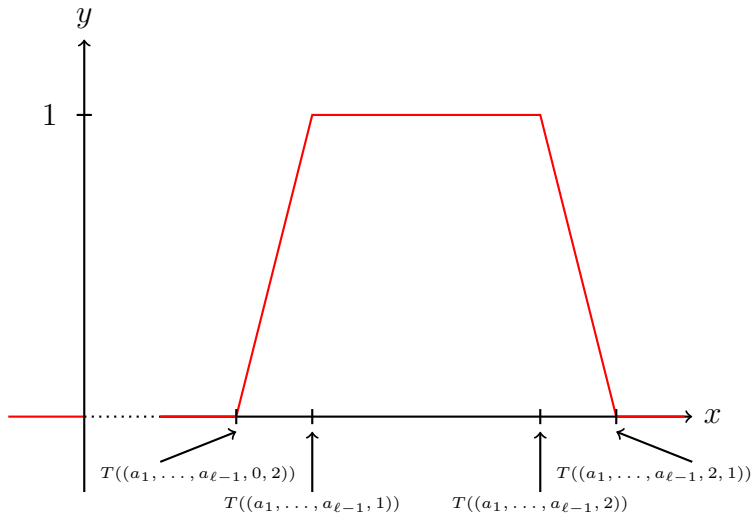
\begin{figure}[H]
						\centering
						\begin{tikzpicture}
						    % Axes
						    \draw[-, thick, red] (-1,0) -- (0,0) node[right]{};
						    \draw[dotted, thick] (0,0) -- (1,0) node[right]{};
						    \draw[->, thick] (1,0) -- (8,0) node[right] {$x$};
						    \draw[->, thick] (0,-1) -- (0,5) node[above] {$y$};
						    
						    % Trapezoid
						    \draw[red, thick] (1,0) -- (2,0) -- (3,4) -- (6,4) -- (7,0) -- (7.9,0);

						    % Ticks
						    \draw[-, thick] (2,-0.1) -- (2,0.1);
						    \draw[-, thick] (3,-0.1) -- (3,0.1);
						    \draw[-, thick] (6,-0.1) -- (6,0.1);
						    \draw[-, thick] (7,-0.1) -- (7,0.1);
						    
						     % Labels with arrows
						    \draw[->, thick] (1,-0.6) -- (2,-0.2) node[midway, below = 1mm, font=\tiny] {$T(( a_1,\dots,a_{\ell-1},0,2))$};
						    \draw[->, thick] (3,-1) -- (3,-0.2) node[midway, below =3mm, font=\tiny] {$T((a_1,\dots,a_{\ell-1},1) )$};
						    \draw[->, thick] (6,-1) -- (6,-0.2) node[midway, below = 3mm, font=\tiny] {$T((a_1,\dots,a_{\ell-1},2))$};
						    \draw[->, thick] (8,-0.6) -- (7,-0.2) node[midway, below = 1mm, font=\tiny] {$T((a_1,\dots,a_{\ell-1},2,1) )$};
						    % Y axis 
						    \draw[-, thick] (-0.1,4) -- (0.1,4) node[left=3mm] {1};
						\end{tikzpicture}
						\caption{The function $r^{( a_1,\dots, a_{\ell - 1} )}$.}
						\label{fig:indication_function}
					\end{figure}

					% \begin{figure}[H]
					% 	\centering
					% 	\includegraphics[width=0.4\textwidth]{images/indication_function.jpg}
					% 	\caption{Image of $h^{\{ a_1,\dots, a_{\ell - 1} \}}$.}
					% 	\label{fig:indication_function}
					% \end{figure}

					\item The ReLU networks realizing $m^\ell: \mathbb{R} \mapsto \mathbb{R}$ are given by $m^\ell  ( z ) = \rho( z - (\ell -1)) - \rho(z - \ell)$, $z \in \mathbb{R}$. It is readily seen that
					\begin{equation}
					\label{eq:property_m_ell}
						m^\ell ( \bit ) = 1_{\{\ell\, \leq\, \bit\} }.
					\end{equation}

				% 	\item Upon noting that $\rho( c + d -1) = c d$ for $c,d \in \{ 0,1 \}$, we have 
				% 	\begin{equation}
				% 		1_{\ell \leq j} 1_{\theta_\ell =1} = \rho (1_{\ell \leq j} +   1_{\theta_\ell =1} -1 ).
				% 	\end{equation}. 

				% 	This, together with the ReLU network realization for $1_{i \leq j}$ and $1_{\theta_i =1}$ identified in (a) and (b), respectively, yields a ReLU network realization for $1_{i \leq j} 1_{\theta_i =1}$ of depth 3, width  $2 + 2^{i+1} \leq 2^{i+2}$, and weight magnitude upper-bounded than $3^{i+2}$. Formally, there exists $t^{i}: \mathbb{R}^2 \mapsto \mathbb{R}$ by given by 
				% 	\begin{equation*}
				% 		t^{i} ( y,z ) = \rho \biggl(\rho( z - (i -1)) - \rho(z - i) + \sum_{j = 1}^{2^{i+1}} u_{i,j} \rho (y - v_{i,j}) - 1  \biggr)
				% 	\end{equation*}
				% 	such that $t^i ( \sum_{i = 1}^\infty \theta_{i} 3^{-i}, j ) = \sum_{i = 1}^{N} 1_{i \leq j} 1_{\theta_i =1}$.
				\end{enumerate}
				% As the choice of $\ell \in \{ 1,\dots, N \}$ is the arbitrary, we have defined $h^\ell, m^\ell$, such that $h^\ell ( \sum_{i = 1}^\infty \theta_{i} 3^{-i}) = 1_{\theta_\ell = 1}$ and $m^\ell ( \bit ) = 1_{\ell\, \leq\, \bit} $, $\ell = 1,\dots, N$. 
				Now, define $g^1: \mathbb{R}^3 \mapsto \mathbb{R}$ according to
				\begin{align}
					g^1 ( x,y,z ) & :=  \rho(x) + \sum_{\ell = 1}^N \rho ( h^\ell ( y ) + m^{\ell} ( z ) - 1  ) \label{eq:expression_g_1_1}\\
					& =  \rho ( \rho(x) ) + \sum_{\ell = 1}^N \rho \biggl(\sum_{j = 1}^{2^{\ell+1}} u_{\ell,j} \rho (y - v_{\ell,j})   + \rho( z - (\ell -1)) - \rho(z - \ell) - 1  \biggr), \label{eq:expression_g_1_2}
				\end{align}
                for $x,y,z \in \mathbb{R}$,    
				where in \eqref{eq:expression_g_1_2} we used the ReLU network realizations of $h^\ell$ and $m^\ell$, $\ell = 1,\dots, N$, along with $\rho\circ \rho = \rho$. This shows that $g^1$ can be realized by a ReLU network of depth $3$, input dimension $3$, with $1 + \sum_{\ell = 1}^N ( 2^{\ell+1} + 2 ) \leq 2^{N+3}$ nodes in the first layer, $N+1$ nodes in the second layer, and, owing to $| u_{\ell,j} |, | v_{\ell,j} | \leq 3^{\ell + 1} \leq 3^{N+1} $, $\ell= 1,\dots, N$, $j = 1, \dots, 2^{\ell +1}$,
    weight magnitude upper-bounded by $3^{N+1}$.
    %upon recalling that $| u_{\ell,j} |, | v_{\ell,j} | \leq 3^{\ell + 1} \leq 3^{N+1} $, $\ell= 1,\dots, N$, $j = 1, \dots, 2^{\ell +1}$. 
    Formally, we have established that
				\begin{equation}
				\label{eqline:g_1_complexity}
					g^1 \in \mathcal{R} ( ( 3,1 ), 2^{N+3}, 3, 3^{N+1}). 
				\end{equation}
				We finally note that
				\begin{align}
					g^1 \biggl( \zero,\sum_{i = 1}^\infty \theta_{i} 3^{-i},\bit \biggr) =&\, \rho(\zero) + \sum_{\ell = 1}^N \rho \biggl( h^\ell \biggl( \sum_{i = 1}^\infty \theta_{i} 3^{-i} \biggr) + m^{\ell} ( \bit ) - 1  \biggr) \\
					=&\, \rho(\zero) + \sum_{\ell = 1}^N \rho ( 1_{\{\theta_\ell = 1\}} + 1_{\{\ell\, \leq\, \bit\}} - 1)\label{eqline:g_1_1} \\
					=&\, \rho(\zero) + \sum_{\ell = 1}^N 1_{\{\theta_\ell = 1\}} 1_{\{\ell\, \leq\, \bit\}} \label{eqline:g_1_2} \\
					=&\, \rho(\zero) +  \sum_{i = 1}^{\min \{ N,\bit \}} \theta_{i},
				\end{align}
				as desired, where in \eqref{eqline:g_1_1} we used %the properties of $h^\ell$ and  $m^\ell$ given in 
    \eqref{eqline:property_h_ell_1}-\eqref{eqline:property_h_ell_2} and \eqref{eq:property_m_ell}, and \eqref{eqline:g_1_2} follows from $\rho( c + d -1) = c d$, for $c,d \in \{ 0,1 \}$. As the choice of  $\zero \in \mathbb{R}$, $\sum_{i = 1}^\infty \theta_{i} 3^{-i} \in \mathbb{T}$, $\bit \in \mathbb{N} \cup \{ 0 \}$ was arbitrary and the construction of $g^1$ does not depend on $\zero,\sum_{i = 1}^\infty \theta_{i}3^{-i},\bit$, this concludes the argument.

				\item Construction of $g^2$. We first show how to realize the bit shifting operation $\sum_{i = 1}^\infty  \theta_i 3^{-i} \in \mathbb{T} \mapsto \sum_{i = 1}^\infty  \theta_{N+i} 3^{-i} \in \mathbb{T}$ by a ReLU network. This will be accomplished by decomposing the mapping into submappings, realizing the individual submappings by ReLU networks, and then putting these networks together to obtain a ReLU network construction for the overall mapping. Specifically, we work with the decomposition
				\begin{equation}
				 	\sum_{i = 1}^\infty  \theta_{N+ i}\, 3^{-i} = \sum_{(a_1,\dots, a_N) \in \{ 0,1 \}^N} \biggl(1_{\{\theta_1 = a_1, \dots, \theta_N =a_N\} } \sum_{i = 1}^\infty  \theta_{N+ i} \,3^{-i}\biggr), 
				 \end{equation}
				for $\sum_{i=1}^\infty \theta_i 3^{-i} \in \mathbb{T}$. %We proceed to the ReLU network construction of $g_2$. 
                Now, for $a_i \in \{ 0,1 \}$, $i = 1,\dots, N$, consider the function 
				\begin{align*}
					f^{(a_1,\dots, a_N)} (y) & = 3^{N} \rho(y - T ( ( a_1,\dots, a_N ) )) - 7\cdot 3^{N} \rho( y - T ( ( a_1,\dots, a_N, 2) ) ) \\
					& + 6\cdot 3^{N} \rho( y - T ( ( a_1,\dots, a_N,2,1) ) ), \quad y \in \mathbb{R},
				\end{align*} 
				illustrated in Figure~\ref{fig:image_of_f}. %for an illustration of $f^{(a_1,\dots, a_N)}$. 
                For $\sum_{i = 1}^\infty  \theta_i 3^{-i} \in \mathbb{T}$, with $\theta_i = a_i$, for $i = 1,\dots, N$, we have  
				\begin{equation*}
					\sum_{i = 1}^\infty  \theta_i 3^{-i} \in \Biggl[\sum_{i = 1}^N  a_i 3^{-i}, \sum_{i = 1}^N  a_i 3^{-i} + 2\cdot 3^{-(N+1)} \Biggr)= [ T ( ( a_1,\dots, a_N ) ),T ( ( a_1,\dots, a_N, 2 ) )  ),
				\end{equation*}
				which together with the definition of $f^{a_1,\dots, a_N}$ implies 
				\begin{align*}
				 	f^{(a_1,\dots, a_N)} \biggl(\sum_{i = 1}^\infty  \theta_i 3^{-i}\biggr) & = 3^{N} \rho\biggl(\sum_{i = 1}^\infty  \theta_i 3^{-i} - T ( ( a_1,\dots, a_N ) )\biggr)\\
				 	 & =  3^{N} \cdot \sum_{i = {N+1}}^\infty  \theta_i 3^{-i}\\
				 	& =  \sum_{i = 1}^\infty  \theta_{N+i}\, 3^{-i}.
				\end{align*} 
				Moreover, for numbers $\sum_{i = 1}^\infty  \theta_i 3^{-i} \in \mathbb{T}$ whose sequence of $N$-leading digits differ from $a_1,\dots, a_N$, we have $f^{(a_1,\dots, a_N)} (\sum_{i = 1}^\infty  \theta_i 3^{-i}) = 0$, as every number in the support of $f^{(a_1,\dots, a_N)}$, i.e., in $[ T ( ( a_1,\dots, a_N ) ),T ( a_1,\dots, a_N, 2,1 ) )  )$, has leading digits $a_1, \dots, a_N$ in its ternary representation. In summary, we therefore get, for $\sum_{i = 1}^\infty  \theta_i 3^{-i} \in \mathbb{T}$,
				\begin{equation*}
					f^{(a_1,\dots, a_N)} \biggl(\sum_{i = 1}^\infty  \theta_i 3^{-i}\biggr) = 1_{\{\theta_1 = a_1, \dots, \theta_N =a_N\} } \sum_{i = 1}^\infty  \theta_{N+i}\, 3^{-i}.
				\end{equation*}
				% and let $\mathbb{T}_{NL}^{a_1,\dots, a_N}$ be the set of numbers in $\mathbb{T}_{NL}$ with leading digits $a_1\dots a_N$, i.e., numbers of the form $ter \, 0.a_1a_2\dots a_N \theta_{N+1} \dots \theta_{NL}$, $\theta_i \in \left\{ 0,1 \right\}$, $i = N+1,\dots,NL$. Consider function $f_{NL}^{a_1,\dots, a_N} (x) := 3^{N} \rho(x - ter\, 0.a_1\dots a_N) - 7\cdot 3^{N} \rho( x - ter\, 0.a_1\dots a_N 2 ) + 7\cdot 3^{N} \rho( x - ter\, 0.a_1\dots a_N 21 )$ depicted in Fig~\ref{fig:image_of_f}. Then, for $b = ter \, 0.a_1a_2\dots a_N \theta_{N+1} \dots \theta_{NL}$ with $\theta_i \in \left\{ 0,1 \right\}$, we have
				% \begin{align*}
				%  	f_{NL}^{a_1,\dots, a_N} (b) =&\, 3^{N} \rho(b - ter\, 0.a_1\dots a_N)\\
				%  	 =&\, 3^{N} \times ter \, 0.\underbrace{0\dots0}_{N}\theta_{N+1}\dots \theta_{NL}\\
				%  	=&\,  ter \, 0. \theta_{N+1}\dots \theta_{NL}
				% \end{align*} 
				% Moreover, for numbers $b \in \mathbb{T}_{NL}$ whose sequence of N-leading digits differ from $a_1\dots a_N$, $f_{NL}^{a_1,\dots, a_N} (b) = 0$ since any ternary number in the interval $[ter \, 0.a_1\dots a_{N}, 0.a_1\dots,a_N 21]$ has leading digits $a_1 \dots a_N$.
				% \begin{figure}[H]
				% 	\centering
				% 	\includegraphics[width=0.4\textwidth]{images/remove_first_N.jpg}
				% 	\caption{Image of $f^{a_1,\dots, a_N}$.}
				% 	\label{fig:image_of_f}
				% \end{figure}

				\begin{figure}[H]
				\centering
				\begin{tikzpicture}
				    % Axes
				    \draw[-, thick, red] (-1,0) -- (0,0);
				    \draw[dotted, thick] (0,0) -- (1,0);
				    \draw[->, thick] (1,0) -- (10,0) node[right] {$x$};
				    \draw[->, thick] (0,-1) -- (0,5) node[above] {$y$};

				    %tiks
				    \draw[-, thick] (-0.1,4) -- (0.1,4) node[left=3mm] {$T ( (2) )$};
				    
				    % Triangle
				    \draw[red, thick] (1,0) -- (2,0) --(8,4) -- (9,0) -- (10,0);
				   
				    % Ticks
				    \draw[-, thick] (2,-0.1) -- (2,0.1);
				    \draw[-, thick] (8,-0.1) -- (8,0.1);
				    \draw[-, thick] (9,-0.1) -- (9,0.1);
				    
				     % Labels with arrows
				    \draw[->, thick] (2,-0.8) -- (2,-0.2) node[midway, below = 1mm, font=\tiny] {$T(( a_1, a_2, \dots, a_{N} ))$};
				    \draw[->, thick] (8,0.8) -- (8,0.2) node[midway, above = 5mm, left = -8mm, font=\tiny] {$T(( a_1, a_2, \dots, a_{N}, 2 ) )$};
				    \draw[->, thick] (9,-0.8) -- (9,-0.2) node[midway, below = 2mm, font=\tiny] {$T(( a_1, a_2, \dots, a_{N},2,1 ) )$};
				\end{tikzpicture}
				\caption{The function $f^{(a_1,\dots, a_N)}$.}
				\label{fig:image_of_f}
			\end{figure}
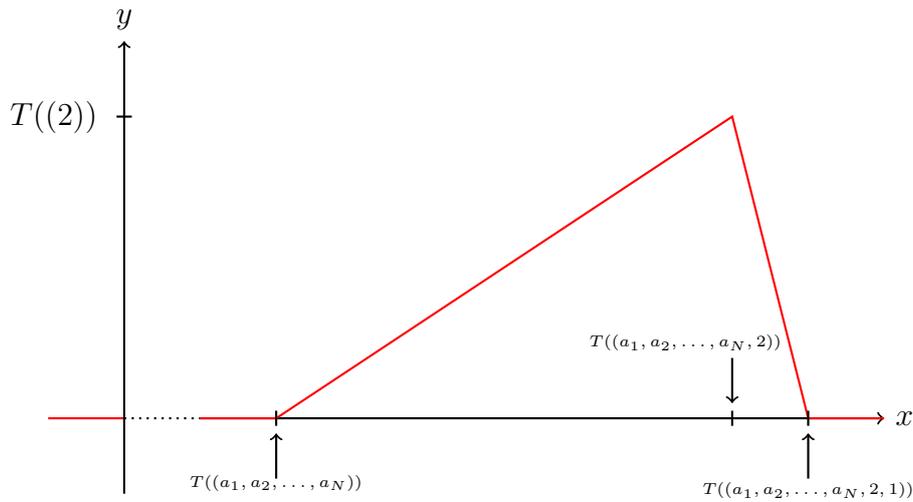

				% Considering all possible $a_i \in \left\{ 0,1 \right\}$, $i = 1,\dots, N$, we have that the function $ter\, 0.\theta_{1}\dots \theta_{NL} \in \mathbb{T}_{NL} \mapsto ter\, 0.\theta_{N+1}\dots \theta_{NL}$ can be written as 

				% Summarizing, we have thus established that, for fixed $a_1,\dots, a_n$, the function $f_{NL}^{a_1,\dots, a_N} (b)$ realizes the desired bit shifting operation \todo{check comments}
				% while being $0$ for all $x$ with $N$ leading digits different from $a_1,\dots,a_N$. 

				% Summarizing, we have thus established that, for fixed $\{ a_1,\dots, a_N \}\in \{ 0,1 \}^N$,the function $f_{N}^{a_1,\dots, a_N}$ realized desired bit shifting operation
				% \begin{equation*}
				% 	\sum_{i =1}^\infty \theta_i 3^{-i} \mapsto \sum_{i =1}^\infty \theta_{N+i} 3^{-i}
				% \end{equation*}
				% for all $\sum_{i =1}^\infty \theta_i 3^{-i} \in \mathbb{T}$ with $N$ leading digits $a_1, \dots, a_N$ while being zero for all $\sum_{i =1}^\infty \theta_i 3^{-i} \in \mathbb{T}$ with $N$ leading digits different from $a_1, \dots, N$. 

				Summing over all $( a_1,\dots, a_N )\in \{ 0,1 \}^N$, we finally obtain $g^2: \mathbb{R}^3 \mapsto \mathbb{R}$, according to
				\begin{align}
					g^2 ( x,y,z ) :=&  \sum_{\substack{ ( a_1,\dots, a_N ) \in \{ 0,1 \}^{N}}} f^{(a_1,\dots, a_N)}(y)  \\
					=& \sum_{\substack{ ( a_1,\dots, a_N ) \in \{ 0,1 \}^{N}}} \bigl ( 3^{N} \rho(y - T ( ( a_1,\dots, a_N ) ) )\label{def:g_2_1}\\
					 &\quad  - 7 \cdot 3^{N} \rho( y - T ( ( a_1,\dots, a_N,2) ) ) \\
					 &\quad + 6 \cdot 3^{N} \rho( y - T ( ( a_1,\dots, a_N,2,1 ) ) )\bigr).\label{def:g_2_2}
				\end{align}
				Then, for all $\zero \in \mathbb{R}$, $\sum_{i = 1}^\infty \theta_{i} 3^{-i} \in \mathbb{T}$, $\bit \in \mathbb{N} \cup \{ 0 \}$, we have 
				\begin{align*}
					g^2 \biggl( \zero ,\sum_{i = 1}^\infty \theta_{i} 3^{-i}, \bit\biggr) =& \sum_{\substack{ ( a_1,\dots, a_N ) \in \{ 0,1 \}^{N}}} f^{(a_1,\dots, a_N)} \biggl( \sum_{i = 1}^\infty \theta_{i} 3^{-i} \biggr) \\
					=&  \sum_{\substack{ ( a_1,\dots, a_N ) \in \{ 0,1 \}^{N}}} \biggl(1_{\{\theta_1 = a_1, \dots, \theta_N =a_N\}} \sum_{i = 1}^\infty  \theta_{N+ i} 3^{-i}\biggr)\\
					=& \sum_{i = 1}^\infty \theta_{N  + i} 3^{-i},
				\end{align*}
				as desired. Moreover, it follows from \eqref{def:g_2_1}-\eqref{def:g_2_2} that $g^2$ can be realized by a 2-layer ReLU network with $3$ nodes in the input layer, $3 \cdot 2^N$ nodes in the first layer, and weight magnitude $7\cdot 3^N$, formally 
				\begin{equation}
				\label{eqline:g_2_complexity}
					g^2 \in ( ( 3,1 ), 3\cdot 2^N, 2, 7 \cdot 3^N ).
				\end{equation}

				% Hence, summing over all $a_i \in \{ 0,1 \} $, $i = 1,\dots, N$, give a ReLU network realization of functions ??? given $\sum_{i = 1}^\infty  \theta_i 3^{-i} \in \mathbb{T} \mapsto \sum_{i = 1}^\infty  \theta_{N+i} 3^{-i}$ according to 
				% \begin{align*}
				% 	x \mapsto&  \sum_{\substack{a_i \in \{ 0,1 \} \\ i = 1,\dots, N}} f_{N}^{a_1,\dots, a_N}(x) \\
				% 	=& \sum_{\substack{a_i \in \{ 0,1 \} \\ i = 1,\dots, N}} \bigl ( 3^{N} \rho(x - ter\, 0.a_1\dots a_N) \\
				% 	 &\quad  - 7\times 3^{N} \rho( x - ter\, 0.a_1\dots a_N 2 ) \\
				% 	 &\quad + 7\times 3^{N} \rho( x - ter\, 0.a_1\dots a_N 21 )\bigr).
				% \end{align*}
				% The resulting ReLU network is hence of depth $2$, width $3 \times 2^N \leq 2^{N+2}$, and weight magnitude upper-bounded by $7\times 3 ^{N} \leq 3^{N+2}$. \todo{See comments 230}

				\item Construction of $g^3$. Setting $g^3 ( x,y,z ) :=  \rho ( z -N )$, for $x,y,z \in \mathbb{R}$, we get $g^3 ( \zero, \sum_{i = 1}^\infty  \theta_i 3^{-i},\allowbreak \bit )  = \rho( \bit -N)$, for all $\zero \in \mathbb{R}$, $\sum_{i = 1}^\infty \theta_{i} 3^{-i} \in \mathbb{T}$, $\bit \in \mathbb{N} \cup \{ 0 \}$. Moreover, $g^3$ can be realized by a ReLU network according to
				\begin{equation}
				\label{eqline:g_3_complexity}
				 	g^3 \in \mathcal{R} ( ( 3,1), 3, 2, N ).
				 \end{equation} 
			\end{enumerate}
			Finally, we note that $g=  ( ( g^1, g^2 ), g^3)$ satisfies \eqref{eq:expression_of_g_bit_extraction}, and application of Lemma~\ref{lem:algebra_on_ReLU_networks}, considering \eqref{eqline:g_1_complexity}, \eqref{eqline:g_2_complexity}, and \eqref{eqline:g_3_complexity}, yields
			% \todo{Now we have three ReLU networks, parallizing according to Lemma, extend, and parallizing} we obtain a ReLU network realization of width $2^{N+3} + 1 + 2^{N+2} \leq 2^{N+4}$, depth $3$, and weight magnitute upperbounded by $3^{N+2}$. and results in 
			\begin{align*}
				g \in&\, \mathcal{R} ( ( 3,3 ),2^{N+3} + 3\cdot 2^{N} + 3, 3,  \max \{ 3^{N+1},7\cdot 3^N, N \}) \\
				\subseteq&\, \mathcal{R} ( ( 3,3 ),2^{N+4}, 3,  3^{N+2}). \qedhere
			\end{align*}

			% \todo{Reformulation the following in a more formal form. }

			% The above construction, together with the flowchart given in Fig~\ref{fig:how_to_compute_1}, shows how to compute $f_{N,L}$ by a ReLU network given how to compute $f_{N-1,L}$ by a ReLU networks. Therefore, by induction, we have constructed how to compute $f_{N,L}$ by a ReLU network. Here we expand the flowchart in Fig~\ref{fig:how_to_compute_1} as the following figure to give a clear view, where the whole computation is decomposed into $L$ similar steps. According to our construction as above, the first step can be computed by a ReLU network of depth $3$, width $ 2^{N+2} + N 2^{N+2} + 1 \leq (N+2) 2^{N+2}$, and weight magnitude less than $3^{N+2}$. Due to the similarity between these steps, the whole computation can be done by a ReLU network of depth $3L$, width $ 2^{N+2} + N 2^{N+2} + 1 \leq (N+2) 2^{N+2}$, and weight magnitude less than $3^{N+2}$.

			% \begin{figure}[H]
			% 	\centering
			% 	\includegraphics[width=0.4\textwidth]{images/flowchart_expanded.jpg}
			% 	\caption{Expanded flowchart}
			% 	\label{fig:flowchart_expanded}
			% \end{figure}
		% \end{proof}

		\subsection{Proof of Lemma~\ref{lem:C2}} % (fold)
		\label{sub:proof_of_lemma_lem:c2}
		
			We prove \eqref{eq:C2_1} and \eqref{eq:C2_2} by induction on $L \in \mathbb{N}$. The base case $L = 1$ follows from $G_{N, 1} = g$, with $g$ per Lemma~\ref{lem:constructive_proof_of_g} and the properties of $g$ as specified in \eqref{eq:g_complexity} and \eqref{eq:expression_of_g_bit_extraction}.

			For the induction step $L - 1 \mapsto L$, $L \geq 2$, we start with the induction assumption
			\begin{equation}
			\label{eq:induction_assumption_G_N0}
				G_{N,L - 1} \in \, \mathcal{R} ( (3,3),2^{N+4},3 ( L-1 ),3^{N+2}),
			\end{equation}
			and, for $\sum_{i = 1}^\infty \theta_{i} 3^{-i} \in \mathbb{T}$, $\bit \in \mathbb{N} \cup \{ 0 \}$,
			\begin{equation}
				\label{eq:induction_assumption_G_N}
				G_{N,L-1} \biggl(0, \sum_{i = 1}^\infty \theta_i 3^{-i}, \bit  \biggr) = \biggl(\,\sum_{i = 1}^{\min \{ N ( L -1 ),\bit \}} \theta_i ,  \sum_{i =1}^\infty \theta_{N ( L -1 ) + i}\, 3^{-i},  \max \{ \bit - N(L-1),\, 0 \}   \biggr).
			\end{equation}

			Next, note that $G_{N,L } = g \circ G_{N, L - 1}$, and both $g$ and $G_{N, L - 1}$ can be realized by ReLU networks according to \eqref{eq:g_complexity} and \eqref{eq:induction_assumption_G_N0}. Application of Lemma~\ref{lem:algebra_on_ReLU_networks} yields
			\begin{align*}
				G_{N,L} \in &\, \mathcal{R} ( (3,3),2^{N+4},3 L,3^{N+2}).
			\end{align*}
			% Upon noting that $G_{N,L}$ $L \in \mathbb{N}$, can be written as composition of functions realized by ReLU networks according to \eqref{eq:complexity_G_N}, application of the parts of Lemma~\ref{lem:algebra_on_ReLU_networks} regarding composition of ReLU networks realization yields, for $L \in \mathbb{N}$,
			% \begin{align*}
			% 	G_{N,L} \in &\, \mathcal{R} ( (3,3),2^{N+4},3 L,3^{N+2}).
			% \end{align*}
			% Next, that, for all $L \in \mathbb{N}$, for $\sum_{i = 1}^\infty \theta_{i} 3^{-i} \in \mathbb{T}$, and $\bit \in \mathbb{N} \cup \{ 0 \}$,
			% 	\begin{equation*}
			% 		G_{N,L} \biggl(0, \sum_{i = 1}^\infty \theta_i 3^{-i}, \bit  \biggr) = \biggl(\sum_{i = 1}^{\min \{ NL,\bit \}} \theta_i ,  \sum_{i =1}^\infty \theta_{NL + i} 3^{-i},  \max \{ \bit - NL, 0 \}   \biggr),
			% 	\end{equation*}
			Furthermore, for $\sum_{i = 1}^\infty \theta_{i} 3^{-i} \in \mathbb{T}$, $\bit \in \mathbb{N} \cup \{ 0 \}$, we have
				\begin{align}
					G_{N,L} \biggl( 0,\sum_{i = 1}^\infty \theta_i 3^{-i}, \bit \biggr) & =  g \biggl(   G_{N,L-1} \biggl(  0,\sum_{i = 1}^\infty \theta_i 3^{-i}, \bit \biggr)  \biggr) \\
					& = g \biggl(\sum_{i = 1}^{\min \{ N ( L -1 ),\bit \}} \theta_i ,  \sum_{i =1}^\infty \theta_{N ( L -1 ) + i} \, 3^{-i},  \max \{ \bit - N(L-1), 0 \}   \biggr) \label{eqline:apply_def_G_ell_1}\\ 
					& = \biggl(  \rho \biggl ( \sum_{i = 1}^{\min \{ N ( L -1 ),\bit \}} \theta_i \biggr) + \sum_{i =1}^{\min \{N,\, \max \{ \bit - N(L-1), 0 \}\}} \theta_{N ( L -1 ) + i},  \nonumber \\
					&\: \: \sum_{i =1}^\infty \theta_{N + N ( L -1 ) + i} \, 3^{-i}, \max \{ \max \{ \bit - N(L-1), 0 \} - N ,0 \} \biggr) \label{eqline:apply_def_g} \\
					& = \biggl(\sum_{i = 1}^{\min \{ NL,\bit \}} \theta_i ,  \sum_{i =1}^{\infty} \theta_{NL + i} \, 3^{-i},  \max \{ \bit - NL, 0 \}   \biggr),
				\end{align}
				where \eqref{eqline:apply_def_G_ell_1} follows from the induction assumption \eqref{eq:induction_assumption_G_N}, and in \eqref{eqline:apply_def_g} we used \eqref{eq:expression_of_g_bit_extraction}. This finishes the proof.
		% subsection subsection_name (end)

	% subsection bit_extraction (end)

	%!TEX root = ../draft_quantized_weight_networks.tex

\section{Proof of Lemma~\ref{lem:quantization_error}} % (fold)
		\label{sub:quantization_error}
		Lemma~\ref{lem:quantization_error} is a special case, with input dimension $d = 1$ and weight magnitude $B = 1$, of the more general result Lemma~\ref{lem:quantization_error_general_dimension} stated and proved here. We provide this more general result as its proof is no longer than that for $d = 1$ and $B = 1$, while offering deeper insights into the mechanisms at play.

		\begin{lemma}
			\label{lem:quantization_error_general_dimension}
			Let $d,W,L, \ell \in \mathbb{N}$ with $W \geq d$ and $\ell \leq L$, $B \in \mathbb{R}_+$ with $B \geq 1$, and let 
			\begin{equation*}
				\Phi^i = ( ( A_j^i,\nmathbf{b}_j^i ) )_{j =1 }^{\ell} \in \mathcal{N} ( ( d,1 ),W,L,B ), \quad i = 1,2,
			\end{equation*}
			have the same architecture.
			% , where $  A_j^1$ and $ A_j^2$ share the same dimension, and   $ \nmathbf{b}_j^1$ and $ \nmathbf{b}_j^2$ have the same length, for $j = 1,\dots ,\ell$.
			Then,
			\begin{equation}
			\label{eq:quantization_error_general_dimension_000}
				\|  R ( \Phi^1 )  -  R ( \Phi^2 )  \|_{L^\infty ( [0,1]^d )} \leq L (W+1)^L B^{L-1} \| \Phi^1 - \Phi^2 \|.
			\end{equation}
			% where 
			% \begin{equation}
			% \label{_new_eq:weightwise_difference}
			% 	\| \Phi^1 - \Phi^2 \| := \max_{j = 1}^\ell \max \bigl\{ \| A_j^1 - A_j^2\|_\infty,  \| \nmathbf{b}_j^1 - \nmathbf{b}_j^2\|_\infty  \bigr\} .
			% \end{equation}

			\begin{proof}
				Fix $x \in [0,1]^d$. For $i = 1,2$ and $k = 1,\dots, \ell$, let $y^i_k(x) :=\Phi ( ( ( A_j^i,b_j^i ) )_{j =1 }^{k} )(x)$, denote the output of the $k$-th layer of $\Phi^i$.
				% which is the output, given input $x$, of the function realized by $\Phi^i$ truncated at its $k$-th layer. 
				% The vector $\nmathbf{y}^i_k$  depends on $x$, but for the sake of simplicity, we will omit the explicit dependency on $x$ in our notation.  We shall upper-bound $\nleft\| y_k^1 - y_k^2 \nright\|_\infty$, for  $k = 1,\dots, \ell$. This will, in turn, give us an upper bound on $|  R ( \Phi^1 )(x)  -  R ( \Phi^2 )(x)  | = \|  R ( \Phi^1 )(x)  -  R ( \Phi^2 )(x)  \|_{\infty} =|y^1_\ell - y^2_\ell  | = \|y^1_\ell - y^2_\ell  \|_\infty$.

				We start with a preparatory result upper-bounding $\| y^1_k (x) \|_\infty $, for $k = 1,\dots, \ell$. Specifically, we prove, by induction, that 
				\begin{equation}
					\label{_new_eq:proof_quantization_0}
				 	\| \nmathbf{y}^1_k (x) \|_\infty \leq (W+1)^k B^k,
				\end{equation} 
				for $k =1,2,\dots, \ell$. The base case  $k = 1$ follows by noting that
				% \begin{equation}
				% \label{_new_eq:induction_assumption_1}
				\begin{align}
					\| \nmathbf{y}^1_1 (x) \|_\infty =&\, \| A_1^1 \nmathbf{x} + \nmathbf{b}^1_1   \|_\infty \\
					\leq &\, W \| A_1^1 \|_\infty \| \nmathbf{x} \|_\infty + \| \nmathbf{b}^1_1 \|_\infty \label{eqline:induction_assumption_width}\\
					\leq &\, WB + B = (W + 1) B,
				\end{align}
				where in \eqref{eqline:induction_assumption_width} we used the fact that  $A_1^1$ has at most $W$ columns.
				% \end{equation}
				We proceed to establish the induction step $k-1 \mapsto k$ with the induction assumption given by
				\begin{equation}
					\label{_new_eq:induction_assumption}
				 	\| \nmathbf{y}^1_{k - 1} (x) \|_\infty \leq (W+1)^{k - 1} B^{k - 1}.
				\end{equation} 
				As
				\begin{align}
					\| \nmathbf{y}^1_k (x) \|_\infty  =&\, \| A^1_k \, \rho(\nmathbf{y}_{k-1}^1 (x))  + \nmathbf{b}_k^1 \|_\infty\\
					\leq &\, W\| A_k^1 \|_\infty \| \rho(\nmathbf{y}_{k-1}^1 (x))  \|_\infty + \| b_k^1 \|_\infty \label{eqline:proof_quantization_1} \\
					\leq &\, W\| A_k^1 \|_\infty \| \nmathbf{y}_{k-1}^1 (x) \|_\infty + \| b_k^1 \|_\infty  \label{eqline:proof_quantization_2}\\
					\leq &\,  WB (W+1)^{k - 1} B^{k - 1} + B  \label{eqline:proof_quantization_3} \\
					\leq &\,(W+1)^k B^k,\label{eqline:proof_quantization_31}
				\end{align}
				where in \eqref{eqline:proof_quantization_1} we used that  $A_k^1$ has at most $W$ columns, \eqref{eqline:proof_quantization_2} follows from the $1$-Lipschitz continuity of $\rho$, in \eqref{eqline:proof_quantization_3} we employed the induction assumption  \eqref{_new_eq:induction_assumption}, and \eqref{eqline:proof_quantization_31} is by $B \geq 1$. %, we complete the induction.
				% For $k =2,\dots, \ell$, if $\| \nmathbf{y}_{k-1} \|_\infty \leq (W+1)^{k-1} \leq (W+1)^{k-1}B^{k-1}$, then it follows from \eqref{_new_eq:induction1} that $\| \nmathbf{y}^1_k \|_\infty \leq WB(W+1)^{k-1}B^{k-1} +B \leq (W+1)^k B^k $. By induction, this, together with \eqref{_new_eq:induction_assumption_1}, implies $\| \nmathbf{y}^1_k \|_\infty \leq (W+1)^k B^k $ for $k =1,2,\dots, \ell$. Similarly, we have $\| \nmathbf{y}^2_k \|_\infty \leq (W+1)^k B^k $ for $k =1,2,\dots, \ell$.

				Next, we bound the difference $\| y^1_k (x) - y^2_k (x) \|_\infty$, for $k = 1,\dots, \ell$.  Specifically, we show, again by induction, that
    			\begin{equation*}
					\| y^1_k (x) - y^2_k (x) \|_\infty  \leq k (W+1)^{k} B^{k-1} \| \Phi^1 - \Phi^2 \|,
				\end{equation*}
				for $k = 1,\dots, \ell$. The base case $k = 1$ follows according to
				\begin{equation}
				% \label{_new_eq:induction_assumption_2}
				\begin{aligned}
					\| \nmathbf{y}^1_1 (x) - \nmathbf{y}^2_1 (x) \|_\infty = &\, \| ( A^1_1 - A_2^1 )\nmathbf{x} + ( \nmathbf{b}_1^1 - \nmathbf{b}_2^1 ) \|_\infty\\
					\leq &\,  W \|  A^1_1 - A_2^1 \|_\infty \|\nmathbf{x}\|_\infty + \|\nmathbf{b}_1^1 - \nmathbf{b}_2^1 \|_\infty \\
					\leq &\, W \| \Phi^1 - \Phi^2 \|  + \| \Phi^1 - \Phi^2 \|\\
					\leq &\, (W+1) \| \Phi^1 - \Phi^2 \|.
				\end{aligned}
				\end{equation}
				To establish the induction step $k-1 \mapsto k$ starting from the induction assumption given by 
				\begin{equation}
					\label{_new_eq:induction_assumption_2}
				 	\| y^1_{k-1} (x) - y^2_{k-1} (x) \|_\infty  \leq (k-1) (W+1)^{k-1} B^{k-2} \| \Phi^1 - \Phi^2 \|,
				\end{equation} 
				we note that
    				\begin{align}
					&\| \nmathbf{y}^1_k (x) - \nmathbf{y}^2_k (x) \|_\infty \\
					&= \, \| A^1_k \rho(\nmathbf{y}_{k-1}^1 (x)) + \nmathbf{b}_{k}^1 - A^2_k \rho(\nmathbf{y}_{k-1}^2 (x)) - \nmathbf{b}_{k}^2 \|_\infty \label{eqline:definition_y_k}\\
					&\leq \, \| ( A_k^1 - A_k^2 ) \rho ( \nmathbf{y}_{k-1}^1 (x) ) \|_\infty + \| \nmathbf{b}_{k}^1  - \nmathbf{b}_{k}^2\|_\infty \nonumber \\
					&\, +  \| A_k^2 ( \rho(\nmathbf{y}_{k-1}^1 (x)) -  \rho(\nmathbf{y}_{k-1}^2) (x)  ) \|_\infty \label{eqline:apply_triangle}\\
					&\leq \, W \|  A_k^1 - A_k^2  \|_\infty \| \rho ( \nmathbf{y}_{k-1}^1 (x) ) \|_\infty +  \| \nmathbf{b}_{k}^1  - \nmathbf{b}_{k}^2\|_\infty  \nonumber \\
					&\, + W \| A_k^2 \|_\infty \| \rho(\nmathbf{y}_{k-1}^1 (x)) -  \rho(\nmathbf{y}_{k-1}^2 (x))   \|_\infty \\
					&\leq \, W \| \Phi^1 - \Phi^2 \|  (W+1)^{k-1} B^{k-1} + \| \Phi^1 - \Phi^2 \| \nonumber \\
                    &\,\,+ W B \| \nmathbf{y}_{k-1}^1 (x) -  \nmathbf{y}_{k-1}^2 (x) \|_\infty   \label{eqline:applying_assumption} \\
					&\leq \,  W \| \Phi^1 - \Phi^2 \|  (W+1)^{k-1} B^{k-1} + \| \Phi^1 - \Phi^2 \| \nonumber \\
                    &\,\,+ W B (k-1) (W+1)^{k-1} B^{k-2} \| \Phi^1 - \Phi^2 \| \label{eqline:applying_assumption_2} \\
					&\leq \,  k (W+1)^{k} B^{k-1} \| \Phi^1 - \Phi^2 \|,
				\end{align}
				where \eqref{eqline:apply_triangle} follows from the triangle inequality, in \eqref{eqline:applying_assumption} we used (i) $\| ( A_k^1 - A_k^2 ) \|_\infty$, $\| \nmathbf{b}_{k}^1  - \nmathbf{b}_{k}^2\|_\infty \leq \| \Phi^1 - \Phi^2 \|$, (ii) $\| \rho ( \nmathbf{y}_{k-1}^1 (x) ) \|_\infty \leq \| \nmathbf{y}_{k-1}^1 (x)  \|_\infty \leq (W+1)^{k-1} B^{k-1}$ owing to \eqref{_new_eq:proof_quantization_0},  (iii) $\| A_k^2 \|_\infty \leq B$, and (iv) $\| \rho(\nmathbf{y}_{k-1}^1 (x)) -  \rho(\nmathbf{y}_{k-1}^2 (x))   \|_\infty \leq \| \nmathbf{y}_{k-1}^1 (x) -  \nmathbf{y}_{k-1}^2 (x) \|_\infty$  thanks to the $1$-Lipschitz continuity of $\rho$, and \eqref{eqline:applying_assumption_2} follows from the induction assumption \eqref{_new_eq:induction_assumption_2} along with $W+1, B \geq 1$.
				% Suppose for $k \in \{ 2,\dots, \ell \}$ we can show $ \|  \nmathbf{y}^1_{k-1} - \nmathbf{y}^2_{k-1}  \|_\infty \leq  (k-1) (W+1)^{k-1} B^{k-2} \varepsilon$, then the above computation show 
				% \begin{align}
				% 	&\| \nmathbf{y}^1_k - \nmathbf{y}^2_k \|_\infty \\
				% 	\leq &\, W \varepsilon  (W+1)^{k-1} B^{k-1} + \varepsilon + W B \| \nmathbf{y}_{k-1}^1 -  \nmathbf{y}_{k-1}^2 \|_\infty \\
				% 	\leq & \, W \varepsilon  (W+1)^{k-1} B^{k-1} + \varepsilon \nonumber\\
				% 	&+ W B  (k-1) (W+1)^{k-1} B^{k-2} \varepsilon\\
				% 	\leq  &\, k (W+1)^{k} B^{k-1} \varepsilon.
				% \end{align}
				In particular, we get
				\begin{equation}
				\label{_new_eq:proof_quantization_3}
				\begin{aligned}
					\| R ( \Phi^1 ) ( \nmathbf{x} ) - R ( \Phi^2 ) ( \nmathbf{x} ) \|_\infty = &\, \| \nmathbf{y}^1_\ell (x) - \nmathbf{y}^2_\ell (x) \|_\infty\\
					% \leq &\,  \ell (W+1)^{\ell} B^{\ell-1} \varepsilon \\
					\leq &\, \ell (W+1)^{\ell} B^{\ell-1} \| \Phi_1 - \Phi_2 \|\\
					\leq &\, L (W+1)^{L} B^{L-1} \| \Phi_1 - \Phi_2 \|,
				\end{aligned}
				\end{equation}
				where the last step follows from $\ell \leq L$ with $W+1, B \geq 1$. The proof is concluded by noting that \eqref{_new_eq:proof_quantization_3}  holds for all $\nmathbf{x} \in [0,1]^d$.
			\end{proof}
		\end{lemma}

\section{Proof of Proposition~\ref{prop:tradeoff_binary}} % (fold)
\label{sec:proof_of_proposition_prop:tradeoff_binary}

	We prove a result, Proposition~\ref{thm:general_representation}, that is more general than Proposition~\ref{prop:tradeoff_binary}, namely we consider general weight sets and general input-output dimensions. This result is then particularized to the setting of Proposition~\ref{prop:tradeoff_binary}. %Concretely, we show the following. 

	% , which will enable the discussion on the depth-precision tradeoff for ReLU network with general weight set 
	% \footnote{will be discussed in the unfinished chapter Appendix~\ref{sec:approximation_error_upper_bounds_of_relu_networks_with_general_quantized_weights_}}. 
	% The formulation is as follows.

	% \todo{why do we consider this general cases? }
	\begin{proposition}
		\label{thm:general_representation}
		Let $d,d',W,L \in \mathbb{N}$, and let $\mathbb{A} \subseteq \mathbb{R}$ be a finite set satisfying $\{ -1,0,1 \} \subseteq \mathbb{A}$. Then, for every $k \in \mathbb{N}$ and all $u,v \in \mathbb{A} \cap \mathbb{R}_{\geq 0}$, it holds that
		\begin{equation*}
			\mathcal{R}_{\mathcal{T}_1 ( \mathbb{A},u,v,k)} (( d,d' ),  W,  L  ) \subseteq  \mathcal{R}_\mathbb{A} ( ( d,d' ), 16W,(k+3) L )
		\end{equation*}
		with
		\begin{equation}
			\label{eq:def_Auvk}
			\mathcal{T}_1 ( \mathbb{A},u,v,k) := \biggl\{ \sum_{i = 0}^{k} (u^i \alpha_i + v^i \beta_i) : | \alpha_i |,| \beta_i | \in \mathbb{A} , i = 0,\dots, k \biggr\}.	
		\end{equation} 
		\begin{proof}
			See Appendix~\ref{sub:representation}.
		\end{proof}
	\end{proposition}
	Proposition~\ref{thm:general_representation}  illustrates that a network with weights in the  set $\mathcal{T}_1 ( \mathbb{A},u,v,k)$ can equivalently be realized by networks with weights in the simpler underlying set $\mathbb{A}$, at the cost of increased network depth and width. We next demonstrate how Proposition~\ref{prop:tradeoff_binary} follows from Proposition \ref{thm:general_representation}.

	\begin{proof}
		[Proof of Proposition~\ref{prop:tradeoff_binary}] Let $a,b \in \mathbb{N}$. For $k =1$, \eqref{eq:embedding_binary_weights} is trivially satisfied.  For $k \geq 2$, we note that
		\begin{align}
			&\,\mathcal{T}_1 (\mathbb{Q}^a_b,2^{-b},2^a,k-1) \label{eqline:def_A_binary_10} \\
			&=\, \Biggl\{ \sum_{i = 0}^{k-1} (2^{-bi} \alpha_i + 2^{ai} \beta_i) : | \alpha_i |,| \beta_i | \in \mathbb{Q}_{b}^a , i = 0,\dots, k-1 \Biggr\} \\ 
			&=\, \Biggl\{ \sum_{i = 0}^{k-1} (2^{-bi} \alpha_i + 2^{ai} \beta_i) : \alpha_i , \beta_i \in \mathbb{Q}_{b}^a , i = 0,\dots, k-1 \Biggr\} \label{eqline:def_A_binary_1}\\ 
			&\supseteq\, \Biggl\{ \pm \sum_{i = -kb}^{ka} 2^{i} c_i: c_i \in \{ 0,1 \} \Biggr\}\label{eqline:def_A_binary_2}\\
			&= \, \mathbb{Q}_{kb}^{ka}, \label{eqline:def_A_binary_19}
		\end{align}
		where in \eqref{eqline:def_A_binary_2} we used $\mathbb{Q}^{a}_b =  \{ \pm \sum_{i = -b}^a \theta_i 2^{i}: \theta_i \in \{ 0,1 \} \}$. Based on \eqref{eqline:def_A_binary_10}-\eqref{eqline:def_A_binary_19}, it now follows that 
		\begin{equation}
			\label{eq:proof_tradeoff_inclusion_1}
			\mathcal{R}^{ka}_{kb} (W,  L) \subseteq \mathcal{R}_{\mathcal{T}_1 (\mathbb{Q}^a_b,2^{-b},2^a,k-1)} (W,  L  ).
		\end{equation}
		Application of Proposition~\ref{thm:general_representation} with $d= d' = 1$, $u = 2^{-b}$, $v = 2^{a}$, $\mathbb{A} = \mathbb{Q}_b^a$, and $k$ replaced by $k-1$, yields
		\begin{equation}
			\label{eq:proof_tradeoff_inclusion_2}
			\begin{aligned}
			\mathcal{R}_{\mathcal{T}_1 (\mathbb{Q}_b^a,2^{-b},2^a,k-1)} ( W,  L  ) 
			\subseteq&\,   \mathcal{R}_b^a (16W,(k+2)L).
			\end{aligned}
		\end{equation}
		The proof is finalized by combining \eqref{eq:proof_tradeoff_inclusion_1} and \eqref{eq:proof_tradeoff_inclusion_2} to obtain  \eqref{eq:embedding_binary_weights}.
		% \begin{equation*}
		% 	\mathcal{R}^{ka}_{kb} (d,  W,  L) \subseteq \mathcal{R}_b^a (d,16W,(k+2)L),
		% \end{equation*}
		% which is \eqref{eq:embedding_binary_weights}.
	\end{proof}

	\subsection{Proof of Proposition~\ref{thm:general_representation}} % (fold)
	\label{sub:representation}
		% Since ReLU networks are computed by the composition of affine mapping followed by the ReLU activation functions, it will suffice for us to show how to represent a linear mapping with ReLU networks whose weights are in $\mathbb{A}$, in order to prove Lemma~\ref{thm:general_representation}.

		% We start by noting that every function in $\mathcal{R}_{\mathcal{T}_1 ( \mathbb{A},u,v,k)} (( d,d' ),  W,  L  )$, $k,d,d',W,L \in \mathbb{N}$, $\mathbb{A} \subseteq \mathbb{R}$ with $\{ -1,0,1 \} \subseteq \mathbb{A}$, $u,v \in \mathbb{A} \cap \mathbb{R}_{\geq 0}$, can be written as composition of some affine mappings $\affine ( A,b )$, with $\coef ( A), \coef ( b ) \in  \mathcal{T}_1 ( \mathbb{A},u,v,k)$ and the ReLU activation function $\rho$, alternatively, according to the definition of ReLU network realization given in Definition~\ref{def:ReLU_networks}. We shall provide a method to realize an affine mapping composed with the ReLU activation function $\affine ( A,b )\circ \rho$, with $\coef ( A), \coef ( b ) \in  \mathcal{T}_1 ( \mathbb{A},u,v,k)$ by composition of some affine mapping $\affine ( G,h )$ with $\coef ( G ) ,\coef ( h ) \in \mathbb{A}$ and the ReLU activation function. 

		We start with a technical lemma, which shows that,  for given $w \in \mathbb{A} \cap \mathbb{R}_{\geq 0}$, every affine mapping $\affine ( A,b )$ with weight set $\coef ( A), \coef ( b )$\footnote{Recall that $\coef(A)$ denotes the set comprising all elements of the matrix $A$, while $\coef(b)$ represents the set containing the entries of the vector $b$.} contained in 
		\begin{equation}
		\label{def:T_2}
			\mathcal{T}_2 ( \mathbb{A}, w,k ) := \biggl\{ \sum_{i = 0}^{k} w^i \alpha_i  : \alpha_i  \in \mathbb{A} \cap \mathbb{R}_{\geq 0} , i = 0,\dots, k \biggr\}, 
		\end{equation}
		can be realized by a composition of affine mappings with weights in $\mathbb{A} \cap \mathbb{R}_{\geq 0}$. 
		% We do announce that sets of the form \eqref{def:T_2} are related to $\mathcal{T}_1 ( \mathbb{A},u,v,k)$ through the relation
		% \begin{align*}
		% 	&\mathcal{T}_1 ( \mathbb{A},u,v,k) \\
		% 	 =& \{ x^{(u,+)} - x^{(u,-)} + x^{(v,+)} - x^{(v,-)}:x^{(u,+)},  x^{(u,-)} \in \mathcal{T}_2 ( \mathbb{A}, u,k ), x^{(v,+)},  x^{(v,-)} \in \mathcal{T}_2 ( \mathbb{A}, v,k )  \},
		% \end{align*}
		% as will be proven later. 

		% \todo{Explain goal is to represent a networks. Replacing each layers by multiple layers (We start with realization affine mappings with coefficients in $\mathbb{A}(u,v,k)$ by composition of $k+1$ affine mappings with coefficients in $\mathbb{A}$ as the following Lemma, which will in turn be used to replace affine mappings in ReLU networks. ) for which we need the following lemma.}
		% Concisely, we can write $\mathcal{T}_1 ( \mathbb{A},u,v,k) = \mathcal{T}_2 ( \mathbb{A}, u,k ) - \mathcal{T}_2 ( \mathbb{A}, u,k ) + \mathcal{T}_2 ( \mathbb{A}, v,k ) - \mathcal{T}_2 ( \mathbb{A}, v,k )$, 
		
		\begin{lemma}
			\label{lem:adaptive_quantization_single_layer_decomposition}
			Let $m,n,k \in \mathbb{N}$, $\mathbb{A} \subseteq \mathbb{R}$ with $\{ 0,1 \} \subseteq \mathbb{A}$, and $w \in \mathbb{A} \cap \mathbb{R}_{\geq 0}$. Let  $A \in \mathbb{R}^{m\times n} $, $\nmathbf{b} \in \mathbb{R}^m$ be such that $\coef ( A ), \coef ( \nmathbf{b} ) \subseteq \mathcal{T}_2 ( \mathbb{A}, w,k )$, with  $\mathcal{T}_2 ( \mathbb{A}, w,k )$ as defined in \eqref{def:T_2}.  Then, there exists a neural network configuration $( ( G_i,h_i ) )_{i=1}^{k+1}\in \mathcal{N}_{\mathbb{A} \cap \mathbb{R}_{\geq 0}} ( ( n,m ), m + n, k+1 )$ such that
			 % its realization with the identity activation function $\id : \mathbb{R} \mapsto \mathbb{R}$, given by $\id ( x  ) = x$, satisfies 
			\begin{equation}
			\label{eq:coefficient_property_0}
			\begin{aligned}
				\affine( G_{k+1},h_{k+1} )\circ\cdots \circ \affine ( G_{1},h_1 ) = \, \affine(A,b).
			\end{aligned}
			\end{equation}
			% Here we recall the notation that  $\affine ( G,h ) (x) := Gx + h$, $x \in \mathbb{R}^{q}$, for $G \in \mathbb{R}^{p\times q}$ and $h \in \mathbb{R}^{p}$, $p,q \in \mathbb{N}$. \todo{linear networks}

			\begin{proof}
			[Proof of Lemma~\ref{lem:adaptive_quantization_single_layer_decomposition}]
				We first note that thanks to $\coef ( A ), \coef ( b) \subseteq \mathcal{T}_2 ( \mathbb{A}, w,k )$, $A$ and $\nmathbf{b}$ can be written in the form
				\begin{equation}
				\begin{aligned}
					A = \sum_{i = 0}^k w^i A_i, \quad \nmathbf{b} = \sum_{i = 0}^k w^i \nmathbf{b}_i,
				\end{aligned}
				\end{equation}
				with $A_i \in \mathbb{R}^{m\times n}$, $b_i \in \mathbb{R}^m$, $\coef (A_i), \coef ( \nmathbf{b}_i )  \in \mathbb{A} \cap \mathbb{R}_{\geq 0}$, $i =0,\dots,k$. Next, set\footnote{Note that, if $k =1$,  then $\{ 1,\dots,k \}\backslash \{ 1\} = \emptyset$, so that no assignment is made in \eqref{eqline:assign_skip_depth_precision_tradeoff}. 
				% The way we formulate  \eqref{eqline:assign_skip_depth_precision_tradeoff} is to unify the construction for the case $k =1$ and  cases $k \geq 2$.  
				}
				% $G_1 = ( \begin{smallmatrix} I_n\\A_{k}\end{smallmatrix} )$, $h_1 = ( \begin{smallmatrix} 0_n\\b_{k} \end{smallmatrix} )$,  $Y_j = ( \begin{smallmatrix} I_n & 0 \\ A_{k - j + 1} & w I_m  \end{smallmatrix} )$, $h_j = ( \begin{smallmatrix} 0_n\\\nmathbf{b}_{k - j + 1} \end{smallmatrix} )$, for every $j \in \{ 1,\dots,k \}\backslash \{ 1\}$ 
				% and $Y_{k+1}
				% = ( \begin{smallmatrix} A_{0} & w I_m  \end{smallmatrix} )$, $h_{k+1} =  \begin{smallmatrix} \nmathbf{b}_{0} \end{smallmatrix} $.
				\begin{alignat}{3}
					G_1 =& \begin{pmatrix} I_n\\A_{k}\end{pmatrix} , &h_1 =&  \begin{pmatrix} 0_n\\b_{k} \end{pmatrix}, \label{eqline:assign_skip_depth_precision_tradeoff_0}\\
					G_j =&  \begin{pmatrix} I_n & 0 \\ A_{k - j + 1} & w I_m  \end{pmatrix} , &h_j =&  \begin{pmatrix} 0_n\\\nmathbf{b}_{k - j + 1} \end{pmatrix},\,  \  j \in \{ 1,\dots,k \}\backslash \{ 1\} ,\label{eqline:assign_skip_depth_precision_tradeoff}\\
					G_{k+1} =&  \begin{pmatrix} A_{0} & w I_m  \end{pmatrix} , \qquad \qquad &h_{k+1} =&\,  \nmathbf{b}_{0}.
				\end{alignat}
				% \footnote{We ignore this assignment if $2 \geq k$, i.e. if $k = 1$. This convention is to avoid the discussion on the corner cases.} 
				We have $( ( G_i,h_i ) )_{i=1}^{k+1}\in \mathcal{N}_{\mathbb{A} \cap \mathbb{R}_{\geq 0}} ( ( n,m ), m + n, k+1 )$ as desired. It remains to verify \eqref{eq:coefficient_property_0}. To simplify notation, we define $L ((  ( G_i,h_i ) )_{i=1}^{j}  )$, $j = 1,\dots, k+1$, recursively, according to 
				\begin{equation}
				L ( ( ( G_i,h_i ) )_{i=1}^{j}  ) = \left\{
				\begin{aligned}
					& \affine ( G_1, h_1 ), &&\text{if } j = 1,\\
					& \affine ( G_j, h_j ) \circ L ( (( G_i, h_i )  )_{i = 1}^{j - 1} ), &&\text{if } j \geq 2,
				\end{aligned}
				\right.
				\end{equation}
				and note that $L ( ( G_i, h_i )_{i = 1}^{k+1} )  = \affine( G_{k+1},h_{k+1} )\circ\cdots \circ \affine ( G_{1},h_1 ) $. The verification of \eqref{eq:coefficient_property_0} will be  effected by proving the following relation by induction. Specifically, for $j = 1,\dots, k$,
				\begin{equation}
				\label{eq:proved_by_induction_tradeoff}
					L (( ( G_i,h_i ) )_{i=1}^{j} ) (x)  = \begin{pmatrix} x\\ \sum_{i = {k - j +1 } }^k  w^{i - k + j -1} A_i x + \sum_{i = {k - j +1 } }^k w^{i - k + j -1}\nmathbf{b}_i\end{pmatrix},\quad x \in \mathbb{R}^n. 
				\end{equation}
				The base case $j=1$ follows from
				\begin{equation*}
					L (( ( G_i,h_i ) )_{i=1}^{1} )(x) = \affine ( G_{1},h_1 ) (x) = \begin{pmatrix} x\\A_{k}x + b_k\end{pmatrix}, \quad x \in \mathbb{R}^n.
				\end{equation*}
				If $k = 1$, the induction step is not needed. For $k \geq 2$, we prove the induction step $j - 1 \mapsto j$ with $2\leq j \leq k$ starting from the induction assumption
				\begin{equation}
				\label{eq:T2_induction_assumption}
					L(( ( G_i,h_i ) )_{i=1}^{j-1} ) (x)  = \begin{pmatrix} x\\ \sum_{i = {k - j + 2 } }^k  w^{i - k + j-2} A_i x + \sum_{i = {k - j +2  } }^k w^{i - k + j-2}\nmathbf{b}_i\end{pmatrix},\, x \in \mathbb{R}^n,
				\end{equation}
				through the following chain of arguments
				\begin{align}
					&\,\,L ( ( ( G_i,h_i ) )_{i=1}^{j} ) (x)\\
					&=\, S ( G_j, h_j ) \circ L (  ( ( G_i,h_i ) )_{i=1}^{j-1} ) (x) \\
					&=\,  \begin{pmatrix} I_n & 0 \\ A_{k - j + 1} & w I_m  \end{pmatrix}  \begin{pmatrix} x\\ \sum_{i = {k - j+2 } }^k  w^{i - k + j-2} A_i x + \sum_{i = {k - j+2 } }^k w^{i - k + j-2}\nmathbf{b}_i\end{pmatrix} +  \begin{pmatrix} 0_n\\\nmathbf{b}_{k - j + 1} \end{pmatrix} \label{eqline:T2_apply_induction_assumption} \\
					&= \,\begin{pmatrix} x\\ \sum_{i = {k - j +1 } }^k  w^{i - k + j -1} A_i x + \sum_{i = {k - j +1 } }^k w^{i - k + j -1}\nmathbf{b}_i\end{pmatrix}, \quad x \in \mathbb{R}^n.
				\end{align}
				%We accomplished the induction. 
                The proof is concluded upon noting that
				\begin{align}
					&\, \, \affine( G_{k+1},h_{k+1} )\circ\cdots \circ \affine ( G_{1},h_1 ) (x)\\
					&=\, S ( G_{k+1}, h_{k+1} ) \circ L ( ( ( G_i,h_i ) )_{i=1}^{k}) (x) \\
					&=\,  \begin{pmatrix} A_{0} & w I_m  \end{pmatrix}  \begin{pmatrix} x\\ \sum_{i = 1  }^k  w^{i -1} A_i x + \sum_{i = 1}^k w^{i -1}\nmathbf{b}_i\end{pmatrix} + \nmathbf{b}_{0}  \label{eq:proof_trade_off_induction_assumption}\\
					&=\, \sum_{i = 0}^k w^i A_i x + \sum_{i = 0}^k w^i b_i \\
					&=\, Ax + b\\
					&=\, S ( A,b ) ( x ),\quad x \in \mathbb{R}^n,
				\end{align}
				where in \eqref{eq:proof_trade_off_induction_assumption} we used \eqref{eq:proved_by_induction_tradeoff} with $j = k$.
			\end{proof}
		\end{lemma}

		% With the help of Lemma~\ref{lem:adaptive_quantization_single_layer_decomposition}, we are now able to realize  affine mapping mapping $\affine ( A,b )\circ \rho$, with $\coef ( A), \coef ( b ) \in  \mathcal{T}_1 ( \mathbb{A},u,v,k)$ through composing some affine mapping $\affine ( G,h )$ with $\coef ( G ) ,\coef ( h ) \in \mathbb{A}$ and the ReLU activation function, alternatively.

		We proceed by incorporating the effect of the ReLU activation function. Specifically, the following lemma is in the style of Lemma~\ref{lem:adaptive_quantization_single_layer_decomposition}, but for  $S (A, b ) \circ \rho$.

		% \todo{Why do we add $\rho$ at the end.}

		\begin{lemma}
			\label{lem:adaptive_quantization_single_layer}
			Let $m,n,k \in \mathbb{N}$, $\mathbb{A} \subseteq \mathbb{R}$ with $\{ -1,0,1 \} \subseteq \mathbb{A}$, and $u,v \in \mathbb{A} \cap \mathbb{R}_{\geq 0}$. Let $A \in \mathbb{R}^{m\times n} $, $\nmathbf{b} \in \mathbb{R}^m$ be such that $\coef ( A ), \coef ( \nmathbf{b} ) \subseteq \mathcal{T}_1 (\mathbb{A},u,v,k)$, with $\mathcal{T}_1 (\mathbb{A},u,v,k)$ as defined in \eqref{eq:def_Auvk}.  Then, there exists a neural network configuration $( ( G_i,h_i ) )_{i=1}^{k+2}$ with $\coef (  ( ( G_i,h_i ) )_{i=1}^{k+2}) \subseteq \mathbb{A}$ and $\mathcal{W} ( ( ( G_i,h_i ) )_{i=1}^{k+2} ) \leq 4 ( m + n )$ such that 
			\begin{align}
				% \affine(A,b) =& \affine( G_{k+2},h_{k+2} )\circ\cdots \circ \affine ( G_{1},h_1 ), \label{eq:coefficient_property_0}\\
				 \affine( G_{k+2},h_{k+2} )\circ \rho \circ \cdots \circ \rho \circ \affine ( G_{1},h_1 ) \circ \rho = \affine(A,b)\circ \rho . \label{eq:coefficient_property_4}
			\end{align}
			 % and 
			 % \begin{equation}
			 % \label{eq:coefficient_property_4}
			 % 	\affine(A,b)\circ \rho = \affine( G_{k+2},h_{k+2} )\circ \rho \circ \cdots \circ \rho \circ \affine ( G_{1},h_1 ) \circ \rho. 
			 % \end{equation}
		\end{lemma}
			% Then $( \rho(\nmathbf{x}), - \rho(\nmathbf{x}) ) \mapsto A \nmathbf{x} + b$, $\nmathbf{x} \in \mathbb{R}^m$, can be realized by a linear network\footnote{i.e. a network with identity activation function.} of depth $k+1$, width at most $8 \max \{ m, n \}$, coefficients in $\mathbb{A}$, such that all the intermediate output are positive. In particular, $( \rho(\nmathbf{x}), - \rho(\nmathbf{x}) ) \mapsto A \nmathbf{x} + b$, $\nmathbf{x} \in \mathbb{R}^m$, can be realized by a ReLU network with the same architectures and weights. \todo{linear network}

		\begin{proof}
		[Proof of Lemma~\ref{lem:adaptive_quantization_single_layer}]
			% where
			% % \begin{enumerate}
			% % 	\item $\coef ( G_{j}), \coef ( h_{j} ) \in \mathbb{A} \cap \mathbb{R}_{\geq 0}, \quad j = 1,\dots, k + 1, \label{eq:coefficient_property_1}$,
			% % 	\item $\coef ( G_{k+2}), \coef ( h_{k+2} ) \in \{ -1,1 \} \label{eq:coefficient_property_2}$
			% % \end{enumerate}
			%  \begin{align}
			%  	&\\
			%  	&,\coef ( G_{k+2}), \coef ( h_{k+2} ) \in \{ -1,1 \} \label{eq:coefficient_property_2}\\
			%  	&\label{eq:coefficient_property_3},
			%  \end{align}

			We first represent $S ( A,b )$ in terms of affine mappings with coefficients in $\mathcal{T}_2 ( \mathbb{A}, u,k )$ or $\mathcal{T}_2 ( \mathbb{A}, v,k )$ and then apply Lemma~\ref{lem:adaptive_quantization_single_layer_decomposition} to these mappings. To this end, we note that  $x = \sum_{i =0}^k {u^i \alpha_i + v^i \beta_i} \in \mathcal{T}_1 ( \mathbb{A},u,v,k)$ with $| \alpha_i |,| \beta_i | \in \mathbb{A}  $, $i = 0,\dots, k$, can be decomposed according to 
			\begin{equation}
			\label{eq:decomposition_scalar}
				x = x^{(u,+)} - x^{(u,-)} + x^{(v,+)} - x^{(v,-)},
			\end{equation}
			with $x^{(u,+)} = \sum_{i = 0}^{k} u^i \max \{ \alpha_i,0 \}$, $x^{(u,-)} = \sum_{i = 0}^{k} u^i \max \{ -\alpha_i,0 \}$, $x^{(v,+)} = \sum_{i = 0}^{k} v^i \max \{ \alpha_i,0 \}$, $x^{(v,-)} = \sum_{i = 0}^{k} v^i \max \{ -\alpha_i,0 \}$, such that 
			\begin{align}
				x^{(u,+)},  x^{(u,-)} \in \mathcal{T}_2 ( \mathbb{A}, u,k ),\\
				x^{(v,+)},  x^{(v,-)} \in \mathcal{T}_2 ( \mathbb{A}, v,k ).	
			\end{align} 
			% Applying the same decomposition as in \eqref{eq:decomposition_scalar} to the entries of the matrix $( A_{r,s}  )_{\substack{r = 1,\dots, m \\ s = 1,\dots, n}}$ and the vector $( b_r  )_{r=1,\dots,m}$, we get $A_{r,s} = A_{r,s}^{(u,+)} - A_{r,s}^{(u,-)} + A_{r,s}^{(v,+)} + A_{r,s}^{(v,-)}$ such that $A_{r,s}^{(u,+)}, A_{r,s}^{(u,-)} \in \mathcal{T}_2 ( \mathbb{A}, u,k )$ and $A_{r,s}^{(v,+)}, A_{r,s}^{(v,-)} \in \mathcal{T}_2 ( \mathbb{A}, v,k )$, for $r = 1,\dots,m$ and $s = 1,\dots, n$, and $b_r = b_r^{(u,+)} - b_r^{(u,-)} + b_r^{(v,+)} + b_r^{(v,-)}$ such that $b_r^{(u,+)}, b_r^{(u,-)} \in \mathcal{T}_2 ( \mathbb{A}, u,k )$ and $b_r^{(v,+)}, b_r^{(v,-)} \in \mathcal{T}_2 ( \mathbb{A}, v,k )$, for $r = 1,\dots,m$. Setting $A^{(u,+)} = ( A_{r,s}^{(u,+)}  )_{\substack{r = 1,\dots, m \\ s = 1,\dots, n}}$, $A^{(u,-)} = ( A^{(u,-)}_{r,s}  )_{\substack{r = 1,\dots, m \\ s = 1,\dots, n}}$, $A^{(v,+)} = ( A^{(v,+)}_{r,s}  )_{\substack{r = 1,\dots, m \\ s = 1,\dots, n}}$, $A^{(v,-)} = ( A^{(v,-)}_{r,s}  )_{\substack{r = 1,\dots, m \\ s = 1,\dots, n}}$, $b^{(u,+)} = ( b_r^{(u,+)}  )_{r=1,\dots,m}$, $b^{(u,-)} = ( b^{(u,-)}_r  )_{r=1,\dots,m}$, $b^{(v,+)} = ( b^{(v,+)}_r  )_{r=1,\dots,m}$, $b^{(v,-)} = ( b^{(v,-)}_r  )_{r=1,\dots,m}$ yields the decomposition of matrix $A$ and the vector $b$ according to
			Applying this decomposition entry-wise to the matrix $A$ and the vector $b$ yields
			\begin{alignat}{5}
				A =\, &A^{(u,+)} &&- A^{(u,-)} && + A^{(v,+)} &&-  A^{(v,-)}, \label{eq:decomposition_A}\\
				b =\, &\,b^{(u,+)} &&- b^{(u,-)} && + b^{(v,+)} &&-  b^{(v,-)},\label{eq:decomposition_b}
			\end{alignat}
			where $A^{(u,+)},A^{(u,-)}, A^{(v, + )}, A^{(v,-)} \in \mathbb{R}^{m\times n}$ and $b^{(u,+)},b^{(u,-)}, b^{(v, + )}, b^{(v,-)} \in \mathbb{R}^m$ satisfy
			\begin{align}
				\coef ( A^{(u,+)}), \coef ( A^{(u,-)} ),\coef ( \nmathbf{b}^{(u,+)} ), \coef ( b^{(u,-)} ) \in&\, \mathcal{T}_2 ( \mathbb{A}, u,k ), \label{eq:coefficient_u_matrices}\\
				\coef ( A^{(v,+)}), \coef ( A^{(v,-)} ),\coef ( \nmathbf{b}^{(v,+)} ), \coef ( b^{(v,-)} ) \in&\, \mathcal{T}_2 ( \mathbb{A}, v,k ). \label{eq:coefficient_v_matrices}
			\end{align}
			Consequently, we get
			\begin{equation*}
			% \label{eq:decompose_SAB}
				S ( A,b ) = S ( A^{(u,+)},b^{(u,+)} )- S ( A^{(u,-)},b^{(u,-)} ) + S ( A^{(v,+)},b^{(v,+)} ) - S (A^{(v,-)},b^{(v,-)}).
			\end{equation*}
			For $\omega = (u,+), (u,-),(v,+),(v,-)$, application of Lemma~\ref{lem:adaptive_quantization_single_layer_decomposition} to the  affine mapping $S ( A^{\omega},b^{\omega} )$ now yields the existence of network configurations
			\begin{equation}
			\label{eq:choice_G_alpha}
				( ( G^\omega_i,h^\omega_i ) )_{i=1}^{k+1}\in \mathcal{N}_{\mathbb{A} \cap \mathbb{R}_{\geq 0}} ( ( n,m ), m + n, k+1 )
			\end{equation}
			such that 
			\begin{equation}
			\label{eqline:composition_four_terms}	
				\affine( G^\omega_{k+1},h^\omega_{k+1} )\circ\cdots \circ \affine ( G^\omega_{1},h^\omega_1 ) = \, \affine(A^\omega,b^\omega).
			\end{equation}
			% \begin{align*}
			% 	G_1 = \begin{pmatrix}
			% 		G^{(u,+)}_1 \\ G^{(u,-)}_1 \\ G^{(v,+)}_1 \\ G^{(v,-)}_1
			% 	\end{pmatrix}, \quad 
			% 	h_1 = \begin{pmatrix}
			% 		h^{(u,+)}_1 \\ h^{(u,-)}_1 \\ h^{(v,+)}_1 \\ h^{(v,-)}_1
			% 	\end{pmatrix},	\\
			% 	G_j = \text{diag}\bigl(
			% 		G^{(u,+)}_j, G^{(u,-)}_j, G^{(v,+)}_j, G^{(v,-)}_j\bigr), \quad 
			% 	h_j = \begin{pmatrix}
			% 		h^{(u,+)}_j \\ h^{(u,-)}_j \\ h^{(v,+)}_j \\ h^{(v,-)}_j
			% 	\end{pmatrix},\\
			% 	G_{k+2} = \begin{pmatrix}I_m \\ -I_m\\ I_m \\ -I_m \end{pmatrix},\quad  h_{k+2} = 0_m.
			% \end{align*}
			Set
			\begin{equation}
			\label{eq:choice_of_G1}
				G_1 = \begin{pmatrix}
					G^{(u,+)}_1 \\ G^{(u,-)}_1 \\ G^{(v,+)}_1 \\ G^{(v,-)}_1
				\end{pmatrix}, \quad 
				h_1 = \begin{pmatrix}
					h^{(u,+)}_1 \\ h^{(u,-)}_1 \\ h^{(v,+)}_1 \\ h^{(v,-)}_1
				\end{pmatrix},			
			\end{equation}
			\begin{equation}
			\label{eq:choice_of_Gj}
				G_j = \text{diag}\bigl(
					G^{(u,+)}_j, G^{(u,-)}_j, G^{(v,+)}_j, G^{(v,-)}_j\bigr), \quad 
				h_j = \begin{pmatrix}
					h^{(u,+)}_j \\ h^{(u,-)}_j \\ h^{(v,+)}_j \\ h^{(v,-)}_j
				\end{pmatrix},			
			\end{equation}
			for $j = 2,\dots, k+1$,
			% \begin{equation}
			% \label{eq:choice_of_Gj}
			% 	G_j = \begin{pmatrix}
			% 		G^{(u,+)}_j&0&0&0 \\0& G^{(u,-)}_j&0&0 \\ 0&0&G^{(v,+)}_j&0 \\ 0&0&0&G^{(v,-)}_j
			% 	\end{pmatrix}, \quad 
			% 	h_j = \begin{pmatrix}
			% 		h^{(u,+)}_j \\ h^{(u,-)}_j \\ h^{(v,+)}_j \\ h^{(v,-)}_j
			% 	\end{pmatrix},			
			% \end{equation}
			and 
			\begin{equation*}
				G_{k+2} = \begin{pmatrix}I_m& -I_m& I_m& -I_m \end{pmatrix},\quad  h_{k+2} = 0_m.
			\end{equation*}
			A direct calculation shows that
			\begin{align}
				&\,\affine( G_{k+2},h_{k+2} )\circ\cdots \circ \affine ( G_{1},h_1 ) \label{eqline:SAB_0}\\
				&=\, \affine( G^{(u,+)}_{k+1},h^{(u,+)}_{k+1} )\circ\cdots \circ \affine ( G^{(u,+)}_{1},h^{(u,+)}_1 ) \nonumber  \\ 
				&\,\,- \affine( G^{(u,-)}_{k+1},h^{(u,-)}_{k+1} )\circ\cdots \circ \affine ( G^{(u,-)}_{1},h^{(u,-)}_1 )  \nonumber\\
				&\,\,+ \affine( G^{(v,+)}_{k+1},h^{(v,+)}_{k+1} )\circ\cdots \circ \affine ( G^{(v,+)}_{1},h^{(v,+)}_1 )\nonumber \\
				&\,\,- \affine( G^{(v,-)}_{k+1},h^{(v,-)}_{k+1} )\circ\cdots \circ \affine ( G^{(v,-)}_{1},h^{(v,-)}_1 ) \nonumber\\
				&=\, S ( A^{(u,+)},b^{(u,+)} )- S ( A^{(u,-)},b^{(u,-)} ) \\
                    &\,\,+ S ( A^{(v,+)},b^{(v,+)} ) - S (A^{(v,-)},b^{(v,-)}) \label{eqline:SAB}\\
				&=\, S ( A,b ).\label{eqline:SAB_2}
			\end{align}
			In turn, we directly get
			\begin{equation}
			\label{eq:SAB_rho}
			 	\affine( G_{k+2},h_{k+2} )\circ\cdots \circ \affine ( G_{1},h_1 ) \circ \rho = S ( A,b ) \circ \rho.
			\end{equation}
			% \begin{equation}
			% \begin{aligned}
			% 	G_1 =& ( (G^{(u,+)}_1)^T,(G^{(u,-)}_1)^T,(G^{(v,+)}_1)^T,(G^{(v,-)}_1)^T )^T,\\
			% 	h_1 =& ( (h^{(u,+)}_1)^T,(h^{(u,-)}_1)^T,(h^{(v,+)}_1)^T,(h^{(v,-)}_1)^T )^T,
			% \end{aligned}				
			% \end{equation}
			% \begin{equation}
			% \begin{aligned}
			% 	G_j =& \text{diag} ( G^{(u,+)}_j,G^{(u,-)}_j,G^{(v,+)}_j,G^{(v,-)}_j ),\\
			% 	h_j =& ( (h^{(u,+)}_j)^T,(h^{(u,-)}_j)^T,(h^{(v,+)}_j)^T,(h^{(v,-)}_j)^T )^T,
			% \end{aligned}	
			% \end{equation}
			% for $j = 2, \dots, k+1$, 
			% $G_{k+2} = (\begin{smallmatrix}I_m & I_m& - I_m & -I_m \end{smallmatrix})^T$, and $h_{k+2} = 0_m$, and 
			% we have 
			% \begin{equation*}
			% 	\affine ( G_j, h_j )\circ \cdots\circ \affine ( G_1, h_1 ) (x)  = \begin{pmatrix}
			% 		\affine ( G^{(u,+)}_j, h^{(u,+)}_j )\circ \cdots\circ \affine ( G^{(u,+)}_1, h^{(u,+)}_1 ) \\\vdots\\ \affine ( G^{(v,-)}_j, h^{(v,-)}_j )\circ \cdots\circ \affine ( G^{(v,-)}_1, h^{(v,-)}_1 )
			% 	\end{pmatrix}
			% \end{equation*}
			% for $j = 1,\dots, k+1$, and 
			% \begin{equation*}
			% 	\affine ( G_{k+2},h_{k+2} )\circ\cdots \circ \affine ( G_{1},h_1 ) = f^{(u,+)} + f^{(u,-)} - f^{(v,+)} - f^{(v,-)} = \affine ( A,b ),
			% \end{equation*}
			% which is \eqref{eq:coefficient_property_0}. 
			Moreover, as, by \eqref{eq:choice_G_alpha}, $G_j,h_j$,  for $j = 1,\dots, k+1$, contains only non-negative entries, namely in $\mathbb{A}\cap \mathbb{R}_{\geq 0}$, the affine mapping $\affine ( G_j,h_j )$ takes vectors with nonnegative entries into vectors with nonnegative entries, so that
			\begin{equation}
			\label{eq:nonnegative}
				\affine ( G_j,h_j ) \circ \rho = \rho \circ \affine ( G_j,h_j ) \circ \rho.
			\end{equation}
			Substituting \eqref{eq:nonnegative}, for $j =1,\dots, k$, into \eqref{eq:SAB_rho} finally yields
			\begin{equation*}
				 \affine( G_{k+2},h_{k+2} )\circ \rho \circ \cdots \circ \rho \circ \affine ( G_{1},h_1 ) \circ \rho = \affine(A,b)\circ \rho . \qedhere
			\end{equation*}
		\end{proof}

		We are now ready to prove Proposition~\ref{thm:general_representation}. 
		\begin{proof}
		[Proof of Proposition~\ref{thm:general_representation}]
			% We shall show that for all $\Phi \in \mathcal{N}_{\mathcal{T}_1 ( \mathbb{A},u,v,k)} (( d,d' ),  W,  L  )$, we have 
			% \begin{equation}
			% 	\label{eq:sufficient_condition}
			% 	R ( \Phi ) \in \mathcal{R}_\mathbb{A} (( d,d' ),16W, (k+3)L  ).
			% \end{equation}
			Let
			\begin{equation*}
				\Phi = ( ( A_j,\nmathbf{b}_j ))_{j = 1}^\ell \in \mathcal{N}_{\mathcal{T}_1 ( \mathbb{A},u,v,k)} (( d,d' ),  W,  L  ),
			\end{equation*}
			with $\ell \leq L$.  Set 
			\begin{alignat}{3}
				\tilde{A}_1 =&\,  \begin{pmatrix} I_{d} \\ - I_{d}  \end{pmatrix}, &&\tilde{b}_1 =&\, 0_{2d},\\
				\tilde{A}_2 =&\, \begin{pmatrix} A_1 & - A_1 \end{pmatrix}, \quad &&\tilde{b}_2 =&\, b_1,		
			\end{alignat}
			and note that $\coef ( \tilde{A}_1 ), \coef (\tilde{b}_1  ) \subseteq \{ -1,0, 1 \} \subseteq \mathbb{A}$  and  $\coef ( \tilde{A}_2 ), \coef (\tilde{b}_2  ) \subseteq \mathcal{T}_1 ( \mathbb{A},u,v,k)$.
			% , as $\coef ( A_1 ),\coef ( b_1 ) \subseteq mathcal{T}_1 ( \mathbb{A},u,v,k)$ and  $\mathcal{T}_1 ( \mathbb{A},u,v,k) = -\mathcal{T}_1 ( \mathbb{A},u,v,k)$ according to the definition of $\mathcal{T}_1 ( \mathbb{A},u,v,k)$ given by \eqref{eq:def_Auvk}.  
			Next, using $x = \rho(x) - \rho(-x), \, x \in \mathbb{R}$, we can write  $A_1 x + b_1 = A_1 \rho(x) - A_1 \rho(-x) + b_1$, $x \in \mathbb{R}^d$, and therefore get
			\begin{equation}
			\label{eq:effect_of_A_1}
			 	\affine ( A_1,b_1 ) = \affine ( \tilde{A}_2, \tilde{b}_2)\circ \rho \circ \affine ( \tilde{A}_1, \tilde{b}_1 ).
			\end{equation} 
			Set\footnote{Note that, for $\ell = 1$, the set $ \{ 1,\dots,\ell \} \backslash {1}$ is empty and no $\tilde{A}_{j+1}$ or $\tilde{b}_{j+1}$ are assigned. }
			\begin{equation}
			\label{eq:def_A_j_plus_1}
				\tilde{A}_{j+1} = A_j, \quad \tilde{b}_{j+1} = b_j, \quad \text{for every } j \in \{ 1,\dots,\ell \} \backslash \{1\}.
			\end{equation}
			We have
			\begin{align}
				\coef  (\tilde{A}_1), \coef (\tilde{b}_1) \subseteq&\,  \mathbb{A},\label{eq:property_of_A_0} \\
				\coef  (\tilde{A}_j), \coef (\tilde{b}_j) \subseteq&\,  \mathcal{T}_1 ( \mathbb{A}, u,v,k ), \quad j \in \{ 2,\dots, \ell + 1 \},\label{eq:coefficient_of_A_j}\\
				\mathcal{W} \bigl( \bigl(\bigl( \tilde{A}_j, \tilde{b}_j\bigr)\bigr)_{j=1}^{\ell + 1} \bigr) \leq&\,  2W, \label{eq:width_of_A_j}
			\end{align}
			% \begin{equation}
			% \label{eq:tilde_A_b_parameters}
			% 	( ( \tilde{A}_j,\tilde{b}_j ))_{j = 1}^{\ell + 1} \in \mathcal{N}_{\mathcal{T}_1 ( \mathbb{A},u,v,k)} (( d,d' ),  2W,  L  ),
			% \end{equation}
			and, by \eqref{eq:effect_of_A_1} and \eqref{eq:def_A_j_plus_1},
			\begin{equation}
			\label{eq:realization_Phi}
			\begin{aligned}
				R ( \Phi ) =&\, \affine ( A_\ell, b_\ell )\circ \rho \circ \cdots  \circ \affine ( A_1, b_1 )\\
				=&\, \affine ( \tilde{A}_{\ell+1}, \tilde{b}_{\ell+1} )\circ \rho \circ \cdots \circ \affine ( \tilde{A}_2, \tilde{b}_2 ) \circ \rho \circ \affine ( \tilde{A}_1, \tilde{b}_1 ).
			\end{aligned}
			\end{equation}
			For $j = 2,\dots, \ell+1$, we note that $\affine ( \tilde{A}_j, \tilde{b}_j)$ is an affine mapping with $\tilde{A}_j \in \mathbb{R}^{n_{j}\times n_{j-1}}$, $\tilde{b}_j \in \mathbb{R}^{n_j}$ such that $\coef ( \tilde{A}_j ), \coef ( \tilde{b}_j ) \in \mathcal{T}_1 ( \mathbb{A},u,v,k)$,  and $n_j,n_{j-1} \leq 2W$. Application of Lemma~\ref{lem:adaptive_quantization_single_layer} to $\affine ( \tilde{A}_j, \tilde{b}_j)$, $j = 2,\dots, \ell + 1$, therefore yields network configurations $( ( G_{j,i},h_{j,i} ) )_{i=1}^{k+2}$ with 
			\begin{align}
				\coef (  ( ( G_{j,i},h_{j,i} ) )_{i=1}^{k+2}) \subseteq&\, \mathbb{A}, \label{eq:properties_G_h_representation_1}\\
				\mathcal{W} ( ( ( G_{j,i},h_{j,i} ) )_{i=1}^{k+2} ) \leq&\, 4 ( n_{j} + n_{j-1} ) \leq\, 16 W,\label{eq:properties_G_h_representation_2}
			\end{align}
			such that 
			\begin{equation}
			\label{eq:Gh_substitute}
				\affine( G_{j,k+2},h_{j,k+2} )\circ \rho \circ \cdots \circ \rho \circ \affine ( G_{j,1},h_{j,1} ) \circ \rho = \affine(\tilde{A}_j,\tilde{b}_j)\circ \rho. 
			\end{equation}
			Substituting \eqref{eq:Gh_substitute}, for $j =2,\dots, \ell + 1$, into \eqref{eq:realization_Phi} then yields $R ( \Phi ) = R ( \widetilde{\Phi} )$,
			with 
			\begin{align*}
				\widetilde{\Phi} :=&\, (( G_{\ell + 1,k+2},h_{\ell + 1,k+2} ),\dots, ( G_{\ell + 1,1},h_{\ell + 1,1} ),\dots,\\
                &\,\, ( G_{2,k+2},h_{2,k+2} ), \dots,( G_{2,1},h_{2,1} ),  ( \tilde{A}_1, \tilde{b}_1 ) ) \\
				\in &\, \mathcal{N} ( ( d,d' ),16W, ( k+2 )\ell + 1 ) \label{eq:complexity_measurement_Phi_tilde}\\
				\subseteq &\,\mathcal{N} ( ( d,d' ),16W, ( k+3 )L).
			\end{align*}
			Since the network $\Phi \in \mathcal{N}_{\mathcal{T}_1 ( \mathbb{A},u,v,k)} (( d,d' ),  W,  L  )$ was arbitrary, we have, indeed, established that
			\begin{equation*}
				\{ R ( \Phi ): \Phi \in \mathcal{N}_{\mathcal{T}_1 ( \mathbb{A},u,v,k)} (( d,d' ),  W,  L  ) \} \subseteq \mathcal{R} ( ( d,d' ),16W, ( k+3 )L),
			\end{equation*}
			which is equivalent to
			\begin{equation*}
				\mathcal{R}_{\mathcal{T}_1 ( \mathbb{A},u,v,k)} (( d,d' ),  W,  L  )  \subseteq \mathcal{R} ( ( d,d' ),16W, ( k+3 )L). \qedhere
			\end{equation*}
			% For $i = 1,\dots, \ell$, we apply Lemma~\ref{lem:adaptive_quantization_single_layer} to $( \tilde{A}_j, \tilde{b}_j )$ to obtain $( ( G_i,h_i ) )_{i=1}^{k+2}$ such that \eqref{eq:coefficient_property_0}-\eqref{eq:coefficient_property_4} hold, 
			%  and replace the tuple $( \tilde{A}_j, \tilde{b}_j )$ by $( ( G_i,h_i ) )_{i=1}^{k+2}$, which does not change the realization since, by \eqref{eq:coefficient_property_4},
			% \begin{equation*}
			% \affine ( \tilde{A}_j, \tilde{b}_j )\circ \rho = \affine( G_{k+2},h_{k+2} )\circ \rho \circ \cdots \circ \rho \circ \affine ( G_{1},h_1 ) \circ \rho.
			% \end{equation*}
			% Denote the new network by $\hat{\Phi}$, we have $R ( \hat{\Phi} ) = R ( \widetilde{\Phi} ) = R ( \Phi )$ since the realization is perserved. Moreover, according to the construction, we have $\mathcal{L} (  \hat{\Phi}) = (k+2) \ell +1 \leq (k+3)\ell$, $\mathcal{W} ( \hat{\Phi} ) \leq 4\cdot 2 \cdot \mathcal{W} ( \widetilde{\Phi} ) \leq 16W$, and $\coef ( \Phi ) \subset \mathbb{A}$, which then implies $R ( \Phi ) = R ( \hat{\Phi} ) \in \mathcal{R}_\mathbb{A} (d,16W, (k+3)\ell  )  \subseteq \mathcal{R}_\mathbb{A} (d,16W, (k+3)L  )$ and establishes \eqref{eq:sufficient_condition}.
		\end{proof}

\section{Proof of Theorem~\ref{thm:three_phase_achievability}} % (fold)
	\label{sub:approximating_lipschitz_continuous_functions_under_assumption_basic_assumption}

		% \todo{value set? weight set? coefficient? }

		% \label{exam:quantized_weights_approximation_error_upper_bound}
		Let $D_1,C_1, E_1 $  be the constants specified in Proposition~\ref{prop:sufficient_rate}. Set $E_{2,1} = \max \left\{ E_1,1 \right\}$, $E_{2,2} = 16 E_{2,1}   (D_1 + 2)^2  $, $C_2 =  2^{18}  E_{2,1} C_1$, and $D_2:= \max \{16 ( D_1 + 2 ), 8 ( E_{2,2} + 1)^2, 2^{15} \}$.  Let $b,W,L \in \mathbb{N}$ with $W,L \geq D_2$. We have 
		\begin{equation*}
			\frac{E_{2,1} L  \log (W)}{E_{2,2} \frac{\log (W)}{L}} = L^2 \frac{E_{2,1}}{E_{2,2}} \geq D_2^2 \cdot \frac{1}{16 ( D_1 + 2 )^2} >  1,
		\end{equation*}
		which, in turn, implies $E_{2,1} L  \log (W) > E_{2,2} \frac{\log (W)}{L}$, and therefore the three regimes corresponding to Points $1$-$3$ are well-defined and, indeed, pairwise disjoint. In addition, they exhaust $\mathbb{N}$.

		In the \textit{over-quantization regime}, i.e., for  $b \geq E_{2,1} L  \log (W)  $, we have
		\begin{align}
			\mathcal{A}_\infty	( \lip  ( [0,1] ), \mathcal{R}_{b}^{1} (W, L ) )  \leq&\, \mathcal{A}_\infty	( \lip  ( [0,1] ), \mathcal{R}_{\lceil E_1 L  \log (W)  \rceil}^{1} ( W, L ) )\label{eqline:applying_inclusion}\\
			% \leq & C_1 ( W^2 L^2 \log (W) )^{-1}, \label{eqline:applying_lowerbound} \\
			\leq & \, C_1 ( W^2 L^2 \log (W) )^{-1},\label{eqline:applying_lowerbound} \\
			\leq & \, C_2 ( W^2 L^2 \log (W) )^{-1},\label{eqline:applying_lowerbound_1}
		\end{align}
		where \eqref{eqline:applying_inclusion} follows from $ b = \lceil b \rceil \geq \lceil E_{2,1} L  \log (W)  \rceil \geq \lceil E_1 L  \log (W)  \rceil $, in \eqref{eqline:applying_lowerbound} we applied Proposition~\ref{prop:sufficient_rate}, and \eqref{eqline:applying_lowerbound_1} is by  $C_2 > C_1$.

		In the \textit{proper-quantization regime}, i.e., for $b \in \bigl[E_{2,2} \frac{\log (W)}{L}, E_{2,1} L  \log (W)\bigr)$,
		% \begin{equation}
		% 	\label{eq:another_case}
		% 	E_{2,2} \frac{\log (W)}{L} \leq b < E_{2,1} L  \log (W),
		% \end{equation}
		we set 
		\begin{align}
			\widetilde{W} =& \biggl\lfloor \frac{W}{16} \biggr\rfloor, \\
			\widetilde{L} =& \Biggl\lfloor \sqrt{\frac{bL}{16E_{2,1}  \log (\widetilde{W})}} \Biggr\rfloor, \\
			k=& \Biggl\lceil \frac{E_{2,1} \tilde{L} \log (\widetilde{W}) }{b} \Biggr \rceil.
		\end{align}
		These choices guarantee that $\widetilde{W} = \lfloor \frac{W}{16} \rfloor \geq\lfloor \frac{D_2}{16} \rfloor  \geq \lfloor \frac{16 (D_1 + 2 )  }{16} \rfloor \geq D_1 + 1$, $ \tilde{L} = \biggl \lfloor \sqrt{\frac{bL}{16E_{2,1}  \log (\widetilde{W})}} \biggr\rfloor \geq \biggl \lfloor \sqrt{\frac{L}{16 E_{2,1}  \log (\lfloor \frac{W}{16} \rfloor)} E_{2,2} \frac{\log (W)}{L}} \biggr\rfloor \geq \biggl\lfloor \sqrt{\frac{E_{2,2}  }{16E_{2,1} }} \biggr \rfloor = \biggl \lfloor\sqrt{\frac{16 E_{2,1}  (D_1 + 2)^2}{16E_{2,1} }} \biggr \rfloor \geq D_1 + 1$, and $k \geq 1$. Application of Proposition~\ref{prop:sufficient_rate} with $W,L$ replaced by $\widetilde{W},\tilde{L}$ then yields 
		\begin{equation}
		\label{eq:proof_proper_quantization_0}
			\mathcal{A}_\infty	( \lip  ( [0,1] ), \mathcal{R}^1_{\lceil E_{1} \tilde{L}\log(\widetilde{W}) \rceil} ( \widetilde{W},\tilde{L}   ) ) 
			\leq C_1 (\widetilde{W}^2 \tilde{L}^2 \log (\widetilde{W}))^{- 1}.
		\end{equation}
		We next note that the family of approximants $\mathcal{R}^{1}_{\lceil E_{1} \tilde{L}\log(\widetilde{W}) \rceil} ( \widetilde{W},\tilde{L}   )$ in \eqref{eq:proof_proper_quantization_0} satisfies
		\begin{align}
			\mathcal{R}^{1}_{\lceil E_{1} \tilde{L}\log(\widetilde{W}) \rceil} ( \widetilde{W},\tilde{L}   ) \subseteq&\, \mathcal{R}_{k b}^{k} ( \widetilde{W},\tilde{L}   )  \label{eqline:proof_proper_quantization_3}  \\
			\subseteq&\, \mathcal{R}_{b}^{1} ( 16\widetilde{W}, (k+2)\tilde{L}  )\label{eqline:proof_proper_quantization_2} \\
			\subseteq &\, \mathcal{R}_{b}^{1} (W, (k+2)\tilde{L} ), \label{eqline:proof_proper_quantization_1}
		\end{align}
		% \begin{align}
		% 	\, \mathcal{R}_{b}^{1} (W, L )\supseteq &\, \mathcal{R}_{b}^{1} ( 16\widetilde{W}, (k+2)\tilde{L} ) \label{eqline:proof_proper_quantization_1}\\
		% 	\supseteq &\, \mathcal{R}_{k b}^{k} ( \widetilde{W},\tilde{L}   ) \label{eqline:proof_proper_quantization_2} \\
		% 	\supseteq &\,\mathcal{R}^{1}_{\lceil E_{2,1} \tilde{L}\log(\widetilde{W}) \rceil} ( \widetilde{W},\tilde{L}   ) , \label{eqline:proof_proper_quantization_3} 
		% \end{align}
		where \eqref{eqline:proof_proper_quantization_3} follows from $kb = \lceil kb \rceil   = \bigl\lceil \bigl\lceil \frac{E_{2,1} \tilde{L} \log (\widetilde{W}) }{b}  \bigr\rceil b \bigr\rceil \geq \lceil E_{2,1} \tilde{L} \log (\widetilde{W})\rceil \geq  \lceil E_{1} \tilde{L} \log (\widetilde{W})\rceil$,  in \eqref{eqline:proof_proper_quantization_2}  we used  Proposition~\ref{prop:tradeoff_binary} with $a = 1$ and $( W,L )$ replaced by $( \widetilde{W},\tilde{L} )$,  and \eqref{eqline:proof_proper_quantization_1} is by $16 \widetilde{W} = 16 \lfloor \frac{W}{16} \rfloor \leq W$. We proceed to show that $(k+2)\tilde{L} \leq L$ which will then yield $\mathcal{R}_{b}^{1} (W, (k+2)\tilde{L} ) \subseteq \mathcal{R}_{b}^{1} (W, L )$. This will be done by distinguishing the cases $\frac{E_{2,1} \tilde{L} \log (\widetilde{W}) }{b} \geq 1$ and $\frac{E_{2,1} \tilde{L} \log (\widetilde{W}) }{b} < 1$. For $\frac{E_{2,1} \tilde{L} \log (\widetilde{W}) }{b} \geq 1$, we have  $( k+2 )\tilde{L} = ( \bigl\lceil \frac{E_{2,1} \tilde{L} \log (\widetilde{W}) }{b}  \bigr\rceil + 2 ) \tilde{L} \leq  \biggl(  \frac{E_{2,1} \tilde{L} \log (\widetilde{W}) }{b}  + 3 \biggr) \tilde{L} \leq \biggl(  \frac{E_{2,1} \tilde{L} \log (\widetilde{W}) }{b}  + 3 \frac{E_{2,1} \tilde{L} \log (\widetilde{W}) }{b} \biggr) \tilde{L}  =  \frac{4E_{2,1} \tilde{L}^2 \log (\widetilde{W}) }{b} \leq  \frac{4E_{2,1}  \log (\widetilde{W}) }{b}\biggl(\sqrt{\frac{bL}{16E_{2,1}  \log (\widetilde{W})}} \biggr)^2 \leq  L$. For $\frac{E_{2,1} \tilde{L} \log (\widetilde{W}) }{b} <1$, we have $( k+2 )\tilde{L} = ( \lceil \frac{E_{2,1} \tilde{L} \log (\widetilde{W}) }{b}  \rceil + 2 ) \tilde{L} = 3\tilde{L} = 3\Bigl\lfloor \sqrt{\frac{bL}{16E_{2,1}  \log (\widetilde{W})}} \Bigr\rfloor \leq  3 \sqrt{\frac{bL}{16E_{2,1}  \log (\widetilde{W})}}  \leq 3 \sqrt{\frac{E_{2,1} L  \log (W) L}{16E_{2,1}  \log (\widetilde{W})}} =  \sqrt{\frac{  9 \log (W) }{  16 \log (\lfloor W \slash 16\rfloor) }}  \,L  \leq  \sqrt{\frac{  9 \log (W) }{  16 \log ( (W \slash 16) - 1) }} \,L  < L$, where in the last inequality we used that, for all $x \geq 2^{15}$, $\frac{\log ( x )}{ \log ( (x\slash 16) -1  )} \leq  \frac{\log ( x )}{ \log ( x\slash 32 )} =  \frac{\log ( x )}{ \log ( x) - 5} \leq \frac{3}{2}$, with $x = W \geq D_2 = \max \{16 ( D_1 + 2 ), 8 ( E_{2,2} + 1)^2, 2^{15} \} \geq 2^{15}$. Overall, we have $(k+2)\tilde{L} \leq L$,
		which together with \eqref{eqline:proof_proper_quantization_3}-\eqref{eqline:proof_proper_quantization_1} implies
		\begin{equation}
		\label{eqline:proof_proper_quantization_4}
			\mathcal{R}^{1}_{\lceil E_{1} \tilde{L}\log(\widetilde{W}) \rceil} ( \widetilde{W},\tilde{L}   )\subseteq  \mathcal{R}_{b}^{1} (W, (k+2)\tilde{L} ) \subseteq \mathcal{R}_{b}^{1} (W, L ). 
		\end{equation}
		We then have 
		\begin{align}
			\mathcal{A}_\infty	( \lip  ( [0,1] ), \mathcal{R}_{b}^{1} ( W, L ) ) & \leq \mathcal{A}_\infty	( \lip  ( [0,1] ), \mathcal{R}^{1}_{\lceil E_{1} \tilde{L}\log(\widetilde{W}) \rceil} ( \widetilde{W},\tilde{L}   ) ) \label{eqline:prove_upperbound_3}\\
			& \leq  C_1 (\widetilde{W}^2 \tilde{L}^2 \log (\widetilde{W}))^{- 1} \label{eqline:prove_upperbound_4}\\
			& \leq  C_1 \biggl( \biggl(\frac{W}{32}\biggr)^2 \tilde{L}^2 \log (\widetilde{W})\biggr)^{- 1} \label{eqline:prove_upperbound_5}\\
			& \leq  C_1 \biggl(\biggl(\frac{W}{32}\biggr)^2 \frac{bL}{2^6 \, E_{2,1}  \log (\widetilde{W})}  \log (\widetilde{W})\biggr)^{- 1}\label{eqline:prove_upperbound_6} \\
			& =  2^{16}E_{2,1} C_1  ( W^2 L b)^{- 1} \label{eqline:prove_upperbound_60}\\ 
			& \leq  C_2 (W^2 L b)^{- 1}, \label{eqline:prove_upperbound_61}
		\end{align}
		where in \eqref{eqline:prove_upperbound_3} we used the inclusion relation \eqref{eqline:proof_proper_quantization_4}, \eqref{eqline:prove_upperbound_4} is \eqref{eq:proof_proper_quantization_0}, in \eqref{eqline:prove_upperbound_5} we employed $\widetilde{W} = \lfloor \frac{W}{16} \rfloor \geq \frac{W}{32}$, which is owing to $\lfloor x \rfloor \geq \frac{1}{2} x $, for $x \geq 1$, \eqref{eqline:prove_upperbound_6} follows from 
		\begin{align*}
			\tilde{L}^2 =&\, \Biggl( \Biggl\lfloor \sqrt{\frac{bL}{16E_{2,1}  \log (\widetilde{W})}} \Biggr\rfloor \Biggr)^2 \\
			\geq&\,  \Biggl( \frac{1}{2} \sqrt{\frac{bL}{16E_{2,1}  \log (\widetilde{W})}}\Biggr)^2 \\
			= &\, \frac{bL}{2^6 \, E_{2,1}  \log (\widetilde{W})},
		\end{align*}
		 % due to $\lfloor x \rfloor \geq \frac{1}{2} x $, $x \geq 1$ with $x = \sqrt{\frac{bL}{4E_{2,1}  \log (\widetilde{W})}} \geq \biggl\lfloor \sqrt{\frac{bL}{4E_{2,1}  \log (\widetilde{W})}} \biggr\rfloor = \tilde{L} \geq 1$, and 
		and in \eqref{eqline:prove_upperbound_61} we used the definition of $C_2$.

	   In the \textit{under-quantization regime}, i.e., for $b \in [1, E_{2,2} \frac{\log (W)}{L})$, we reduce the problem to the \textit{proper-quantization regime}. This will be done by finding a natural number $\overline{W}$ satisfying $D_2 \leq \overline{W} \leq W$ and $ E_{2,2} \frac{\log (\overline{W})}{L} \leq b \leq E_{2,1} L\log (\overline{W})$ such that $\mathcal{A}_\infty	( \lip  ( [0,1] ), \mathcal{R}_{b}^{1} ( \overline{W}, L ) )$ can be upper-bounded using the result from the \textit{proper-quantization regime}. Specifically, set $\overline{W} = \bigl\lfloor 2^{\frac{L b}{ E_{2,2}}} \bigr \rfloor$, which implies $\overline{W}  \leq 2^{\frac{L b}{ E_{2,2}}} \leq  W $ and note that
		\begin{align}
			\overline{W} \geq&\, \Bigl\lfloor 2^{\frac{D_2 b}{ E_{2,2}}}  \Bigr\rfloor\label{eqline:prove_upperbound_8}\\
			=&\, \Bigl\lfloor \Bigl(2^{ \frac{ D_2 b}{ 2 E_{2,2}} } \Bigr)^2 \Bigr\rfloor \label{eqline:prove_upperbound_9}\\
			\geq&\, \Bigl\lfloor  \Bigl(\frac{ D_2 b}{ 2 E_{2,2}}\Bigr)^2 \Bigr\rfloor \label{eqline:prove_upperbound_10}\\
			% \geq&\, \Bigl\lfloor  \Bigl(\frac{ D b}{ 2 E_{2,2}}\Bigr)^2 \Bigr\rfloor \label{eqline:prove_upperbound_101}\\
			\geq&\, \bigl\lfloor 2 D_2 \bigr\rfloor \label{eqline:prove_upperbound_102}\\
			\geq&\, D_2,  \label{eqline:prove_upperbound_103}
		\end{align}
		where \eqref{eqline:prove_upperbound_8} follows from $L \geq D_2$, in \eqref{eqline:prove_upperbound_10} we used $2^x \geq x$, for $x \geq 1$, with $x =  \frac{ D_2 b}{ 2 E_{2,2}} = \frac{ \max \{16 ( D_1 + 2 ),\, 8 ( E_{2,2} + 1)^2,\, 2^{15} \} b}{2 E_{2,2}}  \geq \frac{ 8 ( E_{2,2} + 1)^2  b}{ 2 E_{2,2}} \geq 1 $, and \eqref{eqline:prove_upperbound_102} is by $D_2 =  \max \{16 ( D_1 + 2 ), 8 ( E_{2,2} + 1)^2, 2^{15} \}  \geq  8 ( E_{2,2} + 1 )^2 \geq  8 E_{2,2}^2$. We further note that $b  \ge E_{2,2} \frac{\log (\overline{W})}{L}$ and 
		\begin{align}
			b  & \leq  E_{2,2} \frac{\log (\overline{W} + 1) }{L} \label{eqline:prove_upperbound_14}\\
			& = \frac{2 E_{2,2}}{E_{2,1} L^2} \frac{\log( \overline{W} + 1)}{2 \log(\overline{W})} \cdot E_{2,1}  L\log (\overline{W}) \label{eqline:prove_upperbound_15}\\
			& <   E_{2,1} L\log (\overline{W}),\label{eqline:prove_upperbound_16}
		\end{align}
		where \eqref{eqline:prove_upperbound_16} follows from  $\frac{2 E_{2,2}}{E_{2,1} L^2}= \frac{32 E_{2,1}  (D_1 + 2)^2}{E_{2,1} L^2}  \leq \frac{32  (D_1 + 2)^2}{D_2^2} = \frac{32  (D_1 + 2)^2}{ (\max \{16 ( D_1 + 2 ), 8 ( E_{2,2} + 1)^2, 2^{15} \} )^2}        \leq \frac{32   (D_1 + 2)^2}{( 16 ( D_1 +2  ) )^2} \leq 1 $. 
		Then, application of the bound for the \textit{proper-quantization regime}, specifically \eqref{eqline:prove_upperbound_60} with $W$ replaced by $\overline{W}$, upon noting that $ \overline{W},L \geq D_2$, $ E_{2,2} \frac{\log (\overline{W})}{L} \leq b < E_{2,1} L\log (\overline{W})$ as established in \eqref{eqline:prove_upperbound_8}-\eqref{eqline:prove_upperbound_16},  yields 
		\begin{equation}
			\label{eqline:prove_upperbound_161}
			\mathcal{A}_\infty	( \lip  ( [0,1] ), \mathcal{R}_{b}^{1} ( \overline{W}, L ) ) \leq 2^{16}   E_{2,1}C_1 \bigl( \overline{W}^2L b \bigr)^{-1}.
		\end{equation}
		We finalize the proof for the \textit{under-quantization regime} by noting that
		% We shall show that $( \overline{W},L,a,b )$ falls into the second region in this proof and  so that we can use the corresponding bound. 
		\begin{align}
			\mathcal{A}_\infty	( \lip  ( [0,1] ), \mathcal{R}_{b}^{1} (  W, L ) ) & \leq  \mathcal{A}_\infty	( \lip  ( [0,1] ), \mathcal{R}_{b}^{1} ( \overline{W}, L ) ) \label{eqline:prove_upperbound_17}\\
			& \leq  2^{16} E_{2,1} C_1 \bigl( \overline{W}^2L b \bigr)^{-1} \label{eqline:prove_upperbound_18}\\
			& \leq  2^{16} E_{2,1} C_1  \overline{W}^{-2}  \label{eqline:prove_upperbound_19}\\
			& = 2^{16} E_{2,1} C_1  \Bigl( \Bigl\lfloor 2^{\frac{L b}{ E_{2,2}}}  \Bigr\rfloor \Bigr)^{-2}\label{eqline:prove_upperbound_20}\\
			& \leq 2^{16} E_{2,1} C_1   \biggl( \frac{1}{2} 2^{\frac{L b}{ E_{2,2}}}  \biggr)^{-2} \label{eqline:prove_upperbound_21}\\
			& = 2^{18}  E_{2,1} C_1  ( 2^{\frac{2}{E_{2,2} }} )^{-L b}\label{eqline:prove_upperbound_22} \\
			& \leq C_2 \, \alpha^{-Lb},\label{eqline:prove_upperbound_23}
		\end{align}
		where in \eqref{eqline:prove_upperbound_17} we used $\overline{W} \leq W$, \eqref{eqline:prove_upperbound_18} is \eqref{eqline:prove_upperbound_161}, \eqref{eqline:prove_upperbound_19} follows from $Lb \geq 1$, \eqref{eqline:prove_upperbound_21} is by $\lfloor x \rfloor \geq \frac{1}{2}x$, for $x \geq 1$, with $x = 2^{\frac{L b}{ E_{2,2}}} \geq 1$, and in \eqref{eqline:prove_upperbound_23} we set $\alpha :=  2^{\frac{2}{E_{2,2}}}  > 1$ and used $C_2 = 2^{18}  E_{2,1} C_1 $.

% subsection approximating_lipschitz_continuous_functions_under_assumption_basic_assumption (end)
	%!TEX root = ../draft_quantized_weight_networks.tex

\section{Auxiliary results} % (fold)
\label{sec:auxiliary_results}
	\subsection{Triangle Inequality} % (fold)
	\label{sub:triangle_inequality}
		\begin{lemma}
			\label{lem:triangle_inequality}
			Let $( \mathcal{X}, \delta )$ be a metric space and $\mathcal{F},\mathcal{G}, \mathcal{H} \subseteq \mathcal{X}$. Then,
			\begin{equation}
				\label{eq:prove_triangle_0}
				\mathcal{A} ( \mathcal{F}, \mathcal{H}, \delta) \leq  \mathcal{A} ( \mathcal{F}, \mathcal{G}, \delta) + \mathcal{A} ( \mathcal{G}, \mathcal{H}, \delta).
			\end{equation}
			\begin{proof}
				Since $\delta$ is a metric, by the triangle inequality, it follows that
				 	\begin{equation}
				 		\label{eq:prove_triangle_1}
				 		\delta ( f, h ) \leq\, \delta (f, g) + \delta ( g, h ), \text{ for all }f \in \mathcal{F},g \in \mathcal{G},h \in \mathcal{H}. 
				 	\end{equation}
				 	Taking $\inf_{h \in \mathcal{H}}$ in \eqref{eq:prove_triangle_1} yields
				 	\begin{align}
				 		\inf_{h \in \mathcal{H}} \delta ( f, h ) \leq&\; \delta (f, g) + \inf_{h \in \mathcal{H}}\delta ( g, h ) \\
				 		\leq &\; \delta (f, g) + \mathcal{A} ( \mathcal{G}, \mathcal{H}, \delta ), \text{ for all } f \in \mathcal{F}, g \in \mathcal{G}, \label{eqline:prove_triangle_2}
				 	\end{align}
				 	which, upon taking $\inf_{g \in \mathcal{G}}$  on both sides, yields 
				 	% As \eqref{eqline:prove_triangle_2} holds for all $f \in \mathcal{F}$, it follows that
				 	\begin{equation}
				 		\label{eq:prove_triangle_3}
				 		\inf_{h \in \mathcal{H}} \delta ( f, h ) \leq\, \inf_{g \in \mathcal{G}} \delta (f, g) + \mathcal{A} ( \mathcal{G}, \mathcal{H}, \delta ), \text{ for all } f \in \mathcal{F}.
				 	\end{equation}
				 	Taking $\sup_{f \in \mathcal{F}}$ on both sides of \eqref{eq:prove_triangle_3} finalizes the proof.
			\end{proof}

		\end{lemma}

	% subsection triangle_inequality (end)

	\subsection{Operations Over Functions Realized by ReLU Networks} % (fold)
	\label{sub:proof_of_lemma_lem:algebra_on_relu_networks}
	This section is concerned with the construction of ReLU networks realizing the composition, linear combination, and ``parallelization" of functions. We start with a technical lemma which shows how ReLU networks can be augmented to deeper networks while retaining their input-output relation. This result has been documented previously in \cite{deep-it-2019}, but we restate it here in our notation and provide a proof, for
    the sake of clarity of exposition and completeness.
    		\begin{lemma}
			\label{lem:extension}
			Let $ d,d',W,L \in \mathbb{N}$ and $B \in \mathbb{R}_+$ with $B \geq 1$. Then, for $f \in \mathcal{R} ( ( d,d' ),W,L,B )$, there exists a network $\Phi \in \mathcal{N} ( ( d,d' ),\max \{ W, 2d' \},L,B)$ such that $\mathcal{L} ( \Phi ) = L$ and  $R ( \Phi ) = f$.
			\begin{proof}
				By definition, there exists a network $\widetilde{\Phi} = ( ( \tilde{A}_\ell,\tilde{b}_\ell ))_{\ell = 1}^{\tilde{L}} \in \mathcal{N} ( ( d,d' ),W,L,B)$, with $\tilde{L} \leq L$, 
				such that $R ( \widetilde{\Phi} ) = f$. If $\tilde{L} = L$, the proof is finished by taking $\Phi = \widetilde{\Phi}$.  For $L > \tilde{L}$, we set $\Phi := ( ( A_\ell,b_\ell ))_{\ell = 1}^L$
				with 
				\begin{equation}
					\label{eq:extension}
					(A_\ell, b_\ell)  := (\tilde{A}_\ell, \tilde{b}_\ell),\quad \text{for } 1 \leq \ell < \tilde{L},\footnote{Here and in what follows, we use the convention that if there does not exist an $\ell$ satisfying the constraint, the assignment is skipped; in the present case, this would apply if $\tilde{L} = 1$.  This convention is to unify the discussion on general cases and corner cases.}
				\end{equation}
				$A_{\tilde{L}} := \begin{pmatrix} \tilde{A}_{\tilde{L}}\\-\tilde{A}_{\tilde{L}} \end{pmatrix}$, $b_{\tilde{L}} := \begin{pmatrix} 	\tilde{b}_{\tilde{L}}\\ -\tilde{b}_{\tilde{L}} \end{pmatrix},$
				$A_\ell := I_{2d'}$, $b_\ell := 0_{2d'}$ for $\ell$ such that $\tilde{L} < \ell < L  $, and $A_{L} := \begin{pmatrix}
					I_{d'}& -I_{d'}
				\end{pmatrix}$, $b_{L} := 0$. We have  $\Phi \in \mathcal{R} ( ( d,d' ),\max \{ W, 2d' \},L,B)$ and $\mathcal{L} ( \Phi ) = L$. It remains to show that $R (\Phi ) = f$. This will be effected by noting that 
				\begin{align}
				 	\affine ( A_{L}, b_{L} ) \circ \rho \circ \cdots \circ \rho\circ  \affine ( A_{\tilde{L}}, b_{\tilde{L}} ) =&\, \affine ( A_{L}, b_{L} ) \circ \rho\circ  \affine ( A_{\tilde{L}}, b_{\tilde{L}} ) \label{eqline:rho_rho_0}\\
				 	=&\, \affine ( \tilde{A}_{\tilde{L}}, \tilde{b}_{\tilde{L}} ) \label{eqline:rho_rho},
				 \end{align} 
				 where in \eqref{eqline:rho_rho_0} we used  $\rho \circ \affine ( A_\ell, b_\ell ) = \rho \circ \affine ( I_{2d'}, 0_{2d'} ) = \rho$, for $\tilde{L} < \ell < L$, and $\rho\circ \rho = \rho$, and \eqref{eqline:rho_rho} is by $\left(\affine ( A_{L}, b_{L} ) \circ\,  \rho\, \circ  \affine ( A_{\tilde{L}}, b_{\tilde{L}} )\right) (x) = \rho ( \tilde{A}_{\tilde{L}} x + \tilde{b}_{\tilde{L}} ) - \rho ( - \tilde{A}_{\tilde{L}} x  - \tilde{b}_{\tilde{L}} ) = \tilde{A}_{\tilde{L}} x + \tilde{b}_{\tilde{L}} =  \affine ( \tilde{A}_{\tilde{L}}, \tilde{b}_{\tilde{L}} )(x)$, for $x \in \mathbb{R}_{d''}$, with $d''$ denoting the number of rows of $\tilde{A}_{\tilde{L}}$. Combining \eqref{eqline:rho_rho_0}-\eqref{eqline:rho_rho} and \eqref{eq:extension}, we see that $R ( \Phi ) = R ( \widetilde{\Phi} ) = f$. 
			\end{proof}
		\end{lemma}
		We are now ready to state the main result of this section.
		\begin{lemma}
			\label{lem:algebra_on_ReLU_networks}
			Let $a \in \mathbb{R}$, $d,d',d'' \in \mathbb{N}$, $W_i,L_i, \in \mathbb{N}$, and $B_i \in \mathbb{R}_+$ with $B_i \geq 1$, $i = 1,2,3,4$, and let 

			% \begin{align}
			% 	f_1 \in&\, \mathcal{R} ( ( d, d' ), W_1, L_1, B_1 ),\label{eqline:f_1}\\
			% 	f_2 \in&\, \mathcal{R} ( (d, d''), W_2, L_2, B_2 ),\label{eqline:f_2}\\
			% 	f_3 \in&\, \mathcal{R} ( ( d', d'''), W_3, L_3, B_3 ).\label{eqline:f_3}\\
			% 	f_4 \in&\, \mathcal{R} ( ( d, d'), W_4, L_4, B_4 ).\label{eqline:f_4}
			% \end{align}
			% \begin{minipage}{0.4\textwidth}
			% \begin{align}
			%   f_1 \in&\, \mathcal{R} ( ( d, d' ), W_1, L_1, B_1 ), \label{eqline:f_1} \\
			%   f_2 \in&\, \mathcal{R} ( (d, d''), W_2, L_2, B_2 ), \label{eqline:f_2}
			% \end{align}
			% \end{minipage}
			% \begin{minipage}{0.4\textwidth}
			% \begin{align}
			%   f_3 \in&\, \mathcal{R} ( ( d', d'''), W_3, L_3, B_3 ), \label{eqline:f_3}\\
			%   f_4 \in&\, \mathcal{R} ( ( d, d'), W_4, L_4, B_4 ). \label{eqline:f_4}
			% \end{align}
			% \end{minipage}
            \begin{align}
                 f_1 \in&\, \mathcal{R} ( ( d, d' ), W_1, L_1, B_1 ), \label{eqline:f_1}\\
                 f_2 \in&\, \mathcal{R} ( (d, d''), W_2, L_2, B_2 ), \label{eqline:f_2}\\
                 f_3 \in&\, \mathcal{R} ( ( d', d'''), W_3, L_3, B_3 ), \label{eqline:f_3}\\
                 f_4 \in&\, \mathcal{R} ( ( d, d'), W_4, L_4, B_4 ). \label{eqline:f_4}
            \end{align}
		% 	\begin{subequations}
		% 	\begin{minipage}{0.4\textwidth}
		% 	    \begin{align}
		% 	      f_1 \in&\, \mathcal{R} ( ( d, d' ), W_1, L_1, B_1 ), \label{eqline:f_1} \\
		% 	      f_3 \in&\, \mathcal{R} ( ( d', d'''), W_3, L_3, B_3 ), \label{eqline:f_3}
		% 	    \end{align}
		% 	\end{minipage}
		% 	\hspace{5mm} % Adjust space as needed
		% 	\begin{minipage}{0.4\textwidth}
		% 	    \begin{align}
		% 	      f_2 \in&\, \mathcal{R} ( (d, d''), W_2, L_2, B_2 ), \label{eqline:f_2} \\
		% 	      f_4 \in&\, \mathcal{R} ( ( d, d'), W_4, L_4, B_4 ). \label{eqline:f_4}
		% 	    \end{align}
		% 	\end{minipage}
		% 	\end{subequations}

		% 	\begin{subequations}
		%   \begin{alignat}{2}
		%     f_1 \in&\, \mathcal{R} ( ( d, d' ), W_1, L_1, B_1 ), \quad & 
		%     f_2 \in&\, \mathcal{R} ( (d, d''), W_2, L_2, B_2 ), \label{eq:f2} \\
		%     f_3 \in&\, \mathcal{R} ( ( d', d'''), W_3, L_3, B_3 ), \quad & 
		%     f_4 \in&\, \mathcal{R} ( ( d, d'), W_4, L_4, B_4 ). \label{eq:f4}
		%   \end{alignat}
		% \end{subequations}

		% 	\begin{minipage}{0.4\textwidth}
		%     \begin{align}
		%       f_1 \in&\, \mathcal{R} ( ( d, d' ), W_1, L_1, B_1 ), \label{eqline:f_1} \\
		%       f_3 \in&\, \mathcal{R} ( ( d', d'''), W_3, L_3, B_3 ), \label{eqline:f_3}
		%     \end{align}
		% \end{minipage}
		% \begin{minipage}{0.4\textwidth}
		%     \begin{align}
		%       f_2 \in&\, \mathcal{R} ( (d, d''), W_2, L_2, B_2 ), \label{eqline:f_2} \\
		%       f_4 \in&\, \mathcal{R} ( ( d, d'), W_4, L_4, B_4 ). \label{eqline:f_4}
		%     \end{align}
		% \end{minipage}
			\vspace{2pt}
			\noindent  Then,
			\vspace{-10pt}
			\begin{align}
				a \cdot f_1 \in&\, \mathcal{R} ( ( d, d' ), W_1, L_1, \max \{ | a | B_1, B_1  \} ), \label{eq:scalar_product}\\
				( f_1,f_2 ) \in &\, \mathcal{R} ( ( d, d' + d''), \max \{ W_1, 2d'\} + \max \{ W_2, 2d''\},  \max \{  L_1, L_2\},  \max \{ B_1,B_2 \}  ),\label{eq:parallelization}\\
				f_3 \circ f_1 \in&\, \mathcal{R} (( d, d''' ), \max \{ W_1, W_3, 2d' \}, L_1 + L_3, \max \{ B_1,B_3\}), \label{eq:composition}\\
				f_1 + f_4 \in&\, \mathcal{R} ( ( d,d' ), \max \{ W_1, 2d'\} + \max \{ W_4, 2d'\},  \max \{  L_1, L_4\} + 1, \max \{ B_1,B_4 \} ), \label{eq:summation}
			\end{align}
			where, for $x \in \mathbb{R}^d$, $(a \cdot f_1) (x) := a \cdot f_1(x)$, $( f_1,f_2 ) ( x ) := ( f_1 ( x ), f_2 ( x ) )$, $f_3 \circ f_1 := f_3 ( f_1 ( x ) )$, and $( f_1 + f_4  ) ( x ) := f_1 ( x ) + f_4 ( x )$.
			
			\begin{proof}
				We shall prove \eqref{eq:scalar_product}-\eqref{eq:summation} individually as follows.
				\begin{enumerate}
					\item According to \eqref{eqline:f_1}, there exists a $\Phi_1 = ( ( A_{1,\ell},b_{1,\ell} ) )_{\ell = 1}^{\mathcal{L} ( \Phi_1 )} \in \mathcal{R} ( ( d, d' ), W_1, L_1, B_1 )$ such that $R ( \Phi_1 ) = f_1$.  Now, let $\widetilde{\Phi} = ( ( \tilde{A}_\ell,\tilde{b}_\ell ) )_{\ell = 1}^{\mathcal{L} ( \Phi_1 )}$ with $( \tilde{A}_\ell, \tilde{b}_\ell  ) := ( A_{1,\ell}, b_{1,\ell} )$, for $1 \leq \ell < \mathcal{L} ( \Phi_1 )$, and $( \tilde{A}_{\mathcal{L} ( \Phi_1 )}, \tilde{b}_{\mathcal{L} ( \Phi_1 )}  ) = ( a \cdot A_{1, \mathcal{L} ( \Phi_1 )}, a \cdot b_{1, \mathcal{L} ( \Phi_1 )} )$. We hence have $ \widetilde{\Phi}  \in \mathcal{N} ( ( d, d' ), W_1, L_1, \allowbreak  \max \{ | a | B_1, B_1  \} )$ and   $R ( \widetilde{\Phi} )(x)  = a \cdot R ( \Phi_1 )(x) = a \cdot f_1(x)$, $x \in \mathbb{R}^d$, which  establishes \eqref{eq:scalar_product}.

					\item Application of Lemma~\ref{lem:extension} to $f_1 \in \mathcal{R} ( ( d, d' ), W_1, L_1, B_1 ) \subseteq \mathcal{R} ( ( d, d' ), W_1, \max \{ L_1,L_2 \}, B_1 ) $ implies the existence of a network $$\Phi_1 = ( ( A_{1,\ell},b_{1,\ell} ) )_{\ell = 1}^{\max \{ L_1,L_2 \}} \in \mathcal{N} ( ( d, d' ), \max \{ W_1, 2d' \}, \max \{ L_1,L_2 \}, B_1 )$$ such that $R ( \Phi_1  ) = f_1$.    Similarly, application of Lemma~\ref{lem:extension} to $f_2 \in  \mathcal{R} ( ( d, d'' ), W_2, \allowbreak  \max \{ L_1,L_2 \}, B_2 )$ implies the existence of a $\Phi_2 = ( ( A_{2,\ell},b_{2,\ell} ) )_{\ell = 1}^{\max \{ L_1,L_2 \}} \in \mathcal{N} ( ( d, d'' ), \allowbreak  \{ W_2, 2d'' \}, \max \{ L_1,L_2 \}, B_2 )$ such that $R ( \Phi_2  ) = f_2$.  Now set $\widetilde{\Phi} :=  ( ( \tilde{A}_\ell,\tilde{b}_\ell ) )_{\ell = 1}^{\max \{ L_1,L_2 \}} $ with $ \tilde{A}_1 := \begin{pmatrix} A_{1,1}\\ A_{2,1} \end{pmatrix}$, $ \tilde{b}_1 := \begin{pmatrix} b_{1,1}\\ b_{2,1} \end{pmatrix}$, and $ \tilde{A}_\ell := \text{diag} (A_{1,\ell}, A_{2,\ell}) $, $ \tilde{b}_\ell := \begin{pmatrix} b_{1,\ell}\\ b_{2,\ell} \end{pmatrix}$, for $ 1 < \ell \leq  \max \{ L_1,L_2 \}$. We then get $R ( \widetilde{\Phi} )(x) = ( R ( \Phi_1 )(x), R ( \Phi_2)(x)  ) $, $x \in \mathbb{R}^d$, and $$\widetilde{\Phi} \in \mathcal{N} ( ( d, d'+d'' ), \{ W_1, 2d' \} + \{ W_2, 2d'' \}, \max \{ L_1,L_2 \}, \max \{ B_1,B_2 \} ),$$ which establishes \eqref{eq:parallelization}.

					\item It follows from $f_1 \in \mathcal{R} ( ( d, d' ), W_1, L_1, B_1 ) $ that there exists a $\Phi_1 = ( ( A_{1,\ell},b_{1,\ell} ) )_{\ell = 1}^{\mathcal{L} ( \Phi_1 )} \in \mathcal{N} ( ( d, d' ), W_1,L_1, B_1)$ such that $R ( \Phi_1  ) = f_1$.   Similarly, it follows from $f_3 \in \mathcal{R} ( ( d', d''' ), W_3,\allowbreak  L_3, B_3 ) $ that there exists a $\Phi_3 = ( ( A_{3,\ell},b_{3,\ell} ) )_{\ell = 1}^{\mathcal{L} ( \Phi_3 )} \in \mathcal{N} ( ( d', d''' ), W_3, L_3, B_3)$ so that $R ( \Phi_3  ) = f_3$. Now, let $\widetilde{\Phi} :=  ( ( \tilde{A}_\ell,\tilde{b}_\ell ) )_{\ell = 1}^{L_1 + L_3} $ with
					\begin{equation}
						\label{eq:proof_composition_1}
						( \tilde{A}_\ell, \tilde{b}_\ell   ) := ( A_{1,\ell}, b_{1,\ell}   ), \quad \text{for } 1 \leq \ell < \mathcal{L} ( \Phi_1 ),
					\end{equation}
					% $A^r_{\mathcal{L} ( \Phi_1 )} = \begin{pmatrix} A^1_{\mathcal{L} ( \Phi_1 )}\\ -A^1_{\mathcal{L} ( \Phi_1 )} \end{pmatrix} $, $b^r_{\mathcal{L} ( \Phi_1 )} = \begin{pmatrix} b^1_{\mathcal{L} ( \Phi_1 )}\\ - b^1_{\mathcal{L} ( \Phi_1 )} \end{pmatrix} $, $A^r_{\mathcal{L} ( \Phi_1 ) +1} = \begin{pmatrix}
					% 	A^3_{1} &  - A^3_{1}
					% \end{pmatrix}$, $b^r_{\mathcal{L} ( \Phi_1 ) +1} = b^3_{1} $, 
					\begin{align*}
						&\tilde{A}_{\mathcal{L} ( \Phi_1 )} = \begin{pmatrix} A_{1, \mathcal{L} ( \Phi_1 )}\\ -A_{1, \mathcal{L} ( \Phi_1 )} \end{pmatrix},  &&\tilde{b}_{\mathcal{L} ( \Phi_1 )}  = \begin{pmatrix} b_{1, \mathcal{L} ( \Phi_1 )}\\ - b_{1, \mathcal{L} ( \Phi_1 )} \end{pmatrix},\\
                        &\tilde{A}_{\mathcal{L} ( \Phi_1 ) +1} = \begin{pmatrix}
						A_{3, 1}  - A_{3, 1}
					\end{pmatrix}, &&\tilde{b}_{\mathcal{L} ( \Phi_1 ) +1} = b_{3, 1},\quad
                    \end{align*}
					and 
					\begin{equation}
						\label{eq:proof_composition_2}
						( \tilde{A}_\ell, \tilde{b}_\ell   ) = ( A_{3, \ell - \mathcal{L} ( \Phi_1 )}, b_{3, \ell - \mathcal{L} ( \Phi_1 )}   ), \quad \text{for } \mathcal{L} ( \Phi_1 ) + 1 <  \ell \leq \mathcal{L} ( \Phi_1 )  + \mathcal{L} ( \Phi_3 ). 	
					\end{equation} 

					Next, note that
					\begin{align}
						&\,\,(\affine (\tilde{A}_{\mathcal{L} ( \Phi_1 ) + 1}, \tilde{b}_{\mathcal{L} ( \Phi_1 ) + 1}   ) \circ \rho \circ \affine (\tilde{A}_{\mathcal{L} ( \Phi_1 ) }, \tilde{b}_{\mathcal{L} ( \Phi_1 ) }   ))(x) \label{eqline:proof_composition_3}\\
						&=   \left(\affine (\bigl(\begin{matrix}A_{3, 1} &  - A_{3, 1} \end{matrix}\bigr), b^3_{1}   ) \circ \rho \circ \affine \Biggl(\Biggl(\begin{matrix} A_{1, \mathcal{L} ( \Phi_1 )}\\  - A_{1, \mathcal{L} ( \Phi_1 )} \end{matrix}\Biggr), \Biggl( \begin{matrix} b_{1, \mathcal{L} ( \Phi_1 )}\\ -b_{1, \mathcal{L} ( \Phi_1 )} \end{matrix} \Biggr)  \Biggr)\right) (x) \\
						&= A_{3, 1} \rho ( A_{1, \mathcal{L} ( \Phi_1 )} x + b_{1, \mathcal{L} ( \Phi_1 )}  ) - A_{3, 1} \rho ( - A_{1, \mathcal{L} ( \Phi_1 )} x - b_{1,\mathcal{L} ( \Phi_1 )}  ) + b_{3,1}  \\
						&=  A_{3,1} (A_{1,\mathcal{L} ( \Phi_1 )} x + b_{1,\mathcal{L} ( \Phi_1 )}  ) + b_{3,1}\\
						&=  ( \affine (A_{3,1} , b_{3, 1}   ) \circ  \affine ( A_{1, \mathcal{L} ( \Phi_1 )},  b_{1, \mathcal{L} ( \Phi_1 )}   ) ) (x), \quad \text{for }x \in \mathbb{R}^d. \label{eqline:proof_composition_4}
					\end{align}
					Combining \eqref{eq:proof_composition_1}, \eqref{eq:proof_composition_2},  and \eqref{eqline:proof_composition_3}-\eqref{eqline:proof_composition_4}, we get $R ( \widetilde{\Phi}  ) = R ( \Phi_3 ) \circ R ( \Phi_1 ) = f_3 \circ f_1 $, which together with 
					\begin{equation*}
						\widetilde{\Phi} \in \mathcal{N}(( d, d''' ), \max \{ W_1, W_3, 2d' \}, L_1 + L_3, \max \{ B_1,B_3\})
					\end{equation*}
					establishes \eqref{eq:composition}.

					\item Let $f_5 = \affine \bigl( \bigl( \begin{matrix}
						I_{d'} & I_{d'}
					\end{matrix}\bigr), 0_{d'} \bigr) \in \mathcal{R} ( ( 2d', d' ), 2d',1, 1 )$ so that $f_5(y,z) = y + z$, for all $y,z \in \mathbb{R}^{d'}$. We have
					\begin{equation}
					\label{eq:f_1_plus_f_4}
						f_1 + f_4 = f_5 \circ ( f_1, f_4 ).
					\end{equation} 
					Application of \eqref{eq:parallelization} with $f_2$ replaced by $f_4$ yields 
					\begin{equation*}
						( f_1, f_4 ) \in \mathcal{R} ( ( d,2d'), \max \{ W_1, 2d'\} + \max \{ W_4, 2d'\},  \max \{  L_1, L_4\},  \max \{ B_1,B_4 \}  ).
					\end{equation*}
					Finally, using \eqref{eq:composition} with $f_3$ replaced by $f_5$ and $f_1$ replaced by $( f_1, f_4 )$, we obtain
					\begin{align*}
						&\,\,f_1 + f_4\\
                            &=\, f_5 \circ ( f_1, f_4 ) \\
						&\in \, \mathcal{R} (( d, d' ), \max \{ W_1, 2d'\} + \max \{ W_4, 2d'\}, \max \{  L_1, L_4\} + 1, \max \{ B_1,B_4\}).\label{eqline:apply_composition_for_summation}\qedhere
					\end{align*}
				\end{enumerate}
			\end{proof}
		\end{lemma}
	
	% subsection proof_of_lemma_lem:algebra_on_relu_networks (end)

	\subsection{Depth-Weight-Magnitude Tradeoff} % (fold)
	\label{sub:trade_depth_for_weight_magnitute}
		This section is concerned with the realization of given ReLU networks by corresponding deeper networks of smaller weight-magnitude. The main result is the following proposition.
		\begin{proposition}
			\label{prop:depth_weight_magnitude_tradeoff}
			Let $W,L \in \mathbb{N}$ with $W \geq 2$, $L' \in \mathbb{N} \cup \{ 0 \}$, and $B,B' \in \mathbb{R}$ with $B,B'\geq 1$. It holds that
			\begin{equation}
				\label{eq:depth_weight_magnitude_tradeoff}
				\frac{(B')^{L + L'}   \lfloor W\slash 2 \rfloor^{L'}}{B^L}  \cdot \mathcal{R} ( W, L, B ) \subseteq \mathcal{R} ( W, L + L', B' ).
			\end{equation}
			In particular, if $\frac{(B')^{L'+L}   \lfloor W\slash 2 \rfloor^{L'}}{B^L}  \geq 1$, then 
			\begin{equation*}
				\mathcal{R} ( W,L, B ) \subseteq \mathcal{R} ( W, L + L', B' ).
			\end{equation*}
		\end{proposition}
		% The embedding \eqref{eq:depth_weight_magnitude_tradeoff} relates two families of ReLU network realization with difference weight-magnitude and depth constraints. 
		For the proof of Proposition~\ref{prop:depth_weight_magnitude_tradeoff}, we need the following two technical lemmata.

		\begin{lemma}
			\label{lem:trade_depth_for_weight_magnitute_1}
			Let $W,L \in \mathbb{N}$ with $W \geq 2$ and $B,B' \in \mathbb{R}$ with $B,B'\geq 1$. We have
			\begin{equation}
			\label{eq:proof_weight_trade_off_0}
				( B'\slash B )^{L} \cdot \mathcal{R} ( W,L, B ) = \mathcal{R} ( W, L, B' ).
			\end{equation}
			% where $a\cdot \mathcal{F}:= \{ a \cdot f: f \in \mathcal{F} \}$ for a family of function $\mathcal{F}$ and $a \in \mathbb{R}$. 
		\end{lemma}

			\begin{proof}
				Let $g \in \mathcal{R} ( W,L, B )$.
				 % and we shall establish $( B'\slash B )^{L} \cdot g \in \mathcal{R} ( W, L , B' )$.  
				Application of Lemma~\ref{lem:extension} yields the existence of a network configuration $\Phi = ( ( A_\ell, b_\ell ) )_{\ell = 1}^{L} \in \mathcal{N} ( W,L, B )$ with exactly $L$ layers such that $R ( \Phi ) = g$. Now, let $\widetilde{\Phi} := ( (  \frac{B'}{B} A_\ell, \frac{B'}{B} b_\ell ) )_{\ell = 1}^{L} \in \mathcal{R} ( W, L, B' )$. It then follows from the positive homogeneity of the ReLU function, i.e., $\rho(ax) = a \rho (x) $, for all $x \in \mathbb{R}$ and $a \in \mathbb{R}_+$, that $R ( \widetilde{\Phi} ) = (B'\slash B)^{L} \cdot \mathcal{R} ( \Phi ) = (B'\slash B)^{L} \cdot g$, which in turn implies $( B'\slash B )^{L} \cdot g \in \mathcal{R} ( W, L , B' )$. As the choice of $g \in \mathcal{R} ( W,L, B )$ was arbitrary, we have established that 
				\begin{equation}
				\label{eq:proof_weight_trade_off_1}
					( B'\slash B )^{L} \cdot \mathcal{R} ( W,L, B ) \subseteq \mathcal{R} ( W, L, B' ).
				\end{equation}
				Swapping the roles of $B$ and $B'$ in \eqref{eq:proof_weight_trade_off_1} yields 
				\begin{equation}
				\label{eq:proof_weight_trade_off_2}
					( B\slash B' )^{L} \cdot \mathcal{R} ( W,L, B' ) \subseteq \mathcal{R} ( W, L, B ).
				\end{equation}
				Combining  \eqref{eq:proof_weight_trade_off_1} with \eqref{eq:proof_weight_trade_off_2} establishes \eqref{eq:proof_weight_trade_off_0}.
			\end{proof}

		% We proceed to elaborate on the effect of difference in depth constraint.
		\begin{lemma}
		\label{lem:trade_depth_for_weight_magnitute_2}
			Let $W,L \in \mathbb{N}$ with $W \geq 2$ and $L' \in \mathbb{N} \cup \{ 0 \}$. We have
			\begin{equation}
				\label{eq:depth_magnitute_tradeoff}
				\lfloor W \slash 2 \rfloor^{L'}  \cdot \mathcal{R} ( W,L, 1 ) \subseteq \mathcal{R} ( W, L + L', 1 ).
			\end{equation}
			% where $a\cdot \mathcal{F}:= \{ a\cdot f: f \in \mathcal{F} \}$ for a family of function $\mathcal{F}$ and $a \in \mathbb{R}$.

			\begin{proof}
				For $L' =0$, \eqref{eq:depth_magnitute_tradeoff} is trivially satisfied. We continue with $L' \geq 1$. Let $g \in \mathcal{R} ( W,L, 1 )$. Application of Lemma~\ref{lem:extension} then yields the existence of a network configuration $\Phi = ( ( A_\ell, b_\ell ) )_{\ell = 1}^{L} \in \mathcal{N} ( W,L, 1 )$ with exactly $L$ layers such that $R ( \Phi ) = g$. 
    %Recall that, for $m,n \in \mathbb{N}$, $1_{m\times n}$ denotes the $m$-by-$n$ matrix with all entries equal to $1$, and let
	Let
% \begin{equation*}
					$D : = \text{diag} ( 1_{\lfloor W \slash 2 \rfloor \times \lfloor W \slash 2 \rfloor}, 1_{\lfloor W \slash 2 \rfloor \times \lfloor W \slash 2 \rfloor} )$.
				%\end{equation*}
%				be the block-diagonal matrix with diagonal element-matrices $1_{\lfloor W \slash 2 \rfloor \times \lfloor W \slash 2 \rfloor}$. 
Consider $\widetilde{\Phi} := ( (  \tilde{A}_{\ell} ,  \tilde{b}_\ell ) )_{\ell = 1}^{L + L'} $ with 
				\begin{equation}
					\label{eq:construction_multiplication_1}
					(\tilde{A}_\ell, \tilde{b}_\ell) := ( A_\ell, b_\ell  ), \quad \text{for }1 \leq \ell < L,
				\end{equation}
				\begin{align*}
					\tilde{A}_{L} =&  \begin{pmatrix}
						1_{\lfloor W \slash 2 \rfloor \times 1}\\
						-1_{\lfloor W \slash 2 \rfloor \times 1}
					\end{pmatrix} \cdot A_L=  ( \underbrace{A_L^T, \cdots, A_L^T}_{{ \lfloor W \slash 2 \rfloor} \text{ times}},\underbrace{-A_L^T, \cdots, -A_L^T}_{{ \lfloor W \slash 2 \rfloor} \text{ times}}  )^T,\\
					\tilde{b}_{L} =& \begin{pmatrix}
						1_{\lfloor W \slash 2 \rfloor \times 1}\\
						-1_{\lfloor W \slash 2 \rfloor \times 1}
					\end{pmatrix} \cdot b_L = ( \underbrace{b_L, \cdots, b_L}_{{ \lfloor W \slash 2 \rfloor} \text{ times}},\underbrace{-b_L, \cdots, -b_L}_{{ \lfloor W \slash 2 \rfloor} \text{ times}}  )^T,
				\end{align*}
				and 
				\begin{align}
					(\tilde{A}_\ell, \tilde{b}_\ell) =&\, (D, 0_{2 \lfloor W \slash 2  \rfloor}), \quad \text{for } L < \ell < L + L', \label{eqline:construction_multiplication_2}\\
					(\tilde{A}_{L + L'}, \tilde{b}_{L + L'}) =&\, (( 1_{1 \times \lfloor W \slash 2 \rfloor}, -  1_{1 \times \lfloor W \slash 2 \rfloor} ),  0). \label{eqline:construction_multiplication_3}
				\end{align}
				We have $R ( \widetilde{\Phi} ) \in  \mathcal{R} ( W, L + L', 1 )$. Further, it follows from the definition of $( \tilde{A}_{L}, \tilde{b}_L )$ that
				\begin{equation}
				\label{eq:construction_multiplication_4}
					\affine ( \tilde{A}_{L}, \tilde{b}_L ) = \affine \Biggl( \Biggl(\begin{matrix}
						1_{\lfloor W \slash 2 \rfloor \times 1}\\
						-1_{\lfloor W \slash 2 \rfloor \times 1}
					\end{matrix} \Biggr), 0_{2  {\lfloor W \slash 2 \rfloor}} \Biggr)   \circ \affine ( {A}_{L}, {b}_L ).
				\end{equation} 
				% Substituting \eqref{eq:construction_multiplication_1}-\eqref{eq:construction_multiplication_4} into the definition of $R ( \widetilde{\Phi} )$ yields
				Putting everything together, we obtain
				\begin{align*}
					R ( \widetilde{\Phi} ) =&\, \affine \biggl(\begin{pmatrix}
						1_{1 \times\lfloor W \slash 2 \rfloor} &- 1_{1 \times \lfloor W \slash 2 \rfloor} 
					\end{pmatrix}, 0 \biggr) \circ \rho \circ \underbrace{\affine \biggl(D, 0_{2  \lfloor W \slash 2 \rfloor}  \biggr) \circ \cdots \circ \rho \, }_{ \substack {(L' - 1 ) \text{-fold self-composition} \\ \text{of } \affine (D,\, 0_{2  \lfloor W \slash 2 \rfloor}  ) \circ \rho}} \\
					&\, \circ\, \affine \Biggl( \Biggl( \begin{matrix}
						1_{\lfloor W \slash 2 \rfloor \times 1}\\
						-1_{\lfloor W \slash 2 \rfloor \times 1}
					\end{matrix}\Biggr), 0_{2 {\lfloor W \slash 2 \rfloor}} \Biggr) \circ R ( \Phi ) \\
					:=\, &\, h \circ R ( \Phi ),
				\end{align*}
				where we used the convention that $0$-fold self-composition of a function equals the identity function. A direct calculation yields 
				\begin{align*}
					h(x) =&\, \begin{pmatrix}
						1_{1 \times\lfloor W \slash 2 \rfloor} &- 1_{1 \times \lfloor W \slash 2 \rfloor} 
					\end{pmatrix} \cdot \begin{pmatrix}
						\lfloor W \slash 2 \rfloor^{L' -1}\rho(x) \cdot 1_{\lfloor W \slash 2 \rfloor}\\
						\lfloor W \slash 2 \rfloor^{L' -1} \rho(-x) \cdot 1_{\lfloor W \slash 2 \rfloor}
					\end{pmatrix}\\
					=&\, \lfloor W \slash 2 \rfloor^{L'} x, \quad x \in \mathbb{R},
 				\end{align*}
				which together with $R ( \widetilde{\Phi} ) = h \circ R ( \Phi )$ implies $\lfloor W \slash 2 \rfloor^{L'} \cdot R ( \Phi ) = R ( \widetilde{\Phi} ) \in  \mathcal{R} ( W, L + L', 1 )$. Since the choice of $g \in \mathcal{R} ( W,L, 1 )$ was arbitrary, we have established \eqref{eq:depth_magnitute_tradeoff}. \qedhere
				% Then, We have $\rho \circ W ( K, 0_?)  \circ \rho = W ( K, 0_?) \rho $ and therefor 
				% $K^{N} = ( B \lfloor W \slash 2 \rfloor  )^{L' - 1} \cdot  \begin{pmatrix}
				% 		 1_{\lfloor W \slash 2 \rfloor \times \lfloor W \slash 2 \rfloor} & 0_{\lfloor W \slash 2 \rfloor}\\
				% 		0_{\lfloor W \slash 2 \rfloor} &1_{\lfloor W \slash 2 \rfloor \times \lfloor W \slash 2 \rfloor} 
				% 	\end{pmatrix} $
				% 	\todo{a graph here??}
			\end{proof}
		\end{lemma}
		We are now ready to prove Proposition~\ref{prop:depth_weight_magnitude_tradeoff}. 
		\begin{proof}
			[Proof of Proposition~\ref{prop:depth_weight_magnitude_tradeoff}] We have 
			\begin{align}
				\mathcal{R} ( W, L + L', B' ) =&\, (B')^{L'+L}\cdot \mathcal{R} ( W, L + L', 1 ) \label{eqline:weight_trade_1} \\
				\supseteq&\, (B')^{L'+L}   \lfloor W\slash 2 \rfloor^{L'}  \cdot \mathcal{R} ( W, L, 1 ) \label{eqline:weight_trade_2} \\
				=&\,  \frac{(B')^{L'+L}   \lfloor W\slash 2 \rfloor^{L'}}{B^L}  \cdot \mathcal{R} ( W, L, B ), \label{eqline:weight_trade_3}
			\end{align}
			where \eqref{eqline:weight_trade_1} and \eqref{eqline:weight_trade_3} follow from Lemma~\ref{lem:trade_depth_for_weight_magnitute_1}, and in \eqref{eqline:weight_trade_2} we used Lemma~\ref{lem:trade_depth_for_weight_magnitute_2}. This establishes the first part of the proposition, namely \eqref{eq:depth_weight_magnitude_tradeoff}. Next, let $K: = \frac{(B')^{L'+L}   \lfloor W\slash 2 \rfloor^{L'}}{B^L} \geq 1$. For $g \in \mathcal{R} ( W, L, B )$, we have $K\cdot g \in \mathcal{R} ( W, L + L', B' )$ thanks to \eqref{eqline:weight_trade_1}-\eqref{eqline:weight_trade_3}. It then follows from Lemma~\ref{lem:algebra_on_ReLU_networks} that $K^{-1} \cdot  ( K\cdot g ) \in \mathcal{R} ( W, L + L',  \max \{ | K^{-1}|B',   B'\} ) = \mathcal{R} ( W, L + L',   B' )$. Since the choice of $g \in \mathcal{R} ( W, L, B )$ was arbitrary, we have established that  $\mathcal{R} ( W, L,   B ) \subseteq \mathcal{R} ( W, L + L',   B' )$.
		\end{proof}
		
	% subsection trade_depth_for_weight_magnitute (end)

	\subsection{Auxiliary Lemma on Minimax Error} % (fold)
	\label{sub:auxiliary_lemma_minimax_r}
	\begin{lemma}
		\label{lem:auxiliary_lemma_minimax_r}
		Let $[u,v] \subseteq \mathbb{R}$ and  $\mathbb{A} \subseteq \mathbb{R}$ such that $\mathbb{A} \cap [u,v] \neq \emptyset$. Then,
		\begin{equation}
			\label{eq:equivalence_density_approximation}
			\mathcal{A} ( [u,v], \mathbb{A}, | \cdot |  )   \leq \mathcal{A} ( [u,v], \mathbb{A} \cap [u,v], | \cdot | )   \leq 2 \mathcal{A} ( [u,v], \mathbb{A} , | \cdot | ). 
		\end{equation}

		\begin{proof}
			We note that the first inequality in \eqref{eq:equivalence_density_approximation} follows directly upon noting that $ \mathbb{A} \cap [u,v] \subseteq \mathbb{A}$, so we only have to prove the second inequality. Suppose first that $\mathbb{A}$ is a closed set. 
			% \begin{equation}
			% \label{eq:2_difference}
			% 	\mathcal{A} ( [u,v], \mathbb{A} \cap [u,v], | \cdot | )   \leq 2 \mathcal{A} ( [u,v], \mathbb{A} , | \cdot | ).
			% \end{equation}
			Let $u' = \inf (\mathbb{A} \cap [u,v]  )  $ and $v' = \sup ( \mathbb{A}  \cap [u,v] )$. The points $u',v'$ are elements of  $[u,v] \cap \mathbb{A}$, as $[u,v] \cap \mathbb{A}$ is closed and non-empty.
			We first note that, for $f:\mathbb{R} \mapsto \mathbb{R}$, 
			\begin{equation*}
				\sup_{x \in [u,v]} f(x) = \max \biggl\{ \sup_{x \in [u,u']} f(x), \sup_{x \in [u',v']} f(x),\sup_{x \in [v',v]} f(x) \biggr\}.
			\end{equation*}
			Then, for $\mathbb{B} \subseteq \mathbb{R}$, we set $f(x) = \inf_{y \in \mathbb{B}} | x - y | $, $x \in \mathbb{R}$, and get
			\begin{equation}
			\label{eq:general_decompostion_2_times}
				\mathcal{A} ( [u,v], \mathbb{B}, | \cdot | ) = \max \{ \mathcal{A} ( [u,u'], \mathbb{B}, | \cdot | ), \mathcal{A} ( [u',v'], \mathbb{B}, | \cdot | ), \mathcal{A} ( [v',v], \mathbb{B}, | \cdot | ) \}.
			\end{equation}
			The result will be established by showing that 
			\begin{align}
				\mathcal{A} ( [u,u'], \mathbb{A} \cap [u,v], | \cdot | )   \leq&\, 2\, \mathcal{A} ( [u,u'], \mathbb{A} , | \cdot | ), \label{eq:2_difference_1}\\
				\mathcal{A} ( [u',v'], \mathbb{A} \cap [u,v], | \cdot | )   \leq&\, 2\, \mathcal{A} ( [u',v'], \mathbb{A} , | \cdot | ),\label{eq:2_difference_2}\\
				\mathcal{A} ( [v',v], \mathbb{A} \cap [u,v], | \cdot | )   \leq&\, 2\, \mathcal{A} ( [v',v], \mathbb{A} , | \cdot | )\label{eq:2_difference_3}
			\end{align}
			and then combining these three inequalities with \eqref{eq:general_decompostion_2_times} for $\mathbb{B} =\mathbb{A}$ and $\mathbb{B} =  \mathbb{A} \cap [u,v]$. We start by establishing \eqref{eq:2_difference_1}. To this end, we note that 
			\begin{align}
				\mathcal{A} ( [u,u'], \mathbb{A} \cap [u,v], | \cdot | ) \leq&\, \mathcal{A} ( [u,u'], \{ u' \}, | \cdot | ) \label{eq:proof_2_relation_1}\\
				= & \sup_{x \in [u,u']} | x - u' | \\
				= &\, | u' - u |, \label{eq:proof_2_relation_2} \\
				\mathcal{A} ( [u,u'], \mathbb{A} , | \cdot | ) \geq&\, \mathcal{A} \biggl( \biggl\{ \frac{u + u'}{2} \biggr\}, \mathbb{A} , | \cdot | \biggr) \label{eq:proof_2_relation_3} \\
				=&\, \inf_{y \in \mathbb{A}} \biggl| \frac{u + u'}{2} - y\biggr| \\
				=&\,   \frac{1}{2} | u' - u |,  \label{eq:proof_2_relation_4}
			\end{align}
			where 
			% \eqref{eq:proof_2_relation_1} follows from $\{ u' \} \subseteq \mathbb{A} \cap [u,v]$, 
			in \eqref{eq:proof_2_relation_3} we used that $\{  \frac{u + u'}{2} \} \subseteq [u,u']$. Combining \eqref{eq:proof_2_relation_1}-\eqref{eq:proof_2_relation_2} and \eqref{eq:proof_2_relation_3}-\eqref{eq:proof_2_relation_4} then implies $\mathcal{A} ( [u,u'], \mathbb{A} \cap [u,v], | \cdot | )   \leq | u' - u | \leq  2 \mathcal{A} ( [u,u'], \mathbb{A} , | \cdot | )$. A similar line of reasoning shows that $\mathcal{A} ( [v',v], \mathbb{A} \cap [u,v], | \cdot | ) \leq | v - v' | $ and $\mathcal{A} ( [v',v], \mathbb{A} , | \cdot | ) \geq \frac{1}{2} | v - v' | $, which taken together yields $\mathcal{A} ( [v',v], \mathbb{A} \cap [u,v], | \cdot | )   \leq 2 \mathcal{A} ( [v',v], \mathbb{A} , | \cdot | )$ and thereby establishes \eqref{eq:2_difference_3}. It remains to prove \eqref{eq:2_difference_2}. To this end, we first define the mapping $p: \mathbb{A} \mapsto \mathbb{A} \cap [u,v] $ according to
			\begin{equation*}
				p(y) =
				\begin{cases}
					u', & \text{ if } y \in (- \infty,u'],\\
					y, & \text{ if } y \in [u',v'],\\
					v', & \text{ if } y \in [v',\infty).
				\end{cases}
			\end{equation*}
			Next, for $x \in [u',v']$ and $y \in \mathbb{A}$, we have $| x - y | \geq | x - p(y) | $.
			%  as 
			% \begin{align*}
			% 	| x - y | =&\, x - y \geq x - u' = | x - u' |  = | x - p(y) |,  &&\text{ if } y \in (-\infty,u'], \\
			% 	| x - y  | =&\, | x  - p(y)  |, &&\text{ if } y \in [u',v'],  \\
			% 	| x - y | =&\, y - x \geq v' - x = | v' - x |  = |x - p(y) | , &&\text{ if } y \in [v',\infty).
			% \end{align*}
 			\eqref{eq:2_difference_2} now follows from
 			\begin{align}
 				\mathcal{A} ( [u',v'], \mathbb{A} , | \cdot | ) =&\, \sup_{x \in [u',v']} \inf_{y \in \mathbb{A}} | x - y |\\
 				\leq&\, \sup_{x \in [u',v']} \inf_{y \in \mathbb{A}} | x - p(y) |\\
 				\leq&\, \sup_{x \in [u',v']} \inf_{z \in  \mathbb{A} \cap [u,v]} | x - z | \label{eq:asdfe}\\
 				=&\, \mathcal{A} ( [u',v'], \mathbb{A} \cap [u,v], | \cdot | ) \\
 				\leq&\, 2 \mathcal{A} ( [u',v'], \mathbb{A} \cap [u,v], | \cdot | ),
 			\end{align}
 			where \eqref{eq:asdfe} is by $p(y) \in \mathbb{A} \cap [u,v]$, for $y \in \mathbb{A}$. 
 			% In summary, we have
 			% \begin{align}
 			% 	&\,\mathcal{A} ( [u,v], \mathbb{A} \cap [u,v], | \cdot | ) \label{eq:2_relation_1}\\
 			% 	=&\, \max \{ \mathcal{A} ( [u,u'], \mathbb{A} \cap [u,v], | \cdot | ), \mathcal{A} ( [u',v'], \mathbb{A} \cap [u,v], | \cdot | ), \mathcal{A} ( [v',v], \mathbb{A} \cap [u,v], | \cdot | ) \}\label{eq:2_relation_2}\\
 			% 	\leq &\, \max \{ 2\,\mathcal{A} ( [u,u'], \mathbb{A} , | \cdot | ), 2\,\mathcal{A} ( [u',v'], \mathbb{A} , | \cdot | ), 2\,\mathcal{A} ( [v',v], \mathbb{A} , | \cdot | ) \} \label{eq:2_relation_3}\\
 			% 	=&\, 2 \mathcal{A} ( [u,v], \mathbb{A}, | \cdot | ).\label{eq:2_relation_4}
 			% \end{align}
 			% where \eqref{eq:2_relation_2} and \eqref{eq:2_relation_4} follows from \eqref{eq:general_decompostion_2_times} with $\mathbb{B} = \mathbb{A} \cap [u,v]$ and $\mathbb{B} =  \mathbb{A}$, respectively, and in \eqref{eq:2_relation_3} we used \eqref{eq:2_difference_1}-\eqref{eq:2_difference_3}.

			For a general, not necessarily closed, set $\mathbb{A}$, we consider the closure of $\mathbb{A}$, denoted by $\bar{\mathbb{A}}$. As we have already established the second inequality in \eqref{eq:equivalence_density_approximation} for closed sets $\mathbb{A}$, we can conclude that 
			\begin{equation}
			\label{eq:2_relation_closed}
				\mathcal{A} ( [u,v], \overline{\mathbb{A}} \cap [u,v], | \cdot | ) \leq 2 \mathcal{A} ( [u,v], \overline{\mathbb{A}}, | \cdot | ).
			\end{equation}
			The proof is finalized upon using $\mathcal{A} ( [u,v], \overline{\mathbb{A}}, | \cdot |  ) = \sup_{x \in [u,v]} \inf_{y \in \overline{\mathbb{A}}} \nleft| x - y \nright| = \sup_{x \in [u,v]} \inf_{y \in \mathbb{A}} \nleft| x - y \nright| = \mathcal{A} ( [u,v], \mathbb{A}, | \cdot |  )$ and $\mathcal{A} ( [u,v], \overline{\mathbb{A}} \cap [u,v], | \cdot |  ) =  \mathcal{A} ( [u,v], \overline{\mathbb{A} \cap [u,v]}, | \cdot |  ) = \mathcal{A} ( [u,v], \mathbb{A} \cap [u,v], | \cdot |  )$ in \eqref{eq:2_relation_closed} to obtain $\mathcal{A} ( [u,v], \mathbb{A} \cap [u,v], | \cdot | ) \leq 2 \mathcal{A} ( [u,v], \mathbb{A}, | \cdot | )$.
   %and thereby concludes the proof.
   %for general $\mathbb{A}$. 
   \qedhere

		\end{proof}
	\end{lemma}
	% subsection auxiliary_lemma_ (end)
% section auxiliary_results (end)
% Uncomment the following to activate bibliography.

\bibliography{my_bib}
% \newpage

\end{document}